\documentclass[a4paper,reqno]{amsart} 
\usepackage{cmap} 
\usepackage[british]{babel}
\usepackage[T1]{fontenc}
\usepackage{amssymb,amsthm,bm}
\usepackage{mathtools}
\usepackage{nicefrac}
  \mathtoolsset{mathic}
  \allowdisplaybreaks
\usepackage{courier,mathptmx,stmaryrd}
\usepackage{longtable}
\usepackage[numbers,sort&compress]{natbib} 
\usepackage{hyperref,url}
  \hypersetup{  
    bookmarksnumbered
  }
  \hypersetup{  
    breaklinks=false,
    pdfborderstyle={/S/U/W 0},  
    citebordercolor=.235 .702 .443,
    urlbordercolor=.255 .412 .882,
    linkbordercolor=.804 .149 .19,
  }
  \hypersetup{ 
    pdfauthor={Jasper De Bock and Gert de Cooman},
    pdftitle={Interpreting, axiomatising and representing coherent choice functions in terms of desirability},
    pdfkeywords={}
  }
\usepackage[all]{hypcap} 
\usepackage{xcolor}
\usepackage{tikz}
\usepackage{enumitem}
\title[Coherent choice functions in terms of desirability]{Interpreting, axiomatising and representing\\ coherent choice functions\\ in terms of desirability}
\author{Jasper De Bock \and Gert de Cooman}
\address{Ghent University, ELIS, FLip}
\email{\{jasper.debock,gert.decooman\}@ugent.be}

\usepackage{etoolbox}
\newtoggle{ISIPTA}
\toggletrue{ISIPTA}

\DeclarePairedDelimiter{\group}{(}{)}

\DeclarePairedDelimiter{\set}{\{}{\}}
\DeclarePairedDelimiter{\structure}{\langle}{\rangle}

\DeclarePairedDelimiter{\card}{\vert}{\vert}

\newcommand{\naturals}{\mathbb{N}}
\newcommand{\naturalswithzero}{\mathbb{N}_0}
\newcommand{\reals}{\mathbb{R}}
\newcommand{\posreals}{\reals_{>0}}

\newcommand{\unit}{[0,1]}
\newcommand{\states}{\mathcal{X}}
\newcommand{\rewards}{\mathcal{R}}
\newcommand{\rrewards}[1][\rho]{\mathcal{R}_{#1}}
\newcommand{\statesandrewards}{\states\times\rewards}
\newcommand{\statesandrrewards}{\states\times\rrewards}
\newcommand{\gbls}{\mathcal{L}}
\newcommand{\posgbls}{\gbls_{>0}}
\newcommand{\gblson}[1]{\gbls(#1)}
\newcommand{\gblsonstates}{\gblson{\states}}
\newcommand{\gblsonrewards}{\gblson{\rewards}}
\newcommand{\gblsonstatesandrewards}{\gblson{\statesandrewards}}
\newcommand{\gblsonstatesandrrewards}{\gblson{\statesandrrewards}}
\newcommand{\strictposgbls}{\gbls_{\mathrm{sp}}}
\newcommand{\hls}{\mathcal{H}}
\newcommand{\hlson}[1]{\hls(#1)}
\newcommand{\hlsonstatesandrewards}{\hlson{\statesandrewards}}

\newcommand{\difspace}{\mathcal{D}} 
\newcommand{\difspaceon}[1]{\difspace(#1)}
\newcommand{\difspaceonstatesandrewards}{\difspaceon{\statesandrewards}}

\newcommand{\gblgt}{>}

\newcommand{\gblgteq}{\geqslant}
\newcommand{\optlt}[1][]{\prec_{#1}}
\newcommand{\optgt}[1][]{\succ_{#1}}
\newcommand{\optlteq}[1][]{\preceq_{#1}}

\newcommand{\preflt}[1][]{\lhd_{#1}}
\newcommand{\prefgt}[1][]{\rhd_{#1}}

\newcommand{\hlprefgt}[1][]{\rhd^\ast_{#1}}

\newcommand{\hlsetlt}{\lhd^\ast}
\newcommand{\hlsetgt}{\rhd^\ast}

\newcommand{\hlgt}[1][]{\succ^\ast_{#1}}

\newcommand{\bestreward}{\top}
\newcommand{\worstreward}{\bot}
\newcommand{\opt}[1][]{u_{#1}}
\newcommand{\altopt}[1][]{v_{#1}}
\newcommand{\altopttoo}[1][]{w_{#1}}
\newcommand{\aopt}[1][]{a_{#1}}
\newcommand{\bopt}[1][]{b_{#1}}
\newcommand{\opts}{\mathcal{V}} 
\newcommand{\posopts}{\opts_{\optgt0}}
\newcommand{\nonposopts}{\opts_{\optlteq0}}

\newcommand{\singposopts}{\opts_{\optgt0}^{\mathrm{s}}}
\newcommand{\archopts}{\opts_{\mathrm{a}}}
\newcommand{\optset}[1][]{A_{#1}}
\newcommand{\altoptset}[1][]{B_{#1}}
\newcommand{\altoptsettoo}[1][]{C_{#1}}
\newcommand{\optsets}{\mathcal{Q}}

\newcommand{\assessment}{\mathcal{A}}
\newcommand{\desirset}[1][]{D_{#1}}
\newcommand{\maxdesirset}[1][]{\hat{D}_{#1}}
\newcommand{\desirsets}{\mathbf{D}}
\newcommand{\cohdesirsets}{\overline{\desirsets}}
\newcommand{\totcohdesirsets}{\overline{\desirsets}_{\mathrm{T}}}
\newcommand{\convcohdesirsets}{\overline{\desirsets}_{\mathrm{M}}}

\newcommand{\archconvcohdesirsets}{\overline{\desirsets}_{\mathrm{M,A}}}
\newcommand{\archcohdesirsets}{\overline{\desirsets}_{\mathrm{A}}}
\newcommand{\maxdesirsets}{\hat{\desirsets}}
\newcommand{\rejectset}[1][]{K_{#1}}
\newcommand{\maxrejectset}[1][]{\hat{K}_{#1}}
\newcommand{\natexdesirset}{\closure_{\cohdesirsets}}
\newcommand{\natexrejectset}{\closure_{\cohrejectsets}}
\newcommand{\convnatexrejectset}{\closure_{\convcohrejectsets}}
\newcommand{\rejectsets}{\mathbf{K}}
\newcommand{\cohrejectsets}{\overline{\rejectsets}}
\newcommand{\archcohrejectsets}{\overline{\rejectsets}_{\mathrm{A}}}
\newcommand{\totcohrejectsets}{\overline{\rejectsets}_{\mathrm{T}}}
\newcommand{\convcohrejectsets}{\overline{\rejectsets}_{\mathrm{M}}}
\newcommand{\archconvcohrejectsets}{\overline{\rejectsets}_{\mathrm{M,A}}}
\newcommand{\maxrejectsets}{\hat{\rejectsets}}
\newcommand{\chainofrejectsets}{\mathcal{K}}
\newcommand{\setofdesirsets}{\mathcal{D}}
\newcommand{\choicefun}[1][]{C_{#1}}
\newcommand{\rejectfun}[1][]{R_{#1}}
\newcommand{\gamblifier}[1][\rho]{\gbf_{#1}}
\newcommand{\hl}[1][]{H_{#1}}
\newcommand{\althl}[1][]{G_{#1}}
\newcommand{\althltoo}[1][]{F_{#1}}
\newcommand{\cuhl}{U}
\newcommand{\hlrejectfun}[1][]{R^\ast_{#1}}
\newcommand{\hlset}[1][]{A^\ast_{#1}}
\newcommand{\althlset}[1][]{B^\ast_{#1}}
\newcommand{\althlsettoo}[1][]{C^\ast_{#1}}
\newcommand{\hlsets}{\mathcal{Q}^\ast}
\newcommand{\gbl}[1][]{h_{#1}}
\newcommand{\altgbl}[1][]{g_{#1}}
\newcommand{\altgbltoo}[1][]{f_{#1}}
\newcommand{\cset}[3][]{\set[#1]{#2\colon#3}}

\newcommand{\then}{\Rightarrow}

\newcommand{\ifandonlyif}{\Leftrightarrow}

\newcommand{\upset}[1]{{\uparrow\!\!{#1}}}
\newcommand{\tosingletons}[1]{#1^{\mathrm{s}}}
\newcommand{\lowprev}[1][]{\underline{P}_{#1}}
\newcommand{\linprev}[1][]{P_{#1}}
\newcommand{\cohlowprevs}{\underline{\mathbf{P}}}
\newcommand{\linprevs}{\mathbf{P}}
\DeclareMathOperator{\posi}{posi}

\DeclareMathOperator{\setposi}{Posi}
\DeclareMathOperator{\chull}{conv}
\DeclareMathOperator{\linspan}{span}

\DeclareMathOperator{\SU}{Su}

\DeclareMathOperator{\RN}{Rn}

\DeclareMathOperator{\RP}{Rp}
\DeclareMathOperator{\RS}{Rs}
\DeclareMathOperator{\closure}{cl}
\DeclareMathOperator{\arch}{arch}
\DeclareMathOperator{\Arch}{Arch}
\DeclareMathOperator{\gbf}{gbf}

\DeclareMathOperator{\starify}{st}
\DeclareMathOperator{\Starify}{St}
\newtheorem{theorem}{Theorem}
\newtheorem{proposition}[theorem]{Proposition}
\newtheorem{lemma}[theorem]{Lemma}
\newtheorem{corollary}[theorem]{Corollary}
\theoremstyle{definition}
\newtheorem{definition}[theorem]{Definition}
\theoremstyle{remark}

\newtheorem*{runningexample}{Example}

\newcommand{\stilltodo}[1]{\textcolor{orange}{#1}}

\begin{document}
\begin{abstract}
Choice functions constitute a simple, direct and very general mathematical framework for modelling choice under uncertainty. 
In particular, they are able to represent the set-valued choices that appear in imprecise-probabilistic decision making. 
We provide these choice functions with a clear interpretation in terms of desirability, use this interpretation to derive a set of basic coherence axioms, and show that this notion of coherence leads to a representation in terms of sets of strict preference orders. 
By imposing additional properties such as totality, the mixing property and Archimedeanity, we obtain representation in terms of sets of strict total orders, lexicographic probability systems, coherent lower previsions or linear previsions.
\end{abstract}
\maketitle


\section{Introduction}\label{sec:introduction}
Choice functions provide an elegant unifying mathematical framework for studying set-valued choice: when presented with a set of options, they generally return a subset of them.
If this subset is a singleton, it provides a unique optimal choice or decision. 
But if the answer contains multiple options, these are incomparable and no decision is made between them. 
Such set-valued choices are a typical feature of decision criteria based on imprecise-probabilistic uncertainy models, which aim to make reliable decisions in the face of severe uncertainty. 
Maximality and E-admissibility are well-known examples. 
When working with a choice function, however, it is immaterial whether it is based on such a decision criterion.
The primitive objects on this approach are simply the set-valued choices themselves, and the choice function that represents all these choices serves as an uncertainty model in and by itself.

The seminal work by Seidenfeld et al.~\cite{seidenfeld2010} has shown that a strong advantage of working with choice functions is that they allow us to impose axioms on choices, aimed at characterising what it means for choices to be rational and internally consistent. 
This is also what we want to do here, but we believe our angle of approach to be novel and unique: rather than think of choice intuitively, we provide it with a concrete interpretation in terms of desirability~\cite{couso2011,walley2000,cooman2010,cooman2011b} or binary preference~\cite{seidenfeld1995}. 
Another important feature of our approach is that we consider a very general setting, where the options form an abstract real vector space; horse lotteries and gambles correspond to special cases.

The basic structure of our paper is as follows. 
We start in Section~\ref{sec:choice:functions:and:interpretation} by introducing choice functions and our interpretation for them. 
Next, in Section~\ref{sec:coherent:sets:of:option:sets}, we develop an alternative but equivalent way of describing these choice functions: sets of desirable option sets.
We use our interpretation to suggest and motivate a number of rationality, or coherence, axioms for such sets of desirable option sets, and show in Section~\ref{sec:back:to:choice} what are the corresponding coherence axioms for choice (or rejection) functions.
Section~\ref{sec:binary:choice} deals with the special case of binary choice, and its relation to the theory of sets of desirable options \cite{couso2011,walley2000,cooman2010,cooman2011b} and binary preference.
This is important because our main result in Section~\ref{sec:representation} shows that any coherent choice model can be represented in terms of sets of such binary choice models.
In the remaining Sections~\ref{sec:totality}--\ref{sec:archimedeanity}, we consider additional axioms or properties, such as totality, the mixing property, and an Archimedean property, and prove corresponding representation results. 
This includes representations in terms of sets of strict total orders, sets of lexicographic probability systems, sets of coherent lower previsions and sets of linear previsions.
To facilitate the reading, proofs and intermediate results have been relegated to the Appendix.

\section{Choice functions and their interpretation}\label{sec:choice:functions:and:interpretation}

A choice function $\choicefun$ is a set-valued operator on sets of options. 
In particular, for any set of options $\optset$, the corresponding value of $\choicefun$ is a subset $\choicefun\group{\optset}$ of $\optset$. 
The options themselves are typically actions amongst which a subject wishes to choose. We here follow a very general approach where these options constitute an abstract real vector space $\opts$ provided with a---so-called \emph{background}---vector ordering $\optlteq$ and a strict version $\optlt$. 
The elements $\opt$ of $\opts$ are called \emph{options} and $\opts$ is therefore called the \emph{option space}.
We let $\posopts\coloneqq\cset{\opt\in\opts}{\opt\optgt0}$.
The purpose of a choice function is to represent our subject's choices between such options.

Our motivation for adopting this general framework where options are elements of abstract vector spaces, rather than the more familiar one that focuses on choice between, say, horse lotteries \cite{aumann1962,aumann1964,nau2006,seidenfeld1995,seidenfeld2010}, is its applicability to various contexts.

A typical set-up that is customary in decision theory, for example, is one where every option has a corresponding reward that depends on the state of a variable~$X$, about which the subject is typically uncertain. 
Hence, the reward is uncertain too. 
As a special case, therefore, we can consider that the variable~$X$ takes values~$x$ in some set of states~$\states$. 
The reward that corresponds to a given option is then a function $\opt$ on $\states$. 
If we assume that this reward can be expressed in terms of a real-valued linear utility scale, this allows us to identify every act with a real-valued map on $\states$.
These maps are often taken to be bounded and are then called \emph{gambles} on $X$. 
In that context, we can consider the different gambles on $X$ as our options, and the vector space $\opts$ as the set of all such gambles. 
Two popular vector orderings on $\opts$ then correspond to choosing
\begin{equation*}
\posopts\coloneqq\cset{\opt\in\opts}{\opt\gblgteq0\text{ and }\opt\neq0}
\text{~or~}
\posopts\coloneqq\cset{\opt\in\opts}{\inf\opt>0},
\end{equation*}
where $\gblgteq$ represents the point-wise ordering of gambles, defined by
\begin{equation*}
\opt\gblgteq\altopt\ifandonlyif\group{\forall x\in\states}\opt(x)\geq\altopt(x). 
\end{equation*}

A more general framework, which allows us to dispense with the linearity assumption of the utility scale, consists in considering as option space the linear space of all bounded real-valued maps on the set $\statesandrewards$, where $\rewards$ is a (finite) set of rewards.
Zaffalon and Miranda \cite{zaffalon2017:incomplete:preferences} have shown that, in a context of binary preference relations, this leads to a theory that is essentially equivalent to the classical horse lottery approach.
It tends, however, to be more elegant, because a linear space is typically easier to work with than a convex set of horse lotteries.
Van Camp \cite{2017vancamp:phdthesis} has shown that this idea can be straightforwardly extended from binary preference relations to the more general context of choice functions.
We follow his lead in focusing on linear spaces of options here.

In both of the above-mentioned cases, the options are still bounded real-valued maps.
In fairly recent work, Van Camp et al.~\cite{2017vancamp:phdthesis,vancamp2015:indifference} have shown that a notion of indifference can be associated with choice functions quite easily, by moving from the original option space to its quotient space with respect to the linear subspace of all options that are assessed to be equivalent to the zero option.
Even when the original options are real-valued maps, the elements of the quotient space will be equivalence classes of such maps---affine subspaces of the original option space---which can no longer be straightforwardly identified with real-valued maps.
This provides even more incentives for considering options to be vectors in some abstract linear space $\opts$. 

Having introduced and motivated our abstract option space $\opts$, sets of options can now be identified with subsets of $\opts$, which we call \emph{option sets}. 
We restrict our attention here to \emph{finite} option sets and will use $\optsets$ to denote the set of all such finite subsets of $\opts$, including the empty set.

\begin{definition}[Choice function]\label{def:choicefunction}
A \emph{choice function} $\choicefun$ is a map from $\optsets$ to $\optsets$ such that $\choicefun\group{\optset}\subseteq\optset$ for every $\optset\in\optsets$. 
\end{definition}
\noindent
Options in $\optset$ that do not belong to $\choicefun\group{\optset}$ are said to be \emph{rejected}. 
This leads to an alternative but equivalent representation in terms of rejection functions: the \emph{rejection function} $\rejectfun[\choicefun]$ corresponding to a choice function $\choicefun$ is a map from $\optsets$ to $\optsets$, defined by $\rejectfun[\choicefun]\group{\optset}\coloneqq\optset\setminus\choicefun\group{\optset}$ for all $\optset\in\optsets$.

Alternatively, a rejection function $\rejectfun$ can also be independently defined as a map from $\optsets$ to $\optsets$ such that $\rejectfun\group{\optset}\subseteq\optset$ for all $\optset\in\optsets$.
The corresponding choice function $\choicefun[\rejectfun]$ is then clearly defined by $\choicefun[\rejectfun]\group{\optset}\coloneqq\optset\setminus\rejectfun\group{\optset}$ for all $\optset\in\optsets$.
\noindent
Since a choice function is completely determined by its rejection function, any interpretation for rejection functions automatically implies an interpretation for choice functions. 
This allows us to focus on the former.

Our interpretation for rejection functions---and therefore also for choice functions---now goes as follows. 
Consider a subject whose uncertainty is represented by a rejection function $\rejectfun$, or equivalently, by a choice function $\choicefun[\rejectfun]$. 
Then for a given option set $\optset\in\optsets$, the statement that an option $\opt\in\optset$ is rejected from $\optset$---that is, that $\opt\in\rejectfun\group{\optset}$---is taken to mean that \emph{there is at least one option $\altopt$ in $\optset$ that our subject strictly prefers over $\opt$}.

If we denote the strict preference of one option $\altopt$ over another option $\opt$ by $\altopt\prefgt\opt$, this can be written succinctly as
\begin{equation}\label{eq:interpretation:rejectfuns}
\group{\forall\optset\in\optsets}
\group{\forall\opt\in\optset}
\group[\big]{\opt\in\rejectfun\group{\optset}
\ifandonlyif
\group{\exists\altopt\in\optset}\altopt\prefgt\opt}.
\end{equation}
In this paper, such a statement---as well as statements such as those in Equations~\eqref{eq:interpretation:rejectfuns:after:irreflexivity:and:additivity} and~\eqref{eq:interpretation:rejectfuns:after:irreflexivityandadditivity:Gert}---will be interpreted as providing information about a strict preference relation~$\prefgt$, that may or may not be known or specified.
The only requirements that we impose on~$\prefgt$ is that it should be a strict partial order that extends the background ordering~$\optgt$ and is compatible with the vector space operations on $\opts$:
\begin{enumerate}[label=$\protect{\prefgt[{\arabic*}]}$.,ref=$\protect{\prefgt[{\arabic*}]}$,start=0,leftmargin=*]
\item\label{ax:prefgt:irreflexive} 
$\prefgt$ is irreflexive: for all $\opt\in\opts$, $\opt\not\prefgt\opt$;
\item\label{ax:prefgt:transitive} $\prefgt$ is transitive: for all $\opt,\altopt,\altopttoo\in\opts$, $\opt\prefgt\altopt$ and $\altopt\prefgt\altopttoo$ imply that also $\opt\prefgt\altopttoo$;
\item\label{ax:prefgt:background} for all $\opt,\altopt\in\opts$, $\opt\optgt\altopt$ implies that $\opt\prefgt\altopt$;
\item\label{ax:prefgt:addition}
for all $\opt,\altopt,\altopttoo\in\opts$, $\opt\prefgt\altopt$ implies that---so is equivalent with---$\opt+\altopttoo\prefgt\altopt+\altopttoo$;
\item\label{ax:prefgt:multiplication} for all $\opt,\altopt\in\opts$ and all $\lambda>0$, $\opt\prefgt\altopt$ implies that---so is equivalent with---$\lambda\opt\prefgt\lambda\altopt$.
\end{enumerate}
We then call such a preference ordering~$\prefgt$ \emph{coherent}.
It follows from Axioms~\ref{ax:prefgt:irreflexive} and~\ref{ax:prefgt:addition} that we can rewrite Equation~\eqref{eq:interpretation:rejectfuns} as
\begin{equation}\label{eq:interpretation:rejectfuns:after:irreflexivity:and:additivity}
\group{\forall\optset\in\optsets}
\group{\forall\opt\in\optset}
\group[\Big]{\opt\in\rejectfun\group{\optset}
\ifandonlyif
\group{\exists\altopt\in\optset}\altopt-\opt\prefgt0
\ifandonlyif
\group{\exists\altopt\in\optset\setminus\set{\opt}}\altopt-\opt\prefgt0},
\end{equation}
where we use Axiom~\ref{ax:prefgt:addition} for the first equivalence, and Axiom~\ref{ax:prefgt:irreflexive} for the second. 
Both equivalences can be conveniently turned into a single one if we no longer require that $\opt$ should belong to $\optset$ and consider statements of the form $\opt\in\rejectfun\group{\optset\cup\set{u}}$. 
Equation~\eqref{eq:interpretation:rejectfuns:after:irreflexivity:and:additivity} then turns into
\begin{equation}\label{eq:interpretation:rejectfuns:after:irreflexivityandadditivity:Gert}
\group{\forall\opt\in\opts}\group{\forall\optset\in\optsets}
\Big(\opt\in\rejectfun\group{\optset\cup\set{\opt}}
\ifandonlyif
\group{\exists\altopt\in\optset}\altopt-\opt\prefgt0\Big),
\end{equation}
So, according to our interpretation, the statement that $\opt$ is rejected from $\optset\cup\set{\opt}$ is taken to mean that the option set
\begin{equation}\label{eq:specialminus}
\optset-\opt\coloneqq\cset{\altopt-\opt}{\altopt\in\optset}
\end{equation}
contains at least one option that, according to~ $\prefgt$, is strictly preferred to the zero option~$0$.


\section{Coherent sets of desirable option sets}\label{sec:coherent:sets:of:option:sets}

A crucial observation at this point is that our interpretation for rejection functions does not require our subject to specify the strict preference $\prefgt$.
Instead, all that is needed is for her to specify option sets $\optset\in\optsets$ that---to her---contain at least one option that is strictly preferred to the zero option $0$. 
Options that are strictly preferred to zero---so options $\opt$ for which $\opt\prefgt0$---are also called \emph{desirable}, which is why we will call such option sets \emph{desirable option sets} and collect them in a \emph{set of desirable option sets} $\rejectset\subseteq\optsets$. 
Our interpretation therefore allows a modeller to specify her beliefs by specifying a \emph{set of desirable option sets} $\rejectset\subseteq\optsets$.

As can be seen from Equations~\eqref{eq:interpretation:rejectfuns:after:irreflexivityandadditivity:Gert} and~\eqref{eq:specialminus}, such a set of desirable option sets $\rejectset$ completely determines a rejection function $\rejectfun$ and its corresponding choice function $\choicefun[\rejectfun]$:
\begin{equation}\label{eq:interpretation:rejectfuns:after:irreflexivity:and:additivity:intermsof:K:Gert}
\group{\forall\opt\in\opts}
\group{\forall\optset\in\optsets}
\group[\big]{\opt\in\rejectfun\group{\optset\cup\set{\opt}}
\ifandonlyif\optset-\opt\in\rejectset}.
\end{equation}
\emph{Our interpretation, together with the basic Axioms~\ref{ax:prefgt:irreflexive} and~\ref{ax:prefgt:addition}, therefore allows the study of rejection and choice functions to be reduced to the study of sets of desirable option sets.}

We let $\rejectsets$ denote the set of all sets of desirable option sets $\rejectset\subseteq\optsets$, and consider any such~$\rejectset\in\rejectsets$. 
The first question to address is when to call $\rejectset$ \emph{coherent}: which properties should we impose on a set of desirable option sets in order for it to reflect a rational subject's beliefs? We propose the following axiomatisation, using $(\lambda,\mu)>0$ as a shorthand notation for `$\lambda\geq0$, $\mu\geq0$ and $\lambda+\mu>0$'.

\begin{definition}[Coherence for sets of desirable option sets]\label{def:coherence:rejectset}
A set of desirable option sets $\rejectset\subseteq\optsets$ is called \emph{coherent} if it satisfies the following axioms:
\begin{enumerate}[label=$\mathrm{K}_{\arabic*}$.,ref=$\mathrm{K}_{\arabic*}$,leftmargin=*,start=0]
\item\label{ax:rejects:removezero} if $\optset\in\rejectset$ then also $\optset\setminus\set{0}\in\rejectset$, for all $\optset\in\optsets$;
\item\label{ax:rejects:nozero} $\set{0}\notin\rejectset$;
\item\label{ax:rejects:pos} $\set{\opt}\in\rejectset$, for all $\opt\in\posopts$;
\item\label{ax:rejects:cone} if $\optset[1],\optset[2]\in\rejectset$ and if, for all $\opt\in\optset[1]$ and $\altopt\in\optset[2]$, $(\lambda_{\opt,\altopt},\mu_{\opt,\altopt})>0$, then also\footnote{The following simple example might help the reader understand what this axiom allows for. Consider any two $a,b\in\opts$, let $A_1=A_2=A\coloneqq\{a,b\}$ and choose $(\lambda_{a,a},\mu_{a,a})=(1,0)$, $(\lambda_{a,b},\mu_{a,b})=(1,1)$, $(\lambda_{b,a},\mu_{b,a})=(1,1)$ and $(\lambda_{b,b},\mu_{b,b})=(1,1)$. Then if $A\in\rejectset$, it follows from Axiom~\ref{ax:rejects:cone} that also $\{a,a+b,2b\}\in\rejectset$.}
\begin{equation*}
\cset{\lambda_{\opt,\altopt}\opt+\mu_{\opt,\altopt}\altopt}{\opt\in\optset[1],\altopt\in\optset[2]}
\in\rejectset;
\end{equation*}
\item\label{ax:rejects:mono} if $\optset[1]\in\rejectset$ and $\optset[1]\subseteq\optset[2]$, then also $\optset[2]\in\rejectset$, for all $\optset[1],\optset[2]\in\optsets$.
\end{enumerate}
We denote the set of all coherent sets of desirable option sets by $\cohrejectsets$.
\end{definition}
\noindent
This axiomatisation is entirely based on our interpretation and the following three axioms for desirability:
\begin{enumerate}[label=$\mathrm{d}_{\arabic*}$.,ref=$\mathrm{d}_{\arabic*}$,start=1,leftmargin=*]
\item\label{ax:desirability:nozero} 
$0$ is not desirable;
\item\label{ax:desirability:pos} all $\opt\in\posopts$ are desirable;
\item\label{ax:desirability:cone} if $\opt,\altopt$ are desirable and $(\lambda,\mu)>0$, then $\lambda\opt+\mu\altopt$ is desirable.
\end{enumerate}
Each of these three axioms follows trivially from our assumptions on the preference relation $\prefgt$: \ref{ax:desirability:nozero} follows from~\ref{ax:prefgt:irreflexive}, \ref{ax:desirability:pos} follows from~\ref{ax:prefgt:background} and~\ref{ax:desirability:cone} follows from~\ref{ax:prefgt:transitive} and~\ref{ax:prefgt:multiplication}.\footnote{Conversely, under Axiom~\ref{ax:prefgt:addition} for the preference relation~$\prefgt$, the three Axioms~\ref{ax:desirability:nozero}--\ref{ax:desirability:cone} imply the remaining Axioms~\ref{ax:prefgt:irreflexive}--\ref{ax:prefgt:background} and~\ref{ax:prefgt:multiplication}.}

That the coherence Axioms~\ref{ax:rejects:removezero}--\ref{ax:rejects:mono} are implied by our rationality requirements~\ref{ax:desirability:nozero}--\ref{ax:desirability:cone} for the concept of desirability, can now be seen as follows.
Since a desirable option set is by definition a set of options that contains at least one desirable option, Axiom~\ref{ax:rejects:mono} is immediate. 
Axioms~\ref{ax:rejects:removezero} and~\ref{ax:rejects:nozero} follow naturally from \ref{ax:desirability:nozero}, and Axiom~\ref{ax:rejects:pos} is an immediate consequence of \ref{ax:desirability:pos}.
The argument for Axiom~\ref{ax:rejects:cone} is more subtle. 
Since $\optset[1]$ and $\optset[2]$ are two desirable option sets, there must be at least one desirable option $\opt\in\optset[1]$ and one desirable option $\altopt\in\optset[2]$. 
Since for these two options, the positive linear combination $\lambda_{\opt,\altopt}\opt+\mu_{\opt,\altopt}\altopt$ is again desirable by \ref{ax:desirability:cone}, at least one of the elements of the option set $\cset{\lambda_{\opt,\altopt}\opt+\mu_{\opt,\altopt}\altopt}{\opt\in\optset[1],\altopt\in\optset[2]}$ must be a desirable option. 
Hence, it must be a desirable option set.

\section{Coherent rejection functions}\label{sec:back:to:choice}

Now that we have formulated our basic rationality requirements~\ref{ax:rejects:removezero}--\ref{ax:rejects:mono} for sets of desirable option sets $\rejectset$, we are in a position to use their correspondence~\eqref{eq:interpretation:rejectfuns:after:irreflexivity:and:additivity:intermsof:K:Gert} with rejection  functions~$\rejectfun$ to derive the equivalent rationality requirements for the latter.

Equation~\eqref{eq:interpretation:rejectfuns:after:irreflexivity:and:additivity:intermsof:K:Gert} already allows us to derive a first and very basic axiom for rejection functions---and a very similar one for choice functions, left implicit here---without imposing \emph{any} requirements on sets of desirable option sets~$\rejectset$:
\begin{enumerate}[label=$\mathrm{R}_{\arabic*}$.,ref=$\mathrm{R}_{\arabic*}$,start=0,leftmargin=*]
\item\label{ax:rejectfun:addition} for all $\optset\in\optsets$ and $\opt\in\optset$, $\opt\in\rejectfun\group{\optset}$ if and only if $0\in\rejectfun\group{\optset-\opt}$.
\end{enumerate}
Alternatively, we can also consider a slightly different---but clearly equivalent---version that perhaps displays better the invariance of rejection functions under vector addition:
\begin{enumerate}[label=$\mathrm{R}_{\arabic*}'$.,ref=$\mathrm{R}_{\arabic*}'$,start=0,leftmargin=*]
\item\label{ax:rejectfun:addition:too} for all $\optset\in\optsets$, $\opt\in\optset$ and $\altopt\in\opts$, $\opt\in\rejectfun\group{\optset}$ if and only if $\opt+\altopt\in\rejectfun\group{\optset+\altopt}$.
\end{enumerate}

When we do impose requirements on sets of desirable option sets~$\rejectset$, Equation~\eqref{eq:interpretation:rejectfuns:after:irreflexivity:and:additivity:intermsof:K:Gert} allows us to turn them into requirements for rejection (and hence also choice) functions. 
In particular, we will see in Proposition~\ref{prop:axioms:rejection:sets:and:functions:coherence} below that our Axioms~\ref{ax:rejects:removezero}--\ref{ax:rejects:mono} imply that
\begin{enumerate}[label=$\mathrm{R}_{\arabic*}$.,ref=$\mathrm{R}_{\arabic*}$,leftmargin=*]
\item\label{ax:rejectfun:not:everything:rejected} $\rejectfun\group{\emptyset}=\emptyset$, and $\rejectfun\group{\optset}\neq\optset$ for all $\optset\in\optsets\setminus\set{\emptyset}$;
\item\label{ax:rejectfun:pos} $0\in\rejectfun\group{\set{0,\opt}}$, for all $\opt\in\posopts$;
\item\label{ax:rejectfun:cone} if $\optset[1],\optset[2]\in\optsets$ such that $0\in\rejectfun\group{\optset[1]\cup\set{0}}$ and $0\in\rejectfun\group{\optset[2]\cup\set{0}}$ and if $(\lambda_{\opt,\altopt},\mu_{\opt,\altopt})>0$ for all $\opt\in\optset[1]$ and $\altopt\in\optset[2]$, then also
\begin{equation*}
0\in\rejectfun\group{\cset{\lambda_{\opt,\altopt}\opt+\mu_{\opt,\altopt}\altopt}{\opt\in\optset[1],\altopt\in\optset[2]}\cup\set{0}};
\end{equation*}
\item\label{ax:rejectfun:senalpha} 
if $\optset[1]\subseteq\optset[2]$ then also $\rejectfun\group{\optset[1]}\subseteq\rejectfun\group{\optset[2]}$, for all $\optset[1],\optset[2]\in\optsets$.
\end{enumerate}
Axiom~\ref{ax:rejectfun:senalpha} is Sen's condition~$\alpha$ \cite{sen1971,sen1977}. 
Arthur Van Camp (private communication) has proved in a direct manner that Aizermann's condition \cite{aizerman1985} can be derived from our Axioms~\ref{ax:rejects:removezero}--\ref{ax:rejects:mono} as well.
Indirectly, this can also be inferred from our representation results further on; see Theorem~\ref{theo:coherentrepresentation:twosided} in Section~\ref{sec:representation}, and the discussion following it.

We will call \emph{coherent} any rejection function that satisfies the five properties~\ref{ax:rejectfun:addition}--\ref{ax:rejectfun:senalpha} above.

\begin{definition}[Coherence for rejection and choice functions]\label{def:coherence:rejectfun}
A rejection function $\rejectfun$ is called \emph{coherent} if it satisfies the Axioms~\ref{ax:rejectfun:addition}--\ref{ax:rejectfun:senalpha}. 
A choice function $\choicefun$ is called \emph{coherent} if the associated rejection function $\rejectfun[\choicefun]$ is.
\end{definition}

Our next result establishes that these notions of coherence are perfectly compatible with the coherence for sets of desirable option sets that we introduced in Section~\ref{sec:coherent:sets:of:option:sets}.

\begin{proposition}\label{prop:axioms:rejection:sets:and:functions:coherence}
Consider any set of desirable option sets $\rejectset\in\rejectsets$ and any rejection function $\rejectfun$ that are connected by Equation~\eqref{eq:interpretation:rejectfuns:after:irreflexivity:and:additivity:intermsof:K:Gert}. 
Then $\rejectset$ is coherent if and only if $\rejectfun$ is.
\end{proposition}

We will from now on work directly with (coherent) sets of desirable option sets and will use the collective term \emph{(coherent) choice models} for (coherent) choice functions, rejection functions, and sets of desirable option sets. 
Of course, our primary motivation for studying coherent sets of desirable option sets is their connection with the other two choice models. 
This being said, it should however also be clear that our results do not depend on this connection. 
The theory of sets of desirable option sets that we are about to develop can therefore be used independently as well.

\section{The special case of binary choice}\label{sec:binary:choice}

According to our interpretation, the statement that $\optset$ belongs to a set of desirable option sets $\rejectset$ is taken to mean that $\optset$ contains at least one desirable option. 
This implies that singletons play a special role: for any $\opt\in\opts$, stating that $\set{\opt}\in\rejectset$ is equivalent to stating that $\opt$ is desirable. For any set of desirable option sets $\rejectset$, these singleton assessments are captured completely by the set of options
\begin{equation}\label{eq:desir:from:choice}
\desirset[\rejectset]\coloneqq\cset{\opt\in\opts}{\set{\opt}\in\rejectset}
\end{equation}
that, according to $\rejectset$, are definitely desirable---preferred to $0$.
A set of desirable option sets $\rejectset\in\rejectsets$ that is completely determined by such singleton assessments is called \emph{binary}.

\begin{definition}[Binary set of desirable option sets]\label{def:binary}
We call a set of desirable option sets $\rejectset$ \emph{binary} if
\vspace{-5pt}
\begin{equation}\label{eq:binary}
\optset\in\rejectset\ifandonlyif(\exists\opt\in\optset)\set{\opt}\in\rejectset,
\text{ for all $\optset\in\optsets$}.
\end{equation}
\end{definition}


In order to explain how any binary set of desirable option sets $\rejectset$ is indeed completely determined by $\desirset[\rejectset]$, we need a way to associate a rejection function with sets of options such as $\desirset[\rejectset]$. To that end, we consider the notion of a \emph{set of desirable options}: a subset $\desirset$ of $\opts$ whose interpretation will be that it consists of the options $\opt\in\opts$ that our subject considers desirable. We denote the set of all such sets of desirable options $\desirset\subseteq\opts$ by $\desirsets$. 

With any $\desirset\in\desirsets$, our interpretation for rejection functions in Section~\ref{sec:choice:functions:and:interpretation} inspires us to associate a set of desirable option sets $\rejectset[\desirset]$, defined by
\begin{equation}\label{eq:desirset:to:rejectset}
\rejectset[\desirset]
\coloneqq\cset{\optset\in\optsets}{\optset\cap\desirset\neq\emptyset}.
\end{equation}
It turns out that a set of desirable options sets $\rejectset$ is binary if and only if it has the form $\rejectset[\desirset]$, and the unique \emph{representing} $\desirset$ is then given by $\desirset[\rejectset]$.

\begin{proposition}\label{prop:binary:characterisation}
A set of desirable options sets $\rejectset\in\rejectsets$ is binary if and only if there is some~$\desirset\in\desirsets$ such that $\rejectset=\rejectset[\desirset]$. 
This $\desirset$ is then necessarily unique, and equal to $\desirset[\rejectset]$.
\end{proposition}

Just like we did for sets of desirable option sets in Section~\ref{sec:coherent:sets:of:option:sets}, we can use the basic rationality principles~\ref{ax:desirability:nozero}--\ref{ax:desirability:cone} for the notion of desirability---or binary preference---to infer basic rationality criteria for sets of desirable options.
When they do, we call them coherent.

\begin{definition}[Coherence for sets of desirable options]\label{def:cohdesir}
A set of desirable options $\desirset\in\desirsets$ is called \emph{coherent} if it satisfies the following axioms:\footnote{The Axioms~\ref{ax:desirs:nozero}--\ref{ax:desirs:cone} for sets of desirable options should not be confused with the rationality criteria~\ref{ax:desirability:nozero}--\ref{ax:desirability:cone} for our primitive notion of desirability---or binary preference. Like the Axioms~\ref{ax:rejects:removezero}--\ref{ax:rejects:mono}, they are only derived from these primitive assumptions on the basis of their interpretation.}
\begin{enumerate}[label=$\mathrm{D}_{\arabic*}$.,ref=$\mathrm{D}_{\arabic*}$,leftmargin=*]
\item\label{ax:desirs:nozero} $0\notin\desirset$;
\item\label{ax:desirs:pos} $\posopts\subseteq\desirset$;
\item\label{ax:desirs:cone} if $\opt,\altopt\in\desirset$ and $(\lambda,\mu)>0$, then $\lambda\opt+\mu\altopt\in\desirset$.
\end{enumerate}
We denote the set of all coherent sets of desirable options by $\cohdesirsets$.
\end{definition}
\noindent
So a coherent set of desirable options is a convex cone [Axiom~\ref{ax:desirs:cone}] in $\opts$ that does not contain $0$ [Axiom~\ref{ax:desirs:nozero}] and includes~$\posopts$ [Axiom~\ref{ax:desirs:pos}].
Sets of desirable options are an abstract version of the sets of desirable gambles that have an important part in the literature on imprecise probability models \cite{couso2011,cooman2010,quaeghebeur2012:itip,walley2000}.
This abstraction was first introduced and studied in great detail in~\cite{cooman2011b,quaeghebeur2015:statement}.

Our next result shows that the coherence of a binary set of desirable option sets is completely determined by the coherence of its corresponding set of desirable options.

\begin{proposition}\label{prop:coherence:for:binary}
Consider any binary set of desirable option sets $\rejectset\in\rejectsets$ and let $\desirset[\rejectset]\in\desirsets$ be its corresponding set of desirable options. 
Then $\rejectset$ is coherent if and only if $\desirset[\rejectset]$~is.
Conversely, consider any set of desirable options $\desirset\in\desirsets$ and let $\rejectset[\desirset]$ be its corresponding binary set of desirable option sets, then $\rejectset[\desirset]$ is coherent if and only if $\desirset$ is.
\end{proposition}
\noindent
So the binary coherent sets of desirable option sets are given by $\cset{\rejectset[\desirset]}{\desirset\in\cohdesirsets}$, allowing us to call any coherent set of desirable option sets in $\cohrejectsets\setminus\cset{\rejectset[\desirset]}{\desirset\in\cohdesirsets}$ \emph{non-binary}.

What makes coherent sets of desirable options $\desirset\in\cohdesirsets$---and hence also coherent binary sets of desirable option sets---particularly interesting is that they induce a binary preference order $\prefgt[\desirset]$---a strict vector ordering---on $\opts$, defined by
\begin{equation}\label{eq:desirset:to:prefgt}
\opt\prefgt[\desirset]\altopt\ifandonlyif\opt-\altopt\in\desirset
\text{ for all $\opt,\altopt\in\opts$}.
\end{equation}
The preference order $\prefgt[\desirset]$ is coherent---satisfies Axioms~\ref{ax:prefgt:irreflexive}--\ref{ax:prefgt:multiplication}--- and furthermore fully characterises $\desirset$: one can easily see that $\opt\in\desirset$ if and only if $\opt\prefgt[\desirset]0$. 
Hence, coherent sets of desirable options and coherent binary sets of desirable option sets are completely determined by a single binary strict preference order between options. This is of course the reason why we reserve the moniker \emph{binary} for choice models that are essentially based on \emph{singleton} assessments.


\section{Representation in terms of sets of desirable options}\label{sec:representation}

It should be clear---and it should be stressed---at this point that making a direct desirability assessment for an option $\opt$ typically requires more of a subject than making a typical desirability assessment for an option set $\optset$: the former requires that our subject should state that $\opt$ is desirable, while the latter only requires the subject to state that some option in $\optset$ is desirable, but not to specify which.
It is this difference---this greater latitude in making assessments---that guarantees that our account of choice is much richer than one that is purely based on binary preference. In the framework of sets of desirable option sets, it is for instance possible to express the belief that \emph{at least} one of two options $\opt$ or $\altopt$ is desirable, while remaining undecided about which of them actually is; in order to express this belief, it suffices to state that $\set{\opt,\altopt}\in\rejectset$. 
This is not possible in the framework of sets of desirable options. Sets of desirable option sets therefore constitute a much more general uncertainty framework than sets of desirable options.


So while it is nice that there are sets of desirable option sets $\rejectset[\desirset]$ that are completely determined by a set of desirable options $\desirset$, such binary choice models are typically \emph{not} what we are interested in here: using $\rejectset[\desirset]$ is equivalent to using $\desirset$ here, so there is no benefit in using the more convoluted model $\rejectset[\desirset]$ to represent choice.
No, it is the \emph{non-binary} coherent choice models that we have in our sights.
If we replace such a non-binary coherent set of desirable option sets $\rejectset$ by its corresponding set of desirable options $\desirset[\rejectset]$, we lose information, because then necessarily $\rejectset[{\desirset[\rejectset]}]\subset\rejectset$. 
Choice models are therefore more expressive than sets of desirable options.
But it turns out that our coherence axioms lead to a representation result that allows us to still use sets of desirable options, or rather, sets of them, to completely characterise \emph{any} coherent choice model.

\begin{theorem}\label{theo:coherentrepresentation:twosided}
A set of desirable option sets $\rejectset\in\rejectsets$ is coherent if and only if there is a non-empty set $\setofdesirsets\subseteq\cohdesirsets$ of coherent sets of desirable options such that $\rejectset=\bigcap\cset{\rejectset[\desirset]}{\desirset\in\setofdesirsets}$. 
The largest such set $\setofdesirsets$ is then\/ $\cohdesirsets\group{\rejectset}\coloneqq\cset{\desirset\in\cohdesirsets}{\rejectset\subseteq\rejectset[\desirset]}$.
\end{theorem}
\noindent 
Due to the one-to-one correspondence between coherent sets of desirable options~$\desirset$ and coherent preference orders~$\prefgt[\desirset]$, this representation result tells us that working with a coherent set of desirable option sets $\rejectset$ is equivalent to working with the set of those coherent preference orders~$\prefgt[\desirset]$ for which $\rejectset\subseteq\rejectset[\desirset]$. 
For the rejection function $\rejectfun$ that corresponds to $\rejectset$ through Equation~\eqref{eq:interpretation:rejectfuns:after:irreflexivity:and:additivity:intermsof:K:Gert}, $\opt\in\rejectfun(\optset)$ means that $\opt$ is dominated in $\optset$ for all these representing coherent preference orders $\prefgt[\desirset]$. 
Similarly, $\opt\in\choicefun[\rejectfun](\optset)$ means that $\opt$ is undominated according to at least one of these representing coherent preference orders $\prefgt[\desirset]$.
This effectively tells us that our coherence axioms~\ref{ax:rejects:removezero}--\ref{ax:rejects:mono} for choice models characterise a generalised type of choice under Levi's notion of E-admissibility \cite{levi1980a,troffaes2007,vancamp2015:indifference}, but with representing preference orders $\prefgt[\desirset]$ that need not be total orders based on comparing expectations.

Interestingly, any potential property of sets of desirable option sets that is preserved under taking arbitrary intersections, and that the binary choice models satisfy, is inherited from the binary models through the representation result of Theorem~\ref{theo:coherentrepresentation:twosided}. 
It is easy to see that this applies in particular to Aizermann's condition \cite{aizerman1985}.

\section{Imposing totality}\label{sec:totality}

We have just shown that every coherent choice model~$\rejectset$ can be represented by a collection of coherent sets of desirable options~$\desirset$. 
This leads us to wonder whether it is possible to achieve representation using only particular types of coherent~$\desirset$, and, if yes, for which types of coherent sets of desirable option sets~$\rejectset$---and hence for which types of rejection functions~$\rejectfun$ and choice functions~$\choicefun$---this is possible.
In this section, we clear the air by starting with a rather simple case, where we restrict attention to \emph{total} sets of desirable options~$\desirset$, corresponding to total preference orders $\prefgt[\desirset]$. 

\begin{definition}[Totality for sets of desirable options]\label{def:totaldesirs}
We call a set of desirable options $\desirset\in\desirsets$ \emph{total} if it is coherent and
\begin{enumerate}[label=$\mathrm{D}_{\mathrm{T}}$.,ref=$\mathrm{D}_{\mathrm{T}}$,leftmargin=*]
\item\label{ax:desirs:totality} for all $\opt\in\opts\setminus\set{0}$, either $\opt\in\desirset$ or $-\opt\in\desirset$.
\end{enumerate}
The set of all total sets of desirable options is denoted by $\totcohdesirsets$.
\end{definition}
\noindent
That the binary preference order $\prefgt[\desirset]$ corresponding to a total set of desirable options $\desirset$ is indeed a total order can be seen as follows. 
For all $\opt,\altopt\in\opts$ such that $\opt\neq\altopt$, the property~\ref{ax:desirs:totality} implies that either $\opt-\altopt\in\desirset$ or $\altopt-\opt\in\desirset$. 
Hence, for all $\opt,\altopt\in\opts$, we have that either $\opt=\altopt$, $\opt\prefgt[\desirset]\altopt$ or $\altopt\prefgt[\desirset]\opt$, which indeed makes $\prefgt[\desirset]$ a total order.

It was shown in~\cite{couso2011,cooman2010} that what we call \emph{total} sets of desirable options here, are precisely the \emph{maximal} or undominated coherent ones, i.e.~those coherent $\desirset\in\cohdesirsets$ that are not included in any other coherent set of desirable option sets: $\group{\forall\desirset'\in\cohdesirsets}\group{\desirset\subseteq\desirset'\then\desirset=\desirset'}$.
The question of which types of binary sets of desirable option sets $\rejectset[\desirset]$ the total $\desirset$ correspond to, is answered by the following definition and proposition.

\begin{definition}[Totality for sets of desirable option sets]\label{def:totalrejects}
We call a set of desirable option sets $\rejectset\in\rejectsets$ \emph{total} if it is coherent and
\begin{enumerate}[label=$\mathrm{K}_{\mathrm{T}}$.,ref=$\mathrm{K}_{\mathrm{T}}$,leftmargin=*]
\item\label{ax:rejects:totality} 
$\set{\opt,-\opt}\in\rejectset$ for all $\opt\in\opts\setminus\set{0}$.
\end{enumerate}
The set of all total sets of desirable options is denoted by $\totcohrejectsets$.
\end{definition}

\begin{proposition}\label{prop:totalbinaryiff}
For any set of desirable options $\desirset\in\desirsets$, $\desirset$ is total if and only if~$\rejectset[\desirset]$ is, so $\rejectset[\desirset]\in\totcohrejectsets\ifandonlyif\desirset\in\totcohdesirsets$.
\end{proposition}
\noindent
So a binary $\rejectset$ is total if and only if its corresponding $\desirset[\rejectset]$ is. 
For general total sets of desirable option sets $\rejectset\in\totcohrejectsets$, which are not necessarily binary, we nevertheless still have representation in terms of total binary ones.

\begin{theorem}\label{theo:totalrepresentation:twosided}
A set of desirable option sets $\rejectset\in\rejectsets$ is total if and only if there is a non-empty set $\setofdesirsets\subseteq\totcohdesirsets$ of total sets of desirable options such that $\rejectset=\bigcap\cset{\rejectset[\desirset]}{\desirset\in\setofdesirsets}$. 
The largest such set $\setofdesirsets$ is then\/ $\totcohdesirsets\group{\rejectset}\coloneqq\cset{\desirset\in\totcohdesirsets}{\rejectset\subseteq\rejectset[\desirset]}$.
\end{theorem}
\noindent
This representation result shows that our total choice models correspond a generalised type of choice under Levi's notion of E-admissibility \cite{levi1980a,troffaes2007}, but with representing preference orders $\prefgt[\desirset]$ that are now maximal, or undominated.
They correspond what Van Camp et al.~\cite[Section~4]{vancamp2015:indifference} have called \emph{M-admissible} choice models. 
Our discussion above provides an axiomatic characterisation for such choice models.

We conclude our study of totality by characterising what it means for a rejection function to be total.

\begin{proposition}\label{prop:totalityforrejectionfunctions}
Consider any set of desirable option sets~$\rejectset\in\rejectsets$ and any rejection function~$\rejectfun$ that are connected by Equation~\eqref{eq:interpretation:rejectfuns:after:irreflexivity:and:additivity:intermsof:K:Gert}. 
Then $\rejectset$ is total if and only if $\rejectfun$ is coherent and satisfies
\begin{enumerate}[label=$\mathrm{R}_{\mathrm{T}}$.,ref=$\mathrm{R}_{\mathrm{T}}$,leftmargin=*]
\item\label{ax:rejectfun:total} $0\in\rejectfun\group{\set{0,\opt,-\opt}}$, for all $\opt\in\opts\setminus\set{0}$.
\end{enumerate}
\end{proposition}

\section{Imposing the mixing property}\label{sec:mixture}

Totality is, of course, a very strong requirement, and it leads to a very special and restrictive type of representation.
We therefore now turn to weaker requirements, and their consequences for representation.
One such additional property, which sometimes pops up in the literature about choice and rejection functions, is the following \emph{mixing property} \cite{seidenfeld2010,2017vancamp:phdthesis}, which asserts that an option that is rejected continues to be rejected if one removes mixed options---convex combinations of other options in the option set:\footnote{Van Camp \cite{2017vancamp:phdthesis} refers to this property as `convexity', but we prefer to stick to the original name suggested by Seidenfeld et al.~\cite{seidenfeld2010} for the sake of avoiding confusion. We nevertheless want to point out that in a context that focuses on rejection rather than choice, the term `\emph{unmixing}' would be preferable, because the rejection is preserved under \emph{removing} mixed options---whereas the choice is preserved under \emph{adding} mixed options.}
\begin{enumerate}[label=$\mathrm{R}_{\mathrm{M}}$.,ref=$\mathrm{R}_{\mathrm{M}}$,leftmargin=*,start=5]
\item\label{ax:rejectfun:removepositivecombinations} 
if $\optset\subseteq\altoptset\subseteq\chull\group{\optset}$ then also $\rejectfun\group{\altoptset}\cap\optset\subseteq\rejectfun\group{\optset}$, for all $\optset,\altoptset\in\optsets$, 
\end{enumerate}
where $\chull\group{\cdot}$ is the \emph{convex hull operator}, defined by
\begin{equation}\label{eq:convexhulloperator}
\chull(V)
\coloneqq\cset[\bigg]{\sum_{k=1}^n\lambda_k\opt[k]}
{n\in\naturals,\lambda_k\in\posreals,\sum_{k=1}^n\lambda_k=1,\opt[k]\in V}
\text{ for all $V\subseteq\opts$}.
\end{equation}
$\naturals$ is the set of natural numbers, or in other words all positive integers, excluding~$0$, and $\posreals$ is the set of all (strictly) positive reals.
A rejection function that satisfies this mixing property is called \emph{mixing}.

The following result characterises the mixing property in terms of the corresponding set of desirable option sets. 
We provide two equivalent conditions: one in terms of the convex hull operator, and one in terms of the $\posi\group{\cdot}$ operator, which, for any subset~$V$ of~$\opts$, returns the set of all positive linear combinations of its elements:
\begin{equation}\label{eq:posioperator}
\posi(V)\coloneqq\cset[\bigg]{\sum_{k=1}^n\lambda_k\opt[k]}{n\in\naturals,\lambda_k\in\posreals,\opt[k]\in V}.
\end{equation}

\begin{proposition}\label{prop:axioms:rejection:sets:and:functions:mixing}
Consider any set of desirable option sets~$\rejectset\in\rejectsets$ and any rejection function~$\rejectfun$ that are connected by Equation~\eqref{eq:interpretation:rejectfuns:after:irreflexivity:and:additivity:intermsof:K:Gert}. 
Then $\rejectfun$ is coherent and mixing if and only if~$\rejectset$ is coherent and satisfies any---and hence both---of the following conditions:
\begin{enumerate}[label=$\mathrm{K}_{\mathrm{M}}$.,ref=$\mathrm{K}_{\mathrm{M}}$,leftmargin=*]
\item\label{ax:rejects:removepositivecombinations} if $\altoptset\in\rejectset$ and $\optset\subseteq\altoptset\subseteq\posi\group{\optset}$, then also $\optset\in\rejectset$, for all $\optset,\altoptset\in\optsets$;
\end{enumerate}
\begin{enumerate}[label=$\mathrm{K}'_{\mathrm{M}}$.,ref=$\mathrm{K}'_{\mathrm{M}}$,leftmargin=*]
\item\label{ax:rejects:removeconvexcombinations} if $\altoptset\in\rejectset$ and $\optset\subseteq\altoptset\subseteq\chull\group{\optset}$, then also $\optset\in\rejectset$, for all $\optset,\altoptset\in\optsets$.
\end{enumerate}
\end{proposition}
\noindent
In the context of sets of desirable options in linear spaces, we prefer to use the $\posi\group{\cdot}$ operator because it fits more naturally with Axiom~\ref{ax:desirability:cone}. 
We therefore adopt Axiom~\ref{ax:rejects:removepositivecombinations} as our definition for mixingness.

\begin{definition}[Mixing property for sets of desirable option sets]\label{def:mixingrejects}
We call a set of desirable option sets $\rejectset\in\rejectsets$ \emph{mixing} if it is coherent and satisfies \ref{ax:rejects:removepositivecombinations}. 
The set of all mixing sets of desirable options is denoted by $\convcohrejectsets$.
\end{definition}

We now proceed to show that these mixing sets of desirable option sets allow for a representation in terms of sets of desirable options that are themselves mixing, in the following sense.

\begin{definition}[Mixing property for sets of desirable options]\label{def:mixingdesirs}
We call a set of desirable options $\desirset\in\desirsets$ \emph{mixing} if it is coherent and 
\begin{enumerate}[label=$\mathrm{D}_{\mathrm{M}}$.,ref=$\mathrm{D}_{\mathrm{M}}$,leftmargin=*]
\item for all $\optset\in\optsets$, if $\posi(\optset)\cap\desirset\neq\emptyset$, then also $\optset\cap\desirset\neq\emptyset$.\label{ax:desirs:mixing}
\end{enumerate}
We denote the set of all mixing sets of desirable options by $\convcohdesirsets$.
\end{definition}

The \emph{binary} elements of $\convcohrejectsets$ are precisely the ones based on such a mixing set of desirable options; they can be represented by a single element of $\convcohdesirsets$.

\begin{proposition}\label{prop:mixingbinaryiffD}
For any set of desirable options $\desirset\in\desirsets$, $\rejectset[\desirset]$ is mixing if and only if~$\desirset$ is, so $\rejectset[\desirset]\in\convcohrejectsets\ifandonlyif\desirset\in\convcohdesirsets$.
\end{proposition}
\noindent
For general mixing sets of desirable option sets that are not necessarily binary, we nevertheless still obtain a representation theorem analogous to Theorems~\ref{theo:coherentrepresentation:twosided} and~\ref{theo:totalrepresentation:twosided}, where the representing sets of desirable options are now mixing.

\begin{theorem}\label{theo:mixingrepresentation:twosided}
A set of desirable option sets $\rejectset\in\rejectsets$ is mixing if and only if there is a non-empty set $\setofdesirsets\subseteq\convcohdesirsets$ of mixing sets of desirable options such that $\rejectset=\bigcap\cset{\rejectset[\desirset]}{\desirset\in\setofdesirsets}$. 
The largest such set $\setofdesirsets$ is then\/ $\convcohdesirsets\group{\rejectset}\coloneqq\cset{\desirset\in\convcohdesirsets}{\rejectset\subseteq\rejectset[\desirset]}$.
\end{theorem}
\noindent 
This representation result is akin to the one proved by Seidenfeld et al \cite{seidenfeld2010}, but without the additional condition of weak Archimedeanity they impose.
In order to better explain this, and to provide this result with some extra intuition, we take a closer look at the mixing sets of desirable options that make up our representation. 
The following result is an equivalent characterisation of such sets.

\begin{proposition}\label{prop:mixingequalslexicographic}
Consider any set of desirable options $\desirset\in\cohdesirsets$ and let $\desirset^\mathrm{c}\coloneqq\opts\setminus\desirset$. 
Then $\desirset$ is mixing if and only if\/ $\posi(\desirset^\mathrm{c})=\desirset^\mathrm{c}$.
\end{proposition}
\noindent
So we see that the coherent sets of desirable options that are also mixing are precisely those whose complement is again a convex cone.\footnote{Recall that coherent sets of desirable options are convex cones because of Axiom~\ref{ax:desirs:cone}.} 
They are therefore identical to the \emph{lexicographic} sets of desirable options sets introduced by Van Camp et al.~\cite{2017vancamp:phdthesis,2018vancamp:lexicographic}. 
What makes this particularly relevant and interesting is that these authors have shown that when $\opts$ is the set of all gambles on some finite set~$\states$ and $\posopts=\cset{\opt\in\opts}{\opt\geq0\text{ and }\opt\neq0}$, then the sets of desirable options in $\cohdesirsets$ that are lexicographic---and therefore mixing---are exactly the ones that are representable by some lexicographic probability system that has no non-trivial Savage-null events. 
This is, of course, the reason why they decided to call such coherent sets of desirable options \emph{lexicographic}.
Because of this connection, it follows that in their setting, Theorem~\ref{theo:mixingrepresentation:twosided} implies that every mixing choice model can be represented by a set of lexicographic probability systems. 

Due to the equivalence between coherent lexicographic sets of desirable options and mixing ones on the one hand, and between total sets of desirable options and maximal coherent ones on the other, the following proposition is an immediate consequence of a similar result by Van Camp et al.~\cite{2017vancamp:phdthesis,2018vancamp:lexicographic}. 
It shows that the total sets of desirable options constitute a subclass of the mixing ones: mixingness is a weaker requirement than totality.

\begin{proposition}\label{prop:mixingiftotal}
Every total set of desirable options is mixing: $\totcohdesirsets\subseteq\convcohdesirsets$.
\end{proposition}
\noindent
By combining this result with Theorems~\ref{theo:totalrepresentation:twosided} and~\ref{theo:mixingrepresentation:twosided}, it follows that every total set of desirable options sets is mixing, and similarly for rejection and choice functions. 
So mixingness is implied by totality for non-binary choice models as well. 
Since totality is arguably the more intuitive of the two, one might therefore be inclined to discard the mixing property in favour of totality. 
We have nevertheless studied the mixing property in some detail, because it can be combined with other properties, such as the notions of Archimedeanity studied in the next section. 
As we will see, this combination leads to very intuitive representation, where the role of lexicographic probability systems is taken over by expectation operators---called linear previsions.

\section{Imposing Archimedeanity}\label{sec:archimedeanity}

There are a number of ways a notion of Archimedeanity can be introduced for preference relations and choice models \cite{aumann1962,aumann1964,seidenfeld1995,seidenfeld2010,nau2006}. 
Its aim is always to guarantee that the real number system is expressive enough, or more precisely, that the preferences expressed by the models can be represented by (sets of) \emph{real-valued} probabilities and utilities, rather than, say, probabilities and utilities expressed using hyper-reals.
Here, we consider a notion of Archimedeanity that is close in spirit to an idea explored by Walley \cite{walley1991,walley2000} in his discussion of so-called \emph{strict desirability}.

For the sake of simplicity, we will restrict ourselves to a particular case of our abstract framework,\footnote{It is possible to introduce a version of our notion of Archimedeanity in our general framework as well, but explaining how this works would take up much more space than we are allowed in this conference paper.} where $\opts\coloneqq\gbls(\states)$ is the set of all gambles on a set of states $\states$ and $\posopts=\strictposgbls(\states)\coloneqq\cset{\opt\in\gbls(\states)}{\inf\opt>0}$. 
We identify every real number $\mu\in\reals$ with the constant gamble that takes the value $\mu$, and then define Archimedeanity as follows.

\begin{definition}[Archimedean set of desirable options]\label{def:gamblearchimedean:sets}
We call a set of desirable options $\desirset\in\desirsets$ \emph{Archimedean} if it is coherent and satisfies the following openness condition:
\begin{enumerate}[label=$\mathrm{D}_{\mathrm{A}}$.,ref=$\mathrm{D}_{\mathrm{A}}$,leftmargin=*]
\item\label{ax:desirs:gamblearch} for all $\opt\in\desirset$, there is some $\epsilon\in\posreals$ such that $\opt-\epsilon\in\desirset$.
\end{enumerate}
We denote the set of all Archimedean sets of desirable options by $\archcohdesirsets$, and let $\archconvcohdesirsets$ be the set of all Archimedean sets of desirable options that are also mixing.
\end{definition}

What makes Archimedean and mixing Archimedean sets of desirable options particularly interesting, is that they are in a one-to-one correspondence with coherent lower previsions and linear previsions \cite{walley1991}, respectively.

\begin{definition}[Coherent lower prevision and linear prevision]\label{def:cohlp}
A \emph{coherent lower prevision} $\lowprev$ on $\gblson{\states}$ is a real-valued map on $\gbls(\states)$ that satisfies
\begin{enumerate}[label=$\mathrm{LP}_{\arabic*}$.,ref=$\mathrm{LP}_{\arabic*}$,leftmargin=*]
\item\label{ax:lowprev:inf} $\lowprev(\opt)\geq\inf\opt$ for all $\opt\in\gblson{\states}$;
\item\label{ax:lowprev:lambda} $\lowprev(\lambda\opt)=\lambda\lowprev(\opt)$ for all $\opt\in\gblson{\states}$ and $\lambda\in\posreals$;
\item\label{ax:lowprev:super} $\lowprev(\opt+\altopt)\geq\lowprev(\opt)+\lowprev(\altopt)$ for all $\opt,\altopt\in\gblson{\states}$;
\end{enumerate}
A \emph{linear prevision} $\linprev$ on $\gblson{\states}$ is a coherent lower prevision that additionally satisfies
\begin{enumerate}[label=$\mathrm{P}_{\arabic*}$.,ref=$\mathrm{P}_{\arabic*}$,leftmargin=*,start=3]
\item\label{ax:linprev:additive}$\linprev(\opt+\altopt)=\linprev(\opt)+\linprev(\altopt)$ for all $\opt,\altopt\in\gblson{\states}$;
\end{enumerate}
We denote the set of all coherent lower previsions on $\gblson{\states}$ by $\cohlowprevs$ and let $\linprevs$ be the subset of all linear previsions.
\end{definition}

In order to make the above-mentioned one-to-one correspondences explicit, we introduce the following maps.
With any set of desirable options $\desirset$ in $\desirsets$, we associate a (possibly extended) real functional $\lowprev[\desirset]$ on $\gblson{\states}$, defined by
\begin{equation}\label{eq:lpfromdesir}
\lowprev[\desirset](\opt)\coloneqq\sup\cset{\mu\in\reals}{\opt-\mu\in\desirset},
\text{ for all $\opt\in\gblson{\states}$}.
\end{equation}
Conversely, with any (possibly extended) real functional $\lowprev$ on $\gblson{\states}$, we associate a set of desirable options
\begin{equation}\label{eq:desirfromlp}
\desirset[\lowprev]\coloneqq\cset{\opt\in\gblson{\states}}{\lowprev(\opt)>0}.
\end{equation}
Our next result shows that these two maps lead to an isomorphism between $\archcohdesirsets$ and $\cohlowprevs$, and similarly for $\archconvcohdesirsets$ and $\linprevs$.

\begin{proposition}\label{prop:isomorphismsmarch}
For any Archimedean set of desirable options $\desirset$, $\lowprev[\desirset]$ is a coherent lower prevision on $\gbls(\states)$ and $\desirset[{\lowprev[\desirset]}]=\desirset$. 
If $\desirset$ is moreover mixing, then $\lowprev[\desirset]$ is a linear prevision. 
Conversely, for any coherent lower prevision $\lowprev$ on $\gbls(\states)$, $\desirset[\lowprev]$ is an Archimedean set of desirable options and $\lowprev[{\desirset[\lowprev]}]=\lowprev$. 
If $\lowprev$ is furthermore a linear prevision, then $\desirset[\lowprev]$ is mixing.
\end{proposition}

The import of these correspondences is that any representation in terms of sets of Archimedean (mixing) sets of desirable options is effectively a representation in terms of sets of coherent lower (or linear) previsions. 
As we will see, these kinds of representations can be obtained for sets of desirable option sets---and hence also rejection and choice functions---that are themselves Archimedean in the following sense.

\begin{definition}[Archimedean set of desirable option sets]\label{def:gamblearchimedean:setsofsets}
We call a set of desirable option sets $\rejectset\in\rejectsets$ \emph{Archimedean} if it is coherent and satisfies
\begin{enumerate}[label=$\mathrm{K}_{\mathrm{A}}$.,ref=$\mathrm{K}_{\mathrm{A}}$,leftmargin=*]
\item\label{ax:rejects:gamblearch} for all $\optset\in\rejectset$, there is some $\epsilon\in\posreals$ such that $\optset-\epsilon\in\rejectset$.
\end{enumerate}
We denote the set of all Archimedean sets of desirable option sets by $\archcohrejectsets$, and let $\archconvcohrejectsets$ be the set of all Archimedean sets of desirable options that are also mixing.
\end{definition}
\noindent
This notion easily translates from sets of desirable option sets to rejection functions.

\begin{proposition}
\label{prop:gamblearchimedeanityforrejectionfunctions}
Consider any set of desirable option sets~$\rejectset\in\rejectsets$ and any rejection function~$\rejectfun$ that are connected by Equation~\eqref{eq:interpretation:rejectfuns:after:irreflexivity:and:additivity:intermsof:K:Gert}. 
Then $\rejectset$ is Archimedean if and only if $\rejectfun$ is coherent and satisfies
\begin{enumerate}[label=$\mathrm{R}_{\mathrm{A}}$.,ref=$\mathrm{R}_{\mathrm{A}}$,leftmargin=*]
\item\label{ax:rejectfun:gamblearch} for all $\optset\in\optsets$ and $\opt\in\opts$ such that $\opt\in\rejectfun(\optset\cup\set{\opt})$, there is some $\epsilon\in\posreals$ such that $\opt\in\rejectfun((\optset-\epsilon)\cup\set{\opt})$.
\end{enumerate}
\end{proposition}

A first and basic result is that our notion of Archimedeanity for sets of desirable option sets is compatible with that for sets of desirable options.

\begin{proposition}\label{prop:gamblearchbinaryiffD}
For any set of desirable options $\desirset\in\desirsets$, $\rejectset[\desirset]$ is Archimedean (and mixing) if and only if~$\desirset$ is, so $\rejectset[\desirset]\in\archcohrejectsets\ifandonlyif\desirset\in\archcohdesirsets$ and $\rejectset[\desirset]\in\archconvcohrejectsets\ifandonlyif\desirset\in\archconvcohdesirsets$.
\end{proposition}

In order to state our representation results for Archimedean choice models that are not necessarily binary, we require a final piece of machinery: a topology on $\archcohdesirsets$ and $\archconvcohdesirsets$, or equivalently, a notion of closedness. 
We do this by defining the convergence of a sequence of Archimedean sets of desirable options $\set{\desirset[n]}_{n\in\naturals}$ in terms of the point-wise convergence of the corresponding sequence of coherent lower previsions:
\begin{equation}\label{eq:archimedeanlimit}
\lim_{n\to+\infty}\desirset[n]=\desirset
\ifandonlyif
\group[\big]{\forall\opt\in\gblson{\states}}\lim_{n\to+\infty}\lowprev[{\desirset[n]}](\opt)=\lowprev[\desirset](\opt).
\end{equation}
A set $\setofdesirsets\subseteq\archcohdesirsets$ of Archimedean sets of desirable options is then called \emph{closed} if it contains all of its limit points, or equivalently, if the corresponding set of coherent lower previsions---or linear previsions when $\setofdesirsets\subseteq\archconvcohdesirsets$---is closed with respect to point-wise convergence.

Our final representation results state that a set of desirable option sets $\rejectset$ is Archimedean if and only if it can be represented by such a closed set, and if $\rejectset$ is moreover mixing, the elements of the representing closed set are as well.

\begin{theorem}[Representation for Archimedean choice functions]\label{theo:rejectsets:representation:gamblearch:twosided}
A set of desirable option sets $\rejectset\in\rejectsets$ is Archimedean if and only if there is some non-empty closed set $\setofdesirsets\subseteq\archcohdesirsets$ of Archimedean sets of desirable options such that $\rejectset=\bigcap\cset{\rejectset[\desirset]}{\desirset\in\setofdesirsets}$. 
The largest such set~$\setofdesirsets$ is then\/ $\archcohdesirsets(\rejectset)\coloneqq\cset{\desirset\in\archcohdesirsets}{\rejectset\subseteq\rejectset[\desirset]}$.
\end{theorem}

\begin{theorem}[Representation for Archimedean mixing choice functions]\label{theo:convex:rejectsets:representation:gamblearch:twosided}
A set of desirable option sets $\rejectset\in\rejectsets$ is mixing and Archimedean if and only if there is some non-empty closed set $\setofdesirsets\subseteq\archconvcohdesirsets$ of mixing and Archimedean sets of desirable options such that $\rejectset=\bigcap\cset{\rejectset[\desirset]}{\desirset\in\setofdesirsets}$. 
The largest such set~$\setofdesirsets$ is then\/ $\archconvcohdesirsets(\rejectset)\coloneqq\cset{\desirset\in\archconvcohdesirsets}{\rejectset\subseteq\rejectset[\desirset]}$.
\end{theorem}

If we combine Theorem~\ref{theo:convex:rejectsets:representation:gamblearch:twosided} with the correspondence result of Proposition~\ref{prop:isomorphismsmarch}, we see that Axioms~\ref{ax:rejects:removezero}--\ref{ax:rejects:mono} together with~\ref{ax:rejects:removepositivecombinations} and~\ref{ax:rejects:gamblearch} characterise exactly those choice models that are based on E-admissibility with respect to a closed---but not necessarily convex---set of linear previsions. 
In much the same way, Theorem~\ref{theo:rejectsets:representation:gamblearch:twosided} can be seen to characterise a generalised notion of E-admissibility, where the representing objects are coherent lower previsions. 
Walley--Sen maximality \cite{troffaes2007,walley1991} can be regarded as a special case of this generalised notion, where only a single representing coherent lower prevision is needed.

\section{Conclusion}\label{sec:conclusion}

The main conclusion of this paper is that the language of desirability \emph{is} capable of representing non-binary choice models, provided we extend it with a notion of disjunction, allowing statements such as `at least one of these two options is desirable'. 
When we do so, the resulting framework of sets of desirable options turns out to be a very flexible and elegant tool for representing set-valued choice. 
Not only does it include E-admissibility and maximality, it also opens up a range of other types of choice functions that have so far received little to no attention. 
All of these can be represented in terms of sets of strict preference orders or---if additional properties are imposed---in terms of sets of strict total orders, sets of lexicographic probability systems, sets of coherent lower previsions or sets of linear previsions.

Another important conclusion is that our axiomatisation for general (possibly non-binary) choice models allows for representations in terms of `atomic' models, which in themselves represent binary choice.
However, this should of course not be taken to mean that our choice models are essentially binary.
Indeed, it follows readily from our representation theorems that the binary aspects $\desirset[\rejectset]$ of a non-binary choice model $\rejectset$ are captured by the intersections of the representing sets of desirable options, but the representation is much more powerful than that, because it also extends to the non-binary aspects of choice.

This distinction between the binary level and the non-binary one also leads us to the following important words of caution, which are akin to an earlier observation made by Quaeghebeur \cite{quaeghebeur2015:partial:partial}. 
At the binary level, choice is represented by a set of desirable options, which can---under certain assumptions such as Archimedeanity---be identified with a convex closed set of linear previsions.
We have also seen in Theorem~\ref{theo:convex:rejectsets:representation:gamblearch:twosided} that the (binary \emph{and} non-binary) aspects of mixing and Archimedean choice can be fully represented by a closed set of mixing and Archimedean sets of desirable options, each of which is, by Proposition~\ref{prop:isomorphismsmarch}, equivalent to a linear prevision.
So, in this case there is a representation in terms of a set of linear previsions both at the binary level and at the general (binary \emph{and} non-binary) level, but these two sets of linear previsions will typically be different, and they play a very different role.
To put it very bluntly: sets of linear previsions à la Walley \cite{walley1991} should not be confused with sets of linear previsions---credal sets---à la Levi \cite{levi1980a}.  

To conclude, what we have done here, in a very specific sense, is to introduce a way of dealing with statements of the type `there is some option in the option set~$A$ that is strictly preferred to~$0$'.
Axioms such as \ref{ax:rejects:removezero}--\ref{ax:rejects:mono} can then be seen as the logical axioms---for deriving new statements from old---that govern this language.
Our representation theorems provide a \emph{semantics} for this language in terms of desirability, and they show that the corresponding logical system is sound and complete.

In our future work on this topic, we intend to investigate how we can let go of the closedness condition in Theorems~\ref{theo:rejectsets:representation:gamblearch:twosided} and~\ref{theo:convex:rejectsets:representation:gamblearch:twosided}. 
We expect to have to turn to other types of Archimedeanity; variations on Seidenfeld et al.'s weak Archimedeanity \cite{seidenfeld2010,cozman2018:evenlyconvex} come to mind. 
We also intend to show in more detail how the existing work on choice models for horse lotteries \cite{seidenfeld2010} fits nicely within our more general and abstract framework of choice on linear option spaces.
And finally, we intend to further develop conservative inference methods for coherent choice functions, by extending our earlier natural extension results~\cite{debock2018} to the more general setting that we have considered here.

\section*{Acknowledgements}
This work owes a large intellectual debt to Teddy Seidenfeld, who introduced us to the topic of choice functions.
His insistence that we ought to pay more attention to non-binary choice if we wanted to take imprecise probabilities seriously, is what eventually led to this work.

The discussion in Arthur Van Camp's PhD~thesis \cite{2017vancamp:phdthesis} was the direct inspiration for our work here, and we would like to thank Arthur for providing a pair of strong shoulders to stand on.

As with most of our joint work, there is no telling, after a while, which of us two had what idea, or did what, exactly. 
We have both contributed equally to this paper.
But since a paper must have a first author, we decided it should be the one who took the first significant steps: Jasper, in this case.

\bibliographystyle{plain}
\bibliography{general}

\appendix
\section{Proofs and intermediate results}\label{app:proofs}


\subsection{Terminology and notation used only in the appendix}\label{sec:extraterminology}
For any subset $V$ of $\opts$ we consider its set of linear combinations, or linear span
\begin{equation*}
\linspan(V)\coloneqq\cset[\bigg]{\sum_{k=1}^n\lambda_k\opt[k]}{n\in\naturals,\lambda_k\in\reals,\opt[k]\in V}.
\end{equation*}
We also consider several operators on---transformations of---the set~$\rejectsets$ of all sets of desirable option sets. 
The first is denoted by~$\RN\group{\cdot}$, and allows us to add smaller option sets by removing from any option set elements of $\nonposopts\coloneqq\cset{\opt\in\opts}{\opt\optlteq0}$:
\begin{equation*}
\RN\group{\rejectset}\coloneqq\cset{\optset\in\optsets}{(\exists\altoptset\in\rejectset)\altoptset\setminus\nonposopts\subseteq\optset\subseteq\altoptset},
\text{ for all $\rejectset\in\rejectsets$}.
\end{equation*}
The second is denoted by~$\RS\group{\cdot}$, and allows us to add smaller option sets by removing from any option set positive combinations from some of its other elements:
\begin{equation*}
\RP\group{\rejectset}
\coloneqq\cset{\optset\in\optsets}{(\exists\altoptset\in\rejectset)\optset\subseteq\altoptset\subseteq\posi\group{\optset}},
\text{ for all $\rejectset\in\rejectsets$}.
\end{equation*}
The final one is denoted by $\setposi\group{\cdot}$---not to be confused with $\posi\group{\cdot}$---and defined for all $\rejectset\in\rejectsets$ by
\begin{align}
\setposi\group{\rejectset}\coloneqq\bigg\{
\bigg\{
\sum_{k=1}^n\lambda_{k}^{\opt[1:n]}\opt[k]
\colon
\opt[1:n]\in\times_{k=1}^n\optset[k]
\bigg\}
\colon
&n\in\naturals,(\optset[1],\dots,\optset[n])\in\rejectset^n,\notag\\[-11pt]
&\big(\forall\opt[1:n]\in\times_{k=1}^n\optset[k]\big)\,\lambda_{1:n}^{\opt[1:n]}>0
\bigg\},
\label{eq:setposi}
\end{align}
where we use the notations $\opt[1:n]$ and $\lambda_{1:n}^{\opt[1:n]}$ for $n$-tuples of options $\opt[k]$ and real numbers $\lambda_{k}^{\opt[1:n]}$, $k\in\set{1,\dots,n}$, so $\opt[1:n]\in\opts^{n}$ and $\lambda_{1:n}^{\opt[1:n]}\in\reals^{n}$.
We also use `$\lambda_{1:n}^{\opt[1:n]}>0$' as a shorthand for `$\lambda_k^{\opt[1:n]}\geq0$ for all $k\in\set{1,\dots,n}$ and $\sum_{k=1}^n\lambda_k^{\opt[1:n]}>0$'. 

\subsection{Proofs and intermediate results for Section~\ref{sec:binary:choice}}

\begin{proof}[Proof of Proposition~\ref{prop:binary:characterisation}]
First assume that there is some $\desirset\in\desirsets$ such that $\rejectset=\rejectset[\desirset]$. 
Then for all $\optset\in\optsets$:
\begin{align*}
\optset\in\rejectset
\ifandonlyif
\optset\in\rejectset[\desirset]
\ifandonlyif
\optset\cap\desirset\neq\emptyset
&\ifandonlyif
(\exists\opt\in\optset)\set{\opt}\cap\desirset\neq\emptyset\\
&\ifandonlyif
(\exists\opt\in\optset)\set{\opt}\in\rejectset[\desirset]
\ifandonlyif
(\exists\opt\in\optset)\set{\opt}\in\rejectset.
\end{align*}
It therefore follows from Definition~\ref{def:binary} that $\rejectset$ is binary. 

Furthermore, for any $\opt\in\opts$, we find that
\begin{equation*}
\opt\in\desirset
\ifandonlyif
\set{\opt}\cap\desirset\neq\emptyset
\ifandonlyif
\set{\opt}\in\rejectset[\desirset]
\ifandonlyif
\set{\opt}\in\rejectset
\ifandonlyif
\opt\in\desirset[\rejectset].
\end{equation*}
So we find that $\desirset$ is equal to $\desirset[\rejectset]$, and therefore necessarily unique.

Finally, we assume that $\rejectset$ is binary. 
Let $\desirset\coloneqq\desirset[\rejectset]$. 
Then for all $\optset\in\optsets$:
\begin{equation*}
\optset\in\rejectset
\ifandonlyif
(\exists\opt\in\optset)\set{\opt}\in\rejectset
\ifandonlyif
(\exists\opt\in\optset)\opt\in\desirset[\rejectset]
\ifandonlyif
\optset\cap\desirset[\rejectset]\neq\emptyset
\ifandonlyif
\optset\cap\desirset\neq\emptyset
\ifandonlyif\optset\in\rejectset[\desirset],
\end{equation*}
where the first equivalence follows from Definition~\ref{def:binary} and the fact that $\rejectset$ is binary. 
Hence, we find that $\rejectset=\rejectset[\desirset]$.
\end{proof}

\begin{corollary}\label{corol:binaryiff}
A set of desirable option sets $\rejectset\in\rejectsets$ is binary if and only if $\rejectset[{\desirset[\rejectset]}]=\rejectset$. 
\end{corollary}

\begin{proof}
Immediate consequence of Proposition~\ref{prop:binary:characterisation}.
\end{proof}

\begin{proposition}\label{prop:from:rejection:to:desirability}
For any coherent set of desirable option sets $\rejectset$, $\desirset[\rejectset]$ is a coherent set of desirable options, and $\rejectset[{\desirset[\rejectset]}]\subseteq\rejectset$.
\end{proposition}

\begin{proof}
We first prove that $\desirset[\rejectset]$ is coherent, or equivalently, that it satisfies Axioms~\ref{ax:desirs:nozero}--\ref{ax:desirs:cone}.
For Axiom~\ref{ax:desirs:nozero}, observe that $0\in\desirset[\rejectset]$ implies that $\set{0}\in\rejectset$, contradicting Axiom~\ref{ax:rejects:nozero}.
For Axiom~\ref{ax:desirs:pos}, observe that for any $\opt\in\opts$, $\opt\in\desirset[\rejectset]$ is equivalent to $\set{\opt}\in\rejectset$, and take into account Axiom~\ref{ax:rejects:pos}.
And, finally, for Axiom~\ref{ax:desirs:cone}, observe that $\opt,\altopt\in\desirset[\rejectset]$ implies that $\set{\opt},\set{\altopt}\in\rejectset$, and that Axiom~\ref{ax:rejects:cone} then implies that $\set{\lambda\opt+\mu\altopt}\in\rejectset$, or equivalently, that $\lambda\opt+\mu\altopt\in\desirset[\rejectset]$, for any choice of $(\lambda,\mu)>0$.

For the last statement, consider any $\optset\in\rejectset[{\desirset[\rejectset]}]$, meaning that $\optset\cap\desirset[\rejectset]\neq\emptyset$.
Consider any $\opt\in\optset\cap\desirset[\rejectset]$, then on the one hand $\opt\in\desirset[\rejectset]$, so $\set{\opt}\in\rejectset$.
But since on the other hand also $\opt\in\optset$, we see that $\set{\opt}\subseteq\optset$, and therefore Axiom~\ref{ax:rejects:mono} guarantees that $\optset\in\rejectset$.
\end{proof}

\begin{proposition}\label{prop:fromCohDtoCohK}
For any set of desirable options $\desirset\in\desirsets$, $\desirset=\desirset[{\rejectset[\desirset]}]$.
If, moreover, $\desirset$ is coherent, then $\rejectset[\desirset]$ is a coherent set of desirable option sets.
\end{proposition}

\begin{proof}
For the first statement, simply observe that
\begin{equation*}
\opt\in\desirset[{\rejectset[\desirset]}]
\ifandonlyif
\set{\opt}\in\rejectset[\desirset]
\ifandonlyif
\set{\opt}\cap\desirset\neq\emptyset
\ifandonlyif
\opt\in\desirset,
\text{ for all $\opt\in\opts$}.
\end{equation*}
For the second statement, assume that $\desirset$ is coherent, then we  need to prove that $\rejectset[\desirset]$ is coherent, or equivalently, that it satisfies Axioms~\ref{ax:rejects:removezero}--\ref{ax:rejects:mono}.
For Axiom~\ref{ax:rejects:removezero}, observe that $\optset\cap\desirset\neq\emptyset$ implies that $(\optset\setminus\set{0})\cap\desirset\neq\emptyset$ because we know from the coherence of $\desirset$ [Axiom~\ref{ax:desirs:nozero}] that $0\notin\desirset$. 
For Axiom~\ref{ax:rejects:nozero}, observe that Equation~\eqref{eq:desirset:to:rejectset} implies that $\set{0}\in\rejectset[\desirset]\ifandonlyif0\in\desirset$, and use Axiom~\ref{ax:desirs:nozero}. 
For Axiom~\ref{ax:rejects:pos}, observe that $\set{\opt}\in\rejectset[\desirset]$ is equivalent to $\opt\in\desirset$ for all $\opt\in\opts$, and take into account the coherence of $\desirset$ [Axiom~\ref{ax:desirs:pos}].
For Axiom~\ref{ax:rejects:cone}, consider any $\optset[1],\optset[2]\in\rejectset[\desirset]$, and let $\optset\coloneqq\cset{\lambda_{\opt,\altopt}\opt+\mu_{\opt,\altopt}\altopt}{\opt\in\optset[1],\altopt\in\optset[2]}$ for any particular choice of the $(\lambda_{\opt,\altopt},\mu_{\opt,\altopt})>0$ for all $\opt\in\optset[1]$ and $\altopt\in\altopt[2]$.
Then $\optset[1]\cap\desirset\neq\emptyset$ and $\optset[2]\cap\desirset\neq\emptyset$, so we can fix any $\opt[1]\in\optset[1]\cap\desirset$ and $\opt[2]\in\optset[2]\cap\desirset$.
The coherence of $\desirset$ [Axiom~\ref{ax:desirs:cone}] then implies that $\lambda_{\opt[1],\altopt[2]}\opt[1]+\mu_{\opt[1],\altopt[2]}\altopt[2]\in\desirset$, and therefore also $\optset\cap\desirset\neq\emptyset$, whence indeed $\optset\in\rejectset[\desirset]$.
And, finally, that $\rejectset[\desirset]$ satisfies Axiom~\ref{ax:rejects:mono} is an immediate consequence of its definition~\eqref{eq:desirset:to:rejectset}.
\end{proof}

\begin{proof}[Proof of Proposition~\ref{prop:coherence:for:binary}]
We begin with the first statement.
First, suppose that $\desirset[\rejectset]$ is coherent. 
Proposition~\ref{prop:fromCohDtoCohK} then implies that $\rejectset[{\desirset[\rejectset]}]$ is coherent.
Hence, since we know from Proposition~\ref{corol:binaryiff} and the assumed binary character of~$\rejectset$ that $\rejectset=\rejectset[{\desirset[\rejectset]}]$, we find that $\rejectset$ is coherent.
Next, suppose that $\rejectset$ is coherent. 
Proposition~\ref{prop:from:rejection:to:desirability} then implies that $\desirset[\rejectset]$ is coherent as well.

We now turn to the second statement.
First, assume that $\desirset$ is coherent, then Proposition~\ref{prop:fromCohDtoCohK} guarantees that $\rejectset[\desirset]$ is coherent too.
Next, assume that $\rejectset[\desirset]$ is coherent, then we infer from Proposition~\ref{prop:from:rejection:to:desirability} that $\desirset[{\rejectset[\desirset]}]$ is coherent, and from Proposition~\ref{prop:fromCohDtoCohK} that $\desirset[{\rejectset[\desirset]}]=\desirset$. 
\end{proof}

\begin{lemma}\label{lem:binaryalternative}
A coherent set of desirable option sets $\rejectset$ is binary if and only if
\begin{equation}\label{eq:lem:binaryalternative}
\group{\forall\optset\in\rejectset\colon\card{\optset}\geq2}
\group{\exists\opt\in\optset}
\optset\setminus\set{\opt}\in\rejectset.
\end{equation}
\end{lemma}

\begin{proof}
First assume that $\rejectset$ is binary. 
We then know from Corollary~\ref{corol:binaryiff} that $\rejectset=\rejectset[{\desirset[\rejectset]}]$, implying that $\optset\in\rejectset\ifandonlyif\optset\cap\desirset[\rejectset]\neq\emptyset$, for all $\optset\in\optsets$.
Consider any $\optset\in\rejectset$ such that $\card{\optset}\geq2$. 
Then there is some $\altopt\in\optset\cap\desirset[\rejectset]$ such that $\optset=\set{\altopt}\cup(\optset\setminus\set{\altopt})$. 
But then $\card{\optset\setminus\set{\altopt}}\geq1$, so we can consider an element $\opt\in\optset\setminus\set{\altopt}$.
Since clearly $\altopt\in(\optset\setminus\set{\opt})\cap\desirset[\rejectset]$, we see that $(\optset\setminus\set{\opt})\cap\desirset[\rejectset]\neq\emptyset$ and therefore, that $\optset\setminus\set{\opt}\in\rejectset$.

Next assume that Equation~\eqref{eq:lem:binaryalternative} holds. 
Because of Corollary~\ref{corol:binaryiff}, it suffices to show that $\rejectset[{\desirset[\rejectset]}]=\rejectset$. 
We infer from Proposition~\ref{prop:from:rejection:to:desirability} that $\desirset[\rejectset]$ is a coherent set of desirable options, and that $\rejectset[{\desirset[\rejectset]}]\subseteq\rejectset$.
Assume {\itshape ex absurdo} that $\rejectset[{\desirset[\rejectset]}]\subset\rejectset$, so there is some $\optset\in\rejectset$ such that $\optset\notin\rejectset[{\desirset[\rejectset]}]$, or equivalently, such that $\optset\cap\desirset[\rejectset]=\emptyset$.
But then we must have that $\card{\optset}\geq2$, because otherwise $\optset=\set{\altopt}$ with $\altopt\notin\desirset[\rejectset]$ and therefore $\optset=\set{\altopt}\notin\rejectset$, a contradiction.
But then it follows from Equation~\eqref{eq:lem:binaryalternative} that there is some $\opt[1]\in\optset$ such that $\optset[1]\coloneqq\optset\setminus\set{\opt[1]}\in\rejectset$.
Since it follows from $\optset\cap\desirset[\rejectset]=\emptyset$ that also $\optset[1]\cap\desirset[\rejectset]=\emptyset$, we see that also $\optset[1]\notin\rejectset[{\desirset[\rejectset]}]$.
We can now repeat the same argument with $\optset[1]$ instead of $\optset$ to find that it must be that $\card{\optset[1]}\geq2$, so there is some $\opt[2]\in\optset[1]$ such that $\optset[2]\coloneqq\optset[1]\setminus\set{\opt[2]}\in\rejectset$ and $\optset[2]\notin\rejectset[{\desirset[\rejectset]}]$. 
Repeating the same argument over and over again will eventually lead to a contradiction with $\card{\optset[n]}\geq2$.
Hence it must be that $\rejectset[{\desirset[\rejectset]}]=\rejectset$.
\end{proof}

\subsection{Proofs and intermediate results for Section~\ref{sec:representation}}

\begin{lemma}\label{lem:replacing:by:dominating:options}
Consider any set of desirable option sets $\rejectset\in\rejectsets$ that satisfies Axioms~\ref{ax:rejects:pos} and\/~\ref{ax:rejects:cone}.
Consider any $\optset\in\rejectset$.
Then for any $\altopt\in\optset$ and any $\altopt'\in\opts$ such that $\altopt\optlteq\altopt'$, the option set $\altoptset\coloneqq\set{\altopt'}\cup\group{\optset\setminus\set{\altopt}}$ obtained by replacing $\altopt$ in $\optset$ with the dominating option $\altopt'$ still belongs to $\rejectset$: $\altoptset\in\rejectset$.
\end{lemma}

\begin{proof}
We may assume without loss of generality that $\optset\neq\emptyset$ and that $\altopt'\neq\altopt$.
Let $\altopt''\coloneqq\altopt'-\altopt$, then $\altopt''\in\posopts$, and therefore Axiom~\ref{ax:rejects:pos} implies that $\set{\altopt''}\in\rejectset$.
Applying Axiom~\ref{ax:rejects:cone} for $\optset$ and $\set{\altopt''}$ allows us to infer that $\cset{\lambda_{\opt}\opt+\mu_{\opt}\altopt''}{\opt\in\optset}\in\rejectset$ for all possible choices of $(\lambda_{\opt},\mu_{\opt})>0$.
Choosing $(\lambda_{\opt},\mu_{\opt})\coloneqq(1,0)$ for all $\opt\in\optset\setminus\set{\altopt}$ and $(\lambda_{\altopt},\mu_{\altopt})\coloneqq(1,1)$ yields in particular that $\altoptset=\set{\altopt'}\cup(\optset\setminus\set{\altopt})\in\rejectset$.
\end{proof}

\begin{lemma}\label{lem:replacing:nonpositives:by:zero}
Consider any set of desirable option sets $\rejectset\in\rejectsets$ that satisfies Axioms~\ref{ax:rejects:pos} and\/~\ref{ax:rejects:cone}.
Consider any $\optset\in\rejectset$ such that $\optset\cap\nonposopts\neq\emptyset$ and any $\altopt\in\optset\cap\nonposopts$, and construct the option set $\altoptset\coloneqq\set{0}\cup\group{\optset\setminus\set{\altopt}}$ by replacing $\altopt$ with $0$.
Then still $\altoptset\in\rejectset$.
\end{lemma}

\begin{proof}
Immediate consequence of Lemma~\ref{lem:replacing:by:dominating:options}.
\end{proof}

\begin{proposition}\label{prop:ax:rejects:RN:equivalents}
$\RN\group{\rejectset}=\rejectset$ for any coherent set of desirable option sets $\rejectset\in\cohrejectsets$.
\end{proposition}

\begin{proof}
That $\rejectset\subseteq\RN\group{\rejectset}$ is an immediate consequence of the definition of the $\RN$ operator. 
To prove that $\RN\group{\rejectset}\subseteq\rejectset$, consider any $\optset\in\RN\group{\rejectset}$, which means that there is some $\altoptset\in\rejectset$ such that $\altoptset\setminus\nonposopts\subseteq\optset\subseteq\altoptset$. 
We need to prove that $\optset\in\rejectset$. 
Since $\rejectset$ satisfies Axiom~\ref{ax:rejects:mono}, it suffices to prove that $\altoptset\setminus\nonposopts\in\rejectset$.

If $\altoptset\cap\nonposopts=\emptyset$, then $\altoptset\setminus\nonposopts=\altoptset\in\rejectset$. 
Therefore, without loss of generality, we may assume that $\altoptset\cap\nonposopts\neq\emptyset$. 
For any $\opt\in\altoptset\cap\nonposopts$, Lemma~\ref{lem:replacing:nonpositives:by:zero} implies that we may replace $\opt$ by $0$ and still be guaranteed that the resulting set belongs to $\rejectset$. 
Hence, we can replace all elements of $\altoptset\cap\nonposopts$ with $0$ and still be guaranteed that the result $\altoptset'\coloneqq\set{0}\cup(\altoptset\setminus\nonposopts)$ belongs to $\rejectset$.
Applying Axiom~\ref{ax:rejects:removezero} now guarantees that, indeed, $\altoptset\setminus\nonposopts=\altoptset'\setminus\set{0}\in\rejectset$.
\end{proof}

\begin{proposition}\label{prop:removal:of:nonpositives:gewijzigde:axiomas}
Consider any set of desirable option sets $\rejectset\in\rejectsets$.
Then\/ $\RN\group{\rejectset}$ satisfies Axiom~\ref{ax:rejects:removezero}.
Moreover, if $\rejectset$ satisfies Axioms~\ref{ax:rejects:nozero}, \ref{ax:rejects:pos}, \ref{ax:rejects:cone} and\/~\ref{ax:rejects:mono} and does not contain $\emptyset$, then so does\/ $\RN\group{\rejectset}$.
\end{proposition}

\begin{proof}
The proof of the first statement is trivial.
For the second statement, assume that $\rejectset$ does not contain $\emptyset$, and satisfies Axioms~\ref{ax:rejects:nozero}, \ref{ax:rejects:pos}, \ref{ax:rejects:cone} and~\ref{ax:rejects:mono}.

To prove that $\RN\group{\rejectset}$ satisfies Axiom~\ref{ax:rejects:nozero} and does not contain $\emptyset$, assume \emph{ex absurdo} that $\emptyset\in\RN\group{\rejectset}$ or $\set{0}\in\RN\group{\rejectset}$. 
We then find that there is some $\altoptset\in\rejectset$ such that $\altoptset\setminus\nonposopts\subseteq\emptyset\subseteq\altoptset$ or that there is some $\altoptset\in\rejectset$ such that $\altoptset\setminus\nonposopts\subseteq\set{0}\subseteq\altoptset$. 
In both cases, it follows that $\altoptset\subseteq\nonposopts$. 
If $\altoptset=\emptyset$, this contradicts our assumption that $\rejectset$ does not contain $\emptyset$. 
If $\altoptset\neq\emptyset$, it follows from Lemma~\ref{lem:replacing:nonpositives:by:zero} that we can replace every $\opt\in\altoptset$ by $0$ and still be guaranteed that the resulting option set~$\set{0}$ belongs to $\rejectset$, contradicting our assumption that $\rejectset$ satisfies Axiom~\ref{ax:rejects:nozero}.

To prove that $\RN\group{\rejectset}$ satisfies Axiom~\ref{ax:rejects:pos}, simply observe that the operator $\RN$ never removes option sets from a set of desirable option sets, so the option sets $\set{\opt}$, $\opt\in\posopts$, which belong to $\rejectset$ by Axiom~\ref{ax:rejects:pos}, will also belong to the larger $\RN\group{\rejectset}$.

To prove that $\RN\group{\rejectset}$ satisfies Axiom~\ref{ax:rejects:cone}, consider any $\optset[1],\optset[2]\in\RN\group{\rejectset}$, meaning that there are $\altoptset[1],\altoptset[2]\in\rejectset$ such that $\altoptset[1]\setminus\nonposopts\subseteq\optset[1]\subseteq\altoptset[1]$ and $\altoptset[2]\setminus\nonposopts\subseteq\optset[2]\subseteq\altoptset[2]$.
For any $\opt\in\optset[1]$ and $\altopt\in\optset[2]$, we choose $(\lambda_{\opt,\altopt},\mu_{\opt,\altopt})>0$, and let 
\begin{equation*}
\optset
\coloneqq\cset{\lambda_{\opt,\altopt}\opt+\mu_{\opt,\altopt}\altopt}{\opt\in\optset[1],\altopt\in\optset[2]}.
\end{equation*}
Then we have to prove that $\optset\in\RN\group{\rejectset}$.
Since $\rejectset$ satisfies Axiom~\ref{ax:rejects:cone}, we infer from $\altoptset[1],\altoptset[2]\in\rejectset$ that
\begin{align*}
\altoptsettoo
\coloneqq&\cset{\lambda_{\opt,\altopt}\opt+\mu_{\opt,\altopt}\altopt}
{\opt\in\optset[1],\altopt\in\optset[2]}\\
&\qquad\cup\cset{1\opt+0\altopt}{\opt\in\altoptset[1]\setminus\optset[1],\altopt\in\altoptset[2]}
\cup\cset{0\opt+1\altopt}{\opt\in\optset[1],\altopt\in\altoptset[2]\setminus\optset[2]}\\
=&\,\optset\cup
\cset{\opt}{\opt\in\altoptset[1]\setminus\optset[1],\altopt\in\altoptset[2]}
\cup\cset{\altopt}{\opt\in\optset[1],\altopt\in\altoptset[2]\setminus\optset[2]}
\end{align*}
belongs to $\rejectset$ as well. 
Furthermore, since $\altoptset[1]\setminus\nonposopts\subseteq\optset[1]$ and $\altoptset[2]\setminus\nonposopts\subseteq\optset[2]$ imply that $\altoptset[1]\setminus\optset[1]\subseteq\nonposopts$ and $\altoptset[2]\setminus\optset[2]\subseteq\nonposopts$, we see that
\begin{equation*}
\cset{\opt}{\opt\in\altoptset[1]\setminus\optset[1],\altopt\in\altoptset[2]}
\cup\cset{\altopt}{\opt\in\optset[1],\altopt\in\altoptset[2]\setminus\optset[2]}
\subseteq\nonposopts.
\end{equation*}
Hence, $\altoptsettoo\setminus\nonposopts\subseteq\optset\subseteq\altoptsettoo$. 
Since $\altoptsettoo\in\rejectset$, this implies that, indeed, $\optset\in\RN\group{\rejectset}$.

Finally, to prove that $\RN\group{\rejectset}$ satisfies Axiom~\ref{ax:rejects:mono}, consider any $\optset[1]\in\RN\group{\rejectset}$ and any $\optset[2]\in\optsets$ such that $\optset[1]\subseteq\optset[2]$.
We need to prove that $\optset[2]\in\rejectset$.
That $\optset[1]\in\RN\group{\rejectset} $ implies that there is some $\altoptset[1]\in\rejectset$ such that $\altoptset[1]\setminus\nonposopts\subseteq\optset[1]\subseteq\altoptset[1]$.
Let $\altoptset[2]\coloneqq\altoptset[1]\cup\group{\optset[2]\setminus\optset[1]}$, then $\altoptset[1]\subseteq\altoptset[2]$ and therefore also $\altoptset[2]\in\rejectset$, because $\rejectset$ satisfies Axiom~\ref{ax:rejects:mono}.
We now infer from $\altoptset[1]\setminus\nonposopts\subseteq\optset[1]\subseteq\altoptset[1]$ that
\begin{equation*}
\altoptset[2]\setminus\nonposopts
\subseteq\group{\altoptset[1]\setminus\nonposopts}\cup\group{\optset[2]\setminus\optset[1]}
\subseteq\optset[1]\cup\group{\optset[2]\setminus\optset[1]}
\subseteq\altoptset[1]\cup\group{\optset[2]\setminus\optset[1]}.
\end{equation*}
Since $\optset[1]\cup\group{\optset[2]\setminus\optset[1]}=\optset[2]$, this allows us to conclude that $\altoptset[2]\setminus\nonposopts\subseteq\optset[2]\subseteq\altoptset[2]$, and therefore, since $\altoptset[2]\in\rejectset$, that, indeed, $\optset[2]\in\RN\group{\rejectset}$.
\end{proof}

\begin{lemma}\label{lem:affinespaces}
Let $\optset,\altoptset\in\optsets$ be option sets, and consider any non-zero $\opt[o]\in\opts$. 
Then there are $\alpha_{\opt}$, $\opt\in\optset$ such that $\opt-\alpha_{\opt}\opt[o]\notin\altoptset$ and $\opt-\alpha_{\opt}\opt[o]\neq\altopt-\alpha_{\altopt}\opt[o]$ for all $\opt,\altopt$ in $\optset$.
\end{lemma}

\begin{proof}
Partition the finite set $\optset$ into the finite number of disjoint subsets $\optset[k]$ of options $\opt$ belonging to the same affine space $\cset{\opt+\beta\opt[o]}{\beta\in\reals}=\opt+\linspan\set{\opt[o]}$ parallel to $\linspan\set{\opt[o]}$.
When $\optset[k]$ has $n_k$ elements, choose $n_k$ different options in the corresponding affine space that are not in the finite $\altoptset$, which is always possible.
\end{proof}

\begin{lemma}\label{lem:Kstarstar}
Consider a coherent set of desirable option sets $\rejectset\in\cohrejectsets$ and any $\optset[o]\in\rejectset$ such that $\card{\optset[o]}\geq2$ and $\optset[o]\setminus\set{\opt}\notin\rejectset$ for all $\opt\in\optset[o]$. 
Choose any $\opt[o]\in\optset[o]$ and let
\begin{equation}\label{eq:lem:Kstarstar}
\rejectset^{**}
\coloneqq
\cset[\Big]{\cset[\big]{\lambda_{\altopt}\altopt+\mu_{\altopt}\opt[o]}{\altopt\in\altoptset}}
{\altoptset\in\rejectset,(\forall\altopt\in\altoptset)(\lambda_{\altopt},\mu_{\altopt})>0}.
\end{equation}
Then $\rejectset^{*}\coloneqq\RN\group{\rejectset^{**}}$ is a coherent set of desirable option sets that is a superset of $\rejectset$ and contains $\set{\opt[o]}$. 
Furthermore, $\set{\opt[o]}\notin\rejectset$ and $\opt[o]\not\optlteq0$.
\end{lemma}

\begin{proof}[Proof of Lemma~\ref{lem:Kstarstar}]
To prove that $\set{\opt[o]}\notin\rejectset$, assume {\itshape ex absurdo} that $\set{\opt[o]}\in\rejectset$.
Since $\card{\optset[o]\setminus\set{\opt[o]}}\geq1$, we can pick any element $\altopt\in\optset[o]\setminus\set{\opt[o]}$, and then $\set{\opt[o]}\subseteq\optset[o]\setminus\set{\altopt}$ and therefore $\optset[o]\setminus\set{\altopt}\in\rejectset$ by Axiom~\ref{ax:rejects:mono}, contradicting the assumptions.
To prove that $\opt[o]\not\optlteq0$, assume {\itshape ex absurdo} that $\opt[o]\in\nonposopts$, then we infer that also $\optset[o]\setminus\set{\opt[o]}\in\rejectset$ [use Proposition~\ref{prop:ax:rejects:RN:equivalents} and the coherence of $\rejectset$], contradicting the assumptions. 
To prove that $\set{\opt[o]}\in\rejectset^{*}$, it suffices to notice that $\set{\opt[o]}=\cset{0\altopt+1\opt[o]}{\altopt\in\optset[o]}\in\rejectset^{**}$, whence also $\set{\opt[o]}\in\rejectset^{*}$. 
Similarly, since $\rejectset^{**}$ is clearly a superset of $\rejectset$, the same is true for $\rejectset^*$.

It only remains to prove, therefore, that $\rejectset^*$ is coherent. 
To this end, we intend to show that the set of desirable option sets $\rejectset^{**}$ satisfies Axioms~\ref{ax:rejects:nozero}, \ref{ax:rejects:pos}, \ref{ax:rejects:cone} and~\ref{ax:rejects:mono} and that $\emptyset\notin\rejectset^{**}$. 
The coherence of $\rejectset^*$ will then be an immediate consequence of Proposition~\ref{prop:removal:of:nonpositives:gewijzigde:axiomas}.

For Axiom~\ref{ax:rejects:nozero}, assume {\itshape ex absurdo} that $\set{0}\in\rejectset^{**}$, meaning that there is some $\altoptset\in\rejectset$ and, for all $\altopt\in\altoptset$, some choice of $(\lambda_{\altopt},\mu_{\altopt})>0$, such that $\cset{\lambda_{\altopt}\altopt+\mu_{\altopt}\opt[o]}{\altopt\in\altoptset}=\set{0}$. 
Hence, $\altoptset\neq\emptyset$ and $\lambda_{\altopt}\altopt+\mu_{\altopt}\opt[o]=0$ for all $\altopt\in\altoptset$.

Recall that we already know that $\opt[o]\neq0$. 
For any $\altopt\in\altoptset$, $\lambda_{\altopt}\altopt+\mu_{\altopt}\opt[o]=0$ implies that $\lambda_{\altopt}>0$, because otherwise, since $(\lambda_{\altopt},\mu_{\altopt})>0$, $\lambda_{\altopt}=0$ would imply that $\mu_{\altopt}>0$ and therefore $\opt[o]=0$, a contradiction.
Hence, for all $\altopt\in\altoptset$, $\altopt=-\delta_{\altopt}\opt[o]$ with $\delta_{\altopt}\coloneqq\frac{\mu_{\altopt}}{\lambda_{\altopt}}\geq0$. 
Now let $(\kappa_{\opt,\altopt},\rho_{\opt,\altopt})\coloneqq(1,0)$ for all $\opt\in\optset[o]\setminus\set{\opt[o]}$ and $\altopt\in\altoptset$, and let $(\kappa_{\opt[o],\altopt},\rho_{\opt[o],\altopt})\coloneqq(\delta_{\altopt},1)$ for all $\altopt\in\altoptset$. 
Then
\begin{align*}
\cset{\kappa_{\opt,\altopt}\opt+\rho_{\opt,\altopt}\altopt}{\opt\in\optset[o],\altopt\in\altoptset}
&=
\cset{\opt}{\opt\in\optset[o]\setminus\set{\opt[o]},\altopt\in\altoptset}
\cup
\cset{\delta_{\altopt}\opt[o]+\altopt}{\altopt\in\altoptset}\\
&=
\cset{\opt}{\opt\in\optset[o]\setminus\set{\opt[o]},\altopt\in\altoptset}
\cup
\cset{0}{\altopt\in\altoptset}\\
&=\set{0}\cup\group{\optset[o]\setminus\set{\opt[o]}},
\end{align*}
where the last equality follows from $\altoptset\neq\emptyset$. 
However, since $\optset[o]\in\rejectset$ and $\altoptset\in\rejectset$, the coherence of $\rejectset$ [Axiom~\ref{ax:rejects:cone}] implies that $\cset{\kappa_{\opt,\altopt}\opt+\rho_{\opt,\altopt}\altopt}{\opt\in\optset[o],\altopt\in\altoptset}
\in\rejectset$. 
We therefore find that $\set{0}\cup\group{\optset[o]\setminus\set{\opt[o]}}\in\rejectset$. 
The coherence of $\rejectset$ now guarantees that $\optset[o]\setminus\set{\opt[o]}\in\rejectset$ [use Axiom~\ref{ax:rejects:removezero} if $\set{0}\notin\optset[o]\setminus\set{\opt[o]}$], contradicting the assumptions.

For Axiom~\ref{ax:rejects:pos}, consider any $\opt\in\posopts$.
Then $\set{\opt}\in\rejectset$ because $\rejectset$ satisfies Axiom~\ref{ax:rejects:pos}. 
Since $\rejectset^{**}$ is a superset of $\rejectset$, we see that, indeed, also $\set{\opt}\in\rejectset^{**}$.

For Axiom~\ref{ax:rejects:mono}, consider any $\optset[1]\in\rejectset^{**}$ and any $\optset[2]\in\optsets$ such that $\optset[1]\subseteq\optset[2]$, then we must prove that also $\optset[2]\in\rejectset^{**}$.
Since $\optset[1]\in\rejectset^{**}$, we know that there is some $\altoptset[1]\in\rejectset$ and, for all $\altopt\in\altoptset[1]$, some choice of $(\lambda_{\altopt},\mu_{\altopt})>0$, such that
\begin{equation*}
\optset[1]
=\cset{\lambda_{\altopt}\altopt+\mu_{\altopt}\opt[o]}{\altopt\in\altoptset[1]}.
\end{equation*}
For every $\opt\in\optset[2]\setminus\optset[1]$, we now choose some real $\alpha_{\opt}>0$ such that $\opt-\alpha_{\opt}\opt[o]\notin\altoptset[1]$ and such that, for all $\opt,\opt'\in\optset[2]\setminus\optset[1]$, $\opt-\alpha_{\opt}\opt[o]\neq\opt'-\alpha_{\opt'}\opt[o]$. 
Since $\opt[o]\neq0$ and $\optset[1]$, $\optset[2]$ and $\altoptset[1]$ are finite, this is always possible, by Lemma~\ref{lem:affinespaces}.
Let 
\begin{equation*}
\altoptset[2]
\coloneqq\altoptset[1]\cup
\cset{\opt-\alpha_{\opt}\opt[o]}{\opt\in\optset[2]\setminus\optset[1]}
\end{equation*}
and, for each $\altopt\in\altoptset[2]\setminus\altoptset[1]$, let $\opt[\altopt]$ be the unique element of $\optset[2]\setminus\optset[1]$ for which $v=\opt[\altopt]-\alpha_{\opt[\altopt]}\opt[o]$, and let $(\lambda_{\altopt},\mu_{\altopt})\coloneqq(1,\alpha_{\opt[\altopt]})>0$.
We then see that
\begin{align*}
\optset[2]
&=\optset[1]\cup(\optset[2]\setminus\optset[1])\\
&=\cset{\lambda_{\altopt}\altopt+\mu_{\altopt}\opt[o]}{\altopt\in\altoptset[1]}
\cup\cset{\opt-\alpha_{\opt}\opt[o]+\alpha_{\opt}\opt[o]}{\opt\in\optset[2]\setminus\optset[1]}\\ 
&=\cset{\lambda_{\altopt}\altopt+\mu_{\altopt}\opt[o]}{\altopt\in\altoptset[1]}
\cup\cset{\altopt+\alpha_{\opt[\altopt]}\opt[o]}{\altopt\in\altoptset[2]\setminus\altoptset[1]}\\
&=\cset{\lambda_{\altopt}\altopt+\mu_{\altopt}\opt[o]}{\altopt\in\altoptset[2]}. 
\end{align*}
Furthermore, since $\altoptset[1]\in\rejectset$ and $\altoptset[1]\subseteq\altoptset[2]$, it follows from the coherence of $\rejectset$ and Axiom~\ref{ax:rejects:mono} that $\altoptset[2]\in\rejectset$. 
Hence, indeed, $\optset[2]\in\rejectset^{**}$.

For Axiom~\ref{ax:rejects:cone}, consider any $\optset[1],\optset[2]\in\rejectset^{**}$ and, for all $\opt[1]\in\optset[1]$ and $\opt[2]\in\optset[2]$, any choice of $(\alpha_{\opt[1],\opt[2]},\beta_{\opt[1],\opt[2]})>0$. 
Then we must prove that
\begin{equation*}
\altoptsettoo
\coloneqq\cset{\alpha_{\opt[1],\opt[2]}\opt[1]+\beta_{\opt[1],\opt[2]}\opt[2]}
{\opt[1]\in\optset[1],\opt[2]\in\optset[2]}
\in\rejectset^{**}.
\end{equation*}
Since $\optset[1],\optset[2]\in\rejectset^{**}$, there are $\altoptset[1],\altoptset[2]\in\rejectset$ and, for all $\altopt[1]\in\altoptset[1]$ and $\altopt[2]\in\altoptset[2]$, some choices of $(\lambda_{1,\altopt[1]},\mu_{1,\altopt[1]})>0$ and $(\lambda_{2,\altopt[2]},\mu_{2,\altopt[2]})>0$, such that
\begin{equation*}
\optset[1]
=\cset{\lambda_{1,\altopt[1]}\altopt[1]+\mu_{1,\altopt[1]}\opt[o]}
{\altopt[1]\in\altoptset[1]}
\text{ and }
\optset[2]
=\cset{\lambda_{2,\altopt[2]}\altopt[2]+\mu_{2,\altopt[2]}\opt[o]}
{\altopt[2]\in\altoptset[2]}.
\end{equation*}
Now fix any $\altopt[1]\in\altoptset[1]$ and $\altopt[2]\in\altoptset[2]$, and let
$(\alpha'_{\altopt[1],\altopt[2]},\beta'_{\altopt[1],\altopt[2]})\coloneqq(\alpha_{\opt[1],\opt[2]},\beta_{\opt[1],\opt[2]})>0$, with
$\opt[1]\coloneqq\lambda_{1,\altopt[1]}\altopt[1]+\mu_{1,\altopt[1]}\opt[o]$
and
$\opt[2]\coloneqq\lambda_{2,\altopt[2]}\altopt[2]+\mu_{2,\altopt[2]}\opt[o]$. 
Then
\begin{align*}
\altoptsettoo
&=\cset{\alpha'_{\altopt[1],\altopt[2]}\group{\lambda_{1,\altopt[1]}\altopt[1]+\mu_{1,\altopt[1]}\opt[o]}+\beta'_{\altopt[1],\altopt[2]}\group{\lambda_{2,\altopt[2]}\altopt[2]+\mu_{2,\altopt[2]}\opt[o]}}
{\altopt[1]\in\altoptset[1],\altopt[2]\in\altoptset[2]}
\end{align*}
We consider two cases. 
If $\alpha'_{\altopt[1],\altopt[2]}\lambda_{1,\altopt[1]}+\beta'_{\altopt[1],\altopt[2]}\lambda_{2,\altopt[2]}>0$, we let
\begin{align*}
(\kappa_{\altopt[1],\altopt[2]},\rho_{\altopt[1],\altopt[2]})
&\coloneqq(\alpha'_{\altopt[1],\altopt[2]}\lambda_{1,\altopt[1]},\beta'_{\altopt[1],\altopt[2]}\lambda_{2,\altopt[2]})>0,\\
(\gamma_{\altopt[1],\altopt[2]},\delta_{\altopt[1],\altopt[2]})
&\coloneqq(1,\alpha'_{\altopt[1],\altopt[2]}\mu_{1,\altopt[1]}+\beta'_{\altopt[1],\altopt[2]}\mu_{2,\altopt[2]})>0.
\end{align*}
If $\alpha'_{\altopt[1],\altopt[2]}\lambda_{1,\altopt[1]}+\beta'_{\altopt[1],\altopt[2]}\lambda_{2,\altopt[2]}=0$, we let
\begin{align*}
(\kappa_{\altopt[1],\altopt[2]},\rho_{\altopt[1],\altopt[2]})
&\coloneqq(1,1)>0,\\
(\gamma_{\altopt[1],\altopt[2]},\delta_{\altopt[1],\altopt[2]})
&\coloneqq(0,\alpha'_{\altopt[1],\altopt[2]}\mu_{1,\altopt[1]}+\beta'_{\altopt[1],\altopt[2]}\mu_{2,\altopt[2]})>0.
\end{align*}
In both cases, we find that
\begin{multline}\label{eq:uglyproof}
\gamma_{\altopt[1],\altopt[2]}(\kappa_{\altopt[1],\altopt[2]}\altopt[1]+\rho_{\altopt[1],\altopt[2]}\altopt[2])+\delta_{\altopt[1],\altopt[2]}\opt[o]\\
=
\alpha'_{\altopt[1],\altopt[2]}(
\lambda_{1,\altopt[1]}\altopt[1]+\mu_{1,\altopt[1]}\opt[o])
+\beta'_{\altopt[1],\altopt[2]}(
\lambda_{2,\altopt[2]}\altopt[2]+\mu_{2,\altopt[2]}\opt[o])\in\altoptsettoo. 
\end{multline}
Now let
\begin{equation*}
\altoptset\coloneqq\cset[\big]{\kappa_{\altopt[1],\altopt[2]}\altopt[1]+\rho_{\altopt[1],\altopt[2]}\altopt[2]}{\altopt[1]\in\altoptset[1],\altopt[2]\in\altoptset[2]}.
\end{equation*}
Then clearly, for all $\altopttoo\in\altoptset$, there are $\altopt[1]\in\altoptset[1]$ and $\altopt[2]\in\altoptset[2]$ such that $\altopttoo=\kappa_{\altopt[1],\altopt[2]}\altopt[1]+\rho_{\altopt[1],\altopt[2]}\altopt[2]$. 
However, there could be multiple such pairs. 
We choose any one such pair and denote its two elements by $\altopt[1,\altopttoo]$ and $\altopt[2,\altopttoo]$, respectively. 
Using this notation, we now define the set
\begin{equation*}
\altoptsettoo'
\coloneqq
\cset[\big]{\gamma_{\altopt[1,\altopttoo],\altopt[2,\altopttoo]}\altopttoo+\delta_{\altopt[1,\altopttoo],\altopt[2,\altopttoo]}\opt[o]}
{\altopttoo\in\altoptset}.
\end{equation*}
Since $\altoptset[1],\altoptset[2]\in\rejectset$, the coherence of $\rejectset$ [Axiom~\ref{ax:rejects:cone}] implies that $\altoptset\in\rejectset$, which in turn implies that $\altoptsettoo'\in\rejectset^{**}$. 
Also, since
\begin{align*}
\altoptsettoo'
=&
\cset[\big]{\gamma_{\altopt[1,\altopttoo],\altopt[2,\altopttoo]}\altopttoo+\delta_{\altopt[1,\altopttoo],\altopt[2,\altopttoo]}\opt[o]}
{\altopttoo\in\altoptset}\\
=&\cset[\big]{\gamma_{\altopt[1,\altopttoo],\altopt[2,\altopttoo]}\big(\kappa_{\altopt[1,\altopttoo],\altopt[2,\altopttoo]}\altopt[1,\altopttoo]+\rho_{\altopt[1,\altopttoo],\altopt[2,\altopttoo]}\altopt[2,\altopttoo]\big)+\delta_{\altopt[1,\altopttoo],\altopt[2,\altopttoo]}\opt[o]}
{\altopttoo\in\altoptset},
\end{align*}
we infer from Equation~\eqref{eq:uglyproof} that $\altoptsettoo'\subseteq\altoptsettoo$. 
Since we have already proved that $\rejectset^{**}$ satisfies Axiom~\ref{ax:rejects:mono}, this implies that, indeed, $\altoptsettoo\in\rejectset^{**}$.

It therefore now only remains to prove that $\emptyset\notin\rejectset^{**}$. 
Observe that that $\emptyset\notin\rejectset$ because $\rejectset$ is coherent [combine Axioms~\ref{ax:rejects:nozero} and~\ref{ax:rejects:mono}]. 
It therefore follows from Equation~\eqref{eq:lem:Kstarstar} that, indeed, $\emptyset\notin\rejectset^{**}$.
\end{proof}

\begin{proposition}\label{prop:nonbinary:is:dominated}
Any non-binary coherent set of desirable option sets $\rejectset$ is \emph{strictly dominated} by some other coherent set of desirable option sets.
\end{proposition}
\begin{proof}
Consider an arbitrary coherent non-binary set of desirable option sets $\rejectset$. 
We infer from Lemma~\ref{lem:binaryalternative} that there is some $\optset[o]\in\rejectset$ such that $\card{\optset[o]}\geq2$ and $\optset[o]\setminus\set{\opt}\notin\rejectset$ for all $\opt\in\optset[o]$.
Consider any $\opt[o]\in\optset[o]$ and let $\rejectset^*\coloneqq\RN\group{\rejectset^{**}}$, with $\rejectset^{**}$ as in Equation~\eqref{eq:lem:Kstarstar}. 
It then follows from Lemma~\ref{lem:Kstarstar} that $\rejectset^{*}$ is a coherent set of desirable option sets that is a superset of $\rejectset$ and contains $\set{\opt[o]}$, and that $\set{\opt[o]}\notin\rejectset$. 
Hence, $\rejectset\subset\rejectset^*$.
\end{proof}

\begin{lemma}\label{lem:chaincoherence}
For any non-empty chain $\chainofrejectsets$ in\/ $\cohrejectsets$, its union $\rejectset[o]\coloneqq\bigcup\chainofrejectsets$ is a coherent set of desirable option sets.
\end{lemma}

\begin{proof}
For Axiom~\ref{ax:rejects:removezero}, consider any $\optset\in\rejectset[o]$. 
Then there is some $\rejectset'\in\chainofrejectsets$ such that $\optset\in\rejectset'$, and since $\rejectset'$ is coherent, this implies that $\optset\setminus\set{0}\in\rejectset'\subseteq\rejectset[o]$.

For Axiom~\ref{ax:rejects:nozero}, simply observe that since $\set{0}$ belongs to no element of $\chainofrejectsets$ [since they are all coherent], it cannot belong to their union $\rejectset[o]$.

For Axiom~\ref{ax:rejects:pos}, consider any $\opt\optgt0$ and any $\rejectset\in\chainofrejectsets$, then we know that $\set{\opt}\in\rejectset$ [since $\rejectset$ is coherent], and therefore also $\set{\opt}\in\rejectset[o]$, since $\rejectset\subseteq\rejectset[o]$.

For Axiom~\ref{ax:rejects:cone}, consider any $\optset[1],\optset[2]\in\rejectset[o]$ and, for all $\opt\in\optset[1]$ and $\altopt\in\optset[2]$, choose some $(\lambda_{\opt,\altopt},\mu_{\opt,\altopt})>0$. 
Since $\optset[1],\optset[2]\in\rejectset[o]$, we know that there are $\rejectset[1],\rejectset[2]\in\chainofrejectsets$ such that $\optset[1]\in\rejectset[1]$ and $\optset[2]\in\rejectset[2]$.
Since $\chainofrejectsets$ is a chain, we can assume without loss of generality that $\rejectset[1]\subseteq\rejectset[2]$, and therefore $\set{\optset[1],\optset[2]}\subseteq\rejectset[2]$.
Since $\rejectset[2]$ is coherent, it follows that $\cset{\lambda_{\opt,\altopt}\opt+\mu_{\opt,\altopt}\altopt}{\opt\in\optset[1],\altopt\in\optset[2]}\in\rejectset[2]\subseteq\rejectset[o]$.

And finally, for Axiom~\ref{ax:rejects:mono}, consider any $\optset[1]\in\rejectset[o]$ and any $\optset[2]\in\optsets$ such that $\optset[1]\subseteq\optset[2]$.
Then we know that there is some $\rejectset\in\chainofrejectsets$ such that $\optset[1]\in\rejectset$.
Since $\rejectset$ is coherent, this implies that also $\optset[2]\in\rejectset\subseteq\rejectset[o]$.
\end{proof}

\begin{lemma}\label{lem:Zorncoherence}
For any coherent set of desirable option sets $\rejectset\in\cohrejectsets$ and any set of desirable option sets $\rejectset^*\in\rejectsets$ such that $\rejectset\cap\rejectset^*=\emptyset$, the partially ordered set
\begin{equation*}
\upset{\rejectset}\coloneqq\cset{\rejectset'\in\cohrejectsets}{\rejectset\subseteq\rejectset'\text{ and }\rejectset'\cap\rejectset^*=\emptyset}
\end{equation*}
has a maximal element.
\end{lemma}

\begin{proof}
We will use Zorn's Lemma to establish the existence of a maximal element.
So consider any (non-empty) chain $\chainofrejectsets$ in $\upset{\rejectset}$, then we must prove that $\chainofrejectsets$ has an upper bound in $\upset{\rejectset}$.
Since $\rejectset[o]\coloneqq\bigcup\chainofrejectsets$ is clearly an upper bound, we are done if we can prove that $\rejectset[o]\in\upset{\rejectset}$.

That $\rejectset[o]\cap\rejectset^*=\emptyset$ follows from the fact that $\rejectset'\cap\rejectset^*=\emptyset$ for every $\rejectset'\in\chainofrejectsets\subseteq\upset{\rejectset}$. 
That $\rejectset[o]$ is a coherent set of desirable option sets follows from Lemma~\ref{lem:chaincoherence}.
\end{proof}

\begin{proposition}\label{prop:dominatebybinary}
Every coherent set of desirable option sets $\rejectset\in\cohrejectsets$ is dominated by some binary coherent set of desirable option sets.
\end{proposition}

\begin{proof}
Lemma~\ref{lem:Zorncoherence} with $\rejectset^*=\emptyset$ tells us that the partially ordered set $\cset{\rejectset'\in\cohrejectsets}{\rejectset\subseteq\rejectset'}$ has some maximal element~$\maxrejectset$. 
Assume \emph{ex absurdo} that $\maxrejectset$ is non-binary. 
It then follows from Proposition~\ref{prop:nonbinary:is:dominated} that $\maxrejectset$ is strictly dominated by a coherent set of desirable option sets, meaning that there is some $\rejectset^*\in\cohrejectsets$ such that $\maxrejectset\subset\rejectset^*$, and therefore also $\rejectset\subset\rejectset^*$. 
Hence also $\rejectset^*\in\cset{\rejectset'\in\cohrejectsets}{\rejectset\subseteq\rejectset'}$, which contradicts that $\maxrejectset$ is a maximal element of that set. 
We conclude that $\maxrejectset$ is indeed binary. 
\end{proof}

\begin{theorem}\label{theo:rejectsets:representation}
Every coherent set of desirable option sets $\rejectset\in\cohrejectsets$ is dominated by at least one binary coherent set of desirable option sets: $\cohdesirsets\group{\rejectset}\coloneqq\cset{\desirset\in\cohdesirsets}{\rejectset\subseteq\rejectset[\desirset]}\neq\emptyset$.
Moreover, $\rejectset=\bigcap\cset{\rejectset[\desirset]}{\desirset\in\cohdesirsets\group{\rejectset}}$.
\end{theorem}

\begin{proof}
Let $\rejectset[o]$ be any coherent set of desirable option sets. 
We prove that $\cohdesirsets(\rejectset[o])\coloneqq\cset{\desirset\in\cohdesirsets}{\rejectset[o]\subseteq\rejectset[\desirset]}\neq\emptyset$ and that $\rejectset[o]=\bigcap\cset{\rejectset[\desirset]}{\desirset\in\cohdesirsets(\rejectset[o])}$.

For the first statement, recall from Proposition~\ref{prop:dominatebybinary} that $\rejectset[o]$ is dominated by a binary coherent set of desirable option sets $\maxrejectset$.
Proposition~\ref{prop:binary:characterisation} therefore implies that $\maxrejectset=\rejectset[\desirset]$, with $\smash{\desirset=\desirset[\maxrejectset]}$. 
Furthermore, because $\maxrejectset$ is coherent, Proposition~\ref{prop:coherence:for:binary} implies that $\desirset$ is coherent, whence $\desirset\in\cohdesirsets$. 
Since $\rejectset[o]\subseteq\maxrejectset=\rejectset[\desirset]$, we see that $\desirset[\maxrejectset]\in\cohdesirsets(\rejectset[o])\coloneqq\cset{\desirset\in\cohdesirsets}{\rejectset[o]\subseteq\rejectset[\desirset]}\neq\emptyset$.

For the second statement, it is obvious that $\rejectset[o]\subseteq\bigcap\cset{\rejectset[\desirset]}{\desirset\in\cohdesirsets(\rejectset[o])}$, so we concentrate on the proof of the converse inclusion.
Assume {\itshape ex absurdo} that $\rejectset[o]\subset\bigcap\cset{\rejectset[\desirset]}{\desirset\in\cohdesirsets(\rejectset[o])}$, so there is some option set $\altoptset[o]\in\optsets$ such that $\altoptset[o]\notin\rejectset[o]$ and $\altoptset[o]\in\rejectset[\desirset]$ for all $\desirset\in\cohdesirsets(\rejectset[o])$, so $\altoptset[o]\neq\emptyset$.
Then $\altoptset[o]\setminus\nonposopts\notin\rejectset[o]$ [use the coherence of $\rejectset[o]$ and Axiom~\ref{ax:rejects:mono}] and $\altoptset[o]\setminus\nonposopts\in\rejectset[\desirset]$ for all $\desirset\in\cohdesirsets(\rejectset[o])$ [use the coherence of $\rejectset[\desirset]$---which follows from Proposition~\ref{prop:fromCohDtoCohK} and the coherence of $\desirset$---and Proposition~\ref{prop:ax:rejects:RN:equivalents}], so we may assume without loss of generality that $\altoptset[o]$ has no non-positive options: $\altoptset[o]\cap\nonposopts=\emptyset$.

The partially ordered set $\upset{\rejectset[o]^*}\coloneqq\cset{\rejectset\in\cohrejectsets}{\rejectset[o]\subseteq\rejectset\text{ and }\altoptset[o]\notin\rejectset}$ is non-empty because it contains $\rejectset[o]$. 
Furthermore, due to Lemma~\ref{lem:Zorncoherence} [applied for $\rejectset=\rejectset[o]$ and $\rejectset^*=\set{\altoptset[o]}$], it has at least one maximal element.
If we can prove that any such maximal element $\maxrejectset$ is binary, then we know from Propositions~\ref{prop:binary:characterisation} and~\ref{prop:coherence:for:binary} that there is some coherent set of desirable options $\smash{\desirset[o]=\desirset[\maxrejectset]}$ such that $\rejectset[o]\subseteq\rejectset[{\desirset[o]}]$---and therefore $\desirset[o]\in\cohdesirsets(\rejectset[o])$---and $\altoptset[o]\notin\rejectset[{\desirset[o]}]$, a contradiction.
To prove that all maximal elements of $\upset{\rejectset[o]^*}$ are binary, it suffices to prove that any non-binary element of $\upset{\rejectset[o]^*}$ is strictly dominated in that set, which is what we now set out to do.

So consider any non-binary element $\rejectset$ of $\upset{\rejectset[o]^*}$, so in particular $\rejectset\in\cohrejectsets$, $\rejectset[o]\subseteq\rejectset$ and $\altoptset[o]\notin\rejectset$.
Since $\rejectset$ is non-binary, it follows from Lemma~\ref{lem:binaryalternative} that there is some $\optset[o]\in\rejectset$ such that $\card{\optset[o]}\geq2$ and $\optset[o]\setminus\set{\opt}\notin\rejectset$ for all $\opt\in\optset[o]$.
The partially ordered set $\cset{\optset\in\rejectset}{\altoptset[o]\subseteq\optset}$ contains $\optset[o]\cup\altoptset[o]$ [because $\optset[o]\in\rejectset$ and because $\rejectset$ satisfies Axiom~\ref{ax:rejects:mono}] and therefore has some minimal (non-dominating) element $\altoptset^*$ below it, so $\altoptset^*\in\rejectset$ and $\altoptset[o]\subseteq\altoptset^*\subseteq\optset[o]\cup\altoptset[o]$.

Let us first summarise what we know about this minimal element $\altoptset^*$.
It is impossible that $\altoptset^*\subseteq\altoptset[o]$ because otherwise $\altoptset[o]=\altoptset^*\in\rejectset$, a contradiction.
Hence $\altoptset^*\setminus\altoptset[o]\neq\emptyset$, so we can fix some element $\opt[o]$ in $\altoptset^*\setminus\altoptset[o]\subseteq\optset[o]$.
Since $\altoptset[o]\subseteq\altoptset^*\setminus\set{\opt[o]}$ but $\altoptset^*\setminus\set{\opt[o]}\subset\altoptset^*$, it must be that $\altoptset^*\setminus\set{\opt[o]}\notin\rejectset$, by the definition of a minimal element.
Observe that $\altoptset^*\neq\emptyset$.

Let $\rejectset^*\coloneqq\RN\group{\rejectset^{**}}$, with $\rejectset^{**}$ as in Equation~\eqref{eq:lem:Kstarstar}. 
Since $\opt[o]\in\optset[o]$, it then follows from Lemma~\ref{lem:Kstarstar} that $\rejectset^{*}$ is a coherent set of desirable option sets that is a superset of $\rejectset$---and therefore also of $\rejectset[o]$---and contains $\set{\opt[o]}$, and that $\set{\opt[o]}\notin\rejectset$ and $\opt[o]\not\optlteq0$. 
Hence, it follows that $\rejectset\subset\rejectset^*$. 
If we can now prove that $\altoptset[o]\notin\rejectset^*$ and therefore $\rejectset^*\in\upset{\rejectset[o]^*}$, we are done, because then $\rejectset$ is indeed strictly dominated by $\rejectset^*$ in $\upset{\rejectset[o]^*}$.

Assume therefore {\itshape ex absurdo} that $\altoptset[o]\in\rejectset^*=\RN\group{\rejectset^{**}}$.
Taking into account Equation~\eqref{eq:lem:Kstarstar}, this implies that there are $\altoptsettoo\in\rejectset$ and $(\lambda_{\altopt},\mu_{\altopt})>0$ for all $\altopt\in\altoptsettoo$, such that $\cset{b_{\altopt}}{\altopt\in\altoptsettoo}\setminus\nonposopts\subseteq\altoptset[o]\subseteq\cset{b_{\altopt}}{\altopt\in\altoptsettoo}$, where, for all $\altopt\in\altoptsettoo$, $b_{\altopt}\coloneqq\lambda_{\altopt}\altopt+\mu_{\altopt}\opt[o]$.
Given our assumption that $\altoptset[o]\cap\nonposopts=\emptyset$, this also implies that $\cset{b_{\altopt}}{\altopt\in\altoptsettoo}\setminus\altoptset[o]\subseteq\nonposopts$. 
Now let $\altoptsettoo[1]\coloneqq\cset{\altopt\in\altoptsettoo}{b_{\altopt}\in\altoptset[o]}$ and $\altoptsettoo[2]\coloneqq\cset{\altopt\in\altoptsettoo}{b_{\altopt}\notin\altoptset[o]}$. 
Then $\altoptsettoo[1]\neq\emptyset$ [because $\altoptset[o]\neq\emptyset$] and $\cset{b_{\altopt}}{\altopt\in\altoptsettoo[1]}=\altoptset[o]$.
Consider now any $\altopt\in\altoptsettoo[2]$. 
Then $b_{\altopt}\notin\altoptset[o]$. 
Since $\cset{b_{\altopt}}{\altopt\in\altoptsettoo}\setminus\altoptset[o]\subseteq\nonposopts$, this implies that $b_{\altopt}=\lambda_{\altopt}\altopt+\mu_{\altopt}\opt[o]\optlteq0$.
Hence, we must have that $\lambda_{\altopt}>0$, because otherwise $\mu_{\altopt}\opt[o]\optlteq0$ with $\mu_{\altopt}>0$, and therefore also $\opt[o]\optlteq0$, contradicting what we inferred earlier from Lemma~\ref{lem:Kstarstar}.
So we find that
\begin{equation*}
\altopt\optlteq-\frac{\mu_{\altopt}}{\lambda_{\altopt}}\opt[o]
\text{ for all $\altopt\in\altoptsettoo[2]$}.
\end{equation*}
Consequently, and because $\altoptsettoo[1]\cup\altoptsettoo[2]=\altoptsettoo\in\rejectset$, we infer from Lemma~\ref{lem:replacing:by:dominating:options} that
\begin{equation*}
\altoptsettoo'
\coloneqq
\altoptsettoo[1]
\cup\cset[\Big]{-\frac{\mu_{\altopt}}{\lambda_{\altopt}}\opt[o]}{\altopt\in\altoptsettoo[2]}\in\rejectset.
\end{equation*}
Let $\altoptsettoo[3]\coloneqq\altoptsettoo'\setminus\altoptsettoo[1]$. 
Then for all $\altopt\in\altoptsettoo[3]$, there is some $\gamma_{\altopt}\geq0$ such that $\altopt=-\gamma_{\altopt}\opt[o]$.
Now let $(\alpha_{\opt[o],\altopt},\beta_{\opt[o],\altopt})\coloneqq(\mu_{\altopt},\lambda_{\altopt})$ for all $\altopt\in\altoptsettoo[1]$ and $(\alpha_{\opt[o],\altopt},\beta_{\opt[o],\altopt})\coloneqq(\gamma_{\altopt},1)$ for all $\altopt\in\altoptsettoo[3]$ and, for all $\opt\in\altoptset^*\setminus\set{\opt[o]}$ and $\altopt\in\altoptsettoo'$, let $(\alpha_{\opt,\altopt},\beta_{\opt,\altopt})\coloneqq(1,0)$. 
Then
\begin{multline*}
\cset{\alpha_{\opt,\altopt}\opt+\beta_{\opt,\altopt}\altopt}{\opt\in\altoptset^*,\altopt\in\altoptsettoo'}\\
\begin{aligned}
&=\cset{\mu_{\altopt}\opt[o]+\lambda_{\altopt}\altopt}{\altopt\in\altoptsettoo[1]}
\cup
\cset{\gamma_{\altopt}\opt[o]+\altopt}{\altopt\in\altoptsettoo[3]}
\cup
\cset{\opt}{\opt\in\altoptset^*\setminus\set{\opt[o]},\altopt\in\altoptsettoo'}\\
&=
\cset{b_{\altopt}}{\altopt\in\altoptsettoo[1]}
\cup\cset{0}{\altopt\in\altoptsettoo[3]}
\cup\cset{\opt}{\opt\in\altoptset^*\setminus\set{\opt[o]}}\\
&=\altoptset[o]\cup\cset{0}{\altopt\in\altoptsettoo[3]}\cup(\altoptset^*\setminus\set{\opt[o]})\\
&=(\altoptset^*\setminus\set{\opt[o]})\cup\cset{0}{\altopt\in\altoptsettoo[3]},
\end{aligned}
\end{multline*}
where the second equality holds because $\altoptsettoo'\in\rejectset$ and Axioms~\ref{ax:rejects:nozero} and~\ref{ax:rejects:mono} imply that $\emptyset\neq\altoptsettoo'$, and where the fourth equality holds because $\altoptset[o]\subseteq\altoptset^*\setminus\set{\opt[o]}$. 
Since $\altoptset^*\in\rejectset$ and $\altoptsettoo'\in\rejectset$, we can now invoke Axiom~\ref{ax:rejects:cone} to find that
\begin{equation*}
\altoptset^*\setminus\set{\opt[o]}\cup\cset{0}{\altopt\in\altoptsettoo[3]}
=
\cset{\alpha_{\opt,\altopt}\opt+\beta_{\opt,\altopt}\altopt}{\opt\in\altoptset^*,\altopt\in\altoptsettoo'}\in\rejectset.
\end{equation*}
If $\altoptsettoo[3]=\emptyset$, we find that $\altoptset^*\setminus\set{\opt[o]}\in\rejectset$, a contradiction.
If $\altoptsettoo[3]\neq\emptyset$, we find that $\set{0}\cup\altoptset^*\setminus\set{\opt[o]}\in\rejectset$.
If $0\in\altoptset^*\setminus\set{\opt[o]}$, then we get that $\altoptset^*\setminus\set{\opt[o]}\in\rejectset$, a contradiction.
And if $0\notin\altoptset^*\setminus\set{\opt[o]}$, then we can still derive from Axiom~\ref{ax:rejects:removezero} that $\altoptset^*\setminus\set{\opt[o]}\in\rejectset$, again a contradiction.
\end{proof}

\begin{proof}[Proof of Theorem~\ref{theo:coherentrepresentation:twosided}]
If the set of desirable option sets $\rejectset$ is coherent, we infer from Theorem~\ref{theo:rejectsets:representation} that $\cohdesirsets\group{\rejectset}\coloneqq\cset{\desirset\in\cohdesirsets}{\rejectset\subseteq\rejectset[\desirset]}\neq\emptyset$ and $\rejectset=\bigcap\cset{\rejectset[\desirset]}{\desirset\in\cohdesirsets\group{\rejectset}}$. 
This clearly implies that there is at least one non-empty set $\setofdesirsets\subseteq\cohdesirsets$ of coherent sets of desirable options such that $\rejectset=\bigcap\cset{\rejectset[\desirset]}{\desirset\in\setofdesirsets}$, namely the set $\cohdesirsets\group{\rejectset}$. 
Furthermore, for any non-empty set $\setofdesirsets\subseteq\cohdesirsets$ of coherent sets of desirable options such that $\rejectset=\bigcap\cset{\rejectset[\desirset]}{\desirset\in\setofdesirsets}$, we clearly have that $\rejectset\subseteq\rejectset[\desirset]$ for all $\desirset\in\setofdesirsets$.
Since $\setofdesirsets\subseteq\cohdesirsets$, this implies that $\setofdesirsets\subseteq\cohdesirsets\group{\rejectset}$. 
So $\cohdesirsets\group{\rejectset}$ is indeed the largest such set.

It remains to prove the `if' part of the statement. 
So consider any non-empty set $\setofdesirsets\subseteq\cohdesirsets$ of coherent sets of desirable options such that $\rejectset=\bigcap\cset{\rejectset[\desirset]}{\desirset\in\setofdesirsets}$. 
For any $\desirset\in\setofdesirsets\subseteq\cohdesirsets$, it then follows from Proposition~\ref{prop:fromCohDtoCohK} that $\rejectset[\desirset]$ is coherent. 
Because Axioms~\ref{ax:rejects:removezero}-\ref{ax:rejects:mono} are trivially preserved under taking arbitrary non-empty intersections, it follows that $\rejectset$ is coherent.
\end{proof}

\subsection{Proofs and intermediate results for Section~\ref{sec:back:to:choice}}

\begin{proposition}\label{prop:ax:rejects:cone:equivalents}
$\setposi\group{\rejectset}=\rejectset$ for any coherent set of desirable option sets $\rejectset\in\cohrejectsets$.
\end{proposition}

\begin{proof}[Proof of Proposition~\ref{prop:ax:rejects:cone:equivalents}]
That $\rejectset\subseteq\setposi\group{\rejectset}$, is an immediate consequence of the definition of the $\setposi$ operator, and holds for any set of desirable option sets, coherent or not. 
Indeed, consider any $\optset\in\rejectset$, then it is not difficult to see that $\optset\in\setposi\group{\rejectset}$: choose $n\coloneqq1$, $\optset[1]\coloneqq\optset\in\rejectset^1$, and $\lambda^{\opt[1:1]}_{1:1}\coloneqq1$ for all $\opt[1:1]\in\times_{k=1}^1\optset[1]=\optset$ in the definition of the $\setposi$ operator.

For the converse inclusion, that $\setposi\group{\rejectset}\subseteq\rejectset$, we use the coherence of $\rejectset$, and in particular the representation result of Theorem~\ref{theo:rejectsets:representation}, which allows us to write that $\rejectset=\bigcap\cset{\rejectset[\desirset]}{\desirset\in\cohdesirsets\text{ and }\rejectset\subseteq\rejectset[\desirset]}$.

So, if we fix any $\desirset\in\cohdesirsets$ such that $\rejectset\subseteq\rejectset[\desirset]$, then it clearly suffices to prove that also $\setposi\group{\rejectset}\subseteq\rejectset[\desirset]$.
Consider, therefore, any $\optset\in\setposi\group{\rejectset}$, meaning that there are $n\in\naturals$, $(\optset[1],\dots,\optset[n])\in\rejectset^n$ and, for all $\opt[1:n]\in\times_{k=1}^n\optset[k]$, some choice of $\lambda_{1:n}^{\opt[1:n]}>0$ such that
\begin{equation*}
\optset=\cset[\bigg]{\sum_{k=1}^n\lambda_k^{\opt[1:n]}\opt[k]}{\opt[1:n]\in\times_{k=1}^n\optset[k]}.
\end{equation*} 
For any $k\in\set{1,\dots,n}$, since $\optset[k]\in\rejectset\subseteq\rejectset[\desirset]$, we know that $\optset[k]\cap\desirset\neq\emptyset$, so we can fix some $\altopt[k]\in\optset[k]\cap\desirset$. 
Then, on the one hand, we see that $\sum_{k=1}^n\lambda_k^{\altopt[1:n]}\altopt[k]\in\optset$.
On the other hand, since $\lambda_{1:n}^{\altopt[1:n]}>0$, we infer from Axiom~\ref{ax:desirs:cone} [by applying it multiple times] that also $\sum_{k=1}^n\lambda_k^{\altopt[1:n]}\altopt[k]\in\desirset$. 
Therefore, we find that $\optset\cap\desirset\neq\emptyset$, or equivalently, that $\optset\in\rejectset[\desirset]$.
Since $\optset\in\setposi\group{\rejectset}$ was chosen arbitrarily, it follows that, indeed, $\setposi\group{\rejectset}\subseteq\rejectset[\desirset]$.
\end{proof}

\begin{proof}[Proof of Proposition~\ref{prop:axioms:rejection:sets:and:functions:coherence}]
First, suppose that $\rejectset$ satisfies Axioms~\ref{ax:rejects:removezero}--\ref{ax:rejects:mono}, then we show that $\rejectfun$ satisfies Axioms~\ref{ax:rejectfun:addition}--\ref{ax:rejectfun:senalpha}.

Axiom~\ref{ax:rejectfun:addition} follows even without Axioms~\ref{ax:rejects:removezero}--\ref{ax:rejects:mono}, from the chain of equivalences
\begin{align*}
\opt\in\rejectfun\group{\optset}
\ifandonlyif\opt\in\rejectfun\group{\optset\cup\set{\opt}}
&\ifandonlyif\optset-\opt\in\rejectset
\ifandonlyif(\optset-\opt)-0\in\rejectset\\
&\ifandonlyif0\in\rejectfun((\optset-\opt)\cup\set{0})
\ifandonlyif0\in\rejectfun(\optset-\opt),
\text{ for all $\optset\in\optsets$},
\end{align*}
where the first and last equivalences follow because $\opt\in\optset$ and the second and fourth equivalences follow from Equation~\eqref{eq:interpretation:rejectfuns:after:irreflexivity:and:additivity:intermsof:K:Gert}. 
We now concentrate on Axioms~\ref{ax:rejectfun:not:everything:rejected}--\ref{ax:rejectfun:senalpha}.

\ref{ax:rejectfun:not:everything:rejected}.
It is obvious from $\rejectfun\group{\emptyset}\subseteq\emptyset$ that $\rejectfun\group{\emptyset}=\emptyset$.
For any non-empty option set $\optset\in\optsets$, assume {\itshape ex absurdo} that $\rejectfun\group{\optset}=\optset$. 
For all $\opt\in\optset$, it then follows from Equation~\eqref{eq:interpretation:rejectfuns:after:irreflexivity:and:additivity:intermsof:K:Gert} that $\optset-\opt\in\rejectset$.
So if we denote $\optset$ by $\set{\altopt[1],\dots,\altopt[n]}$, with $n\in\naturals$, and let $\altopt[\ell k]\coloneqq\altopt[\ell]-\altopt[k]$, then we find that for all $k\in\set{1,\dots,n}$
\begin{equation*}
\optset[k]\coloneqq\cset{\altopt[\ell k]}{\ell\in\set{1,\dots,n}}\in\rejectset.
\end{equation*}
Proposition~\ref{prop:ax:rejects:cone:equivalents} and Equation~\eqref{eq:setposi} now tell us that, for any choice of the $\lambda_{1:n}^{\opt[1:n]}>0$ in Equation~\eqref{eq:setposi}, the option set
\begin{equation*}
\cset[\bigg]{\sum_{k=1}^n\lambda_{k}^{\opt[1:n]}\opt[k]}
{\opt[1:n]\in\times_{k=1}^n\optset[k]}\in\rejectset.
\end{equation*}
So if we can show that for any $\opt[1:n]\in\times_{k=1}^n\optset[k]$ we can always choose the $\lambda_{1:n}^{\opt[1:n]}>0$ in such a way that $\sum_{k=1}^n\lambda_{k}^{\opt[1:n]}\opt[k]=0$, we will have that $\set{0}\in\rejectset$, contradicting Axiom~\ref{ax:rejects:nozero}.
We now set out to do this.

For any $k\in\set{1,\dots,n}$, since $\opt[k]\in\optset[k]$, there is a unique $\ell\in\set{1,\dots,n}$ such that $\opt[k]=\altopt[\ell k]$. 
Let $\phi(k)$ be this unique index, so $\opt[k]=\altopt[\phi(k)k]$. 
For the resulting map $\phi\colon\set{1,\dots,n}\to\set{1,\dots,n}$, we now consider the sequence---$\phi$-orbit---in $\set{1,\dots,n}$:
\begin{equation*}
1,\phi(1),\phi^2(1),\dots,\phi^r(1),\dots
\end{equation*}
Because $\phi$ can assume at most $n$ different values, this sequence must be periodic, and its fundamental (smallest) period $p$ cannot be larger than $n$, so $1\leq p\leq n$ and $1=\phi^{p}(1)$.
Now let $\lambda^{\opt[1:n]}_{\phi^r(1)}\coloneqq1$ for $r=0,\dots,p-1$, and let all other components be zero, then indeed
\begin{equation*}
\sum_{k=1}^n\lambda_{k}^{\opt[1:n]}\opt[k]
=\sum_{r=0}^{p-1}\altopt[\phi^{r+1}(1)\phi^{r}(1)]
=\sum_{r=0}^{p-1}\group{\altopt[\phi^{r+1}(1)]-\altopt[\phi^{r}(1)]}
=0.
\end{equation*}

\ref{ax:rejectfun:pos}.
Consider any $\opt\in\posopts$, so $\set{\opt}\in\rejectset$ [use~\ref{ax:rejects:pos}]. 
Since $\set{\opt}=\set{\opt}-0$, Equation~\eqref{eq:interpretation:rejectfuns:after:irreflexivity:and:additivity:intermsof:K:Gert} guarantees that, indeed, $0\in\rejectfun\group{\set{0,\opt}}$.

\ref{ax:rejectfun:cone}.
Since $0\in\rejectfun\group{\optset[1]\cup\set{0}}$ and $0\in\rejectfun\group{\optset[2]\cup\set{0}}$, it follows from Equation~\eqref{eq:interpretation:rejectfuns:after:irreflexivity:and:additivity:intermsof:K:Gert} that $\optset[1],\optset[2]\in\rejectset$. 
Axiom~\ref{ax:rejects:cone} therefore implies that $\cset{\lambda_{\opt,\altopt}\opt+\mu_{\opt,\altopt}\altopt}{\opt\in\optset[1],\altopt\in\optset[2]}\in\rejectset$. 
A final application of Equation~\eqref{eq:interpretation:rejectfuns:after:irreflexivity:and:additivity:intermsof:K:Gert} now tells us that, indeed,
\begin{equation*}
0\in\rejectfun\group{\cset{\lambda_{\opt,\altopt}\opt+\mu_{\opt,\altopt}\altopt}{\opt\in\optset[1],\altopt\in\optset[2]}\cup\set{0}}.
\end{equation*}

\ref{ax:rejectfun:senalpha}.
Consider $\optset[1],\optset[2]\in\optsets$ and $\opt\in\optset[1]$, and assume that $\opt\in\rejectfun\group{\optset[1]}$ and $\optset[1]\subseteq\optset[2]$.
Since $\opt\in\optset[1]$, Equation~\eqref{eq:interpretation:rejectfuns:after:irreflexivity:and:additivity:intermsof:K:Gert} then implies that $\optset[1]-\opt\in\rejectset$, and by Axiom~\ref{ax:rejects:mono}, therefore also $\optset[2]-\opt\in\rejectset$, which in turn implies [again using Equation~\eqref{eq:interpretation:rejectfuns:after:irreflexivity:and:additivity:intermsof:K:Gert}] that $\opt\in\rejectfun\group{\optset[2]\cup\set{\opt}}$, and therefore, since $\opt\in\optset[1]\subseteq\optset[2]$, that $\opt\in\rejectfun\group{\optset[2]}$.

Next, we suppose that $\rejectfun$ satisfies Axioms~\ref{ax:rejectfun:not:everything:rejected}--\ref{ax:rejectfun:senalpha}, and show that $\rejectset$ then satisfies Axioms~\ref{ax:rejects:removezero}--\ref{ax:rejects:mono}.
Again, the first axiom [Axiom~\ref{ax:rejects:removezero} in this case] holds without imposing any conditions on $\rejectfun$. 
To see this, it suffices to consider the following chain of equivalences
\begin{equation*}
\optset\in\rejectset
\ifandonlyif0\in\rejectfun\group{\optset\cup\set{0}}
\ifandonlyif0\in\rejectfun\group{\group{\optset\setminus\set{0}}\cup\set{0}}
\ifandonlyif\optset\setminus\set{0}\in\rejectset,
\text{ for all $\optset\in\optsets$,}
\end{equation*}
where the first and third equivalences follow from Equation~\eqref{eq:interpretation:rejectfuns:after:irreflexivity:and:additivity:intermsof:K:Gert}, and the second one from the trivial fact that $\optset\cup\set{0}=\group{\optset\setminus\set{0}}\cup\set{0}$.
We therefore now concentrate on Axioms~\ref{ax:rejects:nozero}--\ref{ax:rejects:mono}.

\ref{ax:rejects:nozero}.
By Equation~\eqref{eq:interpretation:rejectfuns:after:irreflexivity:and:additivity:intermsof:K:Gert}, $\set{0}\in\rejectset$ is equivalent to $0\in\rejectfun\group{\set{0}}$ , which contradicts Axiom~\ref{ax:rejectfun:not:everything:rejected}.

\ref{ax:rejects:pos}.
Consider any $\opt\in\posopts$, so $0\in\rejectfun\group{\set{0,\opt}}$ by Axiom~\ref{ax:rejectfun:pos}. 
Equation~\eqref{eq:interpretation:rejectfuns:after:irreflexivity:and:additivity:intermsof:K:Gert} now guarantees that, indeed, $\set{\opt}\in\rejectset$.

\ref{ax:rejects:cone}.
This is a straightforward translation, using Equation~\eqref{eq:interpretation:rejectfuns:after:irreflexivity:and:additivity:intermsof:K:Gert}.

\ref{ax:rejects:mono}.
Consider any $\optset[1],\optset[2]\in\optsets$, and assume that $\optset[1]\subseteq\optset[2]$ and $\optset[2]\in\rejectset$.
Then it follows from Equation~\eqref{eq:interpretation:rejectfuns:after:irreflexivity:and:additivity:intermsof:K:Gert} that $0\in\rejectfun\group{\optset[1]\cup\set{0}}$, so Axiom~\ref{ax:rejectfun:senalpha} guarantees that also $0\in\rejectfun\group{\optset[2]\cup\set{0}}$.
Equation~\eqref{eq:interpretation:rejectfuns:after:irreflexivity:and:additivity:intermsof:K:Gert} then leads to the conclusion that, indeed, $\optset[2]\in\rejectset$.
\end{proof}

\subsection{Proofs and intermediate results for Section~\ref{sec:totality}}

\begin{proof}[Proof of Proposition~\ref{prop:totalbinaryiff}]
Since for all $\opt\in\opts\setminus\set{0}$,
\begin{equation*}
\group{\opt\in\desirset\text{ or }-\opt\in\desirset}
\ifandonlyif
\set{\opt,-\opt}\cap\desirset\neq\emptyset
\ifandonlyif
\set{\opt,-\opt}\in\rejectset[\desirset],
\end{equation*}
this result is an immediate consequence of Definitions~\ref{def:totaldesirs} and~\ref{def:totalrejects}, and Proposition~\ref{prop:coherence:for:binary}.
\end{proof}

\begin{proof}[Proof of Theorem~\ref{theo:totalrepresentation:twosided}]
First assume that $\rejectset$ is total. 
Since $\rejectset$ is then in particular coherent, we know from Theorem~\ref{theo:coherentrepresentation:twosided} that there is some non-empty set $\setofdesirsets\subseteq\cohdesirsets$ of coherent sets of desirable options such that $\rejectset=\bigcap\cset{\rejectset[\desirset]}{\desirset\in\setofdesirsets}$, and that the largest such set $\setofdesirsets$ is $\cohdesirsets\group{\rejectset}\coloneqq\cset{\desirset\in\cohdesirsets}{\rejectset\subseteq\rejectset[\desirset]}$. 
Consider any such set $\setofdesirsets$. 
Then for all $\desirset\in\setofdesirsets$, since $\rejectset\subseteq\rejectset[\desirset]$, the totality of $\rejectset$ implies that $\rejectset[\desirset]$ is total, and therefore, because of Proposition~\ref{prop:totalbinaryiff}, that $\desirset$ is total as well. 
Hence, $\setofdesirsets\subseteq\totcohdesirsets$ and therefore also $\cohdesirsets\group{\rejectset}=\totcohdesirsets\group{\rejectset}$. 
All this allows us to conclude that there is some non-empty set $\setofdesirsets\subseteq\totcohdesirsets$ of total sets of desirable options such that $\rejectset=\bigcap\cset{\rejectset[\desirset]}{\desirset\in\setofdesirsets}$, and that the largest such set $\setofdesirsets$ is $\totcohdesirsets\group{\rejectset}\coloneqq\cset{\desirset\in\totcohdesirsets}{\rejectset\subseteq\rejectset[\desirset]}$.

To prove the `if' part of the statement, consider any non-empty set $\setofdesirsets\subseteq\totcohdesirsets$ of total sets of desirable options such that $\rejectset=\bigcap\cset{\rejectset[\desirset]}{\desirset\in\setofdesirsets}$. 
For every $\desirset\in\setofdesirsets\subseteq\totcohdesirsets$, it then follows from Proposition~\ref{prop:totalbinaryiff} that $\rejectset[\desirset]$ is total. 
Because Axioms~\ref{ax:rejects:removezero}-\ref{ax:rejects:mono} and~\ref{ax:rejects:totality} are trivially preserved under taking arbitrary non-empty intersections, this implies that $\rejectset$ is total as well.
\end{proof}

\begin{proof}[Proof of Proposition~\ref{prop:totalityforrejectionfunctions}]
First suppose that $\rejectset$ is total. 
Since $\rejectset$ is then in particular coherent, we know from Proposition~\ref{prop:axioms:rejection:sets:and:functions:coherence} that $\rejectfun$ is coherent.
Consider now any $\opt\in\opts\setminus\set{0}$, so $\set{\opt,-\opt}\in\rejectset$ by Axiom~\ref{ax:rejects:totality}. 
Equation~\eqref{eq:interpretation:rejectfuns:after:irreflexivity:and:additivity:intermsof:K:Gert} then guarantees that, indeed, $0\in\rejectfun(\set{0,\opt,-\opt}$.

Next, suppose that $\rejectfun$ is coherent and satisfies Axiom~\ref{ax:rejectfun:total}. 
Since $\rejectfun$ is coherent, we know from Proposition~\ref{prop:axioms:rejection:sets:and:functions:coherence} that $\rejectset$ is coherent too.
Consider now any $\opt\in\opts\setminus\set{0}$, so $0\in\rejectfun(\set{0,\opt,-\opt})$ by Axiom~\ref{ax:rejectfun:total}. 
Equation~\eqref{eq:interpretation:rejectfuns:after:irreflexivity:and:additivity:intermsof:K:Gert} then guarantees that, indeed, $\set{\opt,-\opt}\in\rejectset$.
\end{proof}

\subsection{Proofs and intermediate results for Section~\ref{sec:mixture}}

\begin{lemma}\label{lem:localrescaling}
Consider any set of desirable option sets $\rejectset\in\rejectsets$ that satisfies Axiom~\ref{ax:rejects:cone}.
Consider any $\optset\in\rejectset$ and, for all $\opt\in\optset$, some $\lambda_{\opt}>0$. 
Then also $\cset{\lambda_{\opt}\opt}{\opt\in\optset}\in\rejectset$.
\end{lemma}

\begin{proof}
Axiom~\ref{ax:rejects:cone} for $\optset$ and $\optset$ allows us to infer that $\cset{\lambda_{\opt,\altopt}\opt+\mu_{\opt,\altopt}\altopt}{\opt,\altopt\in\optset}\in\rejectset$ for all possible choices of $(\lambda_{\opt,\altopt},\mu_{\opt,\altopt})>0$.
Choosing $(\lambda_{\opt,\altopt},\mu_{\opt,\altopt})\coloneqq(\lambda_{\opt},0)$ for all $\opt,\altopt\in\optset$, yields in particular that, indeed, $\cset{\lambda_{\opt}\opt}{\opt\in\optset}\in\rejectset$.
\end{proof}

\begin{proof}[Proof of Proposition~\ref{prop:axioms:rejection:sets:and:functions:mixing}]
First assume that $\rejectfun$ is coherent and mixing. 
Proposition~\ref{prop:axioms:rejection:sets:and:functions:coherence} then tells us that $\rejectset$ is coherent. 
We will prove that $\rejectset$ also satisfies Axioms~\ref{ax:rejects:removepositivecombinations} and~\ref{ax:rejects:removeconvexcombinations}. 
Since \ref{ax:rejects:removepositivecombinations} clearly implies~\ref{ax:rejects:removeconvexcombinations}, it suffices to prove the former. 
So consider any $\optset,\altoptset\in\optsets$ such that $\altoptset\in\rejectset$ and $\optset\subseteq\altoptset\subseteq\posi\group{\optset}$.
For every $\opt\in\altoptset\setminus\optset$, since $\opt\in\posi(\optset)$, it follows from Equations~\eqref{eq:convexhulloperator} and~\eqref{eq:posioperator} that there is some $\lambda_{\opt}>0$ such that $\lambda_{\opt}\opt\in\chull(\optset)$. 
Furthermore, for every $\opt\in\optset$, if we let $\lambda_{\opt}\coloneqq1$, then also $\lambda_{\opt}\opt\in\chull(\optset)$. 
Let $\tilde\altoptset\coloneqq\cset{\lambda_{\opt}\opt}{\opt\in\altoptset}$, then, clearly, $\optset\subseteq\tilde\altoptset\subseteq\chull(\optset)$. 
Furthermore, since $\altoptset\in\rejectset$, it follows from Lemma~\ref{lem:localrescaling} that also $\tilde\altoptset\in\rejectset$.
In order to prove that then also $\optset\in\rejectset$, observe that $\optset\subseteq\tilde\altoptset\subseteq\chull(\optset)$ implies that $\optset\cup\set{0}\subseteq\tilde\altoptset\cup\set{0}\subseteq\chull\group{\optset\cup\set{0}}$, so Axiom~\ref{ax:rejectfun:removepositivecombinations} tells us that
\begin{equation}\label{eq:axioms:rejection:sets:and:functions}
\rejectfun\group[\big]{\tilde\altoptset\cup\set{0}}\cap\group{\optset\cup\set{0}}
\subseteq\rejectfun\group{\optset\cup\set{0}}.  
\end{equation}
Since $\tilde\altoptset\in\rejectset$ is equivalent to $\tilde\altoptset-0\in\rejectset$, we infer from Equation~\eqref{eq:interpretation:rejectfuns:after:irreflexivity:and:additivity:intermsof:K:Gert} that $0\in\rejectfun\group{\tilde\altoptset\cup\set{0}}$, and since also $0\in\group{\optset\cup\set{0}}$, we infer from Equation~\eqref{eq:axioms:rejection:sets:and:functions} that $0\in\rejectfun\group{\optset\cup\set{0}}$.
Equation~\eqref{eq:interpretation:rejectfuns:after:irreflexivity:and:additivity:intermsof:K:Gert} then leads us to conclude that, indeed, $\optset=\optset-0\in\rejectset$.

Next, we assume that $\rejectset$ is coherent and satisfies \ref{ax:rejects:removepositivecombinations} or~\ref{ax:rejects:removeconvexcombinations}. 
Since \ref{ax:rejects:removepositivecombinations} implies~\ref{ax:rejects:removeconvexcombinations}, it follows that $\rejectset$ is coherent and satisfies \ref{ax:rejects:removeconvexcombinations}. 
Proposition~\ref{prop:axioms:rejection:sets:and:functions:coherence} already tells us that then $\rejectfun$ is coherent, so we are left to prove that $\rejectfun$ satisfies~\ref{ax:rejectfun:removepositivecombinations}. 
So, consider any $\optset,\altoptset\in\optsets$ such that $\optset\subseteq\altoptset\subseteq\chull\group{\optset}$. 
In order to prove that $\rejectfun\group{\altoptset}\cap\optset\subseteq\rejectfun\group{\optset}$, consider any $\opt\in\rejectfun\group{\altoptset}\cap\optset$.
Then also $\opt\in\altoptset$, so $\altoptset=\altoptset\cup\set{u}$, and therefore $\opt\in\rejectfun\group{\altoptset\cup\set{u}}$, whence $\altoptset-\opt\in\rejectset$ by Equation~\eqref{eq:interpretation:rejectfuns:after:irreflexivity:and:additivity:intermsof:K:Gert}.
Since it follows from the assumptions that also $\optset-\opt\subseteq\altoptset-\opt\subseteq\chull\group{\optset-\opt}$, it follows from Axiom~\ref{ax:rejects:removeconvexcombinations} that $\optset-\opt\in\rejectset$, so Equation~\eqref{eq:interpretation:rejectfuns:after:irreflexivity:and:additivity:intermsof:K:Gert} and $\optset=\optset\cup\set{\opt}$ guarantee that, indeed, $\opt\in\rejectfun\group{\optset}$.
\end{proof}

\begin{proof}[Proof of Proposition~\ref{prop:mixingbinaryiffD}]
First, assume that $\desirset\in\convcohdesirsets$.
Then $\desirset$ is in particular coherent, so $\rejectset[\desirset]$ is coherent too, by Proposition~\ref{prop:coherence:for:binary}. To prove that $\rejectset[\desirset]\in\convcohrejectsets$, it therefore suffices to show that $\rejectset[\desirset]$ satisfies~\ref{ax:rejects:removepositivecombinations}.
So consider any $\optset,\altoptset\in\optsets$ such that $\altoptset\in\rejectset[\desirset]$ and $\optset\subseteq\altoptset\subseteq\posi\group{\optset}$, then we have to prove that also $\optset\in\rejectset[\desirset]$. 
Since $\altoptset\in\rejectset[\desirset]$, it follows from Equation~\eqref{eq:desirset:to:rejectset} that $\altoptset\cap\desirset\neq\emptyset$ and therefore, since $\altoptset\subseteq\posi(\optset)$, that $\posi(\optset)\cap\desirset\neq\emptyset$.
We can now use the assumption that $\desirset$ satisfies~\ref{ax:desirs:mixing} to infer that also $\optset\cap\desirset\neq\emptyset$, which in turn implies that, indeed, $\optset\in\rejectset[\desirset]$, again because of Equation~\eqref{eq:desirset:to:rejectset}.

Conversely, assume that $\rejectset[\desirset]\in\convcohrejectsets$. 
Then $\rejectset[\desirset]$ is in particular coherent, and therefore so is $\desirset$, by Proposition~\ref{prop:coherence:for:binary}. 
To prove that $\desirset\in\convcohdesirsets$, it therefore suffices to show that $\desirset$ satisfies~\ref{ax:desirs:mixing}. 
Consider, therefore, any $\optset\in\optsets$ such that $\posi(\optset)\cap\desirset\neq\emptyset$, then we have to prove that also $\optset\cap\desirset\neq\emptyset$.
It follows from $\posi(\optset)\cap\desirset\neq\emptyset$ that there is some $\opt[o]\in\posi(\optset)\cap\desirset$. 
Let $\altoptset\coloneqq\optset\cup\set{\opt[o]}$, then clearly $\optset\subseteq\altoptset\subseteq\posi\group{\optset}$. 
Also, $\altoptset\cap\desirset\neq\emptyset$, so Equation~\eqref{eq:desirset:to:rejectset} guarantees that $\altoptset\in\rejectset[\desirset]$.
We then infer from the assumption that $\rejectset[\desirset]$ satisfies~\ref{ax:rejects:removepositivecombinations} that also $\optset\in\rejectset[\desirset]$, whence, indeed, $\optset\cap\desirset\neq\emptyset$, by Equation~\eqref{eq:desirset:to:rejectset}.
\end{proof}

\begin{lemma}\label{lem:removing:multiples}
Consider any set of desirable option sets $\rejectset\in\rejectsets$ that satisfies Axiom~\ref{ax:rejects:cone}.
Consider any $\optset\in\rejectset$ such that $\set{\altopttoo,\lambda\altopttoo}\subseteq\optset$ for some $\altopttoo\in\opts\setminus\set{0}$ and $\lambda>0$ such that $\lambda\neq1$.
Then $\optset\setminus\set{\altopttoo}\in\rejectset$.
\end{lemma}

\begin{proof}
Axiom~\ref{ax:rejects:cone} for $\optset$ and $\optset$ allows us to infer that $\cset{\lambda_{\opt,\altopt}\opt+\mu_{\opt,\altopt}\altopt}{\opt,\altopt\in\optset}\in\rejectset$ for all possible choices of $(\lambda_{\opt,\altopt},\mu_{\opt,\altopt})>0$.
Choosing $(\lambda_{\opt,\altopt},\mu_{\opt,\altopt})\coloneqq(1,0)$ for all $\opt\in\optset\setminus\set{\altopttoo}$ and $(\lambda_{\altopttoo,\altopt},\mu_{\altopttoo,\altopt})\coloneqq(\lambda,0)$, for all $\altopt\in\optset$, yields in particular that $\optset\setminus\set{\altopttoo}\in\rejectset$.
\end{proof}



\begin{proposition}\label{prop:removal:of:posis:gewijzigde:axiomas}
Consider any set of desirable option sets $\rejectset\in\rejectsets$.
Then\/ $\RP\group{\rejectset}$ satisfies\/~\ref{ax:rejects:removepositivecombinations}.
Moreover, if $\rejectset$ satisfies\/~\ref{ax:rejects:removezero}, \ref{ax:rejects:nozero}, \ref{ax:rejects:pos}, \ref{ax:rejects:cone} and\/~\ref{ax:rejects:mono} , then so does\/ $\RP\group{\rejectset}$.
\end{proposition}

\begin{proof}
We begin with the first statement.
Consider any $\optset[1]\in\RP\group{\rejectset}$ and any $\optset\in\optsets$ such that $\optset\subseteq\optset[1]\subseteq\posi\group{\optset}$.
Then, on the one hand, $\posi\group{\optset}=\posi(\optset[1])$, and on the other hand, there is some $\altoptset[1]\in\rejectset$ such that $\optset[1]\subseteq\altoptset[1]\subseteq\posi(\optset[1])$.
Hence $\optset\subseteq\optset[1]\subseteq\altoptset[1]\subseteq\posi(\optset[1])=\posi\group{\optset}$, and therefore indeed $\optset\in\RP\group{\rejectset}$. 

For the second statement, assume that $\rejectset$ satisfies~\ref{ax:rejects:removezero}, \ref{ax:rejects:nozero}, \ref{ax:rejects:pos}, \ref{ax:rejects:cone}, and\/~\ref{ax:rejects:mono}.

To prove that $\RP\group{\rejectset}$ satisfies~\ref{ax:rejects:removezero}, consider any $\optset\in\RP\group{\rejectset}$, meaning that there is some $\altoptset\in\rejectset$ such that $\optset\subseteq\altoptset\subseteq\posi\group{\optset}$.
To see that $\optset\setminus\set{0}\in\RP\group{\rejectset}$, it suffices to show that $\altoptset\setminus\set{0}\subseteq\posi(\optset\setminus\set{0})$, because clearly $\optset\setminus\set{0}\subseteq\altoptset\setminus\set{0}$.
Consider any element $\opt$ of $\altoptset\setminus\set{0}$, then it follows from the assumption $\altoptset\subseteq\posi\group{\optset}$ that $\opt=\sum_{k=1}^n\lambda_k\opt[k]$, with $n\geq1$, different $\opt[k]\in\optset$, and $\lambda_k>0$.
Even if $\opt[k]=0$ for some $k$, we still find that $\opt\in\posi(\optset\setminus\set{0})$.

To prove that $\RP\group{\rejectset}$ satisfies~\ref{ax:rejects:nozero}, assume {\itshape ex absurdo} that $\set{0}\in\RP\group{\rejectset}$, so there is some $\altoptset\in\rejectset$ such that $\set{0}\subseteq\altoptset\subseteq\posi\group{\set{0}}$, or in other words, such that $\altoptset=\set{0}$, contradicting that $\rejectset$ satisfies~\ref{ax:rejects:nozero}.

To prove that $\RP\group{\rejectset}$ satisfies~\ref{ax:rejects:pos}, simply observe that the operator $\RP$ never removes option sets from a set of desirable option sets, so the option sets $\set{\opt}$, $\opt\optgt0$ that belong to $\rejectset$ by~\ref{ax:rejects:pos}, will also belong to the larger $\RP\group{\rejectset}$.

To prove that $\RP\group{\rejectset}$ satisfies~\ref{ax:rejects:cone}, consider any $\optset[1],\optset[2]\in\RP\group{\rejectset}$, meaning that there are $\altoptset[1],\altoptset[2]\in\rejectset$ such that $\optset[1]\subseteq\altoptset[1]\subseteq\posi(\optset[1])$ and $\optset[2]\subseteq\altoptset[2]\subseteq\posi(\optset[2])$.
This tells us that any $\altopt[1]\in\altoptset[1]$ can be written as $\altopt[1]=\sum_{\opt[1]\in\optset[1]}\alpha_{\altopt[1],\opt[1]}\opt[1]$ with $\alpha_{\opt[1],\bullet}>0$, and similarly, any $\altopt[2]\in\altoptset[2]$ can be written as $\altopt[2]=\sum_{\opt[2]\in\optset[2]}\beta_{\altopt[2],\opt[2]}\opt[2]$ with $\beta_{\altopt[2],\bullet}>0$.

Choose, for all $\opt\in\optset[1]$ and $\altopt\in\optset[2]$, $(\lambda_{\opt,\altopt},\mu_{\opt,\altopt})>0$, then we must show that
\begin{equation*}
\optset
\coloneqq\cset{\lambda_{\opt,\altopt}\opt+\mu_{\opt,\altopt}\altopt}{\opt\in\optset[1],\altopt\in\optset[2]}
\in\RP\group{\rejectset},
\end{equation*}
or in other words that there is some $\altoptset\in\rejectset$ such that $\optset\subseteq\altoptset\subseteq\posi\group{\optset}$.
We will show that there is some $\altoptset\in\setposi\set{\altoptset[1],\altoptset[2]}$ that does the job, or in other words that there are suitable choices of $(\kappa_{\altopt[1],\altopt[2]},\rho_{\altopt[1],\altopt[2]})>0$ for all $\altopt[1]\in\altoptset[1]$ and $\altopt[2]\in\altoptset[2]$ such that
\begin{equation}\label{eq:removal:of:posis:to:reach}
\optset\subseteq\cset{\kappa_{\altopt[1],\altopt[2]}\altopt[1]+\rho_{\altopt[1],\altopt[2]}\altopt[2]}
{\altopt[1]\in\altoptset[1],\altopt[2]\in\altoptset[2]}
\subseteq\posi\group{\optset}.
\end{equation}
For a start, if we let $\kappa_{\opt[1],\opt[2]}\coloneqq\lambda_{\opt[1],\opt[2]}$ and $\rho_{\opt[1],\opt[2]}\coloneqq\mu_{\opt[1],\opt[2]}$ for all $\opt[1]\in\optset[1]$ and all $\opt[2]\in\optset[2]$, then we are already guaranteed that the first inequality in~\eqref{eq:removal:of:posis:to:reach} holds.
It now remains to choose the remaining $(\kappa_{\altopt[1],\altopt[2]},\rho_{\altopt[1],\altopt[2]})>0$ in such a way that the second inequality in~\eqref{eq:removal:of:posis:to:reach} will also hold, meaning that $\kappa_{\altopt[1],\altopt[2]}\altopt[1]+\rho_{\altopt[1],\altopt[2]}\altopt[2]\in\posi\group{\optset}$.

We consider three mutually exclusive possibilities, for any fixed remaining $\altopt[1]$ and $\altopt[2]$.
The first possibility is that $\altopt[1]=\opt[1]\in\optset[1]$ and $\altopt[2]\in\altoptset[2]\setminus\optset[2]$.
If there is some $\opt[2]\in\optset[2]$ such that $\mu_{\opt[1],\opt[2]}=0$ and therefore $\lambda_{\opt[1],\opt[2]}>0$, then we choose $\kappa_{\opt[1],\altopt[2]}\coloneqq\lambda_{\opt[1],\opt[2]}$ and $\rho_{\opt[1],\altopt[2]}\coloneqq0$, and then indeed
\begin{equation*}
\kappa_{\opt[1],\altopt[2]}\opt[1]+\rho_{\opt[1],\altopt[2]}\altopt[2]
=\lambda_{\opt[1],\opt[2]}\opt[1]
=\lambda_{\opt[1],\opt[2]}\opt[1]+\mu_{\opt[1],\opt[2]}\opt[2]
\in\optset\subseteq\posi\group{\optset}.
\end{equation*}
If, on the other hand, $\mu_{\opt[1],\opt[2]}>0$ for all $\opt[2]\in\optset[2]$, then we choose $\rho_{\opt[1],\altopt[2]}\coloneqq1$ and
\begin{equation*}
\kappa_{\opt[1],\altopt[2]}\coloneqq\sum_{\opt[2]\in\optset[2]}\frac{\lambda_{\opt[1],\opt[2]}}{\mu_{\opt[1],\opt[2]}}\beta_{\altopt[2],\opt[2]},
\end{equation*}
and then indeed
\begin{align*}
\kappa_{\opt[1],\altopt[2]}\opt[1]+\rho_{\opt[1],\altopt[2]}\altopt[2]
&=\sum_{\opt[2]\in\optset[2]}\frac{\lambda_{\opt[1],\opt[2]}}{\mu_{\opt[1],\opt[2]}}\beta_{\altopt[2],\opt[2]}\opt[1]
+\sum_{\opt[2]\in\optset[2]}\beta_{\altopt[2],\opt[2]}\opt[2]\\
&=\sum_{\opt[2]\in\optset[2]}\frac{\beta_{\altopt[2],\opt[2]}}{\mu_{\opt[1],\opt[2]}}
\group{\lambda_{\opt[1],\opt[2]}\opt[1]+\mu_{\opt[1],\opt[2]}\opt[2]}
\in\posi\group{\optset}.
\end{align*}

The second possibility is that $\altopt[1]\in\altoptset[1]\setminus\optset[1]$ and $\altopt[2]=\opt[2]\in\optset[2]$, and this is treated in a similar way as the first.

And the third and last possibility is that $\altopt[1]\in\altoptset[1]\setminus\optset[2]$ and $\altopt[2]\in\altoptset[2]\setminus\optset[2]$.
We now partition both $\optset[1]$ and $\optset[2]$ in three disjoint pieces.
For $\optset[1]$ we let $\optset[1]^1\coloneqq\cset{\opt[1]\in\optset[1]}{\alpha_{\altopt[1],\opt[1]}=0}$, $\optset[1]^2\coloneqq\cset{\opt[1]\in\optset[1]\setminus\optset[1]^1}{(\forall\opt[2]\in\optset[2])\mu_{\opt[1],\opt[2]}>0}$ and $\optset[1]^3\coloneqq\cset{\opt[1]\in\optset[1]\setminus\optset[1]^1}{(\exists\opt[2]\in\optset[2])\mu_{\opt[1],\opt[2]}=0}$.
Similarly, for $\optset[2]$ we let $\optset[2]^1\coloneqq\cset{\opt[2]\in\optset[2]}{\beta_{\altopt[2],\opt[2]}=0}$, $\optset[2]^2\coloneqq\cset{\opt[2]\in\optset[2]\setminus\optset[2]^1}{(\forall\opt[1]\in\optset[1])\lambda_{\opt[1],\opt[2]}>0}$ and $\optset[2]^3\coloneqq\cset{\opt[2]\in\optset[2]\setminus\optset[2]^1}{(\exists\opt[1]\in\optset[1])\lambda_{\opt[1],\opt[2]}=0}$.

If $\optset[1]^2=\emptyset$ then it cannot be that also $\optset[1]^3=\emptyset$, because $\alpha_{\altopt[1],\opt[1]}>0$ for at least one $\opt[1]\in\optset[1]$.
So for each $\opt[1]\in\optset[1]^3\neq\emptyset$, we choose some $\opt[{2,\opt[1]}]\in\optset[2]$ such that $\mu_{\opt[1],\opt[{2,\opt[1]}]}=0$ and therefore $\lambda_{\opt[1],\opt[{2,\opt[1]}]}>0$ [which is always possible, by the definition of $\optset[1]^3$], and we let $\kappa_{\altopt[1],\altopt[2]}\coloneqq1$ and $\rho_{\altopt[1],\altopt[2]}\coloneqq0$.
Then indeed 
\begin{equation*}
\kappa_{\altopt[1],\altopt[2]}\altopt[1]+\rho_{\altopt[1],\altopt[2]}\altopt[2]
=\sum_{\opt[1]\in\optset[1]^3}\alpha_{\altopt[1],\opt[1]}\opt[1]
=\sum_{\opt[1]\in\optset[1]^3}\frac{\alpha_{\altopt[1],\opt[1]}}{\lambda_{\opt[1],\opt[{2,\opt[1]}]}}
\group[\big]{\lambda_{\opt[1],\opt[{2,\opt[1]}]}\opt[1]+\mu_{\opt[1],\opt[{2,\opt[1]}]}\opt[2]}
\in\posi\group{\optset}.
\end{equation*}
A completely symmetrical argument can be made when $\optset[2]^2=\emptyset$, so we may now assume without loss of generality that both $\optset[1]^2\neq\emptyset$ and $\optset[2]^2\neq\emptyset$.
Then, as before, for any $\opt[1]\in\optset[1]^3$ we choose some $\opt[{2,\opt[1]}]\in\optset[2]$ such that $\mu_{\opt[1],\opt[{2,\opt[1]}]}=0$ and therefore $\lambda_{\opt[1],\opt[{2,\opt[1]}]}>0$ [possible by the definition of $\optset[1]^3$], and for any $\opt[2]\in\optset[2]^3$, we choose some $\opt[{1,\opt[2]}]\in\optset[1]$ such that $\lambda_{\opt[{1,\opt[2]}],\opt[2]}=0$ and therefore $\mu_{\opt[{1,\opt[2]}],\opt[2]}>0$ [possible by the definition of $\optset[2]^3$].
We also let
\begin{equation*}
\altopttoo[1]\coloneqq\smashoperator[r]{\sum_{\opt[1]\in\optset[1]^2}}\alpha_{\altopt[1],\opt[1]}\opt[1]
\text{ and }
\altopttoo[1]'\coloneqq\smashoperator[r]{\sum_{\opt[1]\in\optset[1]^3}}\alpha_{\altopt[1],\opt[1]}\opt[1],
\text{ so }\altopt[1]=\altopttoo[1]+\altopttoo[1]',
\end{equation*}
and
\begin{equation*}
\altopttoo[2]\coloneqq\smashoperator[r]{\sum_{\opt[2]\in\optset[2]^2}}\beta_{\altopt[2],\opt[2]}\opt[2]
\text{ and }
\altopttoo[2]'\coloneqq\smashoperator[r]{\sum_{\opt[2]\in\optset[2]^3}}\beta_{\altopt[2],\opt[2]}\opt[2],
\text{ so }\altopt[2]=\altopttoo[2]+\altopttoo[2]'.
\end{equation*}
Then we already know that, by a similar argument as before:
\begin{equation}\label{eq:first:in:posi}
\altopttoo[1]'
=\sum_{\opt[1]\in\optset[1]^3}\frac{\alpha_{\altopt[1],\opt[1]}}{\lambda_{\opt[1],\opt[{2,\opt[1]}]}}
\group[\big]{\lambda_{\opt[1],\opt[{2,\opt[1]}]}\opt[1]+\mu_{\opt[1],\opt[{2,\opt[1]}]}\opt[{2,\opt[1]}]}
\in\posi\group{\optset}
\end{equation}
and
\begin{equation}\label{eq:second:in:posi}
\altopttoo[2]'
=\sum_{\opt[2]\in\optset[2]^3}\frac{\beta_{\altopt[2],\opt[2]}}{\mu_{\opt[{1,\opt[2]}],\opt[2]}}
\group{\lambda_{\opt[{1,\opt[2]}],\opt[2]}\opt[{1,\opt[2]}]+\mu_{\opt[{1,\opt[2]}],\opt[2]}\opt[2]}
\in\posi\group{\optset}.
\end{equation}
Next, recall that $\lambda_{\opt[1],\opt[2]}>0$ and $\mu_{\opt[1],\opt[2]}>0$ for any $\opt[1]\in\optset[1]$ and $\opt[2]\in\optset[2]$.
We then infer from $\lambda_{\opt[1],\opt[2]}\opt[1]+\mu_{\opt[1],\opt[2]}\opt[2]\in\optset$ for all $\opt[1]\in\optset[1]^2$ and $\opt[2]\in\optset[2]^2$ that
\begin{equation*}
\phi_{\altopt[2],\opt[1]}\opt[1]+\altopttoo[2]\in\posi\group{\optset},
\text{ with }
\phi_{\altopt[2],\opt[1]}
\coloneqq\smashoperator[r]{\sum_{\opt[2]\in\optset[2]^2}}\beta_{\altopt[2],\opt[2]}\frac{\lambda_{\opt[1],\opt[2]}}{\mu_{\opt[1],\opt[2]}}>0,
\end{equation*}
and therefore also
\begin{equation}\label{eq:third:in:posi}
\altopttoo[1]+\psi_{\altopt[1],\altopt[1]}\altopttoo[2]\in\posi\group{\optset},
\text{ with }
\psi_{\altopt[1],\altopt[2]}
\coloneqq\sum_{\opt[1]\in\optset[1]^2}\frac{\alpha_{\altopt[1],\opt[1]}}{\phi_{\altopt[2],\opt[1]}}>0.
\end{equation}
Now let $\kappa_{\altopt[1],\altopt[2]}\coloneqq1$ and $\rho_{\altopt[1],\altopt[2]}\coloneqq\psi_{\altopt[1],\altopt[2]}$, then we infer from~\eqref{eq:first:in:posi}--\eqref{eq:third:in:posi} that, indeed,
\begin{align*}
\kappa_{\altopt[1],\altopt[2]}\altopt[1]+\rho_{\altopt[1],\altopt[2]}\altopt[2]
&=\altopt[1]+\psi_{\altopt[1],\altopt[2]}\altopt[2]
=\altopttoo[1]+\altopttoo[1]'+\psi_{\altopt[1],\altopt[2]}\group{\altopttoo[2]+\altopttoo[2]'}\\
&=\altopttoo[1]+\psi_{\altopt[1],\altopt[2]}\altopttoo[2]
+\altopttoo[1]'+\psi_{\altopt[1],\altopt[2]}\altopttoo[2]'
\in\posi\group{\optset}.
\end{align*}

To prove that, finally, $\RP\group{\rejectset}$ satisfies~\ref{ax:rejects:mono}, consider any $\optset[1]\in\RP\group{\rejectset}$ and any $\optset[2]\in\optsets$ such that $\optset[1]\subseteq\optset[2]$.
This implies on the one hand that $\posi(\optset[1])\subseteq\posi(\optset[2])$, and on the other hand that there is some $\altoptset[1]\in\rejectset$ such that $\optset[1]\subseteq\altoptset[1]\subseteq\posi(\optset[1])$, and therefore also $\posi(\optset[1])=\posi(\altoptset[1])$.
But then $\posi(\optset[2]\cup\altoptset[1])=\posi(\optset[2]\cup\posi(\altoptset[1]))=\posi(\optset[2])$ and therefore $\optset[2]\subseteq\optset[2]\cup\altoptset[1]\subseteq\posi(\optset[2])$.
Since, moreover, $\altoptset[1]\in\rejectset$ and $\altoptset[1]\subseteq\optset[2]\cup\altoptset[1]$, \ref{ax:rejects:mono} guarantees that $\optset[2]\cup\altoptset[1]\in\rejectset$, and therefore indeed $\optset[2]\in\RP\group{\rejectset}$.
\end{proof}

\begin{proposition}\label{prop:nonbinarymixing:is:dominated}
Any non-binary mixing set of desirable option sets $\rejectset$ is \emph{strictly dominated} by some other mixing set of desirable option sets.
\end{proposition}

\begin{proof}
Consider an arbitrary mixing non-binary set of desirable option sets~$\rejectset$. 
Since $\rejectset$ is non-binary and coherent, Lemma~\ref{prop:nonbinary:is:dominated} guarantees the existence of a coherent set of desirable option sets~$\rejectset^*$ such that $\rejectset\subset\rejectset^*$. 
Let $\rejectset^{**}\coloneqq\RP(\rejectset^*)$. 
Since $\rejectset^*$ is coherent, we know from Proposition~\ref{prop:removal:of:posis:gewijzigde:axiomas} that $\rejectset^{**}$ is mixing. 
Since $\RP$ does not remove option sets, we also know that $\rejectset^*\subseteq\rejectset^{**}$.
Since $\rejectset\subset\rejectset^*$, this implies that $\rejectset$ is strictly dominated by the mixing set of desirable option sets~$\rejectset^{**}$.
\end{proof}

\begin{lemma}\label{lem:chainmixing}
For any non-empty chain $\chainofrejectsets$ in $\convcohrejectsets$, its union $\rejectset[o]\coloneqq\bigcup\chainofrejectsets$ is a mixing set of desirable option sets.
\end{lemma}

\begin{proof}
Since $\convcohrejectsets\subseteq\cohrejectsets$, it follows from Lemma~\ref{lem:chaincoherence} that $\rejectset[o]$ is coherent. 
To prove that it is also mixing, we consider any $\optset,\altoptset\in\optsets$ such that $\altoptset\in\rejectset[o]$ and $\optset\subseteq\altoptset\subseteq\posi\group{\optset}$. 
Since $\altoptset\in\rejectset[o]$, there is some $\rejectset'\in\chainofrejectsets$ such that $\altoptset\in\rejectset'$, and since $\rejectset'$ is mixing, this implies that $\optset\in\rejectset'\subseteq\rejectset[o]$.
\end{proof}

\begin{lemma}\label{lem:Zornconvexity}
For any mixing set of desirable option sets $\rejectset\in\convcohrejectsets$ and any set of desirable option sets $\rejectset^*\in\rejectsets$ such that $\rejectset\cap\rejectset^*=\emptyset$, the partially ordered set
\begin{equation*}
\upset{\rejectset[\mathrm{M}]}\coloneqq\cset{\rejectset'\in\convcohrejectsets}{\rejectset\subseteq\rejectset'\text{ and }\rejectset'\cap\rejectset^*=\emptyset}
\end{equation*}
has a maximal element.
\end{lemma}

\begin{proof}
We will use Zorn's Lemma to establish the existence of a maximal element.
So consider any non-empty chain $\chainofrejectsets$ in $\upset{\rejectset[\mathrm{M}]}$, then we must prove that $\chainofrejectsets$ has an upper bound in $\upset{\rejectset[\mathrm{M}]}$.
Since $\rejectset[o]\coloneqq\bigcup\chainofrejectsets$ is clearly an upper bound, we are done if we can prove that $\rejectset[o]\in\upset{\rejectset[\mathrm{M}]}$.

That $\rejectset[o]\cap\rejectset^*=\emptyset$ follows from the fact that $\rejectset'\cap\rejectset^*=\emptyset$ for every $\rejectset'\in\chainofrejectsets\subseteq\upset{\rejectset[\mathrm{M}]}$. 
That $\rejectset[o]$ is mixing follows from Lemma~\ref{lem:chainmixing}.
\end{proof}

\begin{proposition}\label{prop:mixingdominatebybinarymixing}
Every mixing set of desirable option sets $\rejectset\in\convcohrejectsets$ is dominated by some binary mixing set of desirable option sets.
\end{proposition}

\begin{proof}
Lemma~\ref{lem:Zornconvexity} for $\rejectset^*=\emptyset$ guarantees that the partially ordered set $\cset{\rejectset'\in\convcohrejectsets}{\rejectset\subseteq\rejectset'}$ has a maximal element. 
Let $\maxrejectset$ be any such maximal element. 
Assume \emph{ex absurdo} that $\maxrejectset$ is non-binary. 
It then follows from Proposition~\ref{prop:nonbinarymixing:is:dominated} that $\maxrejectset$ is strictly dominated by some other mixing set of desirable option sets, meaning that there is some $\rejectset^*\in\convcohrejectsets$ such that $\maxrejectset\subset\rejectset^*$. 
Then clearly also $\rejectset^*\in\cset{\rejectset'\in\cohrejectsets}{\rejectset\subseteq\rejectset'}$, which contradicts the maximal character of $\maxrejectset$. 
Hence, it must be that $\maxrejectset$ is indeed binary. 
\end{proof}

\begin{theorem}[Representation for mixing choice functions]\label{theo:convex:rejectsets:representation}
Any mixing set of desirable option sets $\rejectset\in\convcohrejectsets$ is dominated by some binary mixing set of desirable option sets: $\convcohdesirsets(\rejectset)\coloneqq\cset{\desirset\in\convcohdesirsets}{\rejectset\subseteq\rejectset[\desirset]}\neq\emptyset$.
Moreover, $\rejectset=\bigcap\cset{\rejectset[\desirset]}{\desirset\in\convcohdesirsets(\rejectset)}$.
\end{theorem}

\begin{proof}[Proof of Theorem~\ref{theo:convex:rejectsets:representation}]
Let $\rejectset[o]\in\convcohrejectsets$ be any mixing set of desirable option sets. 
We prove that $\convcohdesirsets(\rejectset[o])\coloneqq\cset{\desirset\in\convcohdesirsets}{\rejectset[o]\subseteq\rejectset[\desirset]}\neq\emptyset$ and that $\rejectset[o]=\bigcap\cset{\rejectset[\desirset]}{\desirset\in\convcohdesirsets(\rejectset[o])}$.

For the first statement, recall from Proposition~\ref{prop:mixingdominatebybinarymixing} that $\rejectset[o]$ is dominated by some binary mixing coherent set of desirable option sets $\maxrejectset$.
Since $\maxrejectset$ is binary, Proposition~\ref{prop:binary:characterisation} implies that $\maxrejectset=\rejectset[\desirset]$, with $\smash{\desirset=\desirset[\maxrejectset]}$. 
Furthermore, because $\maxrejectset$ is mixing, Proposition~\ref{prop:mixingbinaryiffD} implies that $\desirset[\maxrejectset]$ is mixing too, whence $\desirset[\maxrejectset]\in\convcohdesirsets$. 
Since $\rejectset[o]\subseteq\maxrejectset=\rejectset[{\desirset[\maxrejectset]}]$, we find that $\desirset[\maxrejectset]\in\convcohdesirsets(\rejectset[o])\coloneqq\cset{\desirset\in\convcohdesirsets}{\rejectset[o]\subseteq\rejectset[\desirset]}\neq\emptyset$.

For the second statement, it is obvious that $\rejectset[o]\subseteq\bigcap\cset{\rejectset[\desirset]}{\desirset\in\convcohdesirsets(\rejectset[o])}$, so we concentrate on the converse inclusion.
Assume {\itshape ex absurdo} that $\rejectset[o]\subset\bigcap\cset{\rejectset[\desirset]}{\desirset\in\convcohdesirsets(\rejectset[o])}$, so there is some option set $\altoptset[o]\in\optsets$ such that $\altoptset[o]\notin\rejectset[o]$ and $\altoptset[o]\in\rejectset[\desirset]$ for all $\desirset\in\convcohdesirsets(\rejectset[o])$.
Then $\altoptset[o]\setminus\nonposopts\notin\rejectset[o]$ [use the coherence of $\rejectset[o]$ and~\ref{ax:rejects:mono}] and $\altoptset[o]\setminus\nonposopts\in\rejectset[\desirset]$ for all $\desirset\in\convcohdesirsets(\rejectset[o])$ [use the coherence of $\rejectset[\desirset]$---which follows from Proposition~\ref{prop:fromCohDtoCohK} and the coherence of $\desirset$---and Proposition~\ref{prop:ax:rejects:RN:equivalents}], so we may assume without loss of generality that $\altoptset[o]$ has no non-positive options: $\altoptset[o]\cap\nonposopts=\emptyset$.
Moreover, if $\altoptset[o]$ contains some option $\altopttoo$ that is a positive linear combination of \emph{other} elements of $\altoptset[o]$, then still $\altoptset[o]\setminus\set{\altopttoo}\notin\rejectset[o]$ [use the coherence of $\rejectset[o]$ and~\ref{ax:rejects:mono}] and $\altoptset[o]\setminus\set{\altopttoo}\in\rejectset[\desirset]$ for all $\desirset\in\convcohdesirsets(\rejectset[o])$ [use the coherence of $\rejectset[\desirset]$ and~\ref{ax:rejects:removepositivecombinations}], so we may assume without loss of generality that $\altoptset[o]$ has no elements that are positive linear combinations of some of its other elements.

Consider the partially ordered set $\upset{\rejectset[o]}\coloneqq\cset{\rejectset\in\convcohrejectsets}{\rejectset[o]\subseteq\rejectset\text{ and }\altoptset[o]\notin\rejectset}$, which is non-empty because it contains $\rejectset[o]$. 
Due to to Lemma~\ref{lem:Zornconvexity} [applied for $\rejectset=\rejectset[o]$ and $\rejectset^*=\set{\altoptset[o]}$], it has a maximal element.
If we can prove that any such maximal element is binary, then it will follow from Propositions~\ref{prop:binary:characterisation} and~\ref{prop:mixingbinaryiffD} that there is some mixing set of desirable options~$\desirset[o]\in\convcohdesirsets$ such that $\rejectset[o]\subseteq\rejectset[{\desirset[o]}]$---and therefore $\desirset[o]\in\convcohdesirsets(\rejectset[o])$---and $\altoptset[o]\notin\rejectset[{\desirset[o]}]$, a contradiction.
To prove that all maximal elements of $\upset{\rejectset[o]}$ are binary, it suffices to prove that any non-binary element of $\upset{\rejectset[o]}$ is strictly dominated in that set, which is what we now set out to do.

So consider any non-binary mixing coherent element $\rejectset$ of $\upset{\rejectset[o]}$, so $\rejectset[o]\subseteq\rejectset$ and $\altoptset[o]\notin\rejectset$.
It follows from Lemma~\ref{lem:binaryalternative} that there is some $\optset[o]\in\rejectset$ such that $\card{\optset[o]}\geq2$ and $\optset[o]\setminus\set{\opt}\notin\rejectset$ for all $\opt\in\optset[o]$.
The partially ordered set $\cset{\optset\in\rejectset}{\altoptset[o]\subseteq\optset}$ contains $\optset[o]\cup\altoptset[o]$ and therefore has some minimal (undominating) element $\altoptset^*$ below it, so $\altoptset^*\in\rejectset$ and $\altoptset[o]\subseteq\altoptset^*\subseteq\optset[o]\cup\altoptset[o]$.

Let us first summarise what we know about this minimal element $\altoptset^*$.
It is impossible that $\altoptset^*\subseteq\altoptset[o]$ because otherwise $\altoptset[o]=\altoptset^*\in\rejectset$, a contradiction.
Hence $\altoptset^*\setminus\altoptset[o]\neq\emptyset$, so we can fix any element $\opt[o]$ in $\altoptset^*\setminus\altoptset[o]\subseteq\optset[o]$.
Then we know that $\set{\opt[o]}\notin\rejectset$, because otherwise, since $\optset[o]$ has at least one other element $\altopt$ [because $\card{\optset[o]}\geq2$], we would have that $\set{\opt[o]}\subseteq\optset\setminus\set{\altopt}$ and therefore $\optset\setminus\set{\altopt}\in\rejectset$ by the coherence of~$\rejectset$, a contradiction.
And since also $\altoptset[o]\subseteq\altoptset^*\setminus\set{\opt[o]}$ but $\altoptset^*\setminus\set{\opt[o]}\subset\altoptset^*$, we can conclude that $\altoptset^*\setminus\set{\opt[o]}\notin\rejectset$, by the definition of a minimal element.
We also know that $\altoptset^*$, and therefore also $\altoptset[o]$, cannot contain any positive multiple $\mu\opt[o]$ of $\opt[o]$ with $\mu>0$ and $\mu\neq1$, because otherwise we could use Lemma~\ref{lem:removing:multiples} to remove $\opt[o]$ from $\altoptset^*\in\rejectset$ and still be guaranteed that $\altoptset^*\setminus\set{\opt[o]}\in\rejectset$.
Since we know that $\opt[o]\notin\altoptset[o]$, we would then still have that $\altoptset[o]\subseteq\altoptset^*\setminus\set{\opt[o]}$, which would contradict the minimality of $\altoptset^*$ in the partially ordered set $\cset{\optset\in\rejectset}{\altoptset[o]\subseteq\optset}$.
And, finally, we know that $\opt[o]\notin\posi\group{\altoptset[o]}$, because otherwise $\opt[o]\in\altoptset^*$ would be a positive linear combination of elements of $\altoptset[o]\subseteq\altoptset^*$ different from $\opt[o]$ [because $\opt[o]\notin\altoptset[o]$], so we could use Proposition~\ref{prop:removal:of:posis:gewijzigde:axiomas} and~\ref{ax:rejects:removepositivecombinations} to make sure that still $\altoptset^*\setminus\set{\opt[o]}\in\rejectset$, again contradicting the minimality of $\altoptset^*$ in the partially ordered set $\cset{\optset\in\rejectset}{\altoptset[o]\subseteq\optset}$.

Consider now the set of desirable option sets $\rejectset^{*}\coloneqq\RP\group{\rejectset^{**}}$, with $\rejectset^{**}\coloneqq\RN\group{\rejectset^{***}}$, where
\begin{equation*}
\rejectset^{***}
\coloneqq
\cset[\Big]{\cset[\big]{\lambda_{\altopt}\altopt+\mu_{\altopt}\opt[o]}{\altopt\in\altoptset}}
{\altoptset\in\rejectset,(\forall\altopt\in\altoptset)(\lambda_{\altopt},\mu_{\altopt})>0}.
\end{equation*}
We know from Lemma~\ref{lem:Kstarstar} [with $\rejectset^{***}$ in this proof taking on the role of $\rejectset^{**}$ in the statement of Lemma~\ref{lem:Kstarstar}, and with $\rejectset^{**}$ in this proof taking on the role of $\rejectset^{*}$ in the statement of Lemma~\ref{lem:Kstarstar}] that $\rejectset^{**}$ is a coherent set of desirable option sets that is a superset of $\rejectset$ and contains $\set{\opt[o]}$. 
Proposition~\ref{prop:removal:of:posis:gewijzigde:axiomas} now guarantees that $\rejectset^{*}=\RP\group{\rejectset^{**}}$ satisfies~\ref{ax:rejects:removezero}--\ref{ax:rejects:mono} and~\ref{ax:rejects:removepositivecombinations}, and is therefore mixing. 
Furthermore, since $\RP$ never removes option sets, $\rejectset^{*}$ is a superset of $\rejectset$ that contains $\set{\opt[o]}$.
Since we know that $\set{\opt[o]}\notin\rejectset$, this shows that $\rejectset\subset\rejectset^*$.
If we can now prove that $\altoptset[o]\notin\rejectset^*$ and therefore $\rejectset^*\in\upset{\rejectset[o]}$, we are done, because then $\rejectset$ is strictly dominated by $\rejectset^*$ in $\upset{\rejectset[o]}$.

Assume therefore {\itshape ex absurdo} that $\altoptset[o]\in\rejectset^*=\RP\group{\RN\group{\rejectset^{***}}}$, meaning that there are $\altoptsettoo\in\rejectset$, $(\lambda_{\altopt},\mu_{\altopt})>0$ for all $\altopt\in\altoptsettoo$ and $\altoptset\in\optsets$ such that 
\begin{equation}\label{eq:convex:intermediate}
\cset{\lambda_{\altopt}\altopt+\mu_{\altopt}\opt[o]}{\altopt\in\altoptsettoo}\setminus\nonposopts
\subseteq\altoptset
\subseteq\cset{\lambda_{\altopt}\altopt+\mu_{\altopt}\opt[o]}{\altopt\in\altoptsettoo}
\text{ and }
\altoptset[o]\subseteq\altoptset\subseteq\posi\group{\altoptset[o]}.
\end{equation}
We let $\altoptsettoo[2]\coloneqq\cset{\altopt\in\altoptsettoo}{\lambda_{\altopt}\altopt+\mu_{\altopt}\opt[o]\optlteq0}$, and consider any $\altopt\in\altoptsettoo[2]$, so $\lambda_{\altopt}\altopt+\mu_{\altopt}\opt[o]\optlteq0$.
But then we must have that $\lambda_{\altopt}>0$, because otherwise $\mu_{\altopt}\opt[o]\optlteq0$, with $\mu_{\altopt}>0$, and therefore also $\opt[o]\optlteq0$.
But then Lemma~\ref{lem:replacing:nonpositives:by:zero} and Axiom~\ref{ax:rejects:removezero} would imply that we can remove the non-positive $\opt[o]$ from $\optset[o]\in\rejectset$ and guarantee that still $\optset[o]\setminus\set{\opt[o]}\in\rejectset$, which contradicts our assumptions about $\optset[o]$.
So we find that
\begin{equation*}
\altopt\optlteq-\frac{\mu_{\altopt}}{\lambda_{\altopt}}\opt[o]
\text{ for all $\altopt\in\altoptsettoo[2]$}.
\end{equation*}
If we also call $\altoptsettoo[1]\coloneqq\cset{\altopt\in\altoptsettoo}{\lambda_{\altopt}\altopt+\mu_{\altopt}\opt[o]\in\altoptset[o]}$, then we infer from~\eqref{eq:convex:intermediate} that $\altoptsettoo[1]\neq\emptyset$. 
Also, $\altoptsettoo[1]\cap\altoptsettoo[2]=\emptyset$, because we know that $\altoptset[o]\cap\nonposopts=\emptyset$.
We complete the partition of $\altoptsettoo$ by letting $\altoptsettoo[3]\coloneqq\altoptsettoo\setminus\group{\altoptsettoo[1]\cup\altoptsettoo[2]}$.
We know that
\begin{center}
for each $b\in\altoptset$ there are $\altopt[b]\in\altoptsettoo$ such that
$b=\lambda_{\altopt[b]}\altopt[b]+\mu_{\altopt[b]}\opt[o]$.
\end{center}
This means that for any $\altopt\in\altoptsettoo[1]\cup\altoptsettoo[3]$ there is by construction a necessarily unique $b\in\altoptset$ such that $\altopt=\altopt[b]$, which we will denote by $b_{\altopt}$, so $b_{\altopt}=\lambda_{\altopt}\altopt+\mu_{\altopt}\opt[o]$.
But then $\lambda_{\altopt}>0$, because otherwise we would have that $\mu_{\altopt}\opt[o]=b_{\altopt}\in\altoptset$, with $\mu_{\altopt}>0$, so~\eqref{eq:convex:intermediate} would imply that $\opt[o]\in\posi\group{\altoptset[o]}$, and we have argued above that this is impossible.
So we find that
\begin{equation*}
\altopt=\frac{1}{\lambda_{\altopt}}b_{\altopt}-\frac{\mu_{\altopt}}{\lambda_{\altopt}}\opt[o]
\text{ for all $\altopt\in\altoptsettoo[1]\cup\altoptsettoo[3]$}.
\end{equation*}
Let $\altoptset[3]\coloneqq\cset{b_{\altopt}}{\altopt\in\altoptsettoo[3]}$, then it follows from our construction that $\altoptset[3]\subseteq\posi\group{\altoptset[o]}\setminus\altoptset[o]$.

Since $\group{\altoptsettoo[1]\cup\altoptsettoo[3]}\cup\altoptsettoo[2]=\altoptsettoo\in\rejectset$, we infer from Lemma~\ref{lem:replacing:by:dominating:options} that if we construct the option set $\altoptsettoo'$ as follows
\begin{equation*}
\altoptsettoo'
\coloneqq\cset[\bigg]{\frac{1}{\lambda_{\altopt}}b_{\altopt}-\frac{\mu_{\altopt}}{\lambda_{\altopt}}\opt[o]}
{\altopt\in\altoptsettoo[1]\cup\altoptsettoo[3]}
\cup\cset[\bigg]{-\frac{\mu_{\altopt}}{\lambda_{\altopt}}\opt[o]}{\altopt\in\altoptsettoo[2]},
\end{equation*}
then still $\altoptsettoo'\in\rejectset$.
If we recall that also $\altoptset^*\in\rejectset$, we can invoke~\ref{ax:rejects:cone} to find that $\setposi\set{\altoptset^*,\altoptsettoo'}\subseteq\rejectset$.
In other words, we find that
\begin{equation*}
\cset{\alpha_{\opt,\altopt}\opt+\beta_{\opt,\altopt}\altopt}{\opt\in\altoptset^*,\altopt\in\altoptsettoo'}\in\rejectset
\text{ for all choices of $(\alpha_{\opt,\altopt},\beta_{\opt,\altopt})>0$}.
\end{equation*}
If we now choose $(\alpha_{\opt[o],\altopt},\beta_{\opt[o],\altopt})\coloneqq(\mu_{\altopt},\lambda_{\altopt})$ and $(\alpha_{\opt,\altopt},\beta_{\opt,\altopt})\coloneqq(1,0)$ for all $\opt\in\altoptset^*\setminus\set{\opt[o]}$, for all $\altopt\in\altoptsettoo'$, then we find in particular that
\begin{multline*}
\cset{b_{\altopt}}{\altopt\in\altoptsettoo[1]\cup\altoptsettoo[3]}
\cup\cset{0}{\altopt\in\altoptsettoo[2]}
\cup\cset{\opt}{\opt\in\altoptset^*\setminus\set{\opt[o]}}\\
=\altoptset[o]\cup\altoptset[3]\cup\cset{0}{\altopt\in\altoptsettoo[2]}\cup\altoptset^*\setminus\set{\opt[o]}
=\altoptset[3]\cup\altoptset^*\setminus\set{\opt[o]}\cup\cset{0}{\altopt\in\altoptsettoo[2]}
\end{multline*}
belongs to $\rejectset$.
If $\altoptsettoo[2]=\emptyset$, we find that $\altoptset[3]\cup\altoptset^*\setminus\set{\opt[o]}\in\rejectset$.
If $\altoptsettoo[2]\neq\emptyset$, we find that $\set{0}\cup\altoptset[3]\cup\altoptset^*\setminus\set{\opt[o]}\in\rejectset$, and then we can still derive from~\ref{ax:rejects:removezero} that $\altoptset[3]\cup\altoptset^*\setminus\set{\opt[o]}\in\rejectset$. 
Any $b\in\altoptset[3]$ that does not belong to $\altoptset^*\setminus\set{\opt[o]}$ is a positive linear combination of elements of $\altoptset[o]$, and therefore of $\altoptset^*\setminus\set{\opt[o]}$.
Proposition~\ref{prop:removal:of:posis:gewijzigde:axiomas} and~\ref{ax:rejects:removepositivecombinations} then tell us that we can remove such~$b$ and still be guaranteed that the result belongs to $\rejectset$.
This tells us that $\altoptset^*\setminus\set{\opt[o]}\in\rejectset$, a contradiction.
\end{proof}

\begin{proof}[Proof of Theorem~\ref{theo:mixingrepresentation:twosided}]
First assume that $\rejectset$ is mixing. 
It then follows from Theorem~\ref{theo:convex:rejectsets:representation} that $\convcohdesirsets\group{\rejectset}\coloneqq\cset{\desirset\in\convcohdesirsets}{\rejectset\subseteq\rejectset[\desirset]}\neq\emptyset$ and $\rejectset=\bigcap\cset{\rejectset[\desirset]}{\desirset\in\convcohdesirsets\group{\rejectset}}$. 
This clearly implies that there is at least one non-empty set $\setofdesirsets\subseteq\convcohdesirsets$ of mixing sets of desirable options such that $\rejectset=\bigcap\cset{\rejectset[\desirset]}{\desirset\in\setofdesirsets}$: the set $\convcohdesirsets\group{\rejectset}$. 
Furthermore, for any non-empty set $\setofdesirsets\subseteq\convcohdesirsets$ of mixing coherent sets of desirable options such that $\rejectset=\bigcap\cset{\rejectset[\desirset]}{\desirset\in\setofdesirsets}$, we clearly have that $\rejectset\subseteq\rejectset[\desirset]$ for all $\desirset\in\setofdesirsets$.
Since $\setofdesirsets\subseteq\convcohdesirsets$, this implies that $\setofdesirsets\subseteq\convcohdesirsets\group{\rejectset}$. 
So $\convcohdesirsets\group{\rejectset}$ is indeed the largest such set.

It remains to prove the `if' part. 
So consider any non-empty set $\setofdesirsets\subseteq\convcohdesirsets$ of mixing sets of desirable options such that $\rejectset=\bigcap\cset{\rejectset[\desirset]}{\desirset\in\setofdesirsets}$. 
For every $\desirset\in\setofdesirsets\subseteq\convcohdesirsets$, it then follows from Proposition~\ref{prop:mixingbinaryiffD} that $\rejectset[\desirset]$ is mixing. 
Because Axioms~\ref{ax:rejects:removezero}-\ref{ax:rejects:mono} and~\ref{ax:rejects:removepositivecombinations} are trivially preserved under taking arbitrary (non-empty) intersections, this implies that $\rejectset$ is, indeed, also mixing.
\end{proof}

\begin{proof}[Proof of Proposition~\ref{prop:mixingequalslexicographic}]
First assume that $\desirset$ is mixing. 
Since $\desirset^\mathrm{c}$ is trivially a subset of $\posi(\desirset^\mathrm{c})$, we only need to prove that $\posi(\desirset^\mathrm{c})\subseteq\desirset^\mathrm{c}$. 
So consider any $\opt\in\posi(\desirset^\mathrm{c})$ and assume \emph{ex absurdo} that $\opt\notin\desirset^\mathrm{c}$, so $\opt\in\desirset$. 
Since $\opt\in\posi(\desirset^\mathrm{c})$, we know that $\opt$ is positive linear combination of a finite number of elements of $\desirset^\mathrm{c}$. 
Hence, there is some (finite) option set $\optset\in\optsets$ such that $\optset\subseteq\desirset^\mathrm{c}$ and $\opt\in\posi(\optset)$. 
Since we have assumed that $\opt\in\desirset$, this implies that $\posi(\optset)\cap\desirset\neq\emptyset$, so $\optset\cap\desirset\neq\emptyset$ because $\desirset$ satisfies~\ref{ax:desirs:mixing} by assumption. 
Since $\optset\subseteq\desirset^\mathrm{c}$, we find that $\desirset^\mathrm{c}\cap\desirset\neq\emptyset$, a contradiction. 
So it must be that $\opt\in\desirset^\mathrm{c}$, and therefore, indeed, $\posi\group{\desirset^\mathrm{c}}\subseteq\desirset^\mathrm{c}$.

Conversely, assume that $\posi(\desirset^\mathrm{c})=\desirset^\mathrm{c}$, then we need to prove that $\desirset$ satisfies~\ref{ax:desirs:mixing}. 
So consider any $\optset\in\optsets$ such that $\posi(\optset)\cap\desirset\neq\emptyset$ and assume \emph{ex absurdo} that $\optset\cap\desirset=\emptyset$. 
Then $\optset\subseteq\desirset^\mathrm{c}$ and therefore also $\posi(\optset)\subseteq\posi(\desirset^\mathrm{c})=\desirset^\mathrm{c}$. 
So we find that $\posi(\optset)\cap\desirset=\emptyset$, a contradiction.
\end{proof}

\begin{proof}[Proof of Proposition~\ref{prop:mixingiftotal}]
Consider any total set of desirable options $\desirset$, and any $\optset\in\optsets$ such that $\optset\cap\desirset=\emptyset$. 
If we can show that then also $\posi(\optset)\cap\desirset=\emptyset$, if will clearly follow from Definition~\ref{def:mixingdesirs} that $\desirset$ is indeed mixing.

So consider any $\opt\in\posi(\optset)$, then we need to prove that $\opt\notin\desirset$. 
Since $\opt\in\posi(\optset)$, $\opt$ is a finite positive linear combination of elements of $\optset$, meaning that $\opt=\sum_{i=1}^n\lambda_i\opt[i]$, with $n\in\naturals$ and, for all $i\in\set{1,\dots,n}$, $\lambda_i>0$ and $\opt[i]\in\optset$. 
Let $I\coloneqq\cset{i\in\set{1,\dots,n}}{\opt[i]\neq0}$.
Then for all $i\in I$, since $\opt[i]\in\optset$, $\optset\cap\desirset=\emptyset$ and $\opt[i]\neq0$, it follows from the totality of $\desirset$ that $-\opt[i]\in\desirset$. 
We now consider two cases: $I=\emptyset$ and $I\neq\emptyset$. 
If $I=\emptyset$ then $\opt=0$, which, since $\desirset$ is in particular also coherent [use Axiom~\ref{ax:desirs:nozero}], implies that, indeed, $\opt\notin\desirset$. 
If $I\neq\emptyset$ then $-\opt=\sum_{i=1}^n\lambda_i(-\opt[i])=\sum_{i\in I}\lambda_i(-\opt[i])$ is a finite positive linear combination of elements of $\desirset$, and it therefore follows from the coherence of $\desirset$ [Axiom~\ref{ax:desirs:cone}] that $-\opt\in\desirset$. 
This implies that, indeed, $\opt\notin\desirset$, because otherwise, it would follow from Axiom~\ref{ax:desirs:cone} that $0=\opt-\opt\in\desirset$, contradicting Axiom~\ref{ax:desirs:nozero}.
\end{proof}

\subsection{Proofs and intermediate results for Section~\ref{sec:archimedeanity}}

Various proofs in this section make use of the $\arch(\cdot)$ operator, defined for any set of desirable options $\desirset\in\desirsets$ by
\begin{equation*}
\arch\group{\desirset}
\coloneqq\cset{\opt\in\desirset}{\group{\exists\epsilon\in\posreals}\opt-\epsilon\in \desirset},
\end{equation*}
where the identification of $\epsilon\in\posreals$ with an option in $\opts$ is justified by the additional assumptions on $\opts$ that we impose in Section~\ref{sec:archimedeanity}.

\begin{proof}[Proof of Proposition~\ref{prop:isomorphismsmarch}]
First consider any Archimedean---and hence also coherent---set of desirable options~$\desirset$. 
We will prove that $\lowprev[\desirset]$ is a coherent lower prevision on $\gbls(\states)$, that $\desirset[{\lowprev[\desirset]}]=\desirset$, and that if $\desirset$ is moreover mixing, then $\lowprev[\desirset]$ is a linear prevision.

We begin by showing that $\lowprev[\desirset]$ is a coherent lower prevision on $\gbls(\states)$. 
Consider any $\opt\in\gbls(\states)$ and any $\mu\in\reals$ such that $\mu<\inf\opt$. 
Since $\desirset$ is coherent [use~\ref{ax:desirs:pos}] and $\posopts=\strictposgbls(\states)\coloneqq\cset{\opt\in\gbls(\states)}{\inf\opt>0}$, we find that $\opt-\mu\in\desirset$, because $\inf(\opt-\mu)=\inf\opt-\mu>0$. 
Since this is true for any $\mu\in\reals$ such that $\mu<\inf\opt$, we infer from Equation~\eqref{eq:lpfromdesir} that $\lowprev[\desirset](\opt)\geq\inf\opt$. 
So $\lowprev[\desirset]$ satisfies~\ref{ax:lowprev:inf}. 
Next, we prove that $\lowprev[\desirset]$ is a real-valued map. 
To this end, consider again any $\opt\in\gbls(\states)$. 
Since $\opt$ is bounded, $\inf\opt$ and $\sup\opt$ are real. 
Consider any $\mu\in\reals$ such that $\mu>\sup\opt$. 
Then $\inf(\mu-\opt)=\mu-\sup\opt>0$, so~\ref{ax:desirs:pos} implies that $\mu-\opt\in\desirset$. 
Hence, $\opt-\mu\notin\desirset$, because otherwise it would follow from~\ref{ax:desirs:cone} that $0=(\opt-\mu)+(\mu-\opt)\in\desirset$, contradicting~\ref{ax:desirs:nozero}. 
Since this is true for any $\mu\in\reals$ such that $\mu>\sup\opt$, Equation~\eqref{eq:lpfromdesir} implies that $\lowprev[\desirset](\opt)\leq\sup\opt$. 
On the other hand, we know from~\ref{ax:lowprev:inf} that $\lowprev[\desirset](\opt)\geq\inf\opt$. 
Since $\inf\opt$ and $\sup\opt$ are real, we find that $\lowprev[\desirset](\opt)$ is real, so $\lowprev[\desirset]$ is indeed a real-valued map. 
It remains to show that $\lowprev[\desirset]$ satisfies~\ref{ax:lowprev:lambda} and~\ref{ax:lowprev:super}. 
We start with~\ref{ax:lowprev:lambda}. 
Consider any $\opt\in\gbls(\states)$ and $\lambda\in\posreals$. 
For all $\mu\in\reals$, it then follows from the coherence of $\desirset$ [use~\ref{ax:desirs:cone}] that $\opt-\mu\in\desirset$ if and only if $\lambda\opt-\lambda\mu\in\desirset$. 
Equation~\eqref{eq:lpfromdesir} therefore implies that $\lowprev[\desirset](\lambda\opt)=\lambda\lowprev[\desirset](\opt)$. 
So $\lowprev[\desirset]$ satisfies~\ref{ax:lowprev:lambda}. 
Next, for~\ref{ax:lowprev:super}, consider any $\opt,\altopt\in\gbls(\states)$. 
Fix any $\epsilon\in\posreals$. 
It then follows from Equation~\eqref{eq:lpfromdesir} that there is some $\mu_{\opt}\in\reals$ such that $\mu_{\opt}>\lowprev[\desirset](\opt)-\epsilon$ and $\opt-\mu_{\opt}\in\desirset$, and similarly, that there is some $\mu_{\altopt}\in\reals$ such that $\mu_{\altopt}>\lowprev[\desirset](\altopt)-\epsilon$ and $\opt-\mu_{\altopt}\in\desirset$. 
The coherence of $\desirset$ [use~\ref{ax:desirs:cone}] now implies that $\opt+\altopt-\mu_{\opt}-\mu_{\altopt}\in\desirset$, so Equation~\eqref{eq:lpfromdesir} implies that $\lowprev[\desirset](\opt+\altopt)\geq\mu_{\opt}+\mu_{\altopt}>\lowprev[\desirset](\opt)+\lowprev[\desirset](\altopt)-2\epsilon$. 
Since this inequality holds for every $\epsilon\in\posreals$, we infer that $\lowprev[\desirset](\opt+\altopt)\geq\lowprev[\desirset](\opt)+\lowprev[\desirset](\altopt)$. 
Hence, $\lowprev[\desirset]$ satisfies~\ref{ax:lowprev:super}.

Next, we prove that $\desirset[{\lowprev[\desirset]}]=\desirset$. 
First consider any $\opt\in\desirset$. 
Since $\desirset$ is Archimedean, we know that there is some $\epsilon\in\posreals$ such that $\opt-\epsilon\in\desirset$. 
Combined with Equation~\eqref{eq:lpfromdesir}, this implies that $\lowprev[\desirset](\opt)\geq\epsilon>0$. 
So we infer from Equation~\eqref{eq:desirfromlp} that $\opt\in\desirset[{\lowprev[\desirset]}]$. 
Conversely, consider any $\opt\in\desirset[{\lowprev[\desirset]}]$. 
Equation~\eqref{eq:desirfromlp} then implies that $\lowprev[\desirset](\opt)>0$, so we know from Equation~\eqref{eq:lpfromdesir} that there is some $\mu\in\posreals$ such that $\opt-\mu\in\desirset$. 
Therefore, and because $\mu\in\strictposgbls(\states)$, it follows from the coherence of $\desirset$ [use~\ref{ax:desirs:pos} and~\ref{ax:desirs:cone}] that $\opt=(\opt-\mu)+\mu\in\desirset$. 
We conclude that, indeed, $\desirset[{\lowprev[\desirset]}]=\desirset$.

Suppose now that $\desirset$ is also mixing. 
We prove that $\lowprev[\desirset]$ is a linear prevision. 
Since we already know that $\lowprev[\desirset]$ is a coherent lower prevision, it suffices to prove that $\lowprev[\desirset]$ satisfies~\ref{ax:linprev:additive}. 
So consider any $\opt,\altopt\in\gbls(\states)$. 
We only need to prove that $\lowprev[\desirset](\opt+\altopt)\leq\lowprev[\desirset](\opt)+\lowprev[\desirset](\altopt)$ because the converse inequality follows from~\ref{ax:lowprev:super}. 
Fix any $\mu\in\reals$ such that $\mu>\lowprev[\desirset](\opt)+\lowprev[\desirset](\altopt)$ and let $\epsilon\coloneqq\nicefrac{1}{2}\big(\mu-\lowprev[\desirset](\opt)-\lowprev[\desirset](\altopt)\big)>0$.
It then follows from Equation~\eqref{eq:lpfromdesir} that $\opt[\epsilon]\coloneqq\opt-(\lowprev[\desirset](\opt)+\epsilon)\notin\desirset$ and $\altopt[\epsilon]\coloneqq\altopt-(\lowprev[\desirset](\altopt)+\epsilon)\notin\desirset$. 
So if we let $\optset\coloneqq\{\opt[\epsilon],\altopt[\epsilon]\}$, then $\optset\cap\desirset=\emptyset$. 
Therefore, and since $\optset\in\optsets$ and $\{\opt[\epsilon],\altopt[\epsilon],\opt[\epsilon]+\altopt[\epsilon]\}\subseteq\posi(\optset)$, it follows from the mixingness of $\desirset$ that $\{\opt[\epsilon],\altopt[\epsilon],\opt[\epsilon]+\altopt[\epsilon]\}\cap\desirset=\emptyset$. 
Hence, we find that
\begin{equation*}
\opt+\altopt-\mu
=
\opt+\altopt-\lowprev[\desirset](\opt)-\lowprev[\desirset](\altopt)-2\epsilon
=
\opt[\epsilon]+\altopt[\epsilon]\notin\desirset.
\end{equation*}
Since this holds for all $\mu\in\reals$ such that $\mu>\lowprev[\desirset](\opt)+\lowprev[\desirset](\altopt)$, it follows from Equation~\eqref{eq:lpfromdesir} that $\lowprev[\desirset](\opt+\altopt)\leq\lowprev[\desirset](\opt)+\lowprev[\desirset](\altopt)$, as desired.

For the second part of this proposition, we consider any coherent lower prevision $\lowprev$ on $\gbls(\states)$. 
We will prove that $\desirset[\lowprev]$ is an Archimedean set of desirable options, that $\lowprev[{\desirset[\lowprev]}]=\lowprev$, and that if $\lowprev$ furthermore is a linear prevision, then $\desirset[\lowprev]$ is mixing.

We start by proving that $\desirset[\lowprev]$ is Archimedean, meaning that it satisfies~\ref{ax:desirs:nozero}--\ref{ax:desirs:cone} and~\ref{ax:desirs:gamblearch}. 
Since~\ref{ax:lowprev:super} implies that $\lowprev(0)+\lowprev(0)\leq\lowprev(0)$, and since $\lowprev$ is real-valued, we know that $\lowprev(0)\leq0$, so Equation~\eqref{eq:desirfromlp} implies that $0\notin\desirset[\lowprev]$. 
Hence, $\desirset[\lowprev]$ satisfies~\ref{ax:desirs:nozero}. 
For~\ref{ax:desirs:pos}, it suffices to observe that for any $\opt\in\posopts=\strictposgbls(\states)$, $\lowprev(\opt)\geq\inf\opt>0$ because of~\ref{ax:lowprev:inf}, so $\opt\in\desirset[\lowprev]$ because of Equation~\eqref{eq:desirfromlp}. 
Let us now prove~\ref{ax:desirs:cone}. 
Consider any $\opt,\altopt\in\desirset[\lowprev]$ and any $(\lambda,\mu)>0$. 
It then follows from Equation~\eqref{eq:desirfromlp} that $\lowprev(\opt)>0$ and $\lowprev(\altopt)>0$, and from~\ref{ax:lowprev:lambda} and~\ref{ax:lowprev:super} that $\lowprev(\lambda\opt+\mu\altopt)\geq\lambda\lowprev(\opt)+\mu\lowprev(\altopt)$. 
Since $(\lambda,\mu)>0$, this implies that $\lowprev(\lambda\opt+\mu\altopt)>0$, which in turn implies that $\lambda\opt+\mu\altopt\in\desirset[\lowprev]$ because of Equation~\eqref{eq:desirfromlp}. 
So $\desirset[\lowprev]$ satisfies~\ref{ax:desirs:cone}. 
Let us now prove~\ref{ax:desirs:gamblearch}. 
Consider any $\opt\in\desirset[\lowprev]$, so $\lowprev(\opt)>0$ because of Equation~\eqref{eq:desirfromlp}. 
Let $\epsilon\coloneqq\frac{1}{2}\lowprev(\opt)>0$. 
Then since $\lowprev$ is a coherent lower prevision,
\begin{equation*}
\lowprev(\opt-\epsilon)
\geq\lowprev(\opt)+\lowprev(-\epsilon)
\geq\lowprev(\opt)+\inf(-\epsilon)
=\lowprev(\opt)-\epsilon
=2\epsilon-\epsilon
=\epsilon
>0,
\end{equation*}
so $\opt-\epsilon\in\desirset[\lowprev]$ because of Equation~\eqref{eq:desirfromlp}. 
Hence indeed, $\desirset[\lowprev]$ also satisfies~\ref{ax:desirs:gamblearch}.

Next, we prove that $\lowprev[{\desirset[\lowprev]}]=\lowprev$. 
Consider any $\opt\in\gbls(\states)$ and any $\mu\in\reals$. 
Then on the one hand, if $\mu<\lowprev(\opt)$, it follows from~\ref{ax:lowprev:super} and~\ref{ax:lowprev:inf} that
\begin{equation*}
\lowprev(\opt-\mu)\geq\lowprev(\opt)+\lowprev(-\mu)
\geq\lowprev(\opt)+\inf(-\mu)
=\lowprev(\opt)-\mu>0,
\end{equation*}
so Equation~\eqref{eq:desirfromlp} implies that $\opt-\mu\in\desirset[\lowprev]$. 
On the other hand, if $\mu>\lowprev(\opt)$, it follows from~\ref{ax:lowprev:inf} and~\ref{ax:lowprev:super} that
\begin{equation*}
\lowprev(\opt-\mu)
=\lowprev(\opt-\mu)+\mu-\mu
\leq\lowprev(\opt-\mu)+\lowprev(\mu)-\mu
\leq\lowprev(\opt)-\mu<0
\end{equation*}
so Equation~\eqref{eq:desirfromlp} implies that $\opt-\mu\notin\desirset[\lowprev]$. 
We conclude from Equation~\eqref{eq:lpfromdesir} that $\lowprev[{\desirset[\lowprev]}]=\lowprev$.

Finally, suppose that $\lowprev$ is furthermore a linear prevision. 
We prove that $\desirset[\lowprev]$ is mixing. 
Since we already know that $\desirset[\lowprev]$ is coherent, it suffices to prove that $\desirset[\lowprev]$ satisfies~\ref{ax:desirs:mixing}. 
So consider any $\optset\in\optsets$ such that $\posi(\optset)\cap\desirset[\lowprev]\neq\emptyset$, so there is some $\opt\in\posi(\optset)$ such that also $\opt\in\desirset[\lowprev]$. 
Since $\opt\in\posi(\optset)$, it follows from Equation~\eqref{eq:posioperator} that $\opt=\sum_{k=1}^n\lambda_k\opt[k]$, with $n\in\naturals$ and, for all $k\in\{1,\dots,n\}$, $\lambda_k\in\posreals$ and $\opt[k]\in\optset$. 
It therefore follows from~\ref{ax:lowprev:super} and~\ref{ax:lowprev:lambda} that
\begin{equation*}
\lowprev(\opt)
=\lowprev\group[\bigg]{\sum_{k=1}^n\lambda_k\opt[k]}
\geq\sum_{k=1}^n\lambda_k\lowprev(\opt[k]).
\end{equation*}
Furthermore, since $\opt\in\desirset[\lowprev]$, we know from Equation~\eqref{eq:desirfromlp} that $\lowprev(\opt)>0$. 
Hence, since $\lambda_k>0$ for all $k\in\{1,\dots,n\}$, there must be some $k^*\in\{1,\dots,n\}$ such that $\lowprev(\opt[k^*])>0$. 
Equation~\eqref{eq:desirfromlp} then implies that $\opt[k^*]\in\desirset[\lowprev]$. 
Since we also know that $\opt[k^*]\in\optset$, this in turn implies that $\optset\cap\desirset[\lowprev]\neq\emptyset$. 
We conclude that $\desirset[\lowprev]$ satisfies~\ref{ax:desirs:mixing}, as desired.
\end{proof}

\begin{proof}[Proof of Proposition~\ref{prop:gamblearchbinaryiffD}]
We know from Proposition~\ref{prop:coherence:for:binary} that $\desirset$ is coherent if and only if $\rejectset[\desirset]$ is, and from Proposition~\ref{prop:mixingbinaryiffD} tells us that $\desirset$ is mixing if and only if $\rejectset[\desirset]$ is. 
So the only thing left to prove is that $\desirset$ satisfies~\ref{ax:desirs:gamblearch} if and only if $\rejectset[\desirset]$ satisfies~\ref{ax:rejects:gamblearch}.

First assume that $\desirset$ satisfies~\ref{ax:desirs:gamblearch}.
Consider any $\optset\in\rejectset[\desirset]$, meaning that $\optset\cap\desirset\neq\emptyset$. 
Consider any $\opt\in\optset\cap\desirset$. 
Since $\opt\in\desirset$ and $\desirset$ is Archimedean, we know that there is some $\epsilon\in\posreals$ such that $\opt-\epsilon\in\desirset$. 
Since $\opt\in\optset$, we also know that $\opt-\epsilon\in\optset-\epsilon$. 
We therefore find that $\opt\in(\optset-\epsilon)\cap\desirset$, which implies that $(\optset-\epsilon)\cap\desirset\neq\emptyset$, or equivalently, that, indeed $\optset-\epsilon\in\rejectset[\desirset]$.

Conversely, assume that $\rejectset[\desirset]$ satisfies~\ref{ax:rejects:gamblearch}. 
Consider any $\opt\in\desirset$. 
Then $\set{\opt}\in\rejectset[\desirset]$ and therefore, since $\rejectset[\desirset]$ is Archimedean, there is some $\epsilon\in\posreals$ such that $\set{\opt}-\epsilon\in\rejectset[\desirset]$. 
Since $\set{\opt}-\epsilon=\set{\opt-\epsilon}$, this implies that, indeed, $\opt-\epsilon\in\desirset$.
\end{proof}

\begin{lemma}\label{lem:effectofgamblearch:coh}
For any coherent set of desirable options $\desirset\in\cohdesirsets$, $\arch\group{\desirset}$ is coherent.
\end{lemma}

\begin{proof}
Consider any coherent set of desirable options $\desirset\in\cohdesirsets$, $\arch\group{\desirset}$, then we have to prove that $\arch\group{\desirset}$ satisfies Axioms~\ref{ax:desirs:nozero}--\ref{ax:desirs:cone}. 

That $\arch\group{\desirset}$ satisfies \ref{ax:desirs:nozero} follows directly from the fact that $\desirset$ does, because $\arch\group{\desirset}\subseteq\desirset$.
To prove that $\arch\group{\desirset}$ satisfies~\ref{ax:desirs:pos}, we consider any $\opt\in\posopts$. 
Then $\inf\opt>0$, so there is some $\epsilon\in\posreals$ such that $\inf(\opt-\epsilon)>0$. 
Since $\desirset$ is coherent, this implies that $\opt-\epsilon\in\posopts\subseteq\desirset$, and therefore that $\opt\in\arch\group{\desirset}$.
To prove that $\arch\group{\desirset}$ satisfies \ref{ax:desirs:cone}, we consider any $\opt,\altopt\in\arch\group{\desirset}$ and $\lambda,\mu\geq0$ such that $\lambda+\mu>0$, and prove that $\lambda\opt+\mu\altopt\in\arch\group{\desirset}$.
Since $\opt,\altopt\in\arch\group{\desirset}$, there are $\epsilon_{\opt},\epsilon_{\altopt}\in\posreals$ such that $\opt-\epsilon_{\opt}\in\desirset$ and $\altopt-\epsilon_{\altopt}\in\desirset$, which, since $\desirset$ satisfies \ref{ax:desirs:cone}, implies that $\lambda\opt+\mu\altopt-(\lambda\epsilon_{\opt}+\mu\epsilon_{\altopt})=\lambda(\opt-\epsilon_{\opt})+\mu(\altopt-\epsilon_{\altopt})\in\desirset$. 
Since $\lambda\epsilon_{\opt}+\mu\epsilon_{\altopt}\in\posreals$, this implies that $\lambda\opt+\mu\altopt\in\arch\group{\desirset}$.
\end{proof}

\begin{lemma}\label{lem:effectofgamblearch:gamblearch}
For any coherent set of desirable options $\desirset\in\cohdesirsets$, $\arch\group{\desirset}$ is Archimedean.
\end{lemma}

\begin{proof}
Due to Lemma~\ref{lem:effectofgamblearch:coh}, we already know that $\arch\group{\desirset}$ is coherent, so it only remains to prove that $\arch\group{\desirset}$ satisfies~\ref{ax:desirs:gamblearch}.
Consider any $\opt\in\arch\group{\desirset}$. 
This implies that $\opt\in\desirset$ and that there is some $\epsilon\in\posreals$ such that $\opt-\epsilon\in\desirset$. 
Since $\frac{1}{2}\epsilon\in\posopts$ and $\opt-\epsilon\in\desirset$, the coherence of $\desirset$ implies that $\opt-\frac{1}{2}\epsilon=(\opt-\epsilon)+\frac{1}{2}\epsilon\in\desirset$. 
Since $(\opt-\frac{1}{2}\epsilon)-\frac{1}{2}\epsilon=\opt-\epsilon\in\desirset$ and $\frac{1}{2}\epsilon\in\posreals$, this implies that $\opt-\frac{1}{2}\epsilon\in\arch\group{\desirset}$. 
Since this is true for every $\opt\in\arch\group{\desirset}$, we conclude that $\arch\group{\desirset}$ is indeed Archimedean.
\end{proof}

\begin{lemma}\label{lem:effectofgamblearch:lex}
For any mixing set of desirable options $\desirset\in\convcohdesirsets$, $\arch\group{\desirset}$ is mixing.
\end{lemma}

\begin{proof}
Consider any mixing set of desirable options $\desirset\in\convcohdesirsets$. 
Since $\desirset$ is in particular coherent, we infer from Lemma~\ref{lem:effectofgamblearch:coh} that $\arch\group{\desirset}$ is coherent too, so we only need to prove that $\arch\group{\desirset}$ satisfies~\ref{ax:desirs:mixing}. 

So consider any $\optset\in\optsets$ such that $\posi(\optset)\cap\arch\group{\desirset}\neq\emptyset$. 
Then there is at least one $\opt\in\posi(\optset)$ such that $\opt\in\arch\group{\desirset}$. 
Fix any such $\opt$. 
Since $\opt\in\arch\group{\desirset}$, there is some $\epsilon\in\posreals$ such that $\opt-\epsilon\in\desirset$. 
Fix any such $\epsilon$. 
Since $\opt\in\posi(\optset)$, we know that $\opt=\sum_{i=1}^n\lambda_i\opt[i]$, with $n\in\naturals$ and, for all $i\in\set{1,\dots,n}$, $\lambda_i>0$ and $\opt[i]\in\optset$. 
Let $\lambda\coloneqq\sum_{i=1}^n\lambda_i>0$ and $\alpha\coloneqq\nicefrac{1}{\lambda}>0$. 
Then
\begin{equation*}
\opt-\epsilon
=\sum_{i=1}^n\lambda_i\opt[i]-\alpha\group[\bigg]{\sum_{i=1}^n\lambda_i}\epsilon
=\sum_{i=1}^n\lambda_i\group{\opt[i]-\alpha\epsilon},
\end{equation*}
so $\opt-\epsilon\in\posi(\optset-\alpha\epsilon)$. 
Since also $\opt-\epsilon\in\desirset$, it follows that $\posi(\optset-\alpha\epsilon)\cap\desirset\neq\emptyset$. 
Because $\desirset$ is by assumption mixing, we find that $(\optset-\alpha\epsilon)\cap\desirset\neq\emptyset$, meaning that there is some $\tilde\opt\in\optset$ such that $\tilde\opt-\alpha\epsilon\in\desirset$. Since $\alpha\epsilon\in\posreals$, we conclude that $\tilde\opt\in\arch\group{\desirset}$, so, indeed, $\optset\cap\arch\group{\desirset}\neq\emptyset$.
\end{proof}

\begin{lemma}\label{lem:effectofgamblearch:preservegamblearch}
For any Archimedean set of desirable options $\desirset\in\archcohdesirsets$, $\arch\group{\desirset}=\desirset$.
\end{lemma}

\begin{proof}
Since $\arch\group{\desirset}$ is trivially a subset of $\desirset$, we only need to prove that $\desirset\subseteq\arch\group{\desirset}$. 
So consider any $\opt\in\desirset$. 
Then since $\desirset$ is Archimedean, there is some $\epsilon\in\posreals$ such that $\opt-\epsilon\in\desirset$, which implies that $\opt\in\arch\group{\desirset}$. 
Since this is true for every $\opt\in\desirset$, we conclude that indeed, $\desirset\subseteq\arch\group{\desirset}$.
\end{proof}

\begin{lemma}\label{lem:gamblearch:transformrepresentation}
Consider any non-empty set $\setofdesirsets\subseteq\cohdesirsets$ of coherent sets of desirable options such that $\rejectset\coloneqq\bigcap\cset{\rejectset[\desirset]}{\desirset\in\setofdesirsets}$ is Archimedean. 
Then also $\rejectset=\bigcap\cset{\rejectset[\arch\group{\desirset}]}{\desirset\in\setofdesirsets}$.
\end{lemma}

\begin{proof}
For any $\desirset\in\setofdesirsets$, since $\arch\group{\desirset}\subseteq\desirset$, it follows that $\rejectset[\arch\group{\desirset}]\subseteq\rejectset[\desirset]$. 
Hence,
\begin{equation*}
\bigcap\cset{\rejectset[\arch\group{\desirset}]}{\desirset\in\setofdesirsets}
\subseteq\bigcap\cset{\rejectset[\desirset]}{\desirset\in\setofdesirsets}
=\rejectset,
\end{equation*}
so it remains to prove the converse set inclusion. 
To this end, consider any $\desirset\in\setofdesirsets$ and any $\optset\in\rejectset$. 
Then since $\rejectset$ is Archimedean, we know that there is some $\epsilon\in\posreals$ such that $\optset-\epsilon\in\rejectset\subseteq\rejectset[\desirset]$. 
Fix any such $\epsilon$.
$\optset-\epsilon\in\rejectset[\desirset]$ implies that there is some $\opt\in\optset$ such that $\opt-\epsilon\in\desirset$, by Equation~\eqref{eq:desirset:to:rejectset}. 
Since $\desirset$ is coherent [use Axiom~\ref{ax:desirs:cone}] and also $\epsilon\in\desirset$ [use $\epsilon\in\posopts$ and Axiom~\ref{ax:desirs:pos}], this in turn implies that $\opt\in\desirset$. 
Since also $\opt-\epsilon\in\desirset$, it follows that $\opt\in\arch\group{\desirset}$ and therefore, that $\optset\cap\arch\group{\desirset}\neq\emptyset$. 
Hence, $\optset\in\rejectset[\arch\group{\desirset}]$, again by Equation~\eqref{eq:desirset:to:rejectset}. 
Since this is true for every $\desirset\in\setofdesirsets$ and any $\optset\in\rejectset$, we find that, as desired, $\rejectset\subseteq\bigcap\cset{\rejectset[\arch\group{\desirset}]}{\desirset\in\setofdesirsets}$.
\end{proof}

\begin{proof}[Proof of Proposition~\ref{prop:gamblearchimedeanityforrejectionfunctions}]
First assume that $\rejectset$ is Archimedean. 
Then $\rejectset$ is in particular coherent, and therefore so is $\rejectfun$, by Proposition~\ref{prop:axioms:rejection:sets:and:functions:coherence}.
It therefore remains to prove that $\rejectfun$ satisfies~\ref{ax:rejectfun:gamblearch}.
Consider, to this end, any $\optset\in\optsets$ and $\opt\in\opts$ such that $\opt\in\rejectfun(\optset\cup\set{\opt})$. 
Equation~\eqref{eq:interpretation:rejectfuns:after:irreflexivity:and:additivity:intermsof:K:Gert} then tells us that $\optset-\opt\in\rejectset$. 
Since $\rejectset$ is Archimedean, this implies that there is some $\epsilon\in\posreals$ such that $\optset-\opt-\epsilon\in\rejectset$, so it follows from Equation~\eqref{eq:interpretation:rejectfuns:after:irreflexivity:and:additivity:intermsof:K:Gert} that, indeed, $\opt\in\rejectfun((\optset-\epsilon)\cup\set{\opt})$.

Next, assume that $\rejectfun$ is coherent and satisfies~\ref{ax:rejectfun:gamblearch}.
Then it already follows from Proposition~\ref{prop:axioms:rejection:sets:and:functions:coherence} that $\rejectset$ is coherent, so we only need to prove that $\rejectset$ satisfies~\ref{ax:rejects:gamblearch}.
Consider, to this end, any $\optset\in\rejectset$. 
Equation~\eqref{eq:interpretation:rejectfuns:after:irreflexivity:and:additivity:intermsof:K:Gert} then tells us that $0\in\rejectfun(\optset\cup\set{0})$. 
Since $\rejectfun$ satisfies~\ref{ax:rejectfun:gamblearch}, this implies that there is some $\epsilon\in\posreals$ such that $0\in\rejectfun((\optset-\epsilon)\cup\set{0})$. 
Hence, it follows from Equation~\eqref{eq:interpretation:rejectfuns:after:irreflexivity:and:additivity:intermsof:K:Gert} that $\optset-\epsilon\in\rejectset$, as desired.
\end{proof}

\begin{proof}[Proof of Theorem~\ref{theo:rejectsets:representation:gamblearch:twosided}]
First assume that $\rejectset$ is Archimedean. 
It then follows from Theorem~\ref{theo:coherentrepresentation:twosided} that $\cohdesirsets\group{\rejectset}\coloneqq\cset{\desirset\in\cohdesirsets}{\rejectset\subseteq\rejectset[\desirset]}\neq\emptyset$ and $\rejectset=\bigcap\cset{\rejectset[\desirset]}{\desirset\in\cohdesirsets\group{\rejectset}}$.
Let $\archcohdesirsets^*(\rejectset)\coloneqq\cset{\arch\group{\desirset}}{\desirset\in\cohdesirsets\group{\rejectset}}\neq\emptyset$, then Lemma~\ref{lem:gamblearch:transformrepresentation} already guarantees that $\rejectset=\cap\cset{\rejectset[\desirset]}{\desirset\in\archcohdesirsets^*(\rejectset)}$. 

We first prove that $\archcohdesirsets(\rejectset)=\archcohdesirsets^*(\rejectset)$. 
To this end, first consider any $\desirset\in\archcohdesirsets^*(\rejectset)$. 
Then $\rejectset\subseteq\rejectset[\desirset]$ because $\rejectset=\cap\cset{\rejectset[\desirset]}{\desirset\in\archcohdesirsets^*(\rejectset)}$. 
Furthermore, it follows from Lemmas~\ref{lem:effectofgamblearch:coh} and~\ref{lem:effectofgamblearch:gamblearch} that $\desirset\in\archcohdesirsets$. 
Hence $\desirset\in\archcohdesirsets(\rejectset)$, and therefore $\archcohdesirsets^*(\rejectset)\subseteq\archcohdesirsets(\rejectset)$. 
Conversely, consider any $\desirset\in\archcohdesirsets(\rejectset)$. 
Since $\desirset$ is Archimedean, it follows from Lemma~\ref{lem:effectofgamblearch:preservegamblearch} that $\arch\group{\desirset}=\desirset$, and because also $\desirset\in\archcohdesirsets(\rejectset)\subseteq\cohdesirsets(\rejectset)$, we find that $\desirset=\arch\group{\desirset}\in\archcohdesirsets^*(\rejectset)$. 
So $\archcohdesirsets(\rejectset)\subseteq\archcohdesirsets^*(\rejectset)$, and therefore, indeed, $\archcohdesirsets(\rejectset)=\archcohdesirsets^*(\rejectset)$.

Next, we prove that $\archcohdesirsets(\rejectset)$ is closed. 
So consider any convergent sequence $\set{\desirset[n]}_{n\in\naturals}$ in $\archcohdesirsets(\rejectset)$ and let $\desirset[\infty]$ be its limit. 
Then we need to prove that $\desirset[\infty]\in\archcohdesirsets(\rejectset)$.
Equation~\eqref{eq:archimedeanlimit} and Proposition~\ref{prop:isomorphismsmarch} tell us that $\desirset[\infty]=\desirset[\lowprev]$, where $\lowprev$ is the point-wise limit of the corresponding convergent sequence of coherent lower previsions $\set{\lowprev[{\desirset[n]}]}_{n\in\naturals}$. 
Indeed, since for any $n\in\naturals$, $\desirset[n]$ is Archimedean, we know from Proposition~\ref{prop:isomorphismsmarch} that $\lowprev[{\desirset[n]}]$ is a coherent lower prevision. 
Therefore, and because $\cohlowprevs$ is closed and compact in the (weak*) topology induced by point-wise convergence \cite{maass2003:phd,troffaes2013:lp}, the limit $\lowprev$ is a coherent lower prevision as well, so we infer from Proposition~\ref{prop:isomorphismsmarch} that $\desirset[\infty]=\desirset[\lowprev]$ is Archimedean, so $\desirset[\infty]\in\archcohdesirsets$.
It therefore remains to prove that $\rejectset\subseteq\rejectset[{\desirset[\infty]}]$. 
Consider, to this end, any $\optset\in\rejectset$. 
Since $\rejectset$ is Archimedean, there is some $\epsilon\in\posreals$ such that $\optset-\epsilon\in\rejectset$. 
Fix any $n\in\naturals$. 
Then since $\desirset[n]\in\archcohdesirsets(\rejectset)$, we have that $\rejectset\subseteq\rejectset[{\desirset[n]}]$. 
Therefore, and because $\optset-\epsilon\in\rejectset$, we infer from Equation~\eqref{eq:desirset:to:rejectset} that there is some $\opt[n]\in\optset$ such that $\opt[n]-\epsilon\in\desirset[n]$ and therefore also $\lowprev[{\desirset[n]}](\opt[n])\geq\epsilon$. 
So for all $n\in\naturals$, we have some $\opt[n]\in\optset$ such that $\lowprev[{\desirset[n]}](\opt[n])\geq\epsilon$. 
Since $\optset$ is finite, this implies that there is some $\opt^*\in\optset$ and some subsequence $\set{\desirset[n_i]}_{i\in\naturals}$ such that $\lowprev[{\desirset[{n_i}]}](\opt^*)\geq\epsilon$ for all $i\in\naturals$. 
Since $\lowprev$ is the point-wise limit of $\set{\lowprev[{\desirset[n]}]}_{n\in\naturals}$ and therefore also the point-wise limit of the subsequence $\set{\lowprev[{\desirset[n_i]}]}_{i\in\naturals}$, this implies that $\lowprev(\opt^*)\geq\epsilon>0$, so $\opt^*\in\desirset[\lowprev]=\desirset[\infty]$. 
Since $\opt^*\in\optset$, this implies that $\optset\in\rejectset[{\desirset[\infty]}]$, again by Equation~\eqref{eq:desirset:to:rejectset}. 
Since this is true for every $\optset\in\rejectset$, we find that, indeed, $\rejectset\subseteq\rejectset[{\desirset[\infty]}]$, so $\desirset[\infty]\in\archcohdesirsets(\rejectset)$ and $\archcohdesirsets(\rejectset)$ is therefore indeed closed.

What we have found so far is that that there is a least one non-empty closed set $\setofdesirsets\subseteq\archcohdesirsets$ such that $\rejectset=\bigcap\cset{\rejectset[\desirset]}{\desirset\in\setofdesirsets}$, because $\archcohdesirsets(\rejectset)=\archcohdesirsets^*(\rejectset)\neq\emptyset$ satisfies all these properties. 
Furthermore, for any other non-empty closed set $\setofdesirsets\subseteq\archcohdesirsets$ such that $\rejectset=\bigcap\cset{\rejectset[\desirset]}{\desirset\in\setofdesirsets}$, we clearly have that $\rejectset\subseteq\rejectset[\desirset]$ for all $\desirset\in\setofdesirsets$. 
Since $\setofdesirsets\subseteq\archcohdesirsets$, this implies that $\setofdesirsets\subseteq\archcohdesirsets(\rejectset)$, so $\archcohdesirsets(\rejectset)$ is indeed the largest such set.

On to the `if' part. 
Consider any non-empty closed set $\setofdesirsets\subseteq\archcohdesirsets$ of Archimedean sets of desirable options such that $\rejectset=\bigcap\cset{\rejectset[\desirset]}{\desirset\in\setofdesirsets}$. 
For every $\desirset\in\setofdesirsets\subseteq\archcohdesirsets$, it then follows from Proposition~\ref{prop:gamblearchbinaryiffD} that $\rejectset[\desirset]$ is Archimedean. 
Because Axioms~\ref{ax:rejects:removezero}-\ref{ax:rejects:mono} are trivially preserved under taking arbitrary (non-empty) intersections, it already follows that $\rejectset$ is coherent. 
It therefore only remains to prove that $\rejectset$ is Archimedean as well, or in other words that $\rejectset$ satisfies~\ref{ax:rejects:gamblearch}.

To this end, consider any $\optset\in\rejectset$ and assume \emph{ex absurdo} that $\optset-\epsilon\notin\rejectset$ for all $\epsilon\in\posreals$. 
Fix any $n\in\naturals$ and let $\epsilon_n\coloneqq\nicefrac{1}{n}>0$, so by assumption $\optset-\epsilon_n\notin\rejectset$. 
Hence, since $\rejectset=\bigcap\cset{\rejectset[\desirset]}{\desirset\in\setofdesirsets}$, there is some $\desirset[n]\in\setofdesirsets$ such that $\optset-\epsilon_n\notin\rejectset[{\desirset[n]}]$. 
This allows us to consider the sequence $\set{\lowprev[{\desirset[n]}]}_{n\in\naturals}$, which, because of Proposition~\ref{prop:isomorphismsmarch}, belongs to $\cohlowprevs$. 
Since we already know that $\cohlowprevs$ is compact under the (weak*) topology induced by point-wise convergence, this sequence has some  subsequence $\set{\lowprev[{\desirset[n_i]}]}_{i\in\naturals}$ that converges point-wise to a limit $\lowprev\in\cohlowprevs$. 
Since $\lowprev\in\cohlowprevs$, we know from Proposition~\ref{prop:isomorphismsmarch} that $\lowprev=\lowprev[{\desirset[\infty]}]$, with $\desirset[\infty]\coloneqq\desirset[\lowprev]$ an Archimedean set of desirable options. 
Equation~\eqref{eq:archimedeanlimit} tells us that also $\desirset[\infty]=\lim_{i\to+\infty}\desirset[n_i]$, and since $\setofdesirsets$ was assumed to be closed, this implies that $\desirset[\infty]\in\setofdesirsets$. 
Since $\optset\in\rejectset=\bigcap\cset{\rejectset[\desirset]}{\desirset\in\setofdesirsets}$, this in turn implies that $\optset\in\rejectset[{\desirset[\infty]}]$, so there is some $\opt\in\optset$ such that $\opt\in\desirset[\infty]$, by Equation~\eqref{eq:desirset:to:rejectset}. 
Since $\desirset[\infty]=\desirset[\lowprev]$, this implies that $\lowprev(\opt)>0$. 
Fix any $n\in\naturals$. 
Since by our construction $\optset-\epsilon_n\notin\rejectset[{\desirset[n]}]$, we know that $(\optset-\epsilon_n)\cap\desirset[n]=\emptyset$, and therefore, since $\opt\in\optset$, that $\opt-\epsilon_n\notin\desirset[n]$, again by Equation~\eqref{eq:desirset:to:rejectset}. 
Since we know from Proposition~\ref{prop:isomorphismsmarch} that $\desirset[n]=\desirset[{\lowprev[{\desirset[n]}]}]$, this implies that $\lowprev[{\desirset[n]}](\opt-\epsilon_n)\leq0$. 
Because Proposition~\ref{prop:isomorphismsmarch} also tells us that $\lowprev[{\desirset[n]}]$ is a coherent lower prevision, this in turn implies that
\begin{equation*}
0
\geq
\lowprev[{\desirset[n]}](\opt-\epsilon_n)
\geq
\lowprev[{\desirset[n]}](\opt)+\lowprev[{\desirset[n]}](-\epsilon_n)
\geq
\lowprev[{\desirset[n]}](\opt)-\epsilon_n,
\end{equation*}
whence $\lowprev[{\desirset[n]}](\opt)\leq\epsilon_n=\nicefrac{1}{n}$. 
Since this is true for every $n\in\naturals$, we find that
\begin{equation*}
\lowprev(\opt)=\lim_{i\to+\infty}\lowprev[{\desirset[n_i]}](\opt)\leq\lim_{i\to+\infty}\frac{1}{i}=0.
\end{equation*}
However, we had already found that $\lowprev(\opt)>0$, a contradiction. 
Hence, there must be some $\epsilon\in\posreals$ such that $\optset-\epsilon\in\rejectset$. 
Since this is true for every $\optset\in\rejectset$, we find that $\rejectset$ is indeed Archimedean.
\end{proof}

\begin{proof}[Proof of Theorem~\ref{theo:convex:rejectsets:representation:gamblearch:twosided}]
First assume that $\rejectset$ is mixing and Archimedean. 
It then follows from Theorem~\ref{theo:convex:rejectsets:representation} that $\convcohdesirsets\group{\rejectset}\coloneqq\cset{\desirset\in\convcohdesirsets}{\rejectset\subseteq\rejectset[\desirset]}\neq\emptyset$ and $\rejectset=\bigcap\cset{\rejectset[\desirset]}{\desirset\in\convcohdesirsets\group{\rejectset}}$.
Let $\archconvcohdesirsets^*(\rejectset)\coloneqq\cset{\arch\group{\desirset}}{\desirset\in\convcohdesirsets\group{\rejectset}}\neq\emptyset$, then Lemma~\ref{lem:gamblearch:transformrepresentation} already guarantees that $\rejectset=\cap\cset{\rejectset[\desirset]}{\desirset\in\archconvcohdesirsets^*(\rejectset)}$. 

We first prove that $\archconvcohdesirsets(\rejectset)=\archconvcohdesirsets^*(\rejectset)$. 
To this end, first consider any $\desirset\in\archconvcohdesirsets^*(\rejectset)$. 
Then $\rejectset\subseteq\rejectset[\desirset]$ because $\rejectset=\cap\cset{\rejectset[\desirset]}{\desirset\in\archconvcohdesirsets^*(\rejectset)}$. 
Furthermore, it follows from Lemmas~\ref{lem:effectofgamblearch:coh}--\ref{lem:effectofgamblearch:lex} that $\desirset\in\archconvcohdesirsets$. 
Hence $\desirset\in\archconvcohdesirsets(\rejectset)$, and therefore $\archconvcohdesirsets^*(\rejectset)\subseteq\archconvcohdesirsets(\rejectset)$. 
Conversely, consider any $\desirset\in\archconvcohdesirsets(\rejectset)$. 
Since $\desirset$ is Archimedean, it follows from Lemma~\ref{lem:effectofgamblearch:preservegamblearch} that $\arch\group{\desirset}=\desirset$, and because also $\desirset\in\archconvcohdesirsets(\rejectset)\subseteq\convcohdesirsets(\rejectset)$, we find that $\desirset=\arch\group{\desirset}\in\archconvcohdesirsets^*(\rejectset)$. 
So $\archconvcohdesirsets(\rejectset)\subseteq\archconvcohdesirsets^*(\rejectset)$, and therefore, indeed, $\archconvcohdesirsets(\rejectset)=\archconvcohdesirsets^*(\rejectset)$.

Next, we prove that $\archconvcohdesirsets(\rejectset)$ is closed. 
So consider any convergent sequence $\set{\desirset[n]}_{n\in\naturals}$ in $\archconvcohdesirsets(\rejectset)$ and let $\desirset[\infty]$ be its limit. 
Then we need to prove that $\desirset[\infty]\in\archconvcohdesirsets(\rejectset)$.
Equation~\eqref{eq:archimedeanlimit} and Proposition~\ref{prop:isomorphismsmarch} tell us that $\desirset[\infty]=\desirset[\lowprev]$, where $\lowprev$ is the point-wise limit of the corresponding convergent sequence of coherent lower previsions $\set{\lowprev[{\desirset[n]}]}_{n\in\naturals}$. 
Since for any $n\in\naturals$, $\desirset[n]$ is mixing and Archimedean, we know from Proposition~\ref{prop:isomorphismsmarch} that $\lowprev[{\desirset[n]}]$ is a linear prevision. 
Therefore, and because $\linprevs$ is closed and compact in the (weak*) topology induced by point-wise convergence \cite{walley1991,maass2003:phd,troffaes2013:lp}, the limit $\lowprev$ is a linear prevision as well, so we infer from Proposition~\ref{prop:isomorphismsmarch} that $\desirset[\infty]=\desirset[\lowprev]$ is mixing and Archimedean, so $\desirset[\infty]\in\archconvcohdesirsets$.
It therefore remains to prove that $\rejectset\subseteq\rejectset[{\desirset[\infty]}]$. 
Consider, to this end, any $\optset\in\rejectset$. 
Since $\rejectset$ is Archimedean, there is some $\epsilon\in\posreals$ such that $\optset-\epsilon\in\rejectset$. 
Fix any $n\in\naturals$. 
Then since $\desirset[n]\in\archconvcohdesirsets(\rejectset)$, we have that $\rejectset\subseteq\rejectset[{\desirset[n]}]$. 
Therefore, and because $\optset-\epsilon\in\rejectset$, we infer from Equation~\eqref{eq:desirset:to:rejectset} that there is some $\opt[n]\in\optset$ such that $\opt[n]-\epsilon\in\desirset[n]$ and therefore also $\lowprev[{\desirset[n]}](\opt[n])\geq\epsilon$. 
So for all $n\in\naturals$, we have some $\opt[n]\in\optset$ such that $\lowprev[{\desirset[n]}](\opt[n])\geq\epsilon$. 
Since $\optset$ is finite, this implies that there is some $\opt^*\in\optset$ and some subsequence $\set{\desirset[n_i]}_{i\in\naturals}$ such that $\lowprev[{\desirset[{n_i}]}](\opt^*)\geq\epsilon$ for all $i\in\naturals$. 
Since $\lowprev$ is the point-wise limit of $\set{\lowprev[{\desirset[n]}]}_{n\in\naturals}$ and therefore also the point-wise limit of the subsequence $\set{\lowprev[{\desirset[n_i]}]}_{i\in\naturals}$, this implies that $\lowprev(\opt^*)\geq\epsilon>0$, so $\opt^*\in\desirset[\lowprev]=\desirset[\infty]$. 
Since $\opt^*\in\optset$, this implies that $\optset\in\rejectset[{\desirset[\infty]}]$, again by Equation~\eqref{eq:desirset:to:rejectset}. 
Since this is true for every $\optset\in\rejectset$, we find that, indeed, $\rejectset\subseteq\rejectset[{\desirset[\infty]}]$, so $\desirset[\infty]\in\archconvcohdesirsets(\rejectset)$ and $\archconvcohdesirsets(\rejectset)$ is therefore indeed closed.

What we have found so far is that that there is a least one non-empty closed set $\setofdesirsets\subseteq\archconvcohdesirsets$ such that $\rejectset=\bigcap\cset{\rejectset[\desirset]}{\desirset\in\setofdesirsets}$, because $\archconvcohdesirsets(\rejectset)=\archconvcohdesirsets^*(\rejectset)\neq\emptyset$ satisfies all these properties. 
Furthermore, for any other non-empty closed set $\setofdesirsets\subseteq\archconvcohdesirsets$ such that $\rejectset=\bigcap\cset{\rejectset[\desirset]}{\desirset\in\setofdesirsets}$, we clearly have that $\rejectset\subseteq\rejectset[\desirset]$ for all $\desirset\in\setofdesirsets$. 
Since $\setofdesirsets\subseteq\archconvcohdesirsets$, this implies that $\setofdesirsets\subseteq\archconvcohdesirsets(\rejectset)$, so $\archconvcohdesirsets(\rejectset)$ is indeed the largest such set.

We now turn to the `if' part. 
So consider any non-empty closed set $\setofdesirsets\subseteq\archconvcohdesirsets$ of Archimedean and mixing sets of desirable options such that $\rejectset=\bigcap\cset{\rejectset[\desirset]}{\desirset\in\setofdesirsets}$. 
For every $\desirset\in\setofdesirsets\subseteq\archconvcohdesirsets$, it then follows from Proposition~\ref{prop:gamblearchbinaryiffD} that $\rejectset[\desirset]$ is Archimedean and mixing. 
Because Axioms~\ref{ax:rejects:removezero}-\ref{ax:rejects:mono} and~\ref{ax:rejects:removepositivecombinations} are trivially preserved under taking arbitrary (non-empty) intersections, it already follows that $\rejectset$ is mixing. 
It therefore only remains to prove that $\rejectset$ is Archimedean as well, or in other words---because we already know that $\rejectset$ is coherent because mixing---that $\rejectset$ satisfies~\ref{ax:rejects:gamblearch}.

To this end, consider any $\optset\in\rejectset$ and assume \emph{ex absurdo} that $\optset-\epsilon\notin\rejectset$ for all $\epsilon\in\posreals$. 
Fix any $n\in\naturals$ and let $\epsilon_n\coloneqq\nicefrac{1}{n}>0$, so by assumption $\optset-\epsilon_n\notin\rejectset$. 
Hence, since $\rejectset=\bigcap\cset{\rejectset[\desirset]}{\desirset\in\setofdesirsets}$, there is some $\desirset[n]\in\setofdesirsets$ such that $\optset-\epsilon_n\notin\rejectset[{\desirset[n]}]$. 
This allows us to consider the sequence $\set{\lowprev[{\desirset[n]}]}_{n\in\naturals}$, which, because of Proposition~\ref{prop:isomorphismsmarch}, belongs to $\linprevs$. 
Since we already know that $\linprevs$ is compact under the (weak*) topology induced by point-wise convergence, this sequence has some  subsequence $\set{\lowprev[{\desirset[n_i]}]}_{i\in\naturals}$ that converges point-wise to a limit $\linprev\in\linprevs$. 
Since $\linprev\in\linprevs$, we know from Proposition~\ref{prop:isomorphismsmarch} that $\linprev=\lowprev[{\desirset[\infty]}]$, with $\desirset[\infty]\coloneqq\desirset[\linprev]$ a mixing and Archimedean set of desirable options. 
Equation~\eqref{eq:archimedeanlimit} tells us that also $\desirset[\infty]=\lim_{i\to+\infty}\desirset[n_i]$, and since $\setofdesirsets$ was assumed to be closed, this implies that $\desirset[\infty]\in\setofdesirsets$. 
Since $\optset\in\rejectset=\bigcap\cset{\rejectset[\desirset]}{\desirset\in\setofdesirsets}$, this in turn implies that $\optset\in\rejectset[{\desirset[\infty]}]$, so there is some $\opt\in\optset$ such that $\opt\in\desirset[\infty]$, by Equation~\eqref{eq:desirset:to:rejectset}. 
Since $\desirset[\infty]=\desirset[\linprev]$, this implies that $\linprev(\opt)>0$. 
Fix any $n\in\naturals$. 
Since by our construction $\optset-\epsilon_n\notin\rejectset[{\desirset[n]}]$, we know that $(\optset-\epsilon_n)\cap\desirset[n]=\emptyset$, and therefore, since $\opt\in\optset$, that $\opt-\epsilon_n\notin\desirset[n]$, again by Equation~\eqref{eq:desirset:to:rejectset}. 
Since we know from Proposition~\ref{prop:isomorphismsmarch} that $\desirset[n]=\desirset[{\lowprev[{\desirset[n]}]}]$, this implies that $\lowprev[{\desirset[n]}](\opt-\epsilon_n)\leq0$. 
Because Proposition~\ref{prop:isomorphismsmarch} also tells us that $\lowprev[{\desirset[n]}]$ is a coherent lower prevision, this in turn implies that
\begin{equation*}
0
\geq
\lowprev[{\desirset[n]}](\opt-\epsilon_n)
\geq
\lowprev[{\desirset[n]}](\opt)+\lowprev[{\desirset[n]}](-\epsilon_n)
\geq
\lowprev[{\desirset[n]}](\opt)-\epsilon_n,
\end{equation*}
whence $\lowprev[{\desirset[n]}](\opt)\leq\epsilon_n=\nicefrac{1}{n}$. 
Since this is true for every $n\in\naturals$, we find that
\begin{equation*}
\linprev(\opt)
=\lim_{i\to+\infty}\lowprev[{\desirset[n_i]}](\opt)
\leq\lim_{i\to+\infty}\frac{1}{i}=0.
\end{equation*}
However, we had already found that $\linprev(\opt)>0$, a contradiction. 
Hence, there must be some $\epsilon\in\posreals$ such that $\optset-\epsilon\in\rejectset$. 
Since this is true for every $\optset\in\rejectset$, we find that $\rejectset$ is indeed Archimedean.
\end{proof}

\iftoggle{ISIPTA}{}{

\newpage

\section{Abstract Archimedeanity}

\subsection{Main text}

We consider a very general and abstract notion of Archimedeanity that we can later specialize in different ways. 
The starting point is a given cone of options $\archopts\subseteq\posopts$. 
Using this cone, we introduce the $\arch\group{\cdot}$ operator by letting
\begin{equation*}
\arch(V)
\coloneqq\cset{\opt\in V}{\group{\exists\altopt\in\archopts}\opt-\altopt\in V},
\text{ for any subset $V$ of $\opts$}.
\end{equation*}
A set of desirable options is then called \emph{Archimedean} if it is invariant under this operator.

\begin{definition}[Archimedean set of desirable options]\label{def:archimedean:sets}
We call a coherent set of desirable options $\desirset\in\cohdesirsets$ \emph{Archimedean} if
\begin{enumerate}[label=$\mathrm{D}_{\mathrm{A}}$.,ref=$\mathrm{D}_{\mathrm{A}}$,leftmargin=*]
\item\label{ax:desirs:arch} $\arch\group{\desirset}=\desirset$.
\end{enumerate}
We denote the set of all Archimedean sets of desirable options by $\archcohdesirsets$, and let $\archconvcohdesirsets$ be the set of all Archimedean sets of desirable options that are also mixing.
\end{definition}

Similarly, a set of desirable option sets is called \emph{Archimedean} if it is invariant under the $\Arch\group{\cdot}$ operator, defined by
\begin{equation*}
\Arch(\rejectset)
\coloneqq\cset[\big]{\optset\in\rejectset}
{(\exists\overline{\altopt}\in\archopts^{\optset})
\cset{\opt-\overline{\altopt}(\opt)}{\opt\in\optset}\in\rejectset},
\text{ for any subset $\rejectset$ of $\optsets$}.
\end{equation*}
where $\archopts^{\optset}$ is the set of all maps from $\optset$ to $\archopts$.

\begin{definition}[Archimedean set of desirable option sets]\label{def:archimedean:setsofsets}
We call a coherent set of desirable option sets $\rejectset\in\cohrejectsets$ \emph{Archimedean} if
\begin{enumerate}[label=$\mathrm{K}_{\mathrm{A}}$.,ref=$\mathrm{K}_{\mathrm{A}}$,leftmargin=*]
\item$\Arch(\rejectset)=\rejectset$.\label{ax:rejects:arch}
\end{enumerate}
\end{definition}

\begin{theorem}[Representation for Archimedean choice functions]\label{theo:rejectsets:representation:arch}
Consider any Archimedean set of desirable option sets $\rejectset$. 
Then $\rejectset$ is dominated by some binary Archimedean set of desirable option sets: $\archcohdesirsets(\rejectset)\coloneqq\cset{\desirset\in\archcohdesirsets}{\rejectset\subseteq\rejectset[\desirset]}\neq\emptyset$.
Moreover, $\rejectset=\bigcap\cset{\rejectset[\desirset]}{\desirset\in\archcohdesirsets(\rejectset)}$.
\end{theorem}

\begin{theorem}[Representation for Archimedean mixing choice functions]\label{theo:convex:rejectsets:representation:arch}
Any Archimedean mixing set of desirable option sets $\rejectset$ is dominated by some binary Archimedean mixing set of desirable option sets: $\archconvcohdesirsets(\rejectset)\coloneqq\cset{\desirset\in\archconvcohdesirsets}{\rejectset\subseteq\rejectset[\desirset]}\neq\emptyset$.
Moreover, $\rejectset=\bigcap\cset{\rejectset[\desirset]}{\desirset\in\archconvcohdesirsets(\rejectset)}$.
\end{theorem}

*** Voor binaire choice functions bestaan er allerhande verbanden tussen (sommige noties van) Archimedeanity en coherente (conditionele) onderprevisies. Om dat verband te kunnen leggen ben ik van plan om op dat moment (analoog aan Williams) te eisen dat $\opts$ bestaat uit begrensde reelwaardige functies en minstens alle constante functies bevat (of, in het geval van conditionele onderprevisies, alle indicatoren). De rol van $\archopts$ wordt dan opgenomen door de positieve constante functies (of door de positieve schalingen van indicatoren zonder nul). ***

*** twosided representation theorems toevoegen in termen van gesloten verzamelingen van (conditionele) (onder)previsies ***

*** A basic choice for $\archopts$ is to let $\archopts=\posopts$. 
However, interesting notions of Archimedeanity can also be obtained for $\archopts\subset\posopts$; when $\posopts=\posgbls$, the set of all constant positive gambles serves as a nice example. ***

\subsection{Proofs for abstract Archimedeanity}

\begin{lemma}\label{lem:effectofarch:arch}
For any coherent set of desirable options $\desirset\in\cohdesirsets$, $\arch\group{\desirset}$ is Archimedean.
\end{lemma}
\begin{proof}
It clearly suffices to prove that $\arch\group{\desirset}\subseteq\arch(\arch\group{\desirset})$, because the converse set inclusion holds trivially. 

So consider any $\opt\in\arch\group{\desirset}$. 
This implies that $\opt\in\desirset$ and that there is some $\altopt\in\archopts$ such that $\opt-\altopt\in\desirset$. 
Since $\archopts$ is a cone, we know that $\frac{1}{2}\altopt\in\archopts\subseteq\posopts$. 
Therefore, and because $\desirset$ is coherent and $\opt-v\in\desirset$, we find that $\opt-\frac{1}{2}\altopt=(\opt-\altopt)+\frac{1}{2}\altopt\in\desirset$. 
Since $(\opt-\frac{1}{2}\altopt)-\frac{1}{2}\altopt=\opt-\altopt\in\desirset$ and $\frac{1}{2}\altopt\in\archopts$, this implies that $\opt-\frac{1}{2}\altopt\in\arch\group{\desirset}$. 
Since $\opt\in\arch\group{\desirset}$ and $\frac{1}{2}\altopt\in\archopts$, this implies that $\opt\in\arch(\arch\group{\desirset})$.
\end{proof}

\begin{lemma}\label{lem:effectofarch:coh}
For any coherent set of desirable options $\desirset\in\cohdesirsets$, $\arch\group{\desirset}$ is coherent if and only if $\posopts\subseteq\arch\group{\desirset}$.
\end{lemma}
\begin{proof}
Consider any coherent set of desirable options $\desirset\in\cohdesirsets$, $\arch\group{\desirset}$. 
We will prove that $\arch\group{\desirset}$ satisfies \ref{ax:desirs:nozero} and \ref{ax:desirs:cone}. 
The result is then an immediate consequence of Definition~\ref{def:cohdesir}.

That $\arch\group{\desirset}$ satisfies \ref{ax:desirs:nozero} follows directly from the fact that $\desirset$ does, because $\arch\group{\desirset}\subseteq\desirset$.
To prove that $\arch\group{\desirset}$ satisfies \ref{ax:desirs:cone}, we consider any $\opt,\altopt\in\arch\group{\desirset}$ and $\lambda,\mu\geq0$ such that $\lambda+\mu>0$, and prove that $\lambda\opt+\mu\altopt\in\arch\group{\desirset}$.
Since $\opt,\altopt\in\arch\group{\desirset}$, there are $\opt',\altopt'\in\archopts$ such that $\opt-\opt'\in\desirset$ and $\altopt-\altopt'\in\desirset$, which, since $\desirset$ satisfies \ref{ax:desirs:cone}, implies that $\lambda\opt+\mu\altopt-(\lambda\opt'+\mu\altopt')=\lambda(\opt-\opt')+\mu(\altopt-\altopt')\in\desirset$. 
Since $\archopts$ is a cone, we also know that $\lambda\opt'+\mu\altopt'\in\archopts$. 
Hence, $\lambda\opt+\mu\altopt\in\arch\group{\desirset}$.
\end{proof}

\begin{lemma}\label{lem:effectofarch:lex}
For any coherent mixing set of desirable options $\desirset\in\convcohdesirsets$, $\arch\group{\desirset}$ is coherent and mixing if and only if $\posopts\subseteq\arch\group{\desirset}$.
\end{lemma}
\begin{proof}
Consider any coherent mixing set of desirable options $\desirset\in\convcohdesirsets$. We will prove that $\arch\group{\desirset}$ satisfies~\ref{ax:desirs:mixing}. The result then follows from Lemma~\ref{lem:effectofarch:coh} and Definition~\ref{def:mixingdesirs}.

So consider any $\optset\in\optsets$ such that $\posi(\optset)\cap\arch\group{\desirset}\neq\emptyset$. Then there is at least one $\opt\in\posi(\optset)$ such that $\opt\in\arch\group{\desirset}$. Consider any such $\opt$. Since $\opt\in\arch\group{\desirset}$, there is some $\altopt\in\archopts$ such that $\opt-\altopt\in\desirset$. Consider any such $\altopt$. Since $\opt\in\posi(\optset)$, we know that $\opt=\sum_{i=1}^n\lambda_i\opt[i]$, with $n\in\naturals$ and, for all $i\in\set{1,\dots,n}$, $\lambda_i>0$ and $\opt[i]\in\optset$. Let $\lambda\coloneqq\sum_{i=1}^n\lambda_i>0$ and $\alpha\coloneqq\nicefrac{1}{\lambda}>0$. Then
\begin{equation*}
\opt-\altopt
=\sum_{i=1}^n\lambda_i\opt[i]-\alpha\Big(\sum_{i=1}^n\lambda_i\Big)\altopt
=\sum_{i=1}^n\lambda_i\Big(\opt[i]-\alpha\altopt\Big),
\end{equation*}
which implies that $\opt-\altopt\in\posi(\optset-\alpha\altopt)$. Since also $\opt-\altopt\in\desirset$, it follows that $\posi(\optset-\alpha\altopt)\cap\desirset\neq\emptyset$. Therefore, and because $\desirset$ is mixing, we find that $(\optset-\alpha\altopt)\cap\desirset\neq\emptyset$, meaning that there is some $\tilde\opt\in\optset$ such that $\tilde\opt-\alpha\altopt\in\desirset$. Since $\alpha\altopt\in\archopts$ because $\archopts$ is a cone, we conclude that $\tilde\opt\in\arch\group{\desirset}$, so $\optset\cap\arch\group{\desirset}\neq\emptyset$.
\end{proof}

\begin{lemma}\label{lem:arch:transformrepresentation}
Consider any set $\tilde{\desirsets}\subseteq\cohdesirsets$ of coherent sets of desirable options such that $\rejectset\coloneqq\bigcap\cset{\rejectset[\desirset]}{\desirset\in\tilde{\desirsets}}$ is Archimedean. 
Then $\rejectset=\bigcap\cset{\rejectset[\arch\group{\desirset}]}{\desirset\in\tilde{\desirsets}}$.
\end{lemma}
\begin{proof}
For any $\desirset\in\tilde{\desirsets}$, since $\arch\group{\desirset}\subseteq\desirset$, it follows that $\rejectset[\arch\group{\desirset}]\subseteq\rejectset[\desirset]$. 
Hence,
\begin{equation*}
\bigcap\cset{\rejectset[\arch\group{\desirset}]}{\desirset\in\tilde{\desirsets}}
\subseteq\bigcap\cset{\rejectset[\desirset]}{\desirset\in\tilde{\desirsets}}
=\rejectset.
\end{equation*}
It remains to prove the converse set inclusion. 
To this end, consider any $\desirset\in\tilde{\desirsets}$ and any $\optset\in\rejectset$. 
Then since $\rejectset$ is Archimedean, we know that there are $\altopt[\opt]\in\archopts$ for all $\opt\in\optset$ such that $\cset{\opt-\altopt[\opt]}{\opt\in\optset}\in\rejectset\subseteq\rejectset[\desirset]$. 
This implies that there is some $\opt'\in\optset$ such that $\opt'-\altopt[\opt']\in\desirset$. 
Since $\desirset$ is coherent and $\altopt[\opt']\in\archopts\subseteq\posopts$, this in turn implies that $\opt'\in\desirset$. 
Since $\altopt[\opt']\in\archopts$ and $\opt'-\altopt[\opt']\in\desirset$, it follows that $\opt'\in\arch\group{\desirset}$ and therefore, that $\optset\cap\arch\group{\desirset}\neq\emptyset$. 
Hence, $\optset\in\rejectset[\arch\group{\desirset}]$. 
Since this is true for every $\desirset\in\tilde{\desirsets}$ and any $\optset\in\rejectset$, we find that
\begin{equation*}
\rejectset
\subseteq
\bigcap\cset{\rejectset[\arch\group{\desirset}]}{\desirset\in\tilde{\desirsets}},
\end{equation*}
as desired.
\end{proof}

\begin{lemma}\label{lem:posifdominatedbyarch}
Consider a coherent set of desirable option sets $\rejectset\in\rejectsets$ and any set of desirable options $\desirset\in\desirsets$ such that $\rejectset\subseteq\rejectset[\desirset]$. 
Then $\posopts\subseteq\desirset$.
\end{lemma}
\begin{proof}
For all $\opt\in\posopts$, the coherence of $\rejectset$ implies that $\set{\opt}\in\rejectset$, which, since $\rejectset\subseteq\rejectset[\desirset]$, implies that $\set{\opt}\cap\desirset\neq\emptyset$. 
Hence, $\opt\in\desirset$.
\end{proof}

\begin{lemma}\label{lem:rejectsets:representation:arch:abstract}
Let $\cohdesirsets'\subseteq\cohdesirsets$ be any set of coherent sets of desirable options such that, for all $\desirset\in\cohdesirsets'$, $\arch\group{\desirset}\in\cohdesirsets'\cap\archcohdesirsets$ if and only if $\posopts\subseteq\arch\group{\desirset}$, and consider any Archimedean coherent set of desirable option sets $\rejectset$ such that $\tilde{\desirsets}\coloneqq\cset{\desirset\in\cohdesirsets'}{\rejectset\subseteq\rejectset[\desirset]}\neq\emptyset$ and $\rejectset=\bigcap\cset{\rejectset[\desirset]}{\desirset\in\tilde{\desirsets}}$. 
Then $\tilde{\desirsets}_{\mathrm{A}}\coloneqq\cset{\desirset\in\cohdesirsets'\cap\archcohdesirsets}{\rejectset\subseteq\rejectset[\desirset]}\neq\emptyset$ and $\rejectset=\bigcap\cset{\rejectset[\desirset]}{\desirset\in\tilde{\desirsets}_{\mathrm{A}}}$
\end{lemma}
\begin{proof}
Let $\desirsets'\coloneqq\cset{\arch\group{\desirset}}{\desirset\in\tilde{\desirsets}}$. 
Then $\desirsets'\neq\emptyset$ because $\tilde{\desirsets}\neq\emptyset$. 
Furthermore, since $K$ is Archimedean and $\rejectset=\bigcap\cset{\rejectset[\desirset]}{\desirset\in\tilde{\desirsets}}$, it follows from Lemma~\ref{lem:arch:transformrepresentation} that
\begin{equation*}
\rejectset=\bigcap\cset{\rejectset[\arch\group{\desirset}]}{\desirset\in\tilde{\desirsets}}=\bigcap\cset{\rejectset[\desirset']}{\desirset'\in\desirsets'}.
\end{equation*} 
It therefore remains to prove that $\desirsets'=\tilde{\desirsets}_{\mathrm{A}}$.

First consider any $\desirset'\in\desirsets'$. 
Then $\desirset'=\arch\group{\desirset}$ for some $\desirset\in\tilde{\desirsets}\subseteq\cohdesirsets'$ and it therefore follows from our assumption on $\cohdesirsets'$ that $\desirset'\in\cohdesirsets'\cap\archcohdesirsets$ if and only if $\posopts\subseteq\desirset'$.
Furthermore, since $\smash{\rejectset=\bigcap\cset{\rejectset[\desirset']}{\desirset'\in\desirsets'}}$, we also know that $\rejectset\subseteq\rejectset[\desirset']$ and therefore, because of Lemma~\ref{lem:posifdominatedbyarch}, that $\posopts\subseteq\desirset'$. 
We conclude that $\smash{\desirset'\in\cohdesirsets'\cap\archcohdesirsets}$ and $\rejectset\subseteq\rejectset[\desirset']$. 
Hence, $\desirset'\in\tilde{\desirsets}_{\mathrm{A}}$.

Now consider any $\desirset\in\tilde{\desirsets}_{\mathrm{A}}$. 
Then $\desirset$ is Archimedean and belongs to $\tilde{\desirsets}$. 
The former implies that $\arch\group{\desirset}=\desirset$, while the latter implies that $\arch\group{\desirset}\in\desirsets'$. 
Hence, $\desirset\in\desirsets'$.
\end{proof}

\begin{proof}[Proof of Theorem~\ref{theo:rejectsets:representation:arch}]
Let $\cohdesirsets'\coloneqq\cohdesirsets$. 
Then for all $\desirset\in\cohdesirsets'$, it follows from Lemma~\ref{lem:effectofarch:arch} and~\ref{lem:effectofarch:coh} that $\arch\group{\desirset}\in\cohdesirsets'\cap\archcohdesirsets$ if and only if $\posopts\subseteq\arch\group{\desirset}$. 
Furthermore, if we let $\tilde{\desirsets}\coloneqq\cset{\desirset\in\cohdesirsets'}{\rejectset\subseteq\rejectset[\desirset]}$, then because $\rejectset$ is coherent, we know from Theorem~\ref{theo:rejectsets:representation} that $\tilde{\desirsets}\neq\emptyset$ and $\rejectset=\bigcap\cset{\rejectset[\desirset]}{\desirset\in\tilde{\desirsets}}$. 
Since $\rejectset$ is Archimedean and $\cohdesirsets'\cap\archcohdesirsets=\archcohdesirsets$, the result now follows from Lemma~\ref{lem:rejectsets:representation:arch:abstract}.
\end{proof}

\begin{proof}[Proof of Theorem~\ref{theo:convex:rejectsets:representation:arch}]
Let $\cohdesirsets'\coloneqq\convcohdesirsets$. 
Then for all $\desirset\in\cohdesirsets'$, it follows from Lemma~\ref{lem:effectofarch:arch} and~\ref{lem:effectofarch:lex} that $\arch\group{\desirset}\in\cohdesirsets'\cap\archcohdesirsets$ if and only if $\posopts\subseteq\arch\group{\desirset}$. 
Furthermore, if we let $\tilde{\desirsets}\coloneqq\cset{\desirset\in\cohdesirsets'}{\rejectset\subseteq\rejectset[\desirset]}$, then because $\rejectset$ is mixing and coherent, we know from Theorem~\ref{theo:convex:rejectsets:representation} that $\tilde{\desirsets}\neq\emptyset$ and $\rejectset=\bigcap\cset{\rejectset[\desirset]}{\desirset\in\tilde{\desirsets}}$. 
Since $\rejectset$ is Archimedean and $\cohdesirsets'\cap\archcohdesirsets=\archconvcohdesirsets$, the result now follows from Lemma~\ref{lem:rejectsets:representation:arch:abstract}.
\end{proof}

\newpage

\section{Stuff on maximality that we don't need for now}

The final two concepts that we require are maximal sets of desirable options and maximal sets of desirable option sets, which are the maximal elements of $\cohdesirsets$ and $\cohrejectsets$, respectively, when these sets are partially ordered by set inclusion. 
We end by introducing these notions in more detail.

First, for any two sets of desirable options $\maxdesirset,\desirset\in\desirsets$, we say that $\maxdesirset$ is \emph{dominated} by $\desirset$ if $\maxdesirset\subseteq\desirset$, and we say that $\maxdesirset$ is \emph{strictly dominated} by $\desirset$ if $\maxdesirset\subset\desirset$. 
We then call a coherent set of desirable options $\maxdesirset$ \emph{maximal} if it is not strictly dominated by any other coherent set of desirable options, and we collect all maximal coherent sets of desirable options in the set $\maxdesirsets\subseteq\cohdesirsets$: for any $\maxdesirset\in\cohdesirsets$,
\begin{equation*}
\maxdesirset\in\maxdesirsets
\ifandonlyif
\group{\forall\desirset\in\cohdesirsets}
\group{\maxdesirset\subseteq\desirset\then\maxdesirset=\desirset}.
\end{equation*}

For any two sets of desirable option sets $\maxrejectset,\rejectset\in\rejectsets$, we say that $\maxrejectset$ is \emph{dominated} by $\rejectset$ if $\maxrejectset\subseteq\rejectset$, and we say that $\maxrejectset$ is \emph{strictly dominated} by $\rejectset$ if $\maxrejectset\subset\rejectset$. 
We call a coherent set of desirable option sets $\maxrejectset$ \emph{maximal}, if it is not strictly dominated by any other coherent set of desirable option sets, and we collect all maximal coherent sets of desirable option sets in the set $\maxrejectsets\subseteq\cohrejectsets$: for any $\maxrejectset\in\cohrejectsets$,
\begin{equation*}
\maxrejectset\in\maxrejectsets
\ifandonlyif
\group{\forall\rejectset\in\cohrejectsets}
\group{\maxrejectset\subseteq\rejectset\then\maxrejectset=\rejectset}.
\vspace{4pt}
\end{equation*}

\begin{proposition}\label{prop:maximal:is:lexicographic}
If a coherent set of desirable options is maximal, then it is also mixing: $\maxdesirsets\subseteq\convcohdesirsets$.
\end{proposition}
\begin{proof}[Proof of Proposition~\ref{prop:maximal:is:lexicographic}]
*** to be added ***
\end{proof}

\begin{theorem}\label{theo:maximal:coherent:rejectsets}
$\maxrejectsets=\cset{\rejectset[\maxdesirset]}{\maxdesirset\in\maxdesirsets}$.
\end{theorem}
\begin{proof}
It follows from Proposition~\ref{prop:dominatebybinary} that every non-binary set of desirable option sets is strictly dominated by a binary one. 
This implies that  maximal sets of desirable option sets must be binary: $\maxrejectsets\subseteq\cset{\rejectset[\desirset]}{\desirset\in\cohdesirsets}$.
Since, moreover, we can easily infer from Equation~\eqref{eq:desirset:to:rejectset} that $\desirset[1]\subseteq\desirset[2]\ifandonlyif\rejectset[{\desirset[1]}]\subseteq\rejectset[{\desirset[2]}]$ for all $\desirset[1],\desirset[2]\in\cohdesirsets$, \textcolor{red}{*** to be finished ***}
\end{proof}

\begin{theorem}\label{theo:rejectsets:maximality}
Any coherent set of desirable option sets $\rejectset\in\cohrejectsets$ is dominated by some maximal coherent set of desirable option sets: $\cset{\maxrejectset\in\maxrejectsets}{\rejectset\subseteq\maxrejectset}\neq\emptyset$.
\end{theorem}

\begin{proof}[Proof of Theorem~\ref{theo:rejectsets:maximality}]
By applying Lemma~\ref{lem:Zorncoherence} for $\rejectset^*=\emptyset$, we find that the partially ordered set $\cset{\rejectset'\in\cohrejectsets}{\rejectset\subseteq\rejectset'}$ has a maximal element. 
Let $\maxrejectset$ be any such maximal element. 
Then $\rejectset\subseteq\maxrejectset$. 
Assume \emph{ex absurdo} that $\maxrejectset\notin\maxrejectsets$. 
This means that there is some $\rejectset^*\in\cohrejectsets$ such that $\maxrejectset\subset\rejectset^*$. 
But then also $\rejectset\subseteq\maxrejectset\subset\rejectset^*$, which implies that $\rejectset^*\in\cset{\rejectset'\in\cohrejectsets}{\rejectset\subseteq\rejectset'}$. 
Since $\maxrejectset\subset\rejectset^*$, this contradicts the fact that $\maxrejectset$ is a maximal element of $\cset{\rejectset'\in\cohrejectsets}{\rejectset\subseteq\rejectset'}$. 
So it must be that $\maxrejectset\in\maxrejectsets$.
\end{proof}

\newpage
\section{Axioms we no longer explicitly mention}

*** explain that this can be motivated by imposing the following desirability principle ***

\begin{enumerate}[label=$\mathrm{d}_{\mathrm{T}}$.,ref=$\mathrm{d}_{\mathrm{T}}$,start=1,leftmargin=*]
\item\label{ax:desirability:totality} 
for all $\opt\in\opts\setminus\set{0}$, either $\opt$ or $-\opt$ is desirable.
\end{enumerate}

*** which in turn can be motivated by imposing the following property on $\prefgt$

\begin{enumerate}[label=$\protect{\prefgt[\mathrm{T}]}$.,ref=$\protect{\prefgt[\mathrm{T}]}$,start=0,leftmargin=*]
\item\label{ax:prefgt:totality}
for all $\opt,\altopt\in\opts$ such that $\opt\neq\altopt$, either $\opt\prefgt\altopt$ or $\opt\preflt\altopt$.
\end{enumerate}

*** explain that the mixing property can be motivated by imposing the following desirability principle ***

\begin{enumerate}[label=$\mathrm{d}_{\mathrm{M}}$.,ref=$\mathrm{d}_{\mathrm{M}}$,start=1,leftmargin=*]
\item\label{ax:desirability:mixing} 
for all $\optset\in\optsets$, if $\posi(\optset)$ contains a desirable option, then so does $\optset$.
\end{enumerate}
Here too, the $\posi$ operator can be replaced by the convex hull operator to yield an alternative principle $\mathrm{d}'_{\mathrm{M}}$. 
The proof is quite trivial in this case: it suffices to use~\ref{ax:desirability:cone} to rescale the desirable option in $\posi(\optset)$ to yield a desirable option in $\chull(\optset)$.

If desirability is interpreted in terms of $\prefgt$, \ref{ax:desirability:mixing} corresponds to the following property:

\begin{enumerate}[label=$\protect{\prefgt[{\mathrm{M}}]}$.,ref=$\protect{\prefgt[{\mathrm{M}}]}$,start=0,leftmargin=*]
\item\label{ax:prefgt:mixing}
for all $\optset\in\optsets$, if $\posi(\optset)$ contains an option $\opt$ such that $\opt\prefgt0$, then so does $\optset$.
\end{enumerate}
Once more, the $\posi$ operator can be replaced by the convex hull operator. 
In this case, this has the added advantage that the role of $0$ can be taken over by an arbitrary option $\altopt$ because $\chull(\optset)-\altopt=\chull(\optset-\altopt)$. 
This is not the case in our version with the $\posi$ operator, because $\posi(\optset)-\altopt$ may not be equal to $\posi(\optset-\altopt)$.


*** The mixing property can also be imposed on sets of desirable options, by requiring them to adhere to~\ref{ax:desirability:mixing}. This yields the following definition ***

\newpage

\section{Natural extension}\label{sec:natex}
In many practical situations, a subject will typically not specify a full-fledged coherent set of desirable option sets, but will only provide some partial \emph{assessment} $\assessment\subseteq\optsets$, consisting of a number of option sets for which she is comfortable about assessing that they contain at least one desirable option.
We now want to extend this assessment~$\assessment$ to a coherent set of desirable option sets in a manner that is as conservative---or uninformative---as possible.
This is the essence of \emph{conservative inference}.

We say that a set of desirable option sets $\rejectset[1]$ is less informative than (or rather, at most as informative as) a set of desirable option sets $\rejectset[2]$, when \mbox{$\rejectset[1]\subseteq\rejectset[2]$}: a subject whose beliefs are represented by $\rejectset[2]$ has more (or rather, at least as many) desirable option sets---sets of options that definitely contain a desirable option---than a subject with beliefs represented by $\rejectset[1]$.
The resulting partially ordered set $\structure{\rejectsets,\subseteq}$ is a complete lattice with intersection as infimum and union as supremum.
The following theorem, whose proof is trivial, identifies an interesting substructure.

\begin{theorem}\label{theo:conservative:inference:for:rejectsets}
Let $\set{\rejectset[i]}_{i\in I}$ be an arbitrary non-empty family of sets of desirable option sets, with intersection $\rejectset\coloneqq\bigcap_{i\in I}\rejectset[i]$.
If $\rejectset[i]$ is coherent for all $i\in I$, then so is $\rejectset$.
This implies that $\structure{\cohrejectsets,\subseteq}$ is a complete meet-semilattice. 
\end{theorem}
\noindent
This result is important, as it allows us to a extend a partially specified set of desirable option sets to the most conservative coherent one that includes it.
This leads to the conservative inference procedure we will call natural extension.

\begin{definition}[Consistency and natural extension]
For any assessment $\assessment\subseteq\optsets$, let $\cohrejectsets\group{\assessment}\coloneqq\cset{\rejectset\in\cohrejectsets}{\assessment\subseteq\rejectset}$.
We call the assessment~$\assessment$ \emph{consistent} if\/ $\cohrejectsets\group{\assessment}\neq\emptyset$, and we then call\/ $\natexrejectset\group{\assessment}\coloneqq\bigcap\cohrejectsets\group{\assessment}$ the \emph{natural extension} of $\assessment$.
\end{definition}
\noindent
In other words: an assessment $\assessment$ is consistent if it can be extended to some coherent set of desirable option sets, and then its natural extension $\natexrejectset\group{\assessment}$ is the least informative such coherent set of desirable option sets.

We end this section with a more `constructive' expression for this natural extension and a simpler criterion for consistency. 
In order to state it, we need to introduce the set $\singposopts\coloneqq\cset{\set{\opt}}{\opt\in\posopts}$ and two operators on---transformations of---$\rejectsets$. 
The first is denoted by $\RS$, and defined by
\begin{equation*}
\RS\group{\rejectset}\coloneqq\cset{\optset\in\optsets}{(\exists\altoptset\in\rejectset)\altoptset\setminus\nonposopts\subseteq\optset}
\text{ for all $\rejectset\in\rejectsets$},
\end{equation*}
so $\RS\group{\rejectset}$ contains all option sets $\optset$ in $\rejectset$, all versions of $\optset$ with some of their non-positive options removed, and all supersets of such sets. 
The second is denoted by $\setposi$, and defined for all $\rejectset\in\rejectsets$ by
\begin{align}
\setposi\group{\rejectset}\coloneqq\bigg\{
\bigg\{
\sum_{k=1}^n\lambda_{k}^{\opt[1:n]}\opt[k]
\colon
\opt[1:n]\in\times_{k=1}^n\optset[k]
\bigg\}
\colon
&n\in\naturals,(\optset[1],\dots,\optset[n])\in\rejectset^n,\notag\\[-11pt]
&\big(\forall\opt[1:n]\in\times_{k=1}^n\optset[k]\big)\,\lambda_{1:n}^{\opt[1:n]}>0
\bigg\},
\label{eq:setposi:appendix}
\end{align}
where we used the notations $\opt[1:n]$ and $\lambda_{1:n}^{\opt[1:n]}$ for $n$-tuples of options $\opt[k]$ and real numbers $\lambda_{k}^{\opt[1:n]}$, $k\in\set{1,\dots,n}$, so $\opt[1:n]\in\opts^{n}$ and $\lambda_{1:n}^{\opt[1:n]}\in\reals^{n}$.
We also used $\lambda_{1:n}^{\opt[1:n]}>0$ as a shorthand for `$\lambda_k^{\opt[1:n]}\geq0$ for all $k\in\set{1,\dots,n}$ and $\sum_{k=1}^n\lambda_k^{\opt[1:n]}>0$'.

\begin{theorem}[Consistency and natural extension]\label{theo:rejectsets:consistency:and:natex}
An assessment $\assessment\subseteq\optsets$ is consistent if and only if\/ $\emptyset\notin\assessment$ and\/ $\set{0}\notin\setposi\group{\singposopts\cup\assessment}$.
In that case $\natexrejectset\group{\assessment}=\RS\group{\setposi\group{\singposopts\cup\assessment}}$.
\end{theorem}

\newpage

\section{Stuff I removed but need in the sequel}

This appendix also makes use of the following two additional operators:
\vspace{5pt}
\begin{equation*}
\SU\colon\rejectsets\to\rejectsets
\colon\rejectset\mapsto\SU\group{\rejectset}\coloneqq\cset{\optset\in\optsets}{(\exists\altoptset\in\rejectset)\altoptset\subseteq\optset}
\end{equation*}
and
\begin{equation*}
\RN\colon\rejectsets\to\rejectsets
\colon\rejectset\mapsto\RN\group{\rejectset}\coloneqq\cset{\optset\in\optsets}{(\exists\altoptset\in\rejectset)\altoptset\setminus\nonposopts\subseteq\optset\subseteq\altoptset}.
\vspace{5pt}
\end{equation*}
Applying them in sequence has the same effect as applying the operator $\RS$.

\begin{lemma}\label{lem:combineoperators}
Consider any set of desirable option sets $\rejectset\in\rejectsets$. 
Then
\begin{equation*}
\RS\group{\rejectset}
=\RN\group{\SU\group{\rejectset}}.
\end{equation*}
\end{lemma}

\begin{proof}
Consider any $\optset\in\RS\group{\rejectset}$, which means that there is some $\altoptset\in\rejectset$ such that $\altoptset\setminus\nonposopts\subseteq\optset$. 
Then $(\optset\cup\altoptset)\setminus\nonposopts\subseteq(\altoptset\setminus\nonposopts)\cup\optset=\optset\subseteq\optset\cup\altoptset$. 
Hence, if we let $\altoptset''\coloneqq\optset\cup\altoptset$, then $\altoptset''\setminus\nonposopts\subseteq\optset\subseteq\altoptset''$. 
Since $\altoptset\in\rejectset$ and $\altoptset\subseteq\altoptset''$ implies that $\altoptset''\in\SU\group{\rejectset}$, this allows us to conclude that $\optset\in\RN\group{\SU\group{\rejectset}}$.

Conversely, consider any $\optset\in\RN\group{\SU\group{\rejectset}}$, which means that there is some $\altoptset\in\SU\group{\rejectset}$ such that $\altoptset\setminus\nonposopts\subseteq\optset\subseteq\altoptset$. 
Then since $\altoptset\in\SU\group{\rejectset}$, there is some $\altoptset'\in\rejectset$ such that $\altoptset'\subseteq\altoptset$. 
Hence, we find that $\altoptset'\setminus\nonposopts\subseteq\altoptset\setminus\nonposopts\subseteq\optset$, which, since $\altoptset'\in\rejectset$, implies that $\optset\in\RS\group{\rejectset}$.
\end{proof}

\newpage
\section{Proof of Theorem~\ref{theo:rejectsets:consistency:and:natex}}

\begin{proposition}\label{prop:applying:posi}
For any set of desirable option sets $\rejectset\in\rejectsets$, $\setposi\group{\rejectset}$ satisfies Axiom~\ref{ax:rejects:cone}.
\end{proposition}

\begin{proof}[Proof of Proposition~\ref{prop:applying:posi}]
To prove that $\setposi\group{\rejectset}$ satisfies Axiom~\ref{ax:rejects:cone}, consider any $\optset,\altoptset\in\setposi\group{\rejectset}$ and, for all $\opt\in\optset$ and $\altopt\in\altoptset$, any $(\lambda_{\opt,\altopt},\mu_{\opt,\altopt})>0$. 
Then we need to prove that
\begin{equation*}
C
\coloneqq\cset{\lambda_{\opt,\altopt}\opt+\mu_{\opt,\altopt}\altopt}
{\opt\in\optset,\altopt\in\altoptset}\in\setposi\group{\rejectset}
\end{equation*}
Since $\optset,\altoptset\in\setposi\group{\rejectset}$, we know that there are $m,n\in\naturals$, $(\optset[1],\dots,\optset[m])\in\rejectset^m$ and $(\altoptset[1],\dots,\altoptset[n])\in\rejectset^n$ and, for all $\opt[1:m]\in\times_{k=1}^m\optset[k]$ and $\altopt[1:n]\in\times_{\ell=1}^n\altoptset[\ell]$, some choice of $\lambda^{\opt[1:m]}_{1:m}>0$ and $\mu^{\altopt[1:n]}_{1:n}>0$ such that
\begin{equation}\label{eq:applying:posi:first}
\optset
=\cset[\bigg]{\smashoperator[r]{\sum_{k=1}^m}
\lambda^{\opt[1:m]}_k\opt[k]}
{\opt[1:m]\in\times_{k=1}^m\optset[k]}
\text{ and }
\altoptset
=\cset[\bigg]{\smashoperator[r]{\sum_{\ell=1}^n}
\mu^{\altopt[1:n]}_\ell\altopt[\ell]}
{\altopt[1:n]\in\times_{\ell=1}^n\altoptset[\ell]}.
\end{equation}
For all $\opt[1:m]\in\times_{k=1}^m\optset[k]$ and $\altopt[1:n]\in\times_{\ell=1}^n\altoptset[\ell]$, we introduce the simplifying notation
\begin{equation*}
\aopt[{\opt[1:m]}]
\coloneqq\sum_{k=1}^m\lambda^{\opt[1:m]}_k\opt[k]
\text{ and }
\bopt[{\altopt[1:n]}]
\coloneqq\sum_{\ell=1}^n\mu^{\altopt[1:n]}_{\ell}\altopt[\ell],
\end{equation*}
so $\optset=\cset{\aopt[{\opt[1:m]}]}{\opt[1:m]\in \times_{k=1}^m\optset[k]}$ and $\altoptset=\cset{\bopt[{\altopt[1:n]}]}{\altopt[1:n]\in\times_{\ell=1}^n\altoptset[\ell]}$, and therefore
\begin{align*}
\altoptsettoo
&=\cset{\lambda_{\opt,\altopt}\opt+\mu_{\opt,\altopt}\altopt}{\opt\in\optset,\altopt\in\altoptset}\\
&=\cset[\bigg]{\lambda_{\aopt[{\opt[1:m]}],\bopt[{\altopt[1:n]}]}\aopt[{\opt[1:m]}]
+\mu_{\aopt[{\opt[1:m]}],\bopt[{\altopt[1:n]}]}\bopt[{\altopt[1:n]}]}
{\opt[1:m]\in\times_{k=1}^m\optset[k],\altopt[1:n]\in\times_{\ell=1}^n\altoptset[\ell]}.
\end{align*}
If we now introduce the notations
\begin{equation*}
\altoptsettoo[i]
\coloneqq
\begin{cases}
\optset[i] 
&\text{ if }1\leq i\leq m\\
\altoptset[i-m] 
&\text{ if }m+1\leq i\leq m+n
\end{cases}
\end{equation*}
and for any $\altopttoo[1:m+n]\in\times_{i=1}^{m+n}\altoptsettoo[i]$,
\begin{equation*}
\kappa_i^{\altopttoo[1:m+n]}
\coloneqq
\begin{cases}
\lambda_{\aopt[{\altopttoo[1:m]}],\bopt[{\altopttoo[m+1:m+n]}]}\lambda_i^{\altopttoo[1:m]} &\text{ if }1\leq i\leq m\\
\mu_{a_{\altopttoo[1:m]},b_{\altopttoo[m+1:m+n]}}\mu_{i-m}^{\altopttoo[m+1:m+n]} &\text{ if }m+1\leq i\leq m+n,
\end{cases}
\end{equation*}
where we used $\altopttoo[m+1:m+n]$ to denote the tuple $(\altopttoo[m+1],\dots,\altopttoo[m+n])$, then we find that
\begin{align*}
C
&=\cset[\bigg]{\sum_{i=1}^{m+n}\kappa_i^{\altopttoo[1:m+n]}\altopttoo[i]}
{\altopttoo[1:m+n]\in\times_{i=1}^{m+n}\altoptsettoo[i]}.
\end{align*}
Furthermore, since $(\lambda_{\aopt[{\altopttoo[1:m]}],\bopt[{\altopttoo[m+1:m+n]}]},\mu_{\aopt[{\altopttoo[1:m]}],\bopt[{\altopttoo[m+1:m+n]}]})>0$, $\lambda_{1:m}^{\altopttoo[1:m]}>0$ and $\mu_{1:n}^{\altopttoo[m+1:m+n]}>0$, it follows that also
\begin{equation*}
\kappa_{1:m+1}^{\altopttoo[1:m+n]}\coloneqq(\kappa_{1}^{\altopttoo[1:m+n]},\dots,\kappa_{m+1}^{\altopttoo[1:m+n]})>0.
\end{equation*}
Hence, we find that, indeed, $C\in\setposi\group{\rejectset}$.
\end{proof}

\begin{proposition}\label{prop:adding:supersets:gewijzigde:axiomas}
Consider any set of desirable option sets $\rejectset\in\rejectsets$.
Then\/ $\SU\group{\rejectset}$ satisfies Axiom~\ref{ax:rejects:mono}.
Moreover, if $\rejectset$ satisfies Axioms~\ref{ax:rejects:nozero}, \ref{ax:rejects:pos} and\/~\ref{ax:rejects:cone} and does not contain~$\emptyset$, then so does\/ $\SU\group{\rejectset}$.
\end{proposition}

\begin{proof}
For the first statement, consider any $\optset[1]\in\SU\group{\rejectset}$ and any $\optset[2]\in\optsets$ such that $\optset[1]\subseteq\optset[2]$.
Then there is some $\altoptset[1]\in\rejectset$ such that $\altoptset[1]\subseteq\optset[1]$, and therefore also $\altoptset[1]\subseteq\optset[2]$, whence indeed $\optset[2]\in\SU\group{\rejectset}$.

For the second statement, assume that $\rejectset$ satisfies Axioms~\ref{ax:rejects:nozero}, \ref{ax:rejects:pos} and~\ref{ax:rejects:cone}, and does not contain~$\emptyset$.

To prove that $\SU\group{\rejectset}$ satisfies Axiom~\ref{ax:rejects:pos}, simply observe that the operator $\SU$ never removes option sets from a set of desirable option sets, so the option sets $\set{\opt}$, $\opt\in\posopts$, that belong to $\rejectset$ by Axiom~\ref{ax:rejects:pos}, will also belong to the larger $\SU\group{\rejectset}$.

To prove that $\SU\group{\rejectset}$ satisfies Axiom~\ref{ax:rejects:cone}, consider any $\optset[1],\optset[2]\in\SU\group{\rejectset}$, meaning that there are $\altoptset[1],\altoptset[2]\in\rejectset$ such that $\altoptset[1]\subseteq\optset[1]$ and $\altoptset[2]\subseteq\optset[2]$.
For all $\opt\in\optset[1]$ and $\altopt\in\optset[2]$, choose some $(\lambda_{\opt,\altopt},\mu_{\opt,\altopt})>0$, and let 
\begin{equation*}
\optset
\coloneqq\cset{\lambda_{\opt,\altopt}\opt+\mu_{\opt,\altopt}\altopt}{\opt\in\optset[1],\altopt\in\optset[2]}.
\end{equation*}
We then need to prove that $\optset\in\SU\group{\rejectset}$. 
Since $\rejectset$ satisfies Axiom~\ref{ax:rejects:cone}, we infer from $\altoptset[1],\altoptset[2]\in\rejectset$ that
\begin{equation*}
\altoptset
\coloneqq\cset{\lambda_{\opt,\altopt}\opt+\mu_{\opt,\altopt}\altopt}
{\opt\in\altoptset[1],\altopt\in\altoptset[2]}\in\rejectset.
\end{equation*}
Since $\altoptset\subseteq\optset$, this implies that, indeed, $\optset\in\SU\group{\rejectset}$.

Finally, to prove that $\SU\group{\rejectset}$ satisfies Axiom~\ref{ax:rejects:nozero} and does not contain~$\emptyset$, assume \emph{ex absurdo} that $\set{0}\in\SU\group{\rejectset}$ or $\emptyset\in\SU\group{\rejectset}$. 
Then $\set{0}\in\rejectset$ or $\emptyset\in\rejectset$. 
In either case, we obtain a contradiction with the assumption that $\rejectset$ satisfies Axiom~\ref{ax:rejects:nozero} and does not contain $\emptyset$.
\end{proof}

\begin{proposition}\label{prop:ax:rejects:SU:equivalents}
$\SU\group{\rejectset}=\rejectset$ for any coherent set of desirable option sets $\rejectset\in\cohrejectsets$.
\end{proposition}

\begin{proof}
That $\rejectset\subseteq\SU\group{\rejectset}$ is an immediate consequence of the definition of the $\SU$ operator. 
The converse inclusion follows from the fact that $\rejectset$ is coherent and therefore satisfies Axiom~\ref{ax:rejects:mono}.
\end{proof}

\begin{proof}[Proof of Theorem~\ref{theo:rejectsets:consistency:and:natex}]
For notational simplicity, we will denote the set of desirable option sets $\setposi\group{\singposopts\cup\assessment}$ by $\rejectset[o]$.
We begin with the first statement.
First, assume that $\emptyset\notin\assessment$ and $\set{0}\notin\rejectset[o]$.
Observe that $\rejectset[o]$ satisfies Axiom~\ref{ax:rejects:pos} by construction and Axiom~\ref{ax:rejects:cone} by Proposition~\ref{prop:applying:posi}. 
$\rejectset[o]$ also satisfies Axiom~\ref{ax:rejects:nozero} because, by assumption, $\set{0}\notin\rejectset[o]$.
Furthermore, the assumption $\emptyset\notin\assessment$ implies that $\emptyset\notin\rejectset[o]$, and therefore it follows from Proposition~\ref{prop:adding:supersets:gewijzigde:axiomas} that $\SU\group{\rejectset[o]}$ satisfies Axioms~\ref{ax:rejects:nozero}, \ref{ax:rejects:pos}, \ref{ax:rejects:cone} and~\ref{ax:rejects:mono}, and that $\emptyset\notin\SU\group{\rejectset[o]}$, so we gather from Proposition~\ref{prop:removal:of:nonpositives:gewijzigde:axiomas} that $\rejectset[1]\coloneqq\RN\group{\SU\group{\rejectset[o]}}$ satisfies Axioms~\ref{ax:rejects:removezero}--\ref{ax:rejects:mono}.
Since $\rejectset[1]$ includes $\assessment$ [none of the operators $\setposi$, $\SU$ and $\RN$ remove option sets from their arguments, they only add new option sets], this implies that $\rejectset[1]\in\cohrejectsets\group{\assessment}$, and therefore, that $\assessment$ is consistent.

Next, assume that $\assessment$ is consistent, which means that $\cohrejectsets\group{\assessment}\neq\emptyset$. 
Consider any $\rejectset\in \cohrejectsets\group{\assessment}$, which means that $\rejectset$ is coherent and $\assessment\subseteq\rejectset$. 
Then $\singposopts\cup\assessment\subseteq\rejectset$ [use Axiom~\ref{ax:rejects:pos}] and therefore also $\rejectset[o]=\setposi\group{\singposopts\cup\assessment}\subseteq\setposi\group{\rejectset}=\rejectset$ [for the inclusion, use the definition of the $\setposi$ operator, and for the equality, use Proposition~\ref{prop:ax:rejects:cone:equivalents}]. 
Now assume \emph{ex absurdo} that $\set{0}\in\rejectset[o]$. 
Then also $\set{0}\in\rejectset$, which contradicts the coherence of $\rejectset$ [Axiom~\ref{ax:rejects:nozero}].
Hence, we find that $\set{0}\notin\rejectset[o]$. 
Finally, since $\rejectset$ is coherent, Axioms~\ref{ax:rejects:mono} and~\ref{ax:rejects:nozero} imply that $\emptyset\notin\rejectset$, which, since $\assessment\subseteq\rejectset$, implies that $\emptyset\notin\assessment$. 

For the second statement, assume that $\assessment$ is consistent, so we already know that $\emptyset\notin\assessment$ and $\set{0}\notin\rejectset[o]$.
We have to prove that $\natexrejectset\group{\assessment}=\RS\group{\rejectset[o]}$, or equivalently, due to Lemma~\ref{lem:combineoperators}, that $\natexrejectset\group{\assessment}=\rejectset[1]\coloneqq\RN\group{\SU\group{\rejectset[o]}}$.

We already know from the argumentation above that $\emptyset\notin\assessment$ and $\set{0}\notin\rejectset[o]$ implies that $\rejectset[1]\in\cohrejectsets\group{\assessment}$, and therefore also $\natexrejectset\group{\assessment}\subseteq\rejectset[1]$.
To prove the converse inclusion, consider any $\rejectset\in\cohrejectsets\group{\assessment}$.
Then as shown above, $\rejectset[o]\subseteq\rejectset$.
Hence also $\SU\group{\rejectset[o]}\subseteq\SU\group{\rejectset}=\rejectset$ [for the inclusion, use the definition of the $\SU$ operator, and for the equality, use Proposition~\ref{prop:ax:rejects:SU:equivalents}], and therefore also $\rejectset[1]=\RN\group{\SU\group{\rejectset[o]}}\subseteq\RN\group{\rejectset}=\rejectset$ [for the inclusion, use the definition of the $\RN$ operator, and for the equality, use Proposition~\ref{prop:ax:rejects:RN:equivalents}].
So $\rejectset[1]\subseteq\rejectset$. 
Since this is true for every $\rejectset\in\cohrejectsets\group{\assessment}$, and since the consistency of $\assessment$ implies that $\cohrejectsets\group{\assessment}\neq\emptyset$, we conclude that $\rejectset[1]\subseteq\natexrejectset\group{\assessment}$.
\end{proof}

\newpage

\section{Convex natural extension}

We can also do conservative reasoning under the convexity axiom.
To show this, we begin with the following theorem, whose proof is trivial, and which identifies an interesting substructure.

\begin{theorem}\label{theo:convex:conservative:inference}
Let $\set{\rejectset[i]}_{i\in I}$ be an arbitrary non-empty family of sets of desirable option sets, with intersection $\rejectset\coloneqq\bigcap_{i\in I}\rejectset[i]$.
If $\rejectset[i]$ is convex coherent for all $i\in I$, then so is $\rejectset$.
This implies that $\structure{\convcohrejectsets,\subseteq}$ is a complete meet-semilattice. 
\end{theorem}
\noindent
This result is important, as it allows us to a extend a partially specified set of desirable option sets to the most conservative convex coherent one that includes it.
This leads to the conservative inference procedure we will call convex natural extension.

\begin{definition}[Convex consistency and natural extension]
For any assessment $\assessment\subseteq\optsets$, let $\convcohrejectsets\group{\assessment}\coloneqq\cset{\rejectset\in\convcohrejectsets}{\assessment\subseteq\rejectset}$.
We call the assessment~$\assessment$ \emph{convex consistent} if\/ $\convcohrejectsets\group{\assessment}\neq\emptyset$, and we then call\/ $\convnatexrejectset\group{\assessment}\coloneqq\bigcap\convcohrejectsets\group{\assessment}$ the \emph{convex natural extension} of $\assessment$.
\end{definition}
\noindent
In other words: an assessment $\assessment$ is convex consistent if it can be extended to some convex coherent set of desirable option sets, and then its convex natural extension $\natexrejectset\group{\assessment}$ is the least informative such coherent set of desirable option sets.

Our final result provides a more `constructive' expression for this convex natural extension and a simpler criterion for convex consistency. 
In order to state it, we need to introduce yet another operator on---transformation of---$\rejectsets$. 
It is denoted by $\RS$, and allows us to add smaller option sets by removing from any option set positive combinations from some of its other elements:
\begin{equation*}
\RP\group{\rejectset}
\coloneqq\cset{\optset\in\optsets}{(\exists\altoptset\in\rejectset)\optset\subseteq\altoptset\subseteq\posi\group{\optset}}
\text{ for all $\rejectset\in\rejectsets$},
\end{equation*}

\begin{theorem}[Convex consistency and natural extension]\label{theo:rejectsets:convex:consistency:and:natex}
An assessment $\assessment\subseteq\optsets$ is convex consistent if and only if it is consistent, i.e.~if\/ $\emptyset\notin\assessment$ and\/ $\set{0}\notin\setposi\group{\singposopts\cup\assessment}$.
In that case $\convnatexrejectset\group{\assessment}=\RP\group{\natexrejectset\group{\assessment}}=\RP\group{\RS\group{\setposi\group{\singposopts\cup\assessment}}}$.
\end{theorem}

\newpage

\section{Proof of Theorem~\ref{theo:rejectsets:convex:consistency:and:natex}}

\begin{proposition}\label{prop:ax:rejects:RP:equivalents}
$\RP\group{\rejectset}=\rejectset$ for any convex coherent set of desirable option sets $\rejectset\in\convcohrejectsets$.
\end{proposition}

\begin{proof}
That $\rejectset\subseteq\RP\group{\rejectset}$ is an immediate consequence of the definition of the $\RP$ operator. 
For the converse inclusion, consider any $\optset\in\RP\group{\rejectset}$, which means that there is some $\altoptset\in\rejectset$ such that $\optset\subseteq\altoptset\subseteq\posi\group{\optset}$. 
But since $\rejectset$ satisfies Axiom~\ref{ax:rejects:removepositivecombinations}, this implies that also $\optset\in\rejectset$, whence indeed $\RP\group{\rejectset}\subseteq\rejectset$.
\end{proof}

\begin{proof}[Proof of Theorem~\ref{theo:rejectsets:convex:consistency:and:natex}]
We begin with the first statement, where it clearly suffices to prove that consistency implies convex consistency.
So assume that $\assessment$ is consistent, which means that there is some coherent set of desirable gamble sets $\rejectset\in\cohrejectsets$ such that $\assessment\subseteq\rejectset$.
The definition of $\RP$ allows us to extend this to $\assessment\subseteq\rejectset\subseteq\RP\group{\rejectset}$, and it follows from Proposition~\ref{prop:removal:of:posis:gewijzigde:axiomas} that $\RP\group{\rejectset}$ is convex coherent.

For the second statement, we will denote $\setposi\group{\singposopts\cup\assessment}$ by $\rejectset[o]$, $\natexrejectset\group{\assessment}=\RS\group{\rejectset[o]}$ by $\rejectset[1]$ and $\RP\group{\rejectset[1]}$ by $\rejectset[2]$, so we have to prove that $\convnatexrejectset\group{\assessment}=\rejectset[2]$.

It is clear from Theorem~\ref{theo:rejectsets:consistency:and:natex} that $\rejectset[1]$ is coherent and that $\assessment\subseteq\rejectset[1]$.
The definition of $\RP$ allows us to extend this to $\assessment\subseteq\rejectset[1]\subseteq\RP\group{\rejectset[1]}=\rejectset[2]$, and it follows from Proposition~\ref{prop:removal:of:posis:gewijzigde:axiomas} that $\rejectset[2]=\RP\group{\rejectset[1]}$ is convex coherent because $\rejectset[1]$ is coherent.
This already tells us that $\convnatexrejectset\group{\assessment}\subseteq\rejectset[2]$, so we now concentrate on the converse inclusion. 

Consider any convex coherent set of desirable option sets $\rejectset$ that includes $\assessment$, so $\rejectset\in\convcohrejectsets(\assessment)$.
Since the set of desirable option sets $\rejectset$ is in particular also coherent, Theorem~\ref{theo:rejectsets:consistency:and:natex} guarantees that $\rejectset[1]=\natexrejectset\group{\assessment}\subseteq\rejectset$.
Hence also $\rejectset[2]=\RP\group{\rejectset[1]}\subseteq\RP\group{\rejectset}=\rejectset$ [for the inclusion, use the definition of the $\RP$ operator, and for the second equality, use~Proposition~\ref{prop:ax:rejects:RP:equivalents}].
So $\rejectset[2]\subseteq\rejectset$ and therefore indeed also $\rejectset[2]\subseteq\convnatexrejectset(\assessment)$.
\end{proof}


\section{Some material on sets of desirable options that I removed from the main text for now}

\begin{theorem}[Desirability representation theorem \cite{couso2011,cooman2011b,cooman2010,quaeghebeur2012:itip,quaeghebeur2015:statement}]\label{theo:desirsets:maximality:and:representation}
Consider any coherent set of desirable options $\desirset\in\cohdesirsets$, then it is dominated by some maximal coherent set of desirable options: $\maxdesirsets(\desirset)\coloneqq\cset{\maxdesirset\in\maxdesirsets}{\desirset\subseteq\maxdesirset}\neq\emptyset$.
Moreover, $\desirset=\bigcap\maxdesirsets(\desirset)$.
\end{theorem}

Often in practical situations, a subject will typically not specify a full-fledged coherent set of desirable options sets, but will only provide some partial \emph{assessment} $\assessment\subseteq\opts$.
Such an assessment consists of those for which the subject is comfortable assessing that they are desirable.
The task before us, now, consists in extending this assessment~$\assessment$ to a coherent set of desirable options in a manner that is as conservative---or uninformative---as possible.
This is the essence of \emph{conservative inference}.

We say that a set of desirable options $\desirset[1]$ is less informative than (or rather, at most as informative as) a set of desirable options $\desirset[2]$, when $\desirset[1]\subseteq\desirset[2]$: a subject whose beliefs are represented by $\desirset[2]$ has more (or rather, at least as many) desirable options than a subject with beliefs represented by $\desirset[1]$.
The resulting partially ordered set $\structure{\desirsets,\subseteq}$ is a complete lattice with intersection as infimum and union as supremum.
The following theorem identifies an interesting substructure.

\begin{theorem}[\cite{cooman2011b,cooman2010,quaeghebeur2012:itip,quaeghebeur2015:statement}]\label{theo:conservative:inference:for:desirsets}
Let $\set{\desirset[i]}_{i\in I}$ be an arbitrary non-empty family of sets of desirable options, with intersection $\desirset\coloneqq\bigcap_{i\in I}\desirset[i]$.
If $\desirset[i]$ is coherent for all $i\in I$, then so is $\desirset$.
This implies that $\structure{\cohdesirsets,\subseteq}$ is a complete meet-semilattice.
\end{theorem}
\noindent
This result is practically very important, as it allows us to a extend a partially specified set of desirable options to the most conservative coherent one that includes it.
It leads to directly to the following conservative inference procedure, called \emph{natural extension}.

\begin{definition}[Consistency and natural extension \cite{cooman2011b,cooman2010,quaeghebeur2012:itip,quaeghebeur2015:statement}]
For any assessment $\assessment\subseteq\opts$, let $\cohdesirsets\group{\assessment}\coloneqq\cset{\desirset\in\cohdesirsets}{\assessment\subseteq\desirset}$.
We call the assessment~$\assessment$ \emph{consistent} if\/ $\cohdesirsets\group{\assessment}\neq\emptyset$, and we then call\/ $\natexdesirset\group{\assessment}\coloneqq\bigcap\cohdesirsets\group{\assessment}$ the \emph{natural extension} of $\assessment$.
\end{definition}
\noindent
In other words: an assessment $\assessment$ is consistent if it can be extended to some coherent set of desirable options, and then its natural extension $\natexrejectset\group{\assessment}$ is the least informative such coherent set of desirable options.

A constructive expression for this natural extension and a simpler criterion for consistency are the subject of the following theorem.

\begin{theorem}[Consistency and natural extension \cite{cooman2011b,cooman2010,quaeghebeur2012:itip,quaeghebeur2015:statement}]\label{theo:desirsets:consistency:and:natex}
An assessment $\assessment\subseteq\optsets$ is consistent if and only if\/ $0\notin\posi\group{\posopts\cup\assessment}$.
In that case $\natexdesirset\group{\assessment}=\posi\group{\posopts\cup\assessment}$.
\end{theorem}

\newpage

\section{Additional results on closure and natural extension}\label{app:more:on:closure}

\begin{corollary}\label{corol:rejectsets:maximality}
For any coherent set of desirable option sets $\rejectset\in\cohrejectsets$, $\cset{\maxdesirset\in\maxdesirsets}{\rejectset\subseteq\rejectset[\maxdesirset]}\neq\emptyset$.
\end{corollary}
\begin{proof}
*** to be added; follows from Theorem~\ref{theo:rejectsets:maximality} ***
\end{proof}

\begin{theorem}[Consistency, natural extension and representation]\label{theo:rejectsets:consistency:natex:and:representation}
An assessment $\assessment\subseteq\optsets$ is consistent if and only if there is some coherent set of desirable options $\desirset\in\cohdesirsets$ such that $\assessment\subseteq\rejectset[\desirset]$.
In that case $\natexrejectset\group{\assessment}=\bigcap\cset{\rejectset[\desirset]}{\desirset\in\cohdesirsets\text{ and }\assessment\subseteq\rejectset[\desirset]}$.
\end{theorem}

\begin{proof}
For the first statement, assume first that $\assessment$ is consistent, so it is by definition included in some coherent set of desirable option sets $\rejectset\in\cohrejectsets$: $\assessment\subseteq\rejectset$.
Corollary~\ref{corol:rejectsets:maximality} then guarantees that there is some coherent set of desirable options $\desirset\in\cohdesirsets$ such that $\rejectset\subseteq\rejectset[\desirset]$, and therefore also $\assessment\subseteq\rejectset[\desirset]$.

Conversely, assume that there is some coherent set of desirable options $\desirset\in\cohdesirsets$ such that $\assessment\subseteq\rejectset[\desirset]$, then $\rejectset$ is consistent because $\rejectset[\desirset]$ is a coherent set of desirable option sets, by Proposition~\ref{prop:fromCohDtoCohK}.

For the second statement, simply observe that 
\begin{align*}
\natexrejectset\group{\assessment}
&=\natexrejectset\group{\natexrejectset\group{\assessment}}\\
&=\bigcap\cset{\rejectset[\desirset]}
{\desirset\in\cohdesirsets\text{ and }\natexrejectset\group{\assessment}\subseteq\rejectset[\desirset]}
=\bigcap\cset{\rejectset[\desirset]}
{\desirset\in\cohdesirsets\text{ and }\assessment\subseteq\rejectset[\desirset]}.
\end{align*}
To see that the first equality holds, recall that $\natexrejectset\group{\assessment}$ is coherent, and that the natural extension of a coherent $\rejectset\in\cohrejectsets$ is $\rejectset$ itself.
The second equality is a direct consequence of Theorem~\ref{theo:rejectsets:representation}, because $\natexrejectset\group{\assessment}$ is coherent.
For the third equality, we need to prove that $\natexrejectset\group{\assessment}\subseteq\rejectset[\desirset]\ifandonlyif\assessment\subseteq\rejectset[\desirset]$.
Indeed, if $\natexrejectset\group{\assessment}\subseteq\rejectset[\desirset]$ then also $\assessment\subseteq\rejectset[\desirset]$, because $\assessment\subseteq\natexrejectset\group{\assessment}$.
Conversely, if $\assessment\subseteq\rejectset[\desirset]$ then also $\natexrejectset\group{\assessment}\subseteq\natexrejectset\group{\rejectset[\desirset]}=\rejectset[\desirset]$, where the equality follows because $\rejectset[\desirset]$ is coherent.
\end{proof}

Let us denote by $\Phi_\assessment$ the collection of all maps $\phi$ on $\assessment$ such that $\phi\group{\optset}\in\optset$ for all $\optset\in\assessment$.

\begin{theorem}\label{theo:breakdown:of:natural:extension}
Consider any assessment $\assessment\subseteq\optsets$, then
\begin{equation*}
\natexrejectset\group{\assessment}
=\bigcap_{\phi\in\Phi_\assessment}
\natexrejectset\group{\cset{\set{\phi\group{\optset}}}{\optset\in\assessment}}.
\end{equation*}
\end{theorem}

\begin{proof}
Observe the following chain of equivalences, where the first and the last equivalences follow from Theorem~\ref{theo:rejectsets:consistency:natex:and:representation}:
\begin{align*}
\altoptset\in\natexrejectset\group{\assessment}
&\ifandonlyif\group{\forall\desirset\in\cohdesirsets}
\group[\Big]{\group[\big]{\group{\forall\optset\in\assessment}\optset\in\rejectset[\desirset]}
\then\altoptset\in\rejectset[\desirset]}\\
&\ifandonlyif\group{\forall\desirset\in\cohdesirsets}
\group[\Big]{\group[\big]{\group{\forall\optset\in\assessment}\optset\cap\desirset\neq\emptyset}
\then\altoptset\in\rejectset[\desirset]}\\
&\ifandonlyif\group{\forall\desirset\in\cohdesirsets}
\group[\Big]{\group[\big]{\group{\exists\phi\in\Phi_\assessment}\group{\forall\optset\in\assessment}\phi(A)\in\desirset}
\then\altoptset\in\rejectset[\desirset]}\\
&\ifandonlyif\group{\forall\desirset\in\cohdesirsets}\group{\forall\phi\in\Phi_\assessment}
\group[\Big]{\group[\big]{\group{\forall\optset\in\assessment}\phi(A)\in\desirset}
\then\altoptset\in\rejectset[\desirset]}\\
&\ifandonlyif\group{\forall\phi\in\Phi_\assessment}\group{\forall\desirset\in\cohdesirsets}
\group[\big]{\group{\cset{\phi\group{\optset}}{\optset\in\assessment}\subseteq\rejectset[\desirset]}
\then\altoptset\in\rejectset[\desirset]}\\
&\ifandonlyif\group{\forall\phi\in\Phi_\assessment}
\altoptset\in\natexrejectset\group{\cset{\phi\group{\optset}}{\optset\in\assessment}}.\qedhere
\end{align*}
\end{proof}

We now introduce the following notation, which we have already used before when introducing $\singposopts$.
Consider any set $W$, then $\tosingletons{W}\coloneqq\cset{\set{w}}{w\in W}$.
It allows us to rewrite Theorem~\ref{theo:breakdown:of:natural:extension} as $\natexrejectset\group{\assessment}=\bigcap_{\phi\in\Phi_\assessment}\natexrejectset\group{\tosingletons{\phi\group{\assessment}}}$, so we see that $\assessment$ is consistent if and only if $\tosingletons{\phi\group{\assessment}}$ is consistent for some $\phi\in\Phi_\assessment$.

This suggests that we can completely characterise (consistency and) the natural extension operator by its effect $\natexrejectset\group{\tosingletons{\assessment}}$ on sets $\tosingletons{\assessment}$, where $\assessment\subseteq\opts$ is now any set of \emph{options}.

\begin{theorem}\label{theo:consistency:and:natural:extension:on:singletons}
Consider any options assessment $\assessment\subseteq\opts$.
Then
\begin{enumerate}[label=\textup{(\roman*)},leftmargin=*]
\item the option sets assessment $\tosingletons{\assessment}$ is consistent if and only if the options assessment $\assessment$ is consistent, meaning that\/ $0\notin\posi\group{\posopts\cup\assessment}$;
\item if $\assessment$ is consistent, then $\natexrejectset\group{\tosingletons{\assessment}}=\SU\group{\tosingletons{\natexdesirset\group{\assessment}}}$.
\end{enumerate}
\end{theorem}

\begin{proof}
For the first statement, we recall from Theorem~\ref{theo:rejectsets:consistency:and:natex} that, since clearly $\emptyset\notin\tosingletons{\assessment}$, the consistency of the option sets assessment $\tosingletons{\assessment}$ is equivalent to $\set{0}\notin\setposi\group{\singposopts\cup\tosingletons{\assessment}}=\tosingletons{\posi\group{\posopts\cup\assessment}}$, where the equality follows from Lemma~\ref{lem:posi:tosingletons}.
This is equivalent to $0\notin\posi\group{\posopts\cup\assessment}$, and Theorem~\ref{theo:desirsets:consistency:and:natex} tells us is that this is, in turn, indeed equivalent to the consistency of the options assessment $\assessment$.

For the second statement, assume that the options assessment $\assessment$ is consistent, and therefore the option sets assessment $\tosingletons{\assessment}$ is consistent as well. 
But then
\begin{align*}
\natexrejectset\group{\tosingletons{\assessment}}
&=\RN\group{\SU\group{\setposi\group{\singposopts\cup\tosingletons{\assessment}}}}
=\RN\group{\SU\group{\tosingletons{\posi\group{\posopts\cup\assessment}}}}\\
&=\SU\group{\tosingletons{\posi\group{\posopts\cup\assessment}}}
=\SU\group{\tosingletons{\natexdesirset{\group{\assessment}}}},
\end{align*}
where the first equality follows from Theorem~\ref{theo:rejectsets:consistency:and:natex} and Lemma~\ref{lem:combineoperators}, the second one from Lemma~\ref{lem:posi:tosingletons}, and the fourth one from Theorem~\ref{theo:desirsets:consistency:and:natex}.
So it remains to explain where the third equality comes from: since $\assessment$ is consistent, $\posi\group{\posopts\cup\assessment}\cap\nonposopts=\emptyset$, so removing non-positive options from some superset of a singleton in  $\tosingletons{\posi\group{\posopts\cup\assessment}}$ can only result in a smaller superset of this same singleton.
\end{proof} 

\begin{lemma}\label{lem:posi:tosingletons}
For any options assessment $\assessment\subseteq\opts$, $\setposi\group{\tosingletons{\assessment}}=\tosingletons{\posi\group{\assessment}}$.
\end{lemma}

\begin{proof}
Recall from Equation~\eqref{eq:setposi:appendix} that
\begin{align*}
\setposi\group{\tosingletons{\assessment}}
=\bigg\{
\bigg\{
\sum_{k=1}^n\lambda_{k}^{\opt[1:n]}\opt[k]
\colon
\opt[1:n]\in\times_{k=1}^n\optset[k]
\bigg\}
\colon
&n\in\naturals,(\optset[1],\dots,\optset[n])\in\group{\tosingletons{\assessment}}^n,\\[-11pt]
&\big(\forall\opt[1:n]\in\times_{k=1}^n\optset[k]\big)\,\lambda_{1:n}^{\opt[1:n]}>0
\bigg\},
\end{align*}
so if we remember that the option sets in $\tosingletons{\assessment}$ are singletons, this indeed simplifies to
\begin{equation*}
\setposi\group{\tosingletons{\assessment}}
=\cset[\bigg]{\set[\bigg]{\sum_{k=1}^n\lambda_{k}^{\opt[1:n]}\opt[k]}}
{\opt[1:n]\in\assessment^n,\,\lambda_{1:n}^{\opt[1:n]}>0,\,n\in\naturals}
=\tosingletons{\posi\group{\assessment}}.\qedhere
\end{equation*}
\end{proof}

\newpage
\section{On choice functions for horse lotteries}\label{app:horse:lotteries}
The following discussion is inspired by, but goes beyond, a section in Arthur Van Camp's PhD thesis \cite{2017vancamp:phdthesis}, where he shows that an idea by Zaffalon and Miranda \cite[Section~4]{zaffalon2017:incomplete:preferences} can be extended from a desirability to a choice function context.

\subsection{Vector-valued gambles and horse lotteries}
We will now focus on a particular type of option space, made up of so-called \emph{vector-valued gambles}.
We consider a non-empty space $\states$ of possible values of a variable $X$ whose value a subject is uncertain about.
We also consider a non-empty \emph{finite} set $\rewards$ of rewards, and will single out an element $\rho$ of that set for special attention.
We will denote the reduced reward set $\rewards\setminus\set{\rho}$ by $\rrewards$.

We denote by $\gblsonstates$ the set of all bounded real-valued maps---we also call such bounded maps \emph{gambles}---on $\states$, and similarly for $\gblsonstatesandrewards$ and $\gblsonstatesandrrewards$.
These sets are linear spaces, closed under point-wise addition of gambles, and point-wise multiplication of real numbers with gambles. 
They come with a typical and often used background ordering $\gblgt$, which goes back to the point-wise ordering of real-valued functions, denoted in this paper by `$\gblgteq$'.
For instance,
\begin{equation*}
\gbl\gblgt\altgbl
\ifandonlyif\group{\gbl\gblgteq\altgbl\text{ and }\gbl\neq\altgbl}
\ifandonlyif\group[\big]{\group{\forall x\in\states}\gbl(x)\geq\altgbl(x)\text{ and }\gbl\neq\altgbl},
\text{ for any $\gbl,\altgbl\in\gblsonstates$,}
\end{equation*}
and similarly for gambles on $\statesandrewards$ and $\statesandrrewards$.

Now consider any element $\gbl$ of $\gblsonstatesandrewards$---any gamble on $\statesandrewards$. 
For any $x\in\states$, the partial map $\gbl(x,\cdot)$ is an element of the linear space $\gblsonrewards$, and $\gbl$ can therefore be considered as a vector-valued gamble on $\states$.
Observe that now, with obvious notation overload,
\begin{equation*}
\gbl\gblgteq\altgbl
\ifandonlyif\group{\forall(x,r)\in\statesandrewards}\gbl(x,r)\geq\altgbl(x,r)
\ifandonlyif\group{\forall x\in\states}\gbl(x,\cdot)\gblgteq\altgbl(x,\cdot),
\end{equation*}
and therefore 
\begin{equation*}
\gbl\gblgt\altgbl
\ifandonlyif\group[\big]{\group{\forall x\in\states}\gbl(x,\cdot)\gblgteq\altgbl(x,\cdot)
\text{ and } 
\group{\exists x\in\states}\gbl(x,\cdot)\neq\altgbl(x,\cdot)}
\end{equation*}
so `$\gblgteq$' and `$\gblgt$' can also be seen as a point-wise---or rather state-wise---orderings of vector-valued gambles.

Horse lotteries are special vector-valued gambles:

\begin{definition}
We call a horse lottery $\hl$ on $\states$ with reward set $\rewards$ any map from $\statesandrewards$ to $\unit$ such that for all $x\in\states$, the partial map $\hl(x,\cdot)$ is a probability mass function on $\rewards$:
\begin{equation*}
(\forall x\in\states)
\group[\bigg]{\sum_{r\in\rewards}\hl(x,r)=1
\text{ and }
(\forall r\in\rewards)\hl(x,r)\geq0)}.
\end{equation*}
We collect all the horse lotteries on $\states$ with reward set $\rewards$ in the set $\hlsonstatesandrewards$, which is also denoted more simply by $\hls$ when it is clear from the context what the possibility space $\states$ and the reward set $\rewards$ are.
\end{definition}
Clearly, $\hls\subseteq\gblsonstatesandrewards$.
In contrast with $\gblsonstatesandrewards$, $\hls$ is not a linear space, but typically only closed under convex combinations.
This makes working with horse lotteries mathematically much more cumbersome than working with gambles---or even more generally, options. It should be clear why we prefer to work with gambles and options.
In what follows, we will build a mathematical interface between horse lotteries and our approach using abstract options, in order to find out what are the consequences of our results when translated to a horse lotteries context.

We introduce a linear \emph{gamblifier} map $\gamblifier$ from $\gblsonstatesandrewards$ to $\gblsonstatesandrrewards$ by letting
\begin{equation*}
\gamblifier(h)(x,r)\coloneqq h(x,r)\text{ for all $x\in\states$ and all $r\in\rrewards$},
\end{equation*}
so $\gamblifier(h)$ is the restriction of $h$ to the set $\statesandrrewards$, and is therefore clearly onto.
Interestingly, $\gamblifier$ is a \emph{bijection} between $\hlsonstatesandrewards$ and $\gamblifier(\gblsonstatesandrewards)$:

\begin{proposition}\label{prop:gamblifier:one:to:one}
$\gamblifier$ is one-to-one on $\hlsonstatesandrewards$. 
\end{proposition}
\begin{proof}
Consider arbitrary $\hl,\althl\in\hlsonstatesandrewards$, and assume that $\gamblifier\group{\hl}=\gamblifier\group{\althl}$.
This already tells us that $\hl(\cdot,r)=\althl(\cdot,r)$ for all $r\in\rrewards$.
Moreover,
\begin{align*}
\hl(\cdot,\rho)
=1-\sum_{r\in\rrewards}\hl(\cdot,r)
=1-\sum_{r\in\rrewards}\althl(\cdot,r)
=\althl(\cdot,\rho),
\end{align*}
and therefore, indeed, $\hl=\althl$.
\end{proof}

We will assume that there is some \emph{background ordering} $\hlgt$ on $\hls$ for which we will that it should satisfy the \emph{mixture independence property}, which for a general binary ordering $\hlsetgt$ on $\hls$ is given by:
\begin{enumerate}[label=$\protect{\hlprefgt[{\arabic*}]}$.,ref=$\protect{\hlprefgt[{\arabic*}]}$,leftmargin=*,start=0]
\item\label{ax:hlprefgt:sure:thing:principle} $\hlprefgt$ is a strict ordering---an irreflexive and transitive binary relation---on $\hls$ such that $\hl\hlprefgt\althl\ifandonlyif\alpha\hl+(1-\alpha)\althltoo\hlprefgt\alpha\althl+(1-\alpha)\althltoo$ for all $\hl,\althl,\althltoo\in\hls$ and $\alpha\in(0,1]$.
\end{enumerate}

\begin{runningexample}[A simple background ordering for horse lotteries]
Assume that the rewards are partially ordered, with a worst reward $\bot$ that is strictly dominated by all other rewards.
The ordering $\hlgt[\bot]$ on $\hls$ is defined by
\begin{align}
\hl\hlgt[\bot]\althl
&\ifandonlyif\group{\forall(x,r)\in\statesandrrewards[\bot]}\hl(x,r)\geq\althl(x,r)\text{ and }\hl\neq\althl\notag\\
&\ifandonlyif\gamblifier[\bot](\hl)\gblgt\gamblifier[\bot](\althl),\label{eq:background:ordering:hls}
\end{align}
so this background ordering~$\hlgt[\bot]$ is derived from the usual background ordering~$\gblgt$ on gambles on $\statesandrrewards[\bot]$.
It is the specific ordering considered by Zaffalon and Miranda \cite{zaffalon2017:incomplete:preferences} and Van Camp~\cite{2017vancamp:phdthesis}.
\qed
\end{runningexample}

\stilltodo{Give other simple examples for such a background ordering?}

\subsection{Seidenfeld et al.'s coherence axioms for rejection functions on horse lotteries}
Seidenfeld et al.~\cite{seidenfeld2010} consider rejection functions on horse lotteries, rather than the more abstract options we are considering there. 
These \emph{horse lotteries} are maps $\hlrejectfun\colon\hlsets\to\hlsets$ such that $\hlrejectfun\group{\hlset}\subseteq\hlset$ for all $\hlset\in\hlsets$, where $\hlsets$ is the set of all finite\footnote{Seidenfeld et al.~\cite{seidenfeld2010} also allow infinite closed sets of horse lotteries as arguments of rejection functions, but we see no need for this extra complication here.} subsets of $\hls$.
The axioms they impose on such rejection functions can be summarised as follows:
\begin{enumerate}[label=$\mathrm{S}^\ast_{\arabic*}$.,ref=$\mathrm{S}^\ast_{\arabic*}$,leftmargin=*,start=0]
\item\label{ax:seidenfeld:hlrejectfun:sure:thing:principle}
$\hlset\hlsetlt\althlset\ifandonlyif\alpha\hlset+(1-\alpha)\hl\hlsetlt\alpha\althlset+(1-\alpha)\hl$ for all $\hlset,\althlset\in\hlsets$, $\hl\in\hls$ and $\alpha\in(0,1]$;
\item\label{ax:seidenfeld:hlrejectfun:not:everything:rejected}
$\hlrejectfun\group{\emptyset}=\emptyset$ and $\hlrejectfun\group{\hlset}\neq\hlset$ for all $\hlset\in\hlsets\setminus\set{\emptyset}$;
\item\label{ax:seidenfeld:hlrejectfun:pos} for all $\hlset$ in $\hlsets$, all $\hl[1],\hl[2]\in\hls$ such that $\hl[1]\group{\cdot,\bestreward}\leq\hl[2]\group{\cdot,\bestreward}$ and $\hl[1]\group{\cdot,r}=\hl[2]\group{\cdot,r}=0$ for all $r\in\rewards\setminus\set{\worstreward,\bestreward}$, and all $\hl\in\hls\setminus\set{\hl[1],\hl[2]}$:
\begin{enumerate}[noitemsep,leftmargin=*,label=\upshape\alph*.,ref=\theenumi\upshape\alph*]
\item\label{ax:seidenfeld:hlrejectfun:pos:a}
if~$\hl[2]\in\hlset$ and~$\hl\in\hlrejectfun\group{\set{\hl[1]}\cup\hlset}$ then $\hl\in\hlrejectfun\group{\hlset}$;
\item\label{ax:seidenfeld:hlrejectfun:pos:b}
if~$\hl[1]\in\hlset$ and~$\hl\in\hlrejectfun\group{\hlset}$ then $\hl\in\hlrejectfun\group{\set{\hl[2]}\cup\hlset\setminus\set{\hl[1]}}$;
\end{enumerate}
\item\label{ax:seidenfeld:hlrejectfun:aizerman}
if\/ $\althlset\subseteq\hlrejectfun\group{\althlsettoo}$ and $\hlset\subseteq\althlset$ then\/ $\althlset\setminus\hlset\subseteq\hlrejectfun\group{\althlsettoo\setminus\hlset}$, for all $\hlset,\althlset,\althlsettoo\in\hlsets$;
\item\label{ax:seidenfeld:hlrejectfun:senalpha}
if $\hl\in\hlrejectfun\group{\hlset}$ and $\hlset\subseteq\althlset$, then also $\hl\in\hlrejectfun\group{\althlset}$, for all $\hl\in\hls$ and $\hlset,\althlset\in\hlsets$;
\item\label{ax:seidenfeld:hlrejectfun:removing:convex:combinations}
if $\hlset\subseteq\althlset\subseteq\chull\group{\hlset}$ then $\rejectfun\group{\althlset}\cap\hlset\subseteq\rejectfun\group{\althlset}$, for all $\hlset,\althlset\in\hlsets$.
\item\label{ax:seidenfeld:hlrejectfun:archimedean} for all $\hlset$, $\althlset$, $\althlsettoo$, $\althlset[n]$ and $\althlsettoo[n]$ in $\hlsets$ such that the sequence $\althlset[n]$ converges point-wise to~$\althlset$ and the sequence $\althlsettoo[n]$ converges point-wise to~$\althlsettoo$:
\begin{enumerate}[noitemsep,leftmargin=*,label=\upshape\alph*.,ref=\theenumi\upshape\alph*]
\item\label{ax:seidenfeld:hlrejectfun:archimedean:left}
if $\group{\forall n\in\naturals}\althlsettoo[n]\hlsetlt\althlset[n]$ and $\althlset\hlsetlt\hlset$ then $\althlsettoo\hlsetlt\hlset$;
\item\label{ax:seidenfeld:hlrejectfun:archimedean:right}
if $\group{\forall n\in\naturals}\althlsettoo[n]\hlsetlt\althlset[n]$ and $\hlset\hlsetlt\althlsettoo$ then $\hlset\hlsetlt\althlset$,
\end{enumerate}
\end{enumerate}
In these expressions, $\hlsetlt$ is the binary relation on $\hlsets$ defined by
\begin{equation*}
\hlset\hlsetlt\althlset\ifandonlyif\hlset\subseteq\hlrejectfun\group{\hlset\cup\althlset}
\text{ for all $\hlset,\althlset\in\hlsets$},
\end{equation*}
and it is assumed that the reward set is partially ordered with a top $\bestreward$ and bottom $\worstreward$.
Axiom~\ref{ax:seidenfeld:hlrejectfun:sure:thing:principle} is a mixture independence condition, which goes back to Savage's Sure Thing Principle \cite{savage1972,seidenfeld2010}.
Axiom~\ref{ax:seidenfeld:hlrejectfun:senalpha} is Sen's condition~$\alpha$ \cite{sen1971,sen1977} and Axiom~\ref{ax:seidenfeld:hlrejectfun:aizerman} is Aizermann's condition \cite{aizerman1985}.
Axiom~\ref{ax:seidenfeld:hlrejectfun:archimedean} is an Archimedeanity condition.

The reader will have observed that we make no notational distinction between the ordering $\hlsetgt$ on $\hlsets$ and its restriction on $\hls$, defined by $\hl\hlprefgt\althl\ifandonlyif\set{\hl}\hlsetgt\set{\althl}$, which, as a consequence of Axiom~\ref{ax:seidenfeld:hlrejectfun:sure:thing:principle} satisfies our Axiom~\ref{ax:hlprefgt:sure:thing:principle}.

\subsection{From horse lotteries to options and back}\label{sec:from:hls:to:opts:and:back}
We now allow ourselves to be inspired by the discussion in Sections~\ref{sec:choice:functions:and:interpretation}--\ref{sec:back:to:choice}, and in particular by Equation~\eqref{eq:interpretation:rejectfuns:after:irreflexivity:and:additivity:intermsof:K:Gert} to try and represent a rejection function $\hlrejectfun$ on horse lotteries by corresponding sets of desirable option sets $\rejectset$ in some, appropriately chosen, linear option space~$\opts$.
In other words, we would like to find some ordered linear space $\opts$ and a set of option sets $\rejectset\subseteq\optsets(\opts)$ such that
\begin{equation}\label{eq:representation:hlrejectfuns:tentative}
\group{\forall\hlset\in\hlsets}
\group{\forall\hl\in\hls}
\group[\big]{\hl\in\hlrejectfun\group{\hlset\cup\set{\hl}}
\ifandonlyif\hlset-\hl\in\rejectset}.
\end{equation}
Before we try and find out if---or when---this is at all possible, it will help us to identify the appropriate option space $\opts$ for doing this.
It is clear from looking at Equation~\eqref{eq:representation:hlrejectfuns:tentative} that this option space $\opts$ will have to contain differences $\althl-\hl$ for arbitrary $\hl,\althl\in\hls$, so $\hls-\hls\subseteq\opts$.
We will therefore try and represent rejection function on horse lotteries in the smallest linear subspace of $\gblsonstatesandrewards$ that includes $\hls-\hls=\cset{\althl-\hl}{\hl,\althl\in\hls}$:
\begin{equation*}
\difspaceonstatesandrewards
\coloneqq\linspan\group{\hls-\hls}
\subseteq\gblsonstatesandrewards,
\end{equation*}
also denoted more simply by $\difspace$ when it is clear from the context what the possibility space~$\states$ and the reward set~$\rewards$ are.
This option space has an simpler expression:

\begin{proposition}\label{prop:difspace}
$\difspace=\cset{\lambda(\althl-\hl)}{\lambda\in\posreals\text{ and }\hl,\althl\in\hls}$.
\end{proposition}
\begin{proof}
It is clear that $\difspace$ includes $\cset{\lambda(\althl-\hl)}{\lambda\in\posreals\text{ and }\hl,\althl\in\hls}$, because it is by definition closed under taking linear combinations of vectors. 
So it remains to prove that, conversely, $\difspace\subseteq\cset{\lambda(\althl-\hl)}{\lambda\in\posreals\text{ and }\hl,\althl\in\hls}$.
Any element $\gbl$ of $\difspace$ can be written as
\begin{equation*}
\gbl=\sum_{k=1}^n\lambda_k\group{\althl[k]-\hl[k]}
\text{ with $\naturals$, $\lambda_k\in\reals$ and $\hl[k],\althl[k]\in\hls$}.
\end{equation*} 
First, observe that we can assume without loss of generality that all $\lambda_k>0$---otherwise, simply switch the roles of $\hl[k]$ and $\althl[k]$.
It is now a matter of some algebra and straightforward verification that
\begin{equation*}
\gbl
=\group{\lambda_1+\dots+\lambda_n}
\group[\bigg]{\sum_{k=1}^n\frac{\lambda_k}{\lambda_1+\dots+\lambda_n}\althl[k]
-\sum_{k=1}^n\frac{\lambda_k}{\lambda_1+\dots+\lambda_n}\hl[k]},
\end{equation*}
and because $\lambda\coloneqq\lambda_1+\dots+\lambda_n>0$, $\hl\coloneqq\sum_{k=1}^n\frac{\lambda_k}{\lambda_1+\dots+\lambda_n}\hl[k]\in\hls$ and $\althl\coloneqq\sum_{k=1}^n\frac{\lambda_k}{\lambda_1+\dots+\lambda_n}\althl[k]\in\hls$, we see that, indeed $\gbl\in\cset{\lambda(\althl-\hl)}{\lambda\in\posreals\text{ and }\hl,\althl\in\hls}$.
\end{proof}
\noindent Elements $\gbl$ of $\difspace$ are positive multiples of differences between horse lotteries, but this `representation' is not unique: there may be more than one way to write $\gbl$ as such a rescaled difference.

The following result is essentially due to Zaffalon and Miranda \cite{zaffalon2017:incomplete:preferences}. 
It shows that working with the linear option space $\difspace$ is essentially equivalent to working with the linear option space $\gblsonstatesandrrewards$.
This is sufficient reason for Zaffalon and Miranda to turn away from $\difspace$, and work with $\gblsonstatesandrrewards$ instead.
We prefer to stick with $\difspace$ for a while longer, because it does not require us to single out a `special' reward $\rho$.

\begin{proposition}
$\gamblifier$ is a linear isomorphism between $\difspace$ and $\gblsonstatesandrrewards$.
\end{proposition}
\begin{proof}
The map $\gamblifier$ is obviously linear, so it only remains to prove that it is a bijection.

To prove that $\gamblifier$ is one-to-one, consider any $\lambda_1,\lambda_2\in\posreals$ and $\hl[1],\hl[2],\althl[1],\althl[2]\in\hls$, and assume that $\gamblifier\group{\gbl[1]}=\gamblifier\group{\gbl[2]}$, with $\gbl[1]\coloneqq\lambda_1\group{\althl[1]-\hl[1]}$ and $\gbl[2]\coloneqq\lambda_1\group{\althl[2]-\hl[2]}$.
It follows from the linearity of $\gamblifier$ that then also
\begin{equation*}
\frac{\lambda_1\gamblifier(\althl[1])+\lambda_2\gamblifier(\hl[2])}{\lambda_1+\lambda_2}
=\frac{\lambda_2\gamblifier(\althl[2])+\lambda_1\gamblifier(\hl[1])}{\lambda_1+\lambda_2},
\end{equation*}
and therefore also, due to the normalisation property of horse lotteries,
\begin{equation*}
\frac{\lambda_1\althl[1]+\lambda_2\hl[2]}{\lambda_1+\lambda_2}
=\frac{\lambda_2\althl[2]+\lambda_1\hl[1]}{\lambda_1+\lambda_2},
\end{equation*}
which is indeed equivalent to $\gbl[1]=\lambda_1\group{\althl[1]-\hl[1]}=\lambda_1\group{\althl[2]-\hl[2]}=\gbl[2]$.

To prove that $\gamblifier$ is onto, consider any $\altgbl\in\gblsonstatesandrrewards$, then we need to prove that there are $\lambda\in\posreals$ and $\hl,\althl\in\hls$ such that $\altgbl=\gamblifier\group{\lambda\group{\althl-\hl}}$.
If we consider any $\gbl\in\gblsonstatesandrewards$ such that $\altgbl=\gamblifier(\gbl)$---and this is always possible---then it suffices to prove that we can make sure that $\gbl=\lambda\group{\althl-\hl}$.
But this follows readily from Lemma~\ref{lem:from:difspace:to:hls}, by choosing $\hl$ such that $\hl(\cdot,r)$ is bounded away from zero for all $r\in\rewards$.
\end{proof}

\begin{lemma}\label{lem:from:difspace:to:hls}
Consider arbitrary $\gbl\in\difspace$ and any $\hl\in\hls$ such that $\hl(\cdot,r)$ is bounded away from zero for all $r\in\rewards$.
Then there are real $\lambda>0$ and $\althl\in\hls$ such that $\gbl=\lambda(\althl-\hl)$.
\end{lemma}
\begin{proof}
Consider, for any real $\lambda>0$, the real-valued function $\althl[\lambda]\coloneqq\hl+\frac{1}{\lambda}h$ on $\statesandrewards$. 
It suffices to show that there is some $\lambda>0$ for which $\althl[\lambda]\in\hls$.
Observe first of all that 
\begin{equation*}
\sum_{r\in\rewards}\althl[\lambda](\cdot,r)
=\sum_{r\in\rewards}\hl(\cdot,r)+\frac{1}{\lambda}\sum_{r\in\rewards}\gbl(\cdot,r)
=1+0=1,
\end{equation*}
where the second equality holds because $\hl\in\hls$ and $\gbl\in\difspace$.
So it suffices to show that there is some $\lambda>0$ for which $\althl[\lambda](x,r)=\hl(x,r)+\frac{1}{\lambda}\gbl(x,r)\geq0$ for all $(x,r)\in\statesandrewards$.
Since we also have that $\hl(x,r)\geq0$ for all $(x,r)\in\statesandrewards$, this condition is equivalent to
\begin{equation}\label{eq:from:difspace:to:hls}
\hl(x,r)\geq\frac{1}{\lambda}\gbl^-(x,r)\text{ for all $(x,r)\in\statesandrewards$},
\end{equation}
where we let $\gbl^-\coloneqq\max\set{0,-\gbl}$ be the negative part of $\gbl$.
It follows from the assumptions that for any $r\in\rewards$ there is some $\epsilon_r>0$ such that $\hl(x,r)\geq\epsilon_r$ for all $x\in\states$, where, of course, the bounds $\epsilon_r$ must satisfy $\sum_{r\in\rewards}\epsilon_r\leq1$.
So a sufficient condition for~\eqref{eq:from:difspace:to:hls} to hold is that
\begin{equation*}
\epsilon_r\geq\frac{1}{\lambda}\gbl^-(x,r)
\text{ for all $(x,r)\in\statesandrewards$, or equivalently, }
\lambda\geq\sup_{(x,r)\in\statesandrewards}\frac{\gbl^-(x,r)}{\epsilon_r}.
\end{equation*}
The supremum on the right hand side of the last inequality is real, because $\gbl$ is a gamble, and therefore so is $\gbl^-$, and because $\rewards$ is finite.
\end{proof}

We also use the background order $\hlgt$ on $\hls$ to induce a background vector ordering $\optgt$ on the option space $\difspace$, by letting
\begin{equation}\label{eq:background:ordering:on:options:from:horse:lotteries}
\gbl\optgt0\ifandonlyif\hl\hlgt\althl\text{ where $\gbl=\lambda(\hl-\althl)$ is \emph{any} representation of $\gbl\in\difspace$},
\end{equation}
and more generally $\altgbltoo\optgt\altgbl\ifandonlyif\altgbltoo-\altgbl\optgt0$ for all $\altgbltoo,\altgbl\in\difspace$.
Since the horse lottery representation for elements of $\difspace$ is not unique, we need, in order for this definition to be consistent, for $\lambda_1(\hl[1]-\althl[1])=\lambda_2(\hl[1]-\althl[2])$ to imply that $\hl[1]\hlgt\althl[1]\ifandonlyif\hl[2]\hlgt\althl[2]$.
It is suffices to require that the ordering $\hlgt$ satisfies the mixture independence property of Axiom~\ref{ax:hlprefgt:sure:thing:principle}.
A similar argument is made by Zaffalon and Miranda \cite{zaffalon2017:incomplete:preferences}.

\begin{proposition}[\cite{zaffalon2017:incomplete:preferences}]\label{prop:consistency:for:horselotteries:preference:ordering}
Assume that an ordering $\hlprefgt$ on $\hls$ satisfies the mixture independence property of Axiom~\ref{ax:hlprefgt:sure:thing:principle}, and consider any real $\lambda_1,\lambda_2>0$ and any $\hl[1],\hl[2],\althl[1],\althl[2]\in\hls$ such that $\lambda_1(\hl[1]-\althl[1])=\lambda_2(\hl[2]-\althl[2])$.
Then $\hl[1]\hlprefgt\althl[1]\ifandonlyif\hl[2]\hlprefgt\althl[2]$.
\end{proposition}
\begin{proof}
First of all, consider the following chain of equivalences:
\begin{align}
\lambda_1(\hl[1]-\althl[1])=\lambda_2(\hl[2]-\althl[2])
&\ifandonlyif
\lambda_1\hl[1]+\lambda_2\althl[2]
=\lambda_2\hl[2]+\lambda_1\althl[1]\notag\\
&\ifandonlyif
\alpha\hl[1]+(1-\alpha)\althl[2]
=\alpha\althl[1]+(1-\alpha)\hl[2],\label{eq:consistency:for:forselotteries:background:ordering:aux}
\end{align}
where we let $\alpha\coloneqq\frac{\lambda_1}{\lambda_1+\lambda_2}$.
To prove the proposition, we consider another chain of equivalences:
\begin{align*}
\hl[1]\hlprefgt\althl[1]
&\ifandonlyif\alpha\hl[1]+(1-\alpha)\althl[2]\hlprefgt\alpha\althl[1]+(1-\alpha)\althl[2]
&&\text{[Axiom~\ref{ax:hlprefgt:sure:thing:principle}]}\\
&\ifandonlyif\alpha\althl[1]+(1-\alpha)\hl[2]\hlprefgt\alpha\althl[1]+(1-\alpha)\althl[2]
&&\text{[Equation~\eqref{eq:consistency:for:forselotteries:background:ordering:aux}]}\\
&\ifandonlyif\hl[2]\hlprefgt\althl[2].
&&\text{[Axiom~\ref{ax:hlprefgt:sure:thing:principle}]}\qedhere
\end{align*}
\end{proof}
\noindent Moreover, if $\altgbltoo=\lambda_1(\hl[1]-\althl[1])$ and $\altgbl=\lambda_2(\hl[2]-\althl[2])$, then
\begin{align*}
\altgbltoo\optgt\altgbl
\ifandonlyif\altgbltoo-\altgbl\optgt0
&\ifandonlyif\lambda_1(\hl[1]-\althl[1])-\lambda_2(\hl[2]-\althl[2])\optgt0
&&\\
&\ifandonlyif(\lambda_1+\lambda_2)
\group[\bigg]{\frac{\lambda_1\hl[1]+\lambda_2\althl[2]}{\lambda_1+\lambda_2}
-\frac{\lambda_2\hl[2]+\lambda_1\althl[1]}{\lambda_1+\lambda_2}}\optgt0
&&\\
&\ifandonlyif
\frac{\lambda_1\hl[1]+\lambda_2\althl[2]}{\lambda_1+\lambda_2}
\hlgt\frac{\lambda_2\hl[2]+\lambda_1\althl[1]}{\lambda_1+\lambda_2},
&&\text{[Equation~\eqref{eq:background:ordering:on:options:from:horse:lotteries}]}
\end{align*}
which shows how we can derive the background ordering~$\optgt$ on options in $\difspace$ directly form the background ordering~$\hlgt$ on horse lotteries.

\begin{runningexample}[A simple background ordering for options in $\difspace$]
The background ordering $\optgt[\bot]$ on $\difspace$ derived from the background ordering $\hlgt[\bot]$ on $\hls$ can be found as follows.
Consider an $\gbl=\lambda\group{\althl-\hl}\in\difspace$, then
\begin{align*}
\gbl\optgt[\bot]0
&\ifandonlyif\althl\hlgt[\bot]\hl
&&\text{[Equation~\eqref{eq:background:ordering:on:options:from:horse:lotteries}]}\\
&\ifandonlyif\gamblifier[\bot](\althl)\gblgt\gamblifier[\bot](\hl)
&&\text{[Equation~\eqref{eq:background:ordering:hls}]}\\
&\ifandonlyif\gamblifier[\bot](\althl)-\gamblifier[\bot](\hl)\gblgt0
&&\text{[$\gblgt$ is a vector ordering]}\\
&\ifandonlyif\frac{1}{\lambda}\gamblifier[\bot]\group{\gbl}\gblgt0
&&\text{[$\gamblifier[\bot]$ is linear]}\\
&\ifandonlyif\gamblifier[\bot](\gbl)\gblgt0,
&&\text{[$\gblgt$ is a vector ordering]}
\end{align*}
which tells us that $\optgt[\bot]$ is also the ordering induced on $\difspace$ by the linear isomorphism $\gamblifier[\bot]$ from the standard background ordering $\gblgt$ on gambles in $\gblsonstatesandrrewards$.\qed
\end{runningexample}

The next proposition tells us that the ambiguity in horse lottery representations for elements of $\difspace$ is also of no consequence to a rejection function $\hlrejectfun$ that satisfies Axiom~\ref{ax:seidenfeld:hlrejectfun:sure:thing:principle}.
Its proof is based on an argument similar to the one in the proof of Proposition~\ref{prop:consistency:for:horselotteries:preference:ordering}, and is essentially due to Arthur Van Camp~\cite{2017vancamp:phdthesis}.

\begin{proposition}[{\protect\cite[]{2017vancamp:phdthesis}}]\label{prop:consistency:for:horselotteries:choice}
Let $\hlrejectfun$ be any rejection function on $\hls$ that satisfies Axiom~\ref{ax:seidenfeld:hlrejectfun:sure:thing:principle}.
Consider any $\hlset,\althlset\in\hlsets$ and $\hl,\althl\in\hls$ such that $\lambda\group{\hlset-\hl}=\mu\group{\althlset-\althl}$ for some $\lambda,\mu\in\posreals$.
Then
\begin{equation*}
\hl\in\hlrejectfun\group{\hlset\cup\set{\hl}}\ifandonlyif\althl\in\hlrejectfun\group{\althlset\cup\set{\althl}},
\text{ or equivalently, }
\set{\hl}\hlsetlt\hlset\ifandonlyif\set{\althl}\hlsetlt\althlset.
\end{equation*}
\end{proposition}
\begin{proof}
It is clear that the two statements are equivalent, so we concentrate on proving the latter.
First of all, consider the following chain of equivalences:
\begin{align}
&\lambda\group{\hlset-\hl}=\mu\group{\althlset-\althl}\notag\\
&\qquad\ifandonlyif
\lambda\group{\hlset-\hl}+\lambda\set{\hl}+\mu\set{\althl}
=\mu\group{\althlset-\althl}+\lambda\set{\hl}+\mu\set{\althl}\notag\\
&\qquad\ifandonlyif
\lambda\hlset+\mu\set{\althl}
=\mu\althlset+\lambda\set{\hl}\notag\\
&\qquad\ifandonlyif
\alpha\hlset+(1-\alpha)\set{\althl}
=(1-\alpha)\althlset+\alpha\set{\hl},\label{eq:consistency:for:forselotteries:choice:aux}
\end{align}
where we let $\alpha\coloneqq\frac{\lambda}{\lambda+\mu}\in(0,1)$.
To prove the proposition, we consider another chain of equivalences:
\begin{align*}
\set{\hl}\hlsetlt\hlset
&\ifandonlyif\alpha\set{\hl}+(1-\alpha)\set{\althl}\hlsetlt\alpha\hlset+(1-\alpha)\set{\althl}
&&\text{[Axiom~\ref{ax:seidenfeld:hlrejectfun:sure:thing:principle}]}\\
&\ifandonlyif\alpha\set{\hl}+(1-\alpha)\set{\althl}\hlsetlt\alpha\set{\hl}+(1-\alpha)\althlset
&&\text{[Equation~\eqref{eq:consistency:for:forselotteries:choice:aux}]}\\
&\ifandonlyif\set{\althl}\hlsetlt\althlset.
&&\text{[Axiom~\ref{ax:seidenfeld:hlrejectfun:sure:thing:principle}]}\qedhere
\end{align*}
\end{proof}
\noindent It is this proposition---and therefore Axiom~\ref{ax:seidenfeld:hlrejectfun:sure:thing:principle}---that makes sure that using $\difspace$ as an option space $\opts$ to try and represent rejection functions on horse lotteries in, in the sense of Equation~\eqref{eq:representation:hlrejectfuns:tentative}, has any chance of success.
We will now make clear how such a representation comes about.

First of all, starting from a rejection function~$\hlrejectfun$ on $\hls$, we can use Equation~\eqref{eq:consistency:for:forselotteries:choice:aux} to construct a set of desirable option sets $\rejectset[\hlrejectfun]\in\optsets(\difspace)$ as follows.
We begin with
\begin{equation*}
\rejectset[\hlrejectfun]^0
\coloneqq\cset{\hlset-\hl}
{\hlset\in\hlsets,\hl\in\hls\text{ and }\hl\in\hlrejectfun\group{\hlset\cup\set{\hl}}},
\end{equation*}
and then let
\begin{equation}\label{eq:hlrejectfun:to:rejectset}
\rejectset[\hlrejectfun]
\coloneqq\Starify\group[\big]{\rejectset[\hlrejectfun]^0}
=\Starify\group[\big]{\cset{\hlset-\hl}
{\hlset\in\hlsets,\hl\in\hls\text{ and }\hl\in\hlrejectfun\group{\hlset\cup\set{\hl}}}},
\end{equation}
where we let, for any option set $\optset=\set{\gbl[1],\dots,\gbl[n]}\in\optsets(\difspace)$, with $n\in\naturalswithzero$:
\begin{equation*}
\starify\group{\optset}
\coloneqq\cset[\big]{\set{\lambda_1\gbl[1],\dots,\lambda_n\gbl[n]}}
{\lambda_k\in\posreals,k\in\set{1,\dots,n}} 
\end{equation*} 
and, for any set of option sets $\rejectset\subseteq\optsets(\difspace)$, $\Starify\group{\rejectset}\coloneqq\bigcup\cset{\starify\group{\optset}}{\optset\in\rejectset}$.

Why is applying the $\Starify\group{\cdot}$ operator in Equation~\eqref{eq:hlrejectfun:to:rejectset} reasonable?
While working with differences of horse lotteries is essentially what is done on the horse lottery side, Proposition~\ref{prop:difspace} tells us that doing so is not enough on the options side: a single difference $\hl-\althl$ between horse lotteries corresponds to an entire ray of options $\cset{\lambda(\hl-\althl)}{\lambda\in\posreals}$ between which, from the horse lottery point of view, no distinction can be made.
And since we can make no distinction between multiples of the same horse lottery difference on the horse lottery side, a `representation' in terms of options should not be able to do so either.
Therefore, replacing any option in any option set with any of its positive multiples should not matter.
Fortunately, our coherent choice models in terms of options are in perfect agreement with this: Lemma~\ref{lem:localrescaling} tells us that for any coherent set of desirable option sets $\rejectset$, $\optset\in\rejectset\ifandonlyif\starify\group{\optset}\in\rejectset$.

Conversely, if we start from a set of desirable option sets $\rejectset\in\optsets(\difspace)$, we can use Equation~\eqref{eq:consistency:for:forselotteries:choice:aux} to construct a rejection function $\hlrejectfun[\rejectset]$ on $\hls$ as follows:
\begin{equation}\label{eq:rejectset:to:hlrejectfun}
\hlrejectfun[\rejectset]\group{\emptyset}\coloneqq\emptyset
\text{ and }
\hl\in\hlrejectfun[\rejectset]\group{\hlset\cup\set{\hl}}
\ifandonlyif\hlset-\hl\in\rejectset,
\text{ for all $\hlset\in\hlsets$ and $\hl\in\hls$}.
\end{equation}
Our attempts at representing a rejection function $\hlrejectfun$ on $\hls$ by a set $\rejectset$ of desirable option sets on $\difspace$ will only work if using the representation $\rejectset=\rejectset[\hlrejectfun]$ on options to define a rejection function $\hlrejectfun[\rejectset]$ on horse lotteries leads back to the original rejection function $\hlrejectfun$.
We will see further on in Theorem~\ref{theo:from:horselotteries:to:options:and:back} that under some additional requirements, this can indeed be guaranteed.

We now impose our own coherence requirements on rejection functions on horse lotteries, for which we want to make sure that they lead to the coherence requirements~\ref{ax:rejects:removezero}--\ref{ax:rejects:mono} of the representing sets of desirable option sets that are precisely the ones we have considered before.
We allow ourselves to be inspired by the axioms~\ref{ax:rejectfun:addition}--\ref{ax:rejectfun:senalpha} we imposed on rejection functions on options in Section~\ref{sec:back:to:choice}, as well as by Seidenfeld et al.'s axioms~\ref{ax:seidenfeld:hlrejectfun:sure:thing:principle}--\ref{ax:seidenfeld:hlrejectfun:senalpha}.

\begin{enumerate}[label=$\mathrm{R}^\ast_{\arabic*}$.,ref=$\mathrm{R}^\ast_{\arabic*}$,leftmargin=*,start=0]
\item\label{ax:hlrejectfun:sure:thing:principle}
$\hlset\hlsetgt\althlset\ifandonlyif\alpha\hlset+(1-\alpha)\hl\hlsetgt\alpha\althlset+(1-\alpha)\hl$ for all $\hlset,\althlset\in\hlsets$, $\hl\in\hls$ and $\alpha\in(0,1]$;
\item\label{ax:hlrejectfun:not:everything:rejected} $\hlrejectfun\group{\emptyset}=\emptyset$, and $\hlrejectfun\group{\hlset}\neq\hlset$ for all $\hlset\in\hlsets\setminus\set{\emptyset}$;
\item\label{ax:hlrejectfun:pos} $\hl\in\hlrejectfun\group{\set{\hl,\althl}}$ or equivalently, $\set{\althl}\hlsetgt\set{\hl}$, for all $\hl,\althl\in\hls$ such that $\althl\hlgt\hl$;
\item\label{ax:hlrejectfun:cone} consider any $m,n\in\naturals$, any $\hlset[1]\coloneqq\set{\althltoo[1],\dots,\althltoo[m]},\hlset[2]\coloneqq\set{\althl[1],\dots,\althl[n]}\in\hlsets$ and $\althltoo,\althl\in\hls$ such that $\althltoo\in\hlrejectfun\group{\set{\althltoo[1],\dots\althltoo[m],\althltoo}}$ and $\althl\in\hlrejectfun\group{\set{\althl[1],\dots,\althl[n],\althl}}$, and any $(\lambda_{k\ell},\mu_{k\ell})>0$ for all $(k,\ell)\in\set{1,\dots,m}\times\set{1,\dots,n}$, then it holds that $\hl\in\hlrejectfun\group{\set{\hl[11],\dots,\hl[mn],\hl}}$ for all $\hlset\coloneqq\cset{\hl[k\ell]}{(k,\ell)\in\set{1,\dots,m}\times\set{1,\dots,n}}\in\hlsets$ and $\hl\in\hls$ such that $\hl[k\ell]-\hl=\lambda_{k\ell}\group{\althltoo[k]-\althltoo}+\mu_{k\ell}\group{\althl[\ell]-\althl}$ for all $(k,\ell)\in\set{1,\dots,m}\times\set{1,\dots,n}$;
\item\label{ax:hlrejectfun:senalpha} 
if $\hlset[1]\subseteq\hlset[2]$ then also $\hlrejectfun\group{\hlset[1]}\subseteq\hlrejectfun\group{\hlset[2]}$, for all $\hlset[1],\hlset[2]\in\hlsets$.
\end{enumerate}
Axioms~\ref{ax:hlrejectfun:sure:thing:principle}, \ref{ax:hlrejectfun:not:everything:rejected} and~\ref{ax:hlrejectfun:senalpha} are the same as the homologous~\ref{ax:seidenfeld:hlrejectfun:sure:thing:principle}, \ref{ax:seidenfeld:hlrejectfun:not:everything:rejected} and~\ref{ax:seidenfeld:hlrejectfun:senalpha} imposed by Seidenfeld et al.~\cite{seidenfeld2010}.
Axioms~\ref{ax:hlrejectfun:pos} and~\ref{ax:hlrejectfun:cone} are stronger than the homologous~\ref{ax:seidenfeld:hlrejectfun:pos} and~\ref{ax:seidenfeld:hlrejectfun:aizerman}, and go back to the homologous~\ref{ax:rejectfun:pos} and~\ref{ax:rejectfun:cone} for rejection functions on options.

\begin{definition}[Coherence for rejection functions on horse lotteries]\label{def:coherence:hlrejectset}
A rejection function on horse lotteries $\hlrejectfun\colon\hlsets\to\hlsets$ is called \emph{coherent} if it satisfies Axioms~\ref{ax:hlrejectfun:sure:thing:principle}--\ref{ax:hlrejectfun:senalpha}.
It is, moreover, called \emph{total} if it is coherent and also satisfies
\begin{enumerate}[label=$\mathrm{R}^\ast_{\mathrm{T}}$.,ref=$\mathrm{R}^\ast_{\mathrm{T}}$,leftmargin=*]
\item\label{ax:hlrejectfun:totality} 
$\set[\big]{\frac{\hl[1]+\hl[2]}{2}}\hlsetlt\set{\hl[1],\hl[2]}$ for all $\hl[1],\hl[2]\in\hls$ such that $\hl[1]\neq\hl[2]$.
\end{enumerate}
Finally, it is called \emph{mixing} if it is coherent and also satisfies
\begin{enumerate}[label=$\mathrm{R}^\ast_{\mathrm{M}}$.,ref=$\mathrm{R}^\ast_{\mathrm{M}}$,leftmargin=*]
\item\label{ax:hlrejectfun:removing:convex:combinations}
if $\hlset\subseteq\althlset\subseteq\chull\group{\hlset}$ then $\rejectfun\group{\althlset}\cap\hlset\subseteq\rejectfun\group{\hlset}$, for all $\hlset,\althlset\in\hlsets$.
\end{enumerate}
\end{definition}
\noindent Observe that our mixingness condition~\ref{ax:hlrejectfun:removing:convex:combinations} is the same as Seidenfeld et al.'s Axiom~\ref{ax:seidenfeld:hlrejectfun:removing:convex:combinations}.

We now show that the relations~\eqref{eq:hlrejectfun:to:rejectset} and~\eqref{eq:rejectset:to:hlrejectfun} preserve coherence, totality, and mixingness.

\begin{theorem}\label{theo:horselotteries:to:options:and:back}
Let $\rejectset$ be any set of desirable option sets on $\difspace$, and let $\hlrejectfun$ be any rejection function on $\hls$.
Then the following statements hold:
\begin{enumerate}[label=\upshape(\roman*),leftmargin=*]
\item\label{it:difspacecoherence:to:hlcoherence} if $\rejectset$ is coherent, then so is $\hlrejectfun[\rejectset]$;
\item\label{it:hlcoherence:to:difspacecoherence} if $\hlrejectfun$ is coherent, then so is $\rejectset[\hlrejectfun]$;
\item\label{it:difspacetotal:to:hltotal} if $\rejectset$ is total, then so is $\hlrejectfun[\rejectset]$;
\item\label{it:hltotal:to:difspacetotal} if $\hlrejectfun$ is total, then so is $\rejectset[\hlrejectfun]$;
\item\label{it:difspacemixing:to:hlmixing} if $\rejectset$ is mixing, then so is $\hlrejectfun[\rejectset]$;
\item\label{it:hlmixing:to:difspacemixing} if $\hlrejectfun$ is mixing, then so is $\rejectset[\hlrejectfun]$.
\end{enumerate}
\end{theorem}
\begin{proof}
We begin with the proof of~\ref{it:difspacecoherence:to:hlcoherence}.

Assume that $\rejectset$ is coherent.
We prove that $\hlrejectfun[\rejectset]$ satisfies Axioms~\ref{ax:hlrejectfun:sure:thing:principle}--\ref{ax:hlrejectfun:senalpha}.

\ref{ax:hlrejectfun:sure:thing:principle}.
Consider any $\hlset,\althlset\in\hlsets$, $\hl\in\hls$ and $\alpha\in(0,1]$, and assume that 
$\hlset\hlsetlt\althlset$, then it suffices to show that also $\alpha\hlset+(1-\alpha)\hl\hlsetlt\alpha\althlset+(1-\alpha)\hl$.
$\hlset\hlsetlt\althlset$ means that $\hlset\subseteq\hlrejectfun[\rejectset]\group{\althlset\cup\hlset}$, which is equivalent to
\begin{equation}\label{eq:difspacecoherence:to:hlcoherence:sure:thing:principle:one}
\group{\althlset\cup\hlset}-\althltoo\in\rejectset
\text{ for all }\althltoo\in\hlset.
\end{equation}
On the other hand, $\alpha\hlset+(1-\alpha)\hl\hlsetlt\alpha\althlset+(1-\alpha)\hl$ means that $\alpha\hlset+(1-\alpha)\hl\subseteq\hlrejectfun\group{\alpha\group{\althlset\cup\hlset}+(1-\alpha)\hl}$, and since
\begin{equation*}
\alpha\group{\althlset\cup\hlset}+(1-\alpha)\hl-\group{\alpha\althltoo+(1-\alpha)\hl}
=\alpha\group[\big]{\group{\althlset\cup\hlset}-\althltoo}, 
\end{equation*}
this is equivalent to
\begin{equation}\label{eq:difspacecoherence:to:hlcoherence:sure:thing:principle:two}
\alpha\group[\big]{\group{\althlset\cup\hlset}-\althltoo}\in\rejectset
\text{ for all }\althltoo\in\hlset.
\end{equation}
That the conditions~\eqref{eq:difspacecoherence:to:hlcoherence:sure:thing:principle:one} and~\eqref{eq:difspacecoherence:to:hlcoherence:sure:thing:principle:two} are equivalent is now an easy consequence of Axiom~\ref{ax:rejects:cone} [see, for instance, Lemma~\ref{lem:localrescaling} which implies this equivalence].

\ref{ax:hlrejectfun:not:everything:rejected}.
That $\hlrejectfun[\rejectset]\group{\emptyset}=\emptyset$ follows directly from the definition~\eqref{eq:rejectset:to:hlrejectfun} of $\hlrejectfun[\rejectset]$.
The rest of the proof is now a direct adaptation of the homologous part of Proposition~\ref{prop:axioms:rejection:sets:and:functions:coherence}, which we repeat here for the sake of thoroughness.
For any non-empty option set $\hlset\in\hlsets$, assume {\itshape ex absurdo} that $\hlrejectfun\group{\hlset}=\hlset$. 
For all $\hl\in\hlset$, it then follows from Equation~\eqref{eq:rejectset:to:hlrejectfun} that $\hlset-\hl\in\rejectset$.
So if we denote $\hlset$ by $\set{\hl[1],\dots,\hl[n]}$, with $n\in\naturals$, and let $\gbl[\ell k]\coloneqq\hl[\ell]-\hl[k]$, then we find that for all $k\in\set{1,\dots,n}$
\begin{equation*}
\optset[k]\coloneqq\cset{\gbl[\ell k]}{\ell\in\set{1,\dots,n}}\in\rejectset.
\end{equation*}
Proposition~\ref{prop:ax:rejects:cone:equivalents} and Equation~\eqref{eq:setposi} now tell us that, for any choice of the $\lambda_{1:n}^{\altgbl[1:n]}>0$ in Equation~\eqref{eq:setposi}, the option set
\begin{equation*}
\cset[\bigg]{\sum_{k=1}^n\lambda_{k}^{\altgbl[1:n]}\altgbl[k]}
{\altgbl[1:n]\in\times_{k=1}^n\optset[k]}\in\rejectset.
\end{equation*}
So if we can show that for any $\altgbl[1:n]\in\times_{k=1}^n\optset[k]$ we can always choose the $\lambda_{1:n}^{\altgbl[1:n]}>0$ in such a way that $\sum_{k=1}^n\lambda_{k}^{\altgbl[1:n]}\altgbl[k]=0$, we will have that $\set{0}\in\rejectset$, contradicting Axiom~\ref{ax:rejects:nozero}.
We now set out to do this.

For any $k\in\set{1,\dots,n}$, since $\altgbl[k]\in\optset[k]$, there is a unique $\ell\in\set{1,\dots,n}$ such that $\altgbl[k]=\gbl[\ell k]$. 
Let $\phi(k)$ be this unique index, so $\altgbl[k]=\gbl[\phi(k)k]$. 
For the resulting map $\phi\colon\set{1,\dots,n}\to\set{1,\dots,n}$, we now consider the sequence---$\phi$-orbit---in $\set{1,\dots,n}$:
\begin{equation*}
1,\phi(1),\phi^2(1),\dots,\phi^r(1),\dots
\end{equation*}
Because $\phi$ can assume at most $n$ different values, this sequence must be periodic, and its fundamental (smallest) period $p$ cannot be larger than $n$, so $1\leq p\leq n$ and $1=\phi^{p}(1)$.
Now let $\lambda^{\altgbl[1:n]}_{\phi^r(1)}\coloneqq1$ for $r=0,\dots,p-1$, and let all other components be zero, then indeed
\begin{equation*}
\sum_{k=1}^n\lambda_{k}^{\altgbl[1:n]}\altgbl[k]
=\sum_{r=0}^{p-1}\gbl[\phi^{r+1}(1)\phi^{r}(1)]
=\sum_{r=0}^{p-1}\group{\hl[\phi^{r+1}(1)]-\hl[\phi^{r}(1)]}
=0.
\end{equation*}

\ref{ax:hlrejectfun:pos}.
Consider any $\hl,\althl\in\hls$ such that $\althl\hlgt\hl$.
Then $\althl-\hl\optgt0$, so $\set{\althl-\hl}\in\rejectset$, by Axiom~\ref{ax:rejects:pos}. 
Since $\set{\althl-\hl}=\set{\althl}-\hl$ and $\set{\althl}\cup\set{\hl}=\set{\althl,\hl}$, this implies that, indeed, $\hl\in\hlrejectfun[\rejectset]\group{\set{\hl,\althl}}$.

\ref{ax:hlrejectfun:cone}.
Consider any $m,n\in\naturals$, any $\set{\althltoo[1],\dots,\althltoo[m]},\set{\althl[1],\dots,\althl[n]}\in\hlsets$ and $\althltoo,\althl\in\hls$ such that $\althltoo\in\hlrejectfun[\rejectset]\group{\set{\althltoo[1],\dots\althltoo[m],\althltoo}}$ and $\althl\in\hlrejectfun[\rejectset]\group{\set{\althl[1],\dots,\althl[n],\althl}}$, and any $(\lambda_{k\ell},\mu_{k\ell})>0$ for all $(k,\ell)\in\set{1,\dots,m}\times\set{1,\dots,n}$.
Also consider any $\cset{\hl[k\ell]}{(k,\ell)\in\set{1,\dots,m}\times\set{1,\dots,n}}\in\hlsets$ and any $\hl\in\hls$ such that $\hl[k\ell]-\hl=\lambda_{k\ell}\group{\althltoo[k]-\althltoo}+\mu_{k\ell}\group{\althl[\ell]-\althl}$ for all $k\in\set{1,\dots,m}$ and $\ell\in\set{1,\dots,n}$.
Then we must prove that also $\hl\in\hlrejectfun[\rejectset]\group{\set{\hl[11],\dots,\hl[mn],\hl}}$.
Observe that it follows from the assumptions that $\cset{\althltoo[k]-\althltoo}{k\in\set{1,\dots,m}}\in\rejectset$ and $\cset{\althl[\ell]-\althl}{\ell\in\set{1,\dots,n}}\in\rejectset$, and therefore from Axiom~\ref{ax:rejects:cone} that also
\begin{multline*}
\cset{\hl[k\ell]}{(k,\ell)\in\set{1,\dots,m}\times\set{1,\dots,n}}-\hl\\
=\cset{\lambda_{k\ell}\group{\althltoo[k]-\althltoo}+\mu_{k\ell}\group{\althl[\ell]-\althl}}{(k,\ell)\in\set{1,\dots,m}\times\set{1,\dots,n}}
\in\rejectset.
\end{multline*}
Hence, indeed, $\hl\in\hlrejectfun[\rejectset]\group{\set{\hl[11],\dots,\hl[mn],\hl}}$.

\ref{ax:hlrejectfun:senalpha}.
Consider any $\hlset[1],\hlset[2]\in\hlsets$ such that $\hlset[1]\subseteq\hlset[2]$.
Consider any $\hl\in\hlrejectfun[\rejectset]\group{\hlset[1]}$, then we have to prove that also $\hl\in\hlrejectfun[\rejectset]\group{\hlset[2]}$.
$\hl\in\hlrejectfun[\rejectset]\group{\hlset[1]}$ means that $\hlset[1]-\hl\in\rejectset$, and therefore also $\hlset[2]-\hl\in\rejectset$, by Axiom~\ref{ax:rejects:mono}.
Hence, indeed, $\hl\in\hlrejectfun[\rejectset]\group{\hlset[2]}$.

Next, we turn to the proof of~\ref{it:hlcoherence:to:difspacecoherence}.

Assume that $\hlrejectfun$ is coherent, then we prove that $\rejectset[\hlrejectfun]$ satisfies Axioms~\ref{ax:rejects:removezero}--\ref{ax:rejects:mono}.

\ref{ax:rejects:removezero}.
Consider any $\optset\in\rejectset[\hlrejectfun]$, which means that there are $\hlset\in\hlsets$ and $\hl\in\hls$ such that $\optset\in\starify\group{\hlset-\hl}$ and $\hl\in\hlrejectfun\group{\hlset\cup\set{\hl}}$.
We may assume without loss of generality that $0\in\optset$, or in other words, that $\hl\in\hlset$.
But then $\hlset\cup\set{\hl}=\group{\hlset\setminus\set{\hl}}\cup\set{\hl}$ and therefore also $\hl\in\hlrejectfun\group{\group{\hlset\setminus\set{\hl}}\cup\set{\hl}}$.
This implies that $\group{\hlset-\hl}\setminus\set{0}=\hlset\setminus\set{\hl}-\hl\in\rejectset[\hlrejectfun]^0$, and since also $\optset\setminus\set{0}\in\starify\group{\group{\hlset-\hl}\setminus\set{0}}$, we find that, indeed, $\optset\setminus\set{0}\in\rejectset[\hlrejectfun]$, using Equation~\eqref{eq:hlrejectfun:to:rejectset}.

\ref{ax:rejects:nozero}.
Assume {\itshape ex absurdo} that $\set{0}\in\rejectset[\hlrejectfun]$, meaning that there is some $\hl\in\hls$ such that $\hl\in\hlrejectfun\group{\set{\hl}\cup\set{\hl}}$, or in other words, that $\set{\hl}=\hlrejectfun\group{\set{\hl}}$, contradicting Axiom~\ref{ax:hlrejectfun:not:everything:rejected}.

\ref{ax:rejects:pos}.
Consider any $h\in\difspace$ such that $h\optgt0$, meaning that there are real $\lambda>0$ and $\althltoo,\althl\in\hls$ such that $h=\lambda\group{\althltoo-\althl}$ and $\althltoo\hlgt\althl$.
Then it follows from Axiom~\ref{ax:hlrejectfun:pos} that $\althl\in\hlrejectfun\group{\set{\althl,\althltoo}}$, so we find that $\set{\althltoo}-\althl=\set{\althltoo-\althl}\in\rejectset[\hlrejectfun]^0$.
Hence, $h=\lambda\group{\althltoo-\althl}\in\starify\group{\set{\althltoo-\althl}}\subseteq\rejectset[\hlrejectfun]$.

\ref{ax:rejects:cone}.
Consider any $m,n\in\naturals$, any $\optset[1]\coloneqq\set{\altgbltoo[1],\dots,\altgbltoo[m]}$ and $\optset[2]\coloneqq\set{\altgbl[1],\dots,\altgbl[n]}$ in $\rejectset[\hlrejectfun]$ and any $(\lambda_{k\ell},\mu_{k\ell})>0$ for all $(k,\ell)\in\set{1,\dots,m}\times\set{1,\dots,n}$.
Let $\optset\coloneqq\cset{\lambda_{k\ell}\altgbltoo[k]+\mu_{k\ell}\altgbl[\ell]}{(k,\ell)\in\set{1,\dots,m}\times\set{1,\dots,n}}$, then we have to prove that $\optset\in\rejectset[\hlrejectfun]$.
It follows from the assumptions that there are real $\lambda_k>0$ and $\althltoo[k],\althltoo\in\hls$ such that $\althltoo\in\hlrejectfun\group{\set{\althltoo[1],\dots,\althltoo[m],\althltoo}}$ and $\altgbltoo[k]=\lambda_k(\althltoo[k]-\althltoo)$ for all $k\in\set{1,\dots,m}$, and that there are real $\mu_\ell>0$ and $\althl[\ell],\althl\in\hls$ such that $\althl\in\hlrejectfun\group{\set{\althl[1],\dots,\althl[m],\althl}}$ and $\altgbl[\ell]=\mu_\ell(\althl[\ell]-\althl)$ for all $\ell\in\set{1,\dots,n}$.
We can therefore write $\optset=\cset{\gbl[k\ell]}{(k,\ell)\in\set{1,\dots,m}\times\set{1,\dots,n}}$, with $\gbl[k\ell]\coloneqq\lambda_{k\ell}\lambda_k(\althltoo[k]-\althltoo)+\mu_{k\ell}\mu_\ell(\althl[\ell]-\althl)$.
If we now fix any $\hl\in\hls$ such that $\hl(\cdot,r)$ is bounded away from zero for all $r\in\rewards$, which is always possible, then it follows from Lemma~\ref{lem:from:difspace:to:hls} that for any $(k,\ell)\in\set{1,\dots,m}\times\set{1,\dots,n}$ there are $\hl[k\ell]\in\hls$ and real $\kappa_{k\ell}>0$ such that $\gbl[k\ell]=\kappa_{k\ell}(\hl[k\ell]-\hl)$.
If we now let $\hlset\coloneqq\cset{\hl[k\ell]}{(k,\ell)\in\set{1,\dots,m}\times\set{1,\dots,n}}$, $\lambda_{k\ell}'\coloneqq\frac{\lambda_{k\ell}\lambda_k}{\kappa_{k\ell}}$ and $\mu_{k\ell}'\coloneqq\frac{\mu_{k\ell}\mu_\ell}{\kappa_{k\ell}}$, then clearly $(\lambda_{k\ell}',\mu_{k\ell}')>0$ and
\begin{align}
\hlset-\hl
&=\cset{\lambda_{k\ell}'(\althltoo[k]-\althltoo)+\mu_{k\ell}'(\althl[\ell]-\althl)}
{(k,\ell)\in\set{1,\dots,m}\times\set{1,\dots,n}}
\label{eq:hlcoherence:to:difspacecoherence:hls}\\
&=\cset[\Big]{\frac{\gbl[k\ell]}{\kappa_{k\ell}}}
{(k,\ell)\in\set{1,\dots,m}\times\set{1,\dots,n}}.
\label{eq:hlcoherence:to:difspacecoherence:gambles}
\end{align}
Axiom~\ref{ax:hlrejectfun:cone} and Equation~\eqref{eq:hlcoherence:to:difspacecoherence:hls} now guarantees that $\hl\in\hlrejectfun\group{\hlset\cup\set{\hl}}$, or equivalently, that $\hlset-\hl\in\rejectset[\hlrejectfun]^0$.
Together with Equation~\eqref{eq:hlcoherence:to:difspacecoherence:gambles}, this implies that, indeed, $\optset=\cset{\gbl[k\ell]}{(k,\ell)\in\set{1,\dots,m}\times\set{1,\dots,n}}\in\starify\group{\hlset-\hl}\subseteq\rejectset[\hlrejectfun]$.

\ref{ax:rejects:mono}.
Consider any $\optset[1],\optset[2]\in\optsets(\difspace)$ such that $\optset[1]\in\rejectset[\hlrejectfun]$ and $\optset[1]\subseteq\optset[2]$.
Then we know that there are $\hlset[1]\in\hlsets$ and $\hl\in\hls$ such that $\hl\in\hlrejectfun\group{\hlset[1]\cup\set{\hl}}$ and $\optset[1]\in\starify\group{\hlset[1]-\hl}$.
Now let $\cuhl$ be the constant uniform horse lottery, defined by $\cuhl(x,r)\coloneqq\nicefrac{1}{\card{\rewards}}$ for all $(x,r)\in\rewards$.
Also let $\althl\coloneqq\nicefrac{1}{2}\hl+\nicefrac{1}{2}\cuhl$ and $\althlset[1]\coloneqq\nicefrac{1}{2}\hlset[1]+\nicefrac{1}{2}\cuhl$, then $\althl\in\hlrejectfun\group{\althlset[1]\cup\set{\althl}}$, by Axiom~\ref{ax:hlrejectfun:sure:thing:principle}.
Moreover, $\althlset[1]-\althl=\nicefrac{1}{2}\group{\hlset[1]-\hl}$ and therefore also $\optset[1]\in\starify\group{\althlset[1]-\althl}$.
Since $\althl\geq\nicefrac{1}{2\card{\rewards}}$, $\althl(\cdot,r)$ is bounded away from zero for all $r\in\rewards$, so we may apply Lemma~\ref{lem:from:difspace:to:hls} to infer that it is possible to extend $\althlset[1]$ to a set of horse lotteries $\althlset[2]\supseteq\althlset[1]$ such that $\optset[2]\in\starify\group{\althlset[2]-\althl}$.
Since $\althlset[1]\subseteq\althlset[2]$, we infer from Axiom~\ref{ax:hlrejectfun:senalpha} that also $\althl\in\hlrejectfun\group{\althlset[2]\cup\set{\althl}}$, which tells us that $\althlset[2]-\althl\in\rejectset[\hlrejectfun]^0$ and therefore, indeed, $\optset[2]\in\rejectset[\hlrejectfun]$.

We now turn to the proof of~\ref{it:difspacetotal:to:hltotal}.

Assume that $\rejectset$ is total.
Then we already know from Theorem~\ref{theo:horselotteries:to:options:and:back}\ref{it:difspacecoherence:to:hlcoherence} that $\hlrejectfun[\rejectset]$ is coherent, so it suffices to show that $\hlrejectfun[\rejectset]$ satisfies Axiom~\ref{ax:hlrejectfun:totality}.
Consider any $\hl[1],\hl[2]\in\hls$ such that $\hl[1]\neq\hl[2]$.
Then $\gbl\coloneqq\nicefrac{1}{2}\hl[1]-\nicefrac{1}{2}\hl[2]\neq0$ and therefore Axiom~\ref{ax:rejects:totality} implies that $\set{\nicefrac{1}{2}\hl[1]-\nicefrac{1}{2}\hl[2],\nicefrac{1}{2}\hl[2]-\nicefrac{1}{2}\hl[1]}=\set{\gbl,-\gbl}\in\rejectset$.
Since, moreover,
\begin{equation*}
\left\{
\begin{aligned}
\gbl&=\nicefrac{1}{2}\hl[1]-\nicefrac{1}{2}\hl[2]=\hl[1]-\group{\nicefrac{1}{2}\hl[1]+\nicefrac{1}{2}\hl[2]}\\
-\gbl&=\nicefrac{1}{2}\hl[2]-\nicefrac{1}{2}\hl[1]=\hl[2]-\group{\nicefrac{1}{2}\hl[1]+\nicefrac{1}{2}\hl[2]}
\end{aligned}
\right.
\end{equation*}
this tells us that $\set{\hl[1],\hl[2]}-\group{\nicefrac{1}{2}\hl[1]+\nicefrac{1}{2}\hl[2]}\in\rejectset$, and therefore, indeed, $\nicefrac{1}{2}\hl[1]+\nicefrac{1}{2}\hl[2]\in\hlrejectfun[\rejectset]\group{\set{\hl[1],\hl[2],\nicefrac{1}{2}\hl[1]+\nicefrac{1}{2}\hl[2]}}$.

We now turn to the proof of~\ref{it:hltotal:to:difspacetotal}.

Assume that $\hlrejectfun$ is total.
Then we already know from Theorem~\ref{theo:horselotteries:to:options:and:back}\ref{it:hlcoherence:to:difspacecoherence} that $\rejectset[\hlrejectfun]$ is coherent, so it suffices to show that $\rejectset[\hlrejectfun]$ satisfies Axiom~\ref{ax:rejects:totality}.
So consider any $\gbl\in\difspace\setminus\set{0}$, so we know there are real $\lambda>0$ and $\hl,\althl\in\hls$ with $\hl\neq\althl$ such that $\gbl=\lambda\group{\hl-\althl}$.
Observe that 
\begin{equation*}
\set{h,-h}
=2\lambda\set[\bigg]{\frac{\hl-\althl}{2},\frac{\althl-\hl}{2}}
=2\lambda\group[\bigg]{\set{\hl,\althl}-\frac{\hl+\althl}{2}},
\end{equation*}
so, due to the coherence of $\rejectset[\hlrejectfun]$ [use Lemma~\ref{lem:localrescaling}], proving that $\set{h,-h}\in\rejectset[\hlrejectfun]$ is equivalent to proving that $\set{\hl,\althl}-\nicefrac{1}{2}(\hl+\althl)\in\rejectset[\hlrejectfun]$, which follows directly from Axiom~\ref{ax:hlrejectfun:totality} and Equation~\eqref{eq:hlrejectfun:to:rejectset}.

We now turn to the proof of~\ref{it:difspacemixing:to:hlmixing}.

Assume that $\rejectset$ is mixing.
Then we already know from Theorem~\ref{theo:horselotteries:to:options:and:back}\ref{it:difspacecoherence:to:hlcoherence} that $\hlrejectfun[\rejectset]$ is coherent, so it suffices to show that $\hlrejectfun[\rejectset]$ satisfies Axiom~\ref{ax:hlrejectfun:removing:convex:combinations}.
So consider any $\hlset,\althlset\in\hlsets$ such that $\hlset\subseteq\althlset\subseteq\chull\group{\hlset}$, and any $\hl\in\hlset\cap\hlrejectfun[\rejectset]\group{\althlset}$---observe that this already implies that $\hl\in\hlset$ and $\hl\in\althlset$.
Then we must prove that also $\hl\in\hlrejectfun[\rejectset]\group{\hlset}$.
It follows from the assumptions that $\althlset-\hl\in\rejectset$, and also that $\hlset-\hl\subseteq\althlset-\hl\subseteq\chull\group{\hlset-\hl}$, so Axiom~\ref{ax:rejects:removeconvexcombinations} implies that also $\hlset-\hl\in\rejectset$, and therefore, indeed, $\hl\in\rejectfun[\rejectset]\group{\hlset}$.

And, finally, we prove~\ref{it:hlmixing:to:difspacemixing}.

Assume that $\hlrejectfun$ is mixing.
Then we already know from Theorem~\ref{theo:horselotteries:to:options:and:back}\ref{it:hlcoherence:to:difspacecoherence} that $\rejectset[\hlrejectfun]$ is coherent, so it suffices to show that $\rejectset[\hlrejectfun]$ satisfies Axiom~\ref{ax:rejects:removeconvexcombinations}.
So consider any $\optset,\altoptset\in\optsets$ such that $\altoptset\in\rejectset[\hlrejectfun]$ and $\optset\subseteq\altoptset\subseteq\chull\group{\optset}$, then we must show that also $\optset\in\rejectset[\hlrejectfun]$. 
$\altoptset\in\rejectset[\hlrejectfun]$ means that there are $\althlsettoo\in\hlsets$ and $\althltoo\in\hls$ such that $\althltoo\in\hlrejectfun\group{\althlsettoo\cup\set{\althltoo}}$ and $\altoptset\in\starify\group{\althlsettoo-\althltoo}$.
Let, without loss of generality, $\optset=\set{\gbl[1],\dots,\gbl[n]}$ and $\altoptset=\set{\gbl[1],\dots,\gbl[n],\altgbl[1],\dots,\altgbl[m]}$, with $m,n\in\naturalswithzero$.
Then it follows from the assumptions that there are real $\alpha_{k,\ell}\geq0$ such that $\sum_{\ell=1}^n\alpha_{k,\ell}=1$ and $\altgbl[k]=\sum_{\ell=1}^n\alpha_{k,\ell}\gbl[\ell]$, for all $k\in\set{1,\dots,m}$.
If we fix any $\hl\in\hls$ such that $\hl(\cdot,r)$ is bounded away from zero for all $r\in\rewards$, then we infer from Lemma~\ref{lem:from:difspace:to:hls} that there are real $\lambda_\ell$ and $\hl[\ell]\in\hls$ and real $\mu_k$ and $\althl[k]\in\hls$ such that $\gbl[\ell]=\lambda_\ell(\hl[\ell]-\hl)$ and $\altgbl[k]=\mu_k(\althl[k]-\hl)$, for all $k\in\set{1,\dots,m}$ and $\ell\in\set{1,\dots,n}$.
This implies that
\begin{equation}\label{eq:hlmixing:to:difspacemixing}
\althl[k]-\hl
=\sum_{\ell=1}^n\alpha_{k,\ell}\frac{\lambda_\ell}{\mu_k}(\hl[\ell]-\hl)
=\sum_{\ell=1}^n\alpha_{k,\ell}\frac{\lambda_\ell}{\mu_k}\hl[\ell]
-\group[\bigg]{\sum_{\ell=1}^n\alpha_{k,\ell}\frac{\lambda_\ell}{\mu_k}}\hl.
\end{equation}
and it follows readily from the proof of Lemma~\ref{lem:from:difspace:to:hls} that we can always choose the $\lambda_k$ and $\mu_\ell$ in such a way that $\sum_{\ell=1}^n\alpha_{k,\ell}\frac{\lambda_\ell}{\mu_k}=1$ for all $k\in\set{1,\dots,m}$ [indeed, the only requirement there on these coefficients $\lambda_\ell$ and $\mu_k$ is that they should be higher than some number that is determined by $\hl$, and the $\gbl[\ell]$ and $\altgbl[k]$, respectively].
If we let $\hlset\coloneqq\set{\hl[1],\dots,\hl[n]}$ and $\althlset\coloneqq\set{\hl[1],\dots,\hl[n],\althl[1],\dots,\althl[m]}$, then it follows from Equation~\eqref{eq:hlmixing:to:difspacemixing} that $\hlset\subseteq\althlset\subseteq\chull\group{\hlset}$, and therefore also $\hlset\cup\set{\hl}\subseteq\althlset\cup\set{\hl}\subseteq\chull\group{\hlset\cup\set{\hl}}$.
We also see that $\optset\in\starify\group{\hlset-\hl}$ and $\altoptset\in\starify\group{\althlset-\hl}$.
Since also $\altoptset\in\starify\group{\althlsettoo-\althltoo}$, this tells us that $\althlset-\hl\in\starify\group{\althlsettoo-\althltoo}$, whence also $\hl\in\hlrejectfun\group{\althlset\cup\set{\hl}}$, by Lemma~\ref{lem:localrescaling:horselotteries}.
Axiom~\ref{ax:hlrejectfun:removing:convex:combinations} then leads us to conclude that $\group{\hlset\cup\set{\hl}}\cap\hlrejectfun[\rejectset]\group{\althlset\cup\set{\hl}}\subseteq\hlrejectfun[\rejectset]\group{\hlset\cup\set{\hl}}$, and therefore also that $\hl\in\hlrejectfun\group{\hlset\cup\set{\hl}}$.
Hence, $\hlset-\hl\in\rejectset$, and therefore also $\optset\in\rejectset$ by Lemma~\ref{lem:localrescaling}. 
\end{proof}

\begin{lemma}\label{lem:localrescaling:horselotteries}
Consider any rejection function $\hlrejectfun$ on $\hls$ that satisfies Axiom~\ref{ax:hlrejectfun:cone}.
Consider any $\hlset,\althlset\in\hlsets$ and $\hl,\althl\in\hls$ such that $\hlset-\hl\in\starify\group{\althlset-\althl}$ and $\althl\in\hlrejectfun\group{\althlset\cup\set{\althl}}$. 
Then also $\hl\in\hlrejectfun\group{\hlset\cup\set{\hl}}$.
\end{lemma}

\begin{proof}
Let, without loss of generality, $\hlset=\set{\hl[1],\dots,\hl[n]}$ and $\althlset=\set{\althl[1],\dots,\althl[n]}$.
We know from the assumptions that there are real $\lambda_k>0$ such that $\hl[k]-\hl=\lambda_k\group{\althl[k]-\althl}$ for all $k\in\set{1,\dots,n}$.
Axiom~\ref{ax:hlrejectfun:cone} for $\hlset[1]=\hlset[2]=\althlset=\set{\althl[1],\dots,\althl[n]}$ and $(\lambda_{k\ell},\mu_{k\ell})\coloneqq(\lambda_k,0)>0$ for all $k,\ell\in\set{1,\dots,n}$ now leads to the desired result.
\end{proof}

The following theorem establishes that there is a one-to-one relationship between coherent rejection functions on $\hls$ and coherent sets of desirable option sets on $\difspace$: they represent, essentially, the same thing.

\begin{theorem}\label{theo:from:horselotteries:to:options:and:back}
Consider any coherent rejection function $\hlrejectfun$ on $\hls$ and any coherent set of desirable option sets $\rejectset\in\optsets(\difspace)$.
Then $\rejectset=\rejectset[\hlrejectfun]\ifandonlyif\hlrejectfun=\hlrejectfun[\rejectset]$.
\end{theorem}
\begin{proof}
We already know from Theorem~\ref{theo:horselotteries:to:options:and:back}\ref{it:difspacecoherence:to:hlcoherence}\&\ref{it:hlcoherence:to:difspacecoherence} that $\rejectset[\hlrejectfun]$ and $\hlrejectfun[\rejectset]$ are coherent as well.

To prove the direct implication, let $\rejectset\coloneqq\rejectset[\hlrejectfun]$, and consider any $\hlset\in\hlsets$ and $\hl\in\hls$.
To prove that $\hlrejectfun=\hlrejectfun[\rejectset]$, it suffices to prove that $\hl\in\hlrejectfun\group{\hlset\cup\set{\hl}}\ifandonlyif\hl\in\hlrejectfun[\rejectset]\group{\hlset\cup\set{\hl}}$.

First, assume that $\hl\in\hlrejectfun\group{\hlset\cup\set{\hl}}$, then we infer from Equation~\eqref{eq:hlrejectfun:to:rejectset} that $\hlset-\hl\in\rejectset[\hlrejectfun]$, and since $\rejectset=\rejectset[\hlrejectfun]$, we infer from Equation~\eqref{eq:rejectset:to:hlrejectfun} that, indeed, $\hl\in\hlrejectfun[\rejectset]\group{\hlset\cup\set{\hl}}$.

Conversely, assume that $\hl\in\hlrejectfun[\rejectset]\group{\hlset\cup\set{\hl}}$, then we infer from Equation~\eqref{eq:rejectset:to:hlrejectfun} and $\rejectset=\rejectset[\hlrejectfun]$ that $\hlset-\hl\in\rejectset[\hlrejectfun]$.
Equation~\eqref{eq:hlrejectfun:to:rejectset} then tells us that there are $\althlset\in\hlsets$ and $\althl\in\hls$ such that $\hlset-\hl\in\starify\group{\althlset-\althl}$ and $\althl\in\hlrejectfun\group{\althlset\cup\set{\althl}}$.
It now follows from the coherence of $\hlrejectfun$ [see Lemma~\ref{lem:localrescaling:horselotteries}] that, indeed, also $\hl\in\hlrejectfun\group{\hlset\cup\set{\hl}}$.

To prove the converse implication, let $\hlrejectfun=\hlrejectfun[\rejectset]$, and consider any $\optset\in\optsets$.
To prove that $\rejectset=\rejectset[\hlrejectfun]$, we show that $\optset\in\rejectset\ifandonlyif\optset\in\rejectset[\hlrejectfun]$.

First, assume that $\optset\in\rejectset$.
Choose any horse lottery $\hl\in\hlset$ such that the $\hl(\cdot,r)$ are bounded away from zero for all $r\in\rewards$, then Lemma~\ref{lem:from:difspace:to:hls} guarantees that there is some $\althlset\in\hlsets$ such that $\althlset-\hl\in\starify\group{\optset}$.
The coherence of $\rejectset$ [see Lemma~\ref{lem:localrescaling}] then guarantees that also $\althlset-\hl\in\rejectset$, so we infer from Equation~\eqref{eq:rejectset:to:hlrejectfun} that $\hl\in\hlrejectfun[\rejectset]\group{\althlset\cup\set{\hl}}=\hlrejectfun\group{\althlset\cup\set{\hl}}$. 
Equation~\eqref{eq:hlrejectfun:to:rejectset} then guarantees that $\althlset-\hl\in\rejectset[\hlrejectfun]^0$ and therefore also $\optset\in\rejectset[\hlrejectfun]$, because $\althlset-\hl\in\starify\group{\optset}$ implies that $\optset\in\starify\group{\althlset-\hl}$.

Conversely, assume that $\optset\in\rejectset[\hlrejectfun]$, which means, by Equation~\eqref{eq:hlrejectfun:to:rejectset}, that there are $\hlset\in\hlsets$ and $\hl\in\hls$ such that $\optset\in\starify\group{\hlset-\hl}$ and $\hl\in\hlrejectfun\group{\hlset\cup\set{\hl}}=\hlrejectfun[\rejectset]\group{\hlset\cup\set{\hl}}$.
Equation~\eqref{eq:rejectset:to:hlrejectfun} then tells us that $\hlset-\hl\in\rejectset$, and the coherence of $\rejectset$ [see Lemma~\ref{lem:localrescaling}] then guarantees that, indeed, also $\optset\in\rejectset$.
\end{proof}

In order to refine this picture, we show that the relations~\eqref{eq:hlrejectfun:to:rejectset} and~\eqref{eq:rejectset:to:hlrejectfun} between horse lotteries and sets of desirable option sets preserve intersections, and therefore also the entire inference machinery behind them.

\begin{proposition}\label{prop:preserving:intersections}
\begin{enumerate}[label=\upshape(\roman*),leftmargin=*]
\item\label{it:preserving:intersections:rejects} Let $\rejectset[i]$, $i\in I$ be an arbitrary non-empty family of sets of desirable option sets on~$\difspace$, and let $\rejectset\coloneqq\bigcap_{i\in I}\rejectset[i]$ be its intersection.
Then $\hlrejectfun[\rejectset]=\bigcap_{i\in I}\hlrejectfun[{\rejectset[i]}]$.
\item\label{it:preserving:intersections:rejectfuns} Let $\hlrejectfun[i]$, $i\in I$ be an arbitrary non-empty family of \emph{coherent} rejection functions on~$\hls$, and let $\hlrejectfun\coloneqq\bigcap_{i\in I}\hlrejectfun[i]$ be its infimum.
Then $\rejectset[\hlrejectfun]=\bigcap_{i\in I}\rejectset[{\hlrejectfun[i]}]$.
\end{enumerate}
\end{proposition}
\begin{proof}
We begin with the first statement.
Consider any $\hlset\in\hlsets$ and $\hl\in\hls$, then the following chain of equivalences does the job:
\begin{align*}
\hl\in\hlrejectfun[\rejectset]\group{\hlset\cup\set{\hl}}
&\ifandonlyif\hlset-\hl\in\rejectset
&&\text{[Equation~\eqref{eq:rejectset:to:hlrejectfun}]}\\
&\ifandonlyif\group{\forall i\in I}\hlset-\hl\in\rejectset[i]
&&\\
&\ifandonlyif\group{\forall i\in I}\hl\in\hlrejectfun[{\rejectset[i]}]\group{\hlset\cup\set{\hl}}
&&\text{[Equation~\eqref{eq:rejectset:to:hlrejectfun}]}\\
&\ifandonlyif\hl\in\group[\bigg]{\bigcap_{i\in I}\hlrejectfun[{\rejectset[i]}]}\group{\hlset\cup\set{\hl}}.
&&
\end{align*}

To prove the second statement, we use the first one, and Theorem~\ref{theo:from:horselotteries:to:options:and:back}.
Let $\rejectset[i]\coloneqq\rejectset[{\hlrejectfun[i]}]$, then the $\rejectset[i]$ are also coherent [use Theorem~\ref{theo:horselotteries:to:options:and:back}\ref{it:hlcoherence:to:difspacecoherence}], and so is therefore their intersection $\rejectset=\bigcap_{i\in I}\rejectset[i]=\bigcap_{i\in I}\rejectset[{\hlrejectfun[i]}]$ [use Theorem~\ref{theo:conservative:inference:for:rejectsets}], and we infer from~\ref{it:preserving:intersections:rejects} that
\begin{equation*}
\hlrejectfun[\rejectset]
=\bigcap_{i\in I}\hlrejectfun[{\rejectset[{\hlrejectfun[i]}]}]
=\bigcap_{i\in I}\hlrejectfun[i]
=\hlrejectfun,
\end{equation*}
where the second equality follows from the coherence of the $\hlrejectfun[i]$ and the $\rejectset[{\hlrejectfun[i]}]$, and Theorem~\ref{theo:from:horselotteries:to:options:and:back}.
$\hlrejectfun[\rejectset]$ is coherent because $\rejectset$ is [use Theorem~\ref{theo:horselotteries:to:options:and:back}\ref{it:difspacecoherence:to:hlcoherence}],\footnote{Incidentally, this proves indirectly that an arbitrary non-empty infimum of coherent rejection functions on horse lotteries is coherent.} and therefore Theorem~\ref{theo:from:horselotteries:to:options:and:back} guarantees that $\rejectset[\hlrejectfun]=\rejectset=\bigcap_{i\in I}\rejectset[{\hlrejectfun[i]}]$, as desired. 
\end{proof}

\subsection{Binary choice models on horse lotteries}
The correspondence results tell us that we can import all results that we have proved for sets of desirable option sets---such as natural extension and representation theorems---into the theory of rejection functions on horse lotteries.
The only things we still need in order to complete the picture, is to identify the coherent rejection functions on horse lotteries that  correspond to the binary coherent sets of desirable option sets $\rejectset[\desirset]\in\cohrejectsets$, or in other words, to the coherent sets of desirable options $\desirset\in\cohdesirsets$.

So, we fix any set of desirable gambles $\desirset$ on the option space $\difspace$ and the corresponding binary set of desirable option sets $\rejectset[\desirset]$, determined through Equation~\eqref{eq:desirset:to:rejectset}.
We use Equation~\eqref{eq:rejectset:to:hlrejectfun} to derive the corresponding rejection function $\hlrejectfun[{\rejectset[\desirset]}]$, which we will also denote more simply by $\hlrejectfun[\desirset]$: for any $\hl\in\hls$ and any $\hlset\in\hlsets$,
\begin{align}
\hl\in\hlrejectfun[\desirset]\group{\hlset\cup\set{\hl}}
&\ifandonlyif\hlset-\hl\in\rejectset[\desirset]
&&\text{[Equation~\eqref{eq:rejectset:to:hlrejectfun}]}\notag\\
&\ifandonlyif\group{\hlset-\hl}\cap\desirset\neq\emptyset
&&\text{[Equation~\eqref{eq:desirset:to:rejectset}]}\notag\\
&\ifandonlyif\group{\exists\althl\in\hlset}\althl-\hl\in\desirset
&&\label{eq:desirset:to:hlrejectfun}\\
&\ifandonlyif\group{\exists\althl\in\hlset}\althl-\hl\prefgt[\desirset]0,
&&\text{[Equation~\eqref{eq:desirset:to:prefgt}]}\notag
\end{align}
which means that if we define the binary ordering $\hlprefgt[\desirset]$ on $\hls$ by
\begin{equation}\label{eq:desirset:to:hlprefgt} 
\althl\hlprefgt[\desirset]\hl
\ifandonlyif\althl-\hl\in\desirset
\ifandonlyif\althl-\hl\prefgt[\desirset]0
\text{ for all $\althl,\hl\in\hls$}
\end{equation}
then also
\begin{equation}\label{hlprefgt:to:hlrejectfun}
\hl\in\hlrejectfun[\desirset]\group{\hlset\cup\set{\hl}}
\ifandonlyif\group{\exists\althl\in\hlset}\althl\hlprefgt[\desirset]\hl.
\end{equation}
We see that the rejection function $\hlprefgt[\desirset]$ on horse lotteries is completely determined by a binary preference relation on horse lotteries; we will therefore also call it a \emph{binary} rejection function.

It will be interesting to take a closer look at the binary choice models on $\hls$.

\begin{definition}
We call a binary relation $\hlprefgt$ on $\hls$ a \emph{coherent} strict preference relation if it satisfies Axiom~\ref{ax:hlprefgt:sure:thing:principle} and
\begin{enumerate}[label=$\protect{\hlprefgt[{\arabic*}]}$.,ref=$\protect{\hlprefgt[{\arabic*}]}$,leftmargin=*,start=1]
\item\label{ax:hlprefgt:pos} $\althl\hlgt\hl\then\althl\hlprefgt\hl$ for all $\althl,\hl\in\hls$.
\end{enumerate}
It is called \emph{total} if it is coherent and satisfies
\begin{enumerate}[label=$\protect{\hlprefgt[{\mathrm{T}}]}$.,ref=$\protect{\hlprefgt[{\mathrm{T}}]}$,leftmargin=*]
\item\label{ax:hlprefgt:totality} $\althl\hlprefgt\hl$ or $\hl\hlprefgt\althl$ for all $\althl,\hl\in\hls$ such that $\althl\neq\hl$.
\end{enumerate}
It is called \emph{mixing} if it is coherent and satisfies
\begin{enumerate}[label=$\protect{\hlprefgt[{\mathrm{M}}]}$.,ref=$\protect{\hlprefgt[{\mathrm{M}}]}$,leftmargin=*]
\item\label{ax:hlprefgt:mixing} if $\group{\exists\althl\in\chull\group{\hlset}}\althl\hlprefgt\hl$ then also $\group{\exists\althltoo\in\hlset}\althltoo\hlprefgt\hl$ for all $\hl\in\hls$ and $\hlset\in\hlsets$.
\end{enumerate}
\end{definition}

Given a set of desirable gambles $\desirset$ on $\difspace$, the corresponding (binary) strict preference $\hlprefgt[\desirset]$ on $\hls$ is given by Equation~\eqref{eq:desirset:to:hlprefgt}.
Conversely, given a binary strict preference $\hlprefgt$ on $\hls$, the corresponding set of desirable gambles $\desirset[\hlprefgt]$ on $\difspace$ is given by
\begin{equation}\label{eq:hlprefgt:to:desirset}
\lambda\group{\althl-\hl}\in\desirset[\hlprefgt]\ifandonlyif\althl\hlprefgt\hl,
\text{ for all $\lambda\in\posreals$ and $\althl,\hl\in\hls$}.
\end{equation}
In order for this definition to be consistent, Proposition~\ref{prop:consistency:for:horselotteries:preference:ordering} tells us that it suffices for $\hlprefgt$ to satisfy Axiom~\ref{ax:hlprefgt:sure:thing:principle}.

It will be clear that at this point, we can repeat the argumentation in Sections~\ref{sec:binary:choice}--\ref{sec:archimedeanity}, and in Section~\ref{sec:from:hls:to:opts:and:back} for rejection functions and binary preference orderings on horse lotteries.
We restrict ourselves here to mentioning only a few of these results.

\begin{theorem}\label{theo:horselotteries:to:options:and:back:binary}
Let $\desirset$ be any set of desirable options in $\difspace$, and let $\hlprefgt$ be any strict preference ordering on $\hls$.
Then the following statements hold:
\begin{enumerate}[label=\upshape(\roman*),leftmargin=*]
\item\label{it:difspacecoherence:to:hlcoherence:binary} if\/ $\desirset$ is coherent, then so is $\hlprefgt[\desirset]$;
\item\label{it:hlcoherence:to:difspacecoherence:binary} if\/ $\hlprefgt$ is coherent, then so is $\desirset[\hlprefgt]$;
\item\label{it:difspacetotal:to:hltotal:binary} if\/ $\desirset$ is total, then so is $\hlprefgt[\desirset]$;
\item\label{it:hltotal:to:difspacetotal:binary} if\/ $\hlprefgt$ is total, then so is $\desirset[\hlprefgt]$;
\item\label{it:difspacemixing:to:hlmixing:binary} if\/ $\desirset$ is mixing, then so is $\hlprefgt[\desirset]$;
\item\label{it:hlmixing:to:difspacemixing:binary} if\/ $\hlprefgt$ is mixing, then so is $\desirset[\hlprefgt]$.
\end{enumerate}
\end{theorem}
\begin{proof}
We begin with the proof of~\ref{it:difspacecoherence:to:hlcoherence:binary}.
Assume that $\desirset$ is coherent.

\ref{ax:hlprefgt:sure:thing:principle}.
That $\hlprefgt[\desirset]$ is irreflexive follows immediately from Axiom~\ref{ax:desirs:nozero}. 
That it is transitive follows at once from Axiom~\ref{ax:desirs:cone}, and in particular from the fact that $\desirset$ is closed under vector addition.
Consider therefore any $\althltoo,\althl,\hl\in\hls$ and any $\alpha\in(0,1]$, and the following chain of equivalences:
\begin{align*}
\althltoo\hlprefgt[\desirset]\althl
&\ifandonlyif\althltoo-\althl\in\desirset
&&\text{[Equation~\eqref{eq:desirset:to:hlprefgt}]}\\
&\ifandonlyif\alpha\group{\althltoo-\althl}\in\desirset
&&\text{[Axiom~\ref{ax:desirs:cone}]}\\
&\ifandonlyif\alpha\althltoo+(1-\alpha)\hl-\group[\big]{\alpha\althl+(1-\alpha)\hl}\in\desirset
&&\\
&\ifandonlyif\alpha\althltoo+(1-\alpha)\hl\hlprefgt[\desirset]\alpha\althl+(1-\alpha)\hl.
&&\text{[Equation~\eqref{eq:desirset:to:hlprefgt}]}
\end{align*}

\ref{ax:hlprefgt:pos}.
Consider any $\althl,\hl\in\hls$ and assume that $\althl\hlgt\hl$.
Then Equation~\eqref{eq:background:ordering:on:options:from:horse:lotteries} guarantees that $\althl-\hl\optgt0$, whence $\althl-\hl\in\desirset$ by Axiom~\ref{ax:desirs:pos}.
Equation~\eqref{eq:desirset:to:prefgt} then leads us to conclude that, indeed, $\althl\hlprefgt[\desirset]\hl$.

We now turn to the proof of~\ref{it:hlcoherence:to:difspacecoherence:binary}.
Assume that $\hlprefgt$ is coherent.
That $\desirset[\hlprefgt]$ satisfies Axiom~\ref{ax:desirs:nozero} follows at once from the irreflexivity of $\hlprefgt$ [Axiom~\ref{ax:hlprefgt:sure:thing:principle}].
That $\desirset[\hlprefgt]$ satisfies Axiom~\ref{ax:desirs:pos} follows at once from Axiom~\ref{ax:hlprefgt:pos}].

\ref{ax:desirs:cone}.
Consider any $\lambda_1\group{\althl[1]-\hl[1]}$ and $\lambda_2\group{\althl[2]-\hl[2]}$ in $\desirset[\hlprefgt]$, and any $(\alpha,\beta)>0$, then we must show that $\alpha\lambda_1\group{\althl[1]-\hl[1]}+\beta\lambda_2\group{\althl[2]-\hl[2]}\in\desirset[\hlprefgt]$, or equivalently, that
\begin{equation*}
\kappa\althl[1]+(1-\kappa)\hl[2]
\hlprefgt
\kappa\hl[1]+(1-\kappa)\althl[2],
\end{equation*}
where we let $\kappa\coloneqq\frac{\alpha\lambda_1}{\alpha\lambda_1+\beta\lambda_2}\in[0,1]$.
It follows from the assumptions that $\althl[1]\hlprefgt\hl[1]$ and $\althl[2]\hlprefgt\hl[2]$, so if $\kappa=0$ or $\kappa=1$, we are done.
Assume therefore that $\kappa\in(0,1)$, and consider the following equivalences, due to Axiom~\ref{ax:hlprefgt:sure:thing:principle}:
\begin{align*}
\althl[1]\hlprefgt\hl[1]
&\ifandonlyif\kappa\althl[1]+(1-\kappa)\althl[2]\hlprefgt\kappa\hl[1]+(1-\kappa)\althl[2]\\
\althl[2]\hlprefgt\hl[2]
&\ifandonlyif(1-\kappa)\althl[2]+\kappa\hl[1]\hlprefgt(1-\kappa)\hl[2]+\kappa\hl[1]
\end{align*}
and now the transitivity of~$\hlprefgt$ [Axiom~\ref{ax:hlprefgt:sure:thing:principle}] tells us that, indeed, $\kappa\althl[1]+(1-\kappa)\hl[2]\hlprefgt\kappa\hl[1]+(1-\kappa)\althl[2]$.

We now turn to the proof of~\ref{it:difspacetotal:to:hltotal:binary}.
Assume that $\desirset$ is total, then we only need to prove that $\hlprefgt[\desirset]$ satisfies Axiom~\ref{ax:hlprefgt:totality}.
So consider any $\althl,\hl\in\hls$ with $\althl\neq\hl$. 
Then $\gbl=\althl-\hl\neq0$, so Axiom~\ref{ax:desirs:totality} tells us that $\set{-\gbl,\gbl}\cap\desirset\neq\emptyset$, so, indeed, either $\althl\hlprefgt[\desirset]\hl$ or $\hl\hlprefgt[\desirset]\althl$.

We now turn to the proof of~\ref{it:hltotal:to:difspacetotal:binary}.
Assume that $\hlprefgt$ is total, then we only need to prove that $\desirset[\hlprefgt]$ satisfies Axiom~\ref{ax:desirs:totality}.
Consider any non-zero $\gbl=\lambda\group{\althl-\hl}\in\difspace$, so $\althl\neq\hl$ and therefore Axiom~\ref{ax:hlprefgt:totality} tells us that either $\althl\hlprefgt\hl$ or $\hl\hlprefgt\althl$, so indeed, either $\gbl\in\desirset[\hlprefgt]$ or $-\gbl\in\desirset[\hlprefgt]$.

We now turn to the proof of~\ref{it:difspacemixing:to:hlmixing:binary}.
Assume that $\desirset$ is mixing, then we only need to prove that $\hlprefgt[\desirset]$ satisfies Axiom~\ref{ax:hlprefgt:mixing}.
So consider any $\hl\in\hls$ and $\hlset\in\hlsets$, and assume that there is some $\althl\in\chull(\hlset)$ such that $\althl\hlprefgt[\desirset]\hl$, so $\althl-\hl\in\desirset$, and therefore also $\chull\group{\hlset-\hl}\cap\desirset=\group{\chull\group{\hlset}-\hl}\cap\desirset\neq\emptyset$.
Then, since $\chull(\hlset-\hl)\subseteq\posi(\hlset-\hl)$, we infer from Axiom~\ref{ax:desirs:totality} that there is some $\althltoo\in\hlset$ such that $\althltoo-\hl\in\desirset$, and therefore, indeed, $\althltoo\hlprefgt[\desirset]\hl$.

Finally, we prove~\ref{it:hlmixing:to:difspacemixing:binary}.
Assume that $\hlprefgt$ is mixing, then we only need to prove that $\desirset[\hlprefgt]$ satisfies Axiom~\ref{ax:desirs:mixing}.
So consider any $\optset\in\optsets(\difspace)$, and assume that $\desirset[\hlprefgt]\cap\posi(\optset)\neq\emptyset$.
If we consider any $\hl\in\hls$ such that $\hl(\cdot,r)$ is bounded away from zero for all $r\in\rewards$, then we infer from Lemma~\ref{lem:from:difspace:to:hls} that there are $n\in\naturals$ and $\althl[k]\in\hls$ such that $\optset\in\starify\group{\hlset-\hl}$, with $\hlset\coloneqq\cset{\althl[k]}{k\in\set{1,\dots,n}}$.
Since $\chull\group{\hlset-\hl}=\chull\group{\hlset}-\hl$, we infer from the assumptions and Axiom~\ref{ax:desirs:cone} that there is some $\althl\in\chull\group{\hlset}$ such that $\althl-\hl\in\desirset[\hlprefgt]$, and therefore also $\althl\hlprefgt\hl$.
Axiom~\ref{ax:hlprefgt:mixing} then guarantees that there is some $\althl[k]\in\hlset$ such that $\althl[k]\hlprefgt\hl$, and therefore $\althl[k]-\hl\in\desirset[\hlprefgt]$.
Since we know from the construction that there is some real $\lambda_k>0$ such that $\lambda_k\group{\althl[k]-\hl}\in\optset$, Axiom~\ref{ax:desirs:cone} guarantees that, indeed, $\optset\cap\desirset[\hlprefgt]\neq\emptyset$.
\end{proof}

\begin{theorem}\label{theo:from:horselotteries:to:options:and:back:binary}
Consider any coherent strict preference ordering $\hlprefgt$ on $\hls$ and any coherent set of desirable option sets $\desirset\in\cohdesirsets(\difspace)$.
Then $\desirset=\desirset[\hlprefgt]\ifandonlyif\hlprefgt=\hlprefgt[\desirset]$.
\end{theorem}
\begin{proof}
We already know from Theorem~\ref{theo:horselotteries:to:options:and:back:binary}\ref{it:difspacecoherence:to:hlcoherence:binary}\&\ref{it:hlcoherence:to:difspacecoherence:binary} that $\desirset[\hlprefgt]$ and $\hlprefgt[\desirset]$ are coherent as well.

To prove the direct implication, let $\desirset\coloneqq\desirset[\hlprefgt]$, and consider any $\althl,\hl\in\hls$, and the following chain of equivalences:
\begin{equation*}
\althl\hlprefgt\hl
\ifandonlyif\althl-\hl\in\desirset[\hlprefgt]
\ifandonlyif\althl-\hl\in\desirset
\ifandonlyif\althl\hlprefgt[\desirset]\hl,
\end{equation*}
where the first equivalence follows from Equation~\eqref{eq:hlprefgt:to:desirset}, the second from the assumption, and the third from Equation~\eqref{eq:desirset:to:hlprefgt}.

To prove the converse implication, let $\hlprefgt=\hlprefgt[\desirset]$, and consider any $\gbl=\lambda\group{\althl-\hl}\in\difspace$, and and the following chain of equivalences:
\begin{equation*}
\gbl\in\desirset
\ifandonlyif\althl\hlprefgt[\desirset]\hl
\ifandonlyif\althl\hlprefgt\hl
\ifandonlyif\gbl\in\desirset[\hlprefgt],
\end{equation*}
where the first equivalence follows from Equation~\eqref{eq:desirset:to:hlprefgt}, the second from the assumption, and the third from Equation~\eqref{eq:hlprefgt:to:desirset}.
\end{proof}

\stilltodo{Mention the representation results}.

\newpage
\section{Adding weak Archimedeanity}
We consider a set of desirable option sets $\rejectset$ that satisfies Axioms~\ref{ax:rejects:nozero}--\ref{ax:rejects:mono}, or equivalently, a rejection function $\rejectfun$ that satisfies Axioms~\ref{ax:rejectfun:addition}--\ref{ax:rejectfun:senalpha}.

Let us define extend the binary preference ordering $\preflt$ on $\opts$ to a binary ordering on $\optsets$ as follows:
\begin{equation*}
\optset\preflt\altoptset\ifandonlyif\optset\subseteq\rejectfun\group{\optset\cup\altoptset}
\text{ for all $\optset,\altoptset\in\optsets$}.
\end{equation*}
It is clear that, in particular
\begin{equation}\label{eq:order:and:rejectset:basic}
\set{\opt}\preflt\altoptset
\ifandonlyif\opt\in\rejectfun\group{\set{u}\cup\altoptset}
\ifandonlyif\altoptset-\opt\in\rejectset,
\end{equation}
and the following proposition extends this result.

\begin{proposition}
For all $\optset,\altoptset\in\optsets$:
\begin{equation}\label{eq:order:and:rejectset}
\optset\preflt\altoptset
\ifandonlyif\group{\forall\opt\in\optset}\altoptset-\opt\in\rejectset.
\end{equation}
\end{proposition}
\begin{proof}
First, assume that $\optset\preflt\altoptset$, and consider any $\opt\in\optset$.
Then $\group{\optset\setminus\set{\opt}}\subseteq\optset\subseteq\rejectfun\group{\optset\cup\altoptset}$, and therefore
\begin{equation*}
\set{\opt}
=\optset\setminus\group{\optset\setminus\set{\opt}}
\subseteq\rejectfun\group{\group{\optset\cup\altoptset}\setminus\group{\optset\setminus\set{\opt}}}
=\rejectfun\group{\set{\opt}\cup\group{\altoptset\setminus\optset}}
\subseteq\rejectfun\group{\set{\opt}\cup\altoptset},
\end{equation*}
where the first inclusion follows from Axiom~\ref{ax:rejectfun:aizermann} and the second from Axiom~\ref{ax:rejectfun:senalpha}.
This tells us that $\opt\in\rejectfun\group{\set{\opt}\cup\altoptset}$ and therefore, indeed, $\altoptset-\opt\in\rejectset$.

Conversely, assume that $\altoptset-\opt\in\rejectset$, and therefore also $\opt\in\rejectfun\group{\set{\opt}\cup\altoptset}\subseteq\rejectfun\group{\optset\cup\altoptset}$, for all $\opt\in\optset$, where the inclusion follows from Axiom~\ref{ax:rejectfun:senalpha}.
Then indeed $\optset\subseteq\rejectfun\group{\optset\cup\altoptset}$, or equivalently, $\optset\preflt\altoptset$.
\end{proof}

Seidenfeld et al.~\cite{seidenfeld2010} consider the following axioms for so-called \emph{weak Archimedeanity}:
\begin{enumerate}[label=$\mathrm{R}_{7}$.,ref=$\mathrm{R}_{7}$,leftmargin=*]
\item\label{ax:rejectfun:archimedean} for all $\optset$, $\altoptset$, $\altoptsettoo$, $\altoptset[n]$ and $\altoptsettoo[n]$ in $\optsets$ such that the sequence $\altoptset[n]$ converges point-wise to~$\altoptset$ and the sequence $\altoptsettoo[n]$ converges point-wise to~$\altoptsettoo$:
\begin{enumerate}[noitemsep,leftmargin=*,label=\upshape\alph*.,ref=\theenumi\upshape\alph*]
\item\label{ax:rejectfun:archimedean:left}
if $\group{\forall n\in\naturals}\altoptsettoo[n]\preflt\altoptset[n]$ and $\altoptset\preflt\optset$ then $\altoptsettoo\preflt\optset$;
\item\label{ax:rejectfun:archimedean:right}
if $\group{\forall n\in\naturals}\altoptsettoo[n]\preflt\altoptset[n]$ and $\optset\preflt\altoptsettoo$ then $\optset\preflt\altoptset$,
\end{enumerate}
\end{enumerate}

\begin{proposition}
The relation $\preflt$ on $\optsets$ satisfies~\ref{ax:rejectfun:archimedean} if and only if for all $\optset,\altoptset[n],\altoptset\in\optsets$ such that $\altoptset[n]\to\altoptset$ point-wise:
\begin{equation}\label{eq:rejectset:archimedean}
\group[\big]{\group{\forall n\in\naturals}\altoptset[n]\in\rejectset\text{ and }\group{\forall\opt\in\altoptset}\optset-\opt\in\rejectset}
\then\optset\in\rejectset.
\end{equation}
\end{proposition}
\begin{proof}
We begin with the `only if' part. 
Assume that $\preflt$ satisfies~\ref{ax:rejectfun:archimedean}.
Consider any $\optset,\altoptset[n],\altoptset\in\optsets$ such that $\altoptset[n]\to\altoptset$ point-wise, and any $\opt[n],\opt\in\opts$ such that $\opt[n]\to\opt$.
Let $\altoptset[n]'\coloneqq\altoptset[n]+\set{\opt[n]}$, $\altoptset'\coloneqq\altoptset+\set{\opt}$ and $\optset'\coloneqq\optset+\set{\opt}$, then also $\altoptset[n]'\to\altoptset'$ and therefore Axiom~\ref{ax:rejectfun:archimedean:left} guarantees that
\begin{equation*}
\group[\big]{\group{\forall n\in\naturals}\set{\opt[n]}\preflt\altoptset[n]'\text{ and }\altoptset'\preflt\optset'}
\then\set{\opt}\preflt\optset',
\end{equation*}
which is equivalent to
\begin{equation*}
\group[\big]{\group{\forall n\in\naturals}\set{0}\preflt\altoptset[n]\text{ and }\altoptset\preflt\optset}
\then\set{0}\preflt\optset.
\end{equation*}
Taking into account Equations~\eqref{eq:order:and:rejectset:basic} and~\eqref{eq:order:and:rejectset}, this is in turn equivalent to the statement~\eqref{eq:rejectset:archimedean}. 

We now move to the `if' part.
Assume that the statement~\eqref{eq:rejectset:archimedean} holds and consider any $\optset$, $\altoptset$, $\altoptsettoo$, $\altoptset[n]$ and $\altoptsettoo[n]$ in $\optsets$ such that the sequence $\altoptset[n]$ converges point-wise to~$\altoptset$ and the sequence $\altoptsettoo[n]$ converges point-wise to~$\altoptsettoo$.

We begin with Axiom~\ref{ax:rejectfun:archimedean:left}, and assume that $\group{\forall n\in\naturals}\altoptsettoo[n]\preflt\altoptset[n]$ and $\altoptset\preflt\optset$.
Then we must prove that $\altoptsettoo\preflt\optset$, or using Equation~\eqref{eq:rejectset:archimedean}, that $\optset-\opt\in\rejectset$ for all $\opt\in\altoptsettoo$.
So, consider any $\opt\in\altoptsettoo$, then it follows from the assumptions that there is some sequence of options $\opt[n]\in\altoptsettoo[n]$ such that $\opt[n]\to\opt$, and moreover $\set{\opt[n]}\preflt\altoptset[n]$, or equivalently, $\altoptset[n]-\opt[n]\in\rejectset$.
Since $\altoptset[n]-\opt[n]\to\altoptset-\opt$, and since $\altoptset\preflt\optset$ is equivalent to $\altoptset-\opt\preflt\optset-\opt$ and therefore also to $\group{\forall\altopt\in\altoptset-\opt}\group{\optset-\opt}-\altopt\in\rejectset$, we infer from~\eqref{eq:rejectset:archimedean} that, indeed, $\optset-\opt\in\rejectset$. 

Next, we turn to Axiom~\ref{ax:rejectfun:archimedean:right}, and assume that $\group{\forall n\in\naturals}\altoptsettoo[n]\preflt\altoptset[n]$ and $\optset\preflt\altoptsettoo$.
Then we must prove that $\optset\preflt\altoptset$,  or using Equation~\eqref{eq:rejectset:archimedean}, that $\altoptset-\opt\in\rejectset$ for all $\opt\in\optset$.
\end{proof}

\newpage
\section{Aizermann's condition}
The following property for rejection functions is called \emph{Aizermann's condition}~\cite{aizerman1985}:
\begin{enumerate}[label=$\mathrm{R}_{\mathrm{A}}$.,ref=$\mathrm{R}_{\mathrm{A}}$,leftmargin=*]
\item\label{ax:rejectfun:aizermann} if $\optset,\altoptset,\altoptsettoo\in\optsets$ are such that $\optset\subseteq\altoptset\subseteq\rejectfun\group{\altoptsettoo}$ then also $\altoptset\setminus\optset\subseteq\rejectfun\group{\altoptsettoo\setminus\optset}$.
\end{enumerate}
As was proved by Arthur van Camp (private communication), Aizermann's condition holds for rejection functions that correspond to coherent sets of desirable option sets.
We give our own version of the argument here for the sake of completeness.

\begin{proposition}\label{prop:aizermann}
If the set of desirable option sets $\rejectset$ is coherent, then the corresponding rejection function $\rejectfun[\rejectset]$ satisfies Aizermann's condition~\ref{ax:rejectfun:aizermann}.
\end{proposition}

The main idea behind the proof, is to show that coherence implies the following counterpart of the Aizermann condition for sets of desirable option sets:
\begin{enumerate}[label=$\mathrm{K}_{\mathrm{A}}$.,ref=$\mathrm{K}_{\mathrm{A}}$,leftmargin=*]
\item\label{ax:rejects:aizermann} if $\optset\in\rejectset$ and $\opt\in\optset$ then $\optset-\opt\in\rejectset$ implies $\optset\setminus\set{\opt}\in\rejectset$.
\end{enumerate}

\begin{lemma}\label{lem:aizermann}
If the set of desirable option sets $\rejectset$ is coherent, then it satisfies Axiom~\ref{ax:rejects:aizermann}.
\end{lemma}
\begin{proof}
We may assume without loss of generality that $\optset=\set{\altopt[1],\dots,\altopt[n],\opt}\in\rejectset$ and $\optset-\opt=\set{\altopt[1]-\opt,\dots,\altopt[n]-\opt,0}\in\rejectset$, and we must then prove that also $\optset\setminus\set{\opt}=\set{\altopt[1],\dots,\altopt[n]}\in\rejectset$.

We begin by applying Axiom~\ref{ax:rejects:cone} to the option sets $\optset[1]=\set{\altopt[1],\dots,\altopt[n],\opt}$ and $\optset[2]=\set{\altopt[1]-\opt,\dots,\altopt[n]-\opt,0}$, with $\lambda_{\altopt[k],\altopt[\ell]-\opt}=1$ and $\mu_{\altopt[k],\altopt[\ell]-\opt}=0$, $\lambda_{\opt,\altopt[\ell]-\opt}=\mu_{\opt,\altopt[\ell]-\opt}=1$, $\lambda_{\altopt[k],0}=1$ and $\mu_{\altopt[k],0}=0$, and $\lambda_{\opt,0}=0$ and $\mu_{\opt,0}=1$.
Axiom~\ref{ax:rejects:cone} then guarantees that $\set{\altopt[1],\dots,\altopt[n],0}\in\rejectset$, and Axiom~\ref{ax:rejects:removezero} then implies that, indeed, also $\set{\altopt[1],\dots,\altopt[n]}\in\rejectset$. 
\end{proof}

\begin{proof}[Proof of Theorem~\ref{prop:aizermann}]
Assume that $\rejectset$ is coherent, and consider any $\optset,\altoptset,\altoptsettoo\in\optsets$ such that $\optset\subseteq\altoptset\subseteq\rejectfun\group{\altoptsettoo}$.
Then we must prove that $\altoptset\setminus\optset\subseteq\rejectfun\group{\altoptsettoo\setminus\optset}$.

First of all, Proposition~\ref{prop:axioms:rejection:sets:and:functions:coherence} implies that the rejection function~$\rejectfun[\rejectset]$ is coherent, so we know that $\rejectfun[\rejectset]$ satisfies Axioms~\ref{ax:rejectfun:addition}--\ref{ax:rejectfun:senalpha}, and Lemma~\ref{lem:aizermann} implies that $\rejectset$ satisfies~\ref{ax:rejects:aizermann}.

We may assume without loss of generality that $\optset=\set{\opt[1],\dots,\opt[n]}$.
So we know that, in particular $\opt[1]\in\altoptset$ and $\opt[1]\in\rejectfun\group{\altoptsettoo}$, so $\altoptsettoo-\opt[1]\in\rejectset$ [use Axiom~\ref{ax:rejectfun:addition} and Equation~\eqref{eq:interpretation:rejectfuns:after:irreflexivity:and:additivity:intermsof:K:Gert}].
Now consider any $\altopt\in\altoptset$, then on the one hand also $\altopt\in\rejectfun\group{\altoptsettoo}$ and therefore $\altoptsettoo-\altopt\in\rejectset$ [use Axiom~\ref{ax:rejectfun:addition} and Equation~\eqref{eq:interpretation:rejectfuns:after:irreflexivity:and:additivity:intermsof:K:Gert}].
At the same time also $\opt[1]-\altopt\in\altoptsettoo-\altopt$, and $\altoptsettoo-\altopt-(\opt[1]-\altopt)=\altoptsettoo-\opt[1]\in\rejectset$.
Applying~\ref{ax:rejects:aizermann} [with $\optset\coloneqq\altoptsettoo-\altopt$ and $\opt\coloneqq\opt[1]-\altopt$] then allows us to infer that $(\altoptsettoo-\altopt)\setminus\set{\opt[1]-\altopt}\in\rejectset$, or in other words, that $(\altoptsettoo\setminus\set{\opt[1]})-\altopt\in\rejectset$, so $\altopt\in\rejectfun\group{\altoptsettoo\setminus\set{\opt[1]}}$ [again, use Axiom~\ref{ax:rejectfun:addition} and Equation~\eqref{eq:interpretation:rejectfuns:after:irreflexivity:and:additivity:intermsof:K:Gert}].
We conclude that $\altoptset\subseteq\rejectfun\group{\altoptsettoo\setminus\set{\opt[1]}}$, and therefore also $\altoptset\setminus\set{\opt[1]}\subseteq\rejectfun\group{\altoptsettoo\setminus\set{\opt[1]}}$.

Repeating the argument above with $\opt[2]$ instead of $\opt[1]$, $\altoptset\setminus\set{\opt[1]}$ instead of $\altoptset$ and $\altoptsettoo\setminus\set{\opt[1]}$ instead of $\altoptsettoo$, now leads us to conclude that $(\altoptset\setminus\set{\opt[1]})\setminus\set{\opt[2]}\subseteq\rejectfun\group{(\altoptsettoo\setminus\set{\opt[1]})\setminus\set{\opt[2]}}$, or equivalently, $\altoptset\setminus\set{\opt[1],\opt[2]}\subseteq\rejectfun\group{\altoptsettoo\setminus\set{\opt[1],\opt[2]}}$.
Repeating the same argument over and over until we reach $\opt[n]$, eventually leads to the conclusion that, indeed, $\altoptset\setminus\optset\subseteq\rejectfun\group{\altoptsettoo\setminus\optset}$.
\end{proof}

We can use these ideas to establish a relationship between our coherent sets of desirable option sets, and Seidenfelds preference ordering $ $ over option sets, defined by 
\begin{equation*}
\optset\preflt\altoptset\ifandonlyif\optset\subseteq\rejectfun\group{\altoptset\cup\optset}
\text{ for all $\optset,\altoptset\in\optsets$}.
\end{equation*}
The following argument shows how this connection comes about:
\begin{align}
\optset\preflt\altoptset
&\ifandonlyif\optset\subseteq\rejectfun\group{\altoptset\cup\optset}
&&\notag\\
&\ifandonlyif\group{\forall\opt\in\optset}\opt\in\rejectfun\group{\altoptset\cup\optset}
&&\notag\\
&\ifandonlyif\group{\forall\opt\in\optset}\opt\in\rejectfun\group{\altoptset\cup\set{\opt}}
&&\text{[Axioms~\ref{ax:rejectfun:aizermann} and~\ref{ax:rejectfun:senalpha}]}\notag\\
&\ifandonlyif\group{\forall\opt\in\optset}\altoptset-\opt\in\rejectset.
&&\text{[Equation~\eqref{eq:interpretation:rejectfuns:after:irreflexivity:and:additivity:intermsof:K:Gert}]}
\label{eq:preflt:in:terms:of:K}
\end{align}
and, of course, conversely
\begin{equation}\label{eq:K:in:terms:of:preflt}
\optset\in\rejectset
\ifandonlyif0\in\rejectfun\group{\optset\cup\set{0}}
\ifandonlyif\set{0}\preflt\optset.
\end{equation}

\newpage
\section{Questions and ideas}

QUESTION: Can we do natural extension in the context of Archimedeanity? I think not, at least not in general, basically because an intersection of open cones (each of which is Archimedean) can be closed (and hence is not Archimedean). (IS THAT ARGUMENT CORRECT? CAN WE FIND AN EXAMPLE?) This would be quite the argument in favour of our more general approach...

IDEA: as `motivation' for convexity, we can use the interpretation that at least one of the extreme vectors in $\optset$ (with zero NOT extreme) is desirable (extreme can either be in terms of rays of a cone, or in terms of extreme points of the convex hull, depending on how we formulate the convexity axiom)

IDEA: as `motivation' for Archimedeanity, we can use the interpretation that at least one of the vectors $\opt$ in $\optset$ is `strictly desirable', in the sense that $\opt-\altopt$ is desirable for some $\altopt\in\archopts$.

IDEA: how about we generalize our theory as to allow the inclusion of zero? That way, we could also deal with non-strict orderings.

}
\end{document}